\documentclass[10pt,journal,cspaper,compsoc]{IEEEtran}

\usepackage{microtype}
\usepackage{graphicx}
\usepackage{subfigure}
\usepackage{booktabs} 
\usepackage{bbm}
\usepackage{mathtools}
\let\appendices\relax
\usepackage{titletoc}
\usepackage[page,header]{appendix}
\usepackage[left=0.32in,right=0.32in,top=0.32in,bottom=0.32in,%
            footskip=.25in]{geometry}
\usepackage[dvipsnames]{xcolor}
\definecolor{ballblue}{HTML}{338EA7}
\definecolor{lightseagreen}{HTML}{759D39}
\definecolor{lightred}{HTML}{DD7769}
\definecolor{org}{HTML}{F8A145}
\definecolor{blu}{HTML}{63ACE5}
\definecolor{c1}{HTML}{41B3A3}
\definecolor{c2}{HTML}{3500D3}
\usepackage[colorlinks,linkcolor=blue,citecolor=org,urlcolor=org, linktoc=all]{hyperref}

\usepackage{colortbl}
\usepackage{hyperref}       
\usepackage[framemethod=tikz]{mdframed}
\usepackage{lipsum}
\usepackage[noend]{algpseudocode}
\usepackage{algorithmicx,algorithm}

\usepackage{amsthm}
\usepackage{amsmath}
\usepackage{multirow}
\usepackage{amssymb}
\usepackage{subfigure}
\usepackage{graphicx}
\usepackage{bm}
\newtheorem{defi}{Definition}

\newtheorem*{pf}{Proof}
\newtheorem{prop}{Proposition}

\newtheorem{lem}{Lemma}
\newtheorem{rem}{Remark}
\newtheorem{thm}{Theorem}
\newtheorem{asm}{Assumption}
\newtheorem{col}{Corollary}
\newtheorem{exa}{Example}
\DeclareMathOperator*{\expe}{\mathbb{E}}
\DeclareMathOperator*{\expeb}{\bar{\mathbb{E}}}
\usepackage{letltxmacro}
\LetLtxMacro{\originaleqref}{\eqref}
\renewcommand{\eqref}{Eq.~\originaleqref}

\DeclarePairedDelimiter\floor{\lfloor}{\rfloor}
\def \VC {\mathsf{VC}}
\def \x {\bm{x}}
\def \f {f}
\def \tlf {\tilde{f}}

\def \np {n_+}
\def \nn {n_-}
\def \rorg {{\hat{\mathcal{R}}}^\ell_{\alpha,\beta}}
\def \rpsi { {\hat{\mathcal{R}}}^{\ell}_{\psi}}
\def \sumterm {\frac{1}{\np\nn}}
\def \sumpn {\sum_{i=1}^{n_+}\sum_{j=1}^{n_-}}
\def \psumpn {\sum_{i=1}^{n^\alpha_+}\sum_{j=1}^{n^\beta_-}}
\DeclareMathOperator*{\argmax}{argmax}
\DeclareMathOperator*{\argmin}{argmin}
\def \XP {\mathcal{X}_P}
\def \XN {\mathcal{X}_N}
\def \sub {\overset{i.i.d}{\sim}}
\def \xp {\x^+}
\def \xpi {\xp_{i}}
\def \xpio {\xp_{(i)}}
\def \xn {\x^-}
\def \xnj {\xn_{j}}
\def \xnjo {\xn_{(j)}}
\newcommand{\fxpo}[1]{\xp_{(#1)}}
\newcommand{\fxno}[1]{\xn_{(#1)}}
\def \cp {c_+}

\def \cn {c_-}

\def \eP{\expe_{\xp\sim \mathcal{P}}}
\def \eN{\expe_{\xn \sim \mathcal{N}}}
\DeclareMathOperator*{\prob}{\mathbb{P}}
\def \lt {\ell_{0,1}}

\def \vp {{v}_+}
\def \vn {{v}_-}
\def \vpi {{v}^+_i}
\def \vnj {{v}^-_j}

\def \Ionep {\mathcal{I}_1^+}
\def \Ionen {\mathcal{I}_1^-}

\def \npa {\np^\alpha}
\def \nnb {\nn^\beta}
\def \expiop { \expeb_{\xp \in \Ionep } }

\def \expion { \expeb_{\xn \in \Ionen } }

\def \Itwo {\mathcal{I}_2}
\def \expit { \expeb_{\xp, \xn \in \Itwo} }
\def \expite{ \expe_{\xp, \xn \in \Itwo} }

\def \sumtermp {\frac{1}{\nn}}
\def \expiopn {\expeb_{\xp \in \Ionep, \xn \in \Ionen } }
\def \expiopne {\expe_{\xp \in \Ionep, \xn \in \Ionen } }

\def \cfab{\mathcal{C}_f^{\alpha, \beta}}
\def \cfabup{\mathcal{C}_f^{\alpha, \beta,\uparrow}}
\def \cfabdown{\mathcal{C}_f^{\alpha, \beta,\downarrow}}

\def \opo{{(OP_1)}}
\def \opz{{(OP_0)}}
\def \aucs{\hat{\mathcal{R}}^{\alpha, \beta}_{AUC}(\f, \mathcal{S})}
\def \aucf{  {\mathcal{R}^{\alpha, \beta}_{AUC}}(\f, \mathcal{S})}
\newcommand{\ind}[1]{\mathbbm{1}\left[ #1 \right]}
\def \da {\delta_\alpha}
\def \dah {\hat{\delta}_\alpha}
\def \db {\delta_\beta}
\def \dbh {\hat{\delta}_\beta}
\def  \fxp {\f(\xp)}
\def  \fxpi {\f(\xpi)}
\def  \fxn {\f(\xn)}
\def  \fxnj {\f(\xnj)}
\def \lpj {\ell_+(\f, \xnj)}
\def \sump {\sum_{i=1}^{\np}}
\def \sumn {\sum_{j=1}^{\nn}}
\def \rfab{\mathcal{R}^{\alpha, \beta}(f)}

\def \vinf {v_\infty}
\def \linf {\ell_\infty}

\usepackage{enumitem}

\newtheorem*{rthm1}{Reminder of Proposition \ref{prop:reform}}
\newtheorem*{rthm2}{Reminder of Proposition \ref{prop:concon}}
\newtheorem*{rthm3}{Reminder of Proposition \ref{prop:dual}}
\newtheorem*{rthm4}{Reminder of Theorem \ref{thm:reform}}
\newtheorem*{rthm5}{Reminder of Theorem \ref{prop:bayes}}
\newtheorem*{rthm6}{Reminder of Theorem \ref{thm:gen}}
\newtheorem*{rthm7}{Reminder of Theorem \ref{thm:abs}}
\newtheorem*{rthm8}{Reminder of Theorem \ref{thm:cnn}}
\newtheorem*{rthm9}{Reminder of Theorem \ref{thm:simplebayes}}

\newtheorem*{rlem1}{Reminder of Lemma \ref{lem:decomp}}
\newtheorem*{rlem3}{Reminder of Lemma \ref{lem:uppcover}}
\newtheorem*{rlem4}{Reminder of Lemma \ref{lem:rdcomp}}
\newtheorem*{rlem5}{Reminder of Lemma \ref{lem:uppstar}}

\newenvironment{flem}
{\begin{mdframed}[hidealllines=true,backgroundcolor=ballblue!15	,innerleftmargin=3pt,innerrightmargin=3pt,leftmargin=-3pt,rightmargin=-3pt]\begin{lem}}
{\end{lem}\end{mdframed}}

\newenvironment{frthm1}
  {\begin{mdframed}[hidealllines=true,backgroundcolor=ballblue!10	,innerleftmargin=3pt,innerrightmargin=3pt,leftmargin=-3pt,rightmargin=-3pt]\begin{rthm1}}
  {\end{rthm1}\end{mdframed}}

  \newenvironment{frthm2}
  {\begin{mdframed}[hidealllines=true,backgroundcolor=ballblue!10	,innerleftmargin=3pt,innerrightmargin=3pt,leftmargin=-3pt,rightmargin=-3pt]\begin{rthm2}}
  {\end{rthm2}\end{mdframed}}

  \newenvironment{frthm3}
  {\begin{mdframed}[hidealllines=true,backgroundcolor=ballblue!10	,innerleftmargin=3pt,innerrightmargin=3pt,leftmargin=-3pt,rightmargin=-3pt]\begin{rthm3}}
  {\end{rthm3}\end{mdframed}}

  \newenvironment{frthm5}
  {\begin{mdframed}[hidealllines=true,backgroundcolor=ballblue!10	,innerleftmargin=3pt,innerrightmargin=3pt,leftmargin=-3pt,rightmargin=-3pt]\begin{rthm5}}
  {\end{rthm5}\end{mdframed}}

  \newenvironment{frthm6}
  {\begin{mdframed}[hidealllines=true,backgroundcolor=ballblue!10	,innerleftmargin=3pt,innerrightmargin=3pt,leftmargin=-3pt,rightmargin=-3pt]\begin{rthm6}}
  {\end{rthm6}\end{mdframed}}

  \newenvironment{frthm7}
  {\begin{mdframed}[hidealllines=true,backgroundcolor=ballblue!10	,innerleftmargin=3pt,innerrightmargin=3pt,leftmargin=-3pt,rightmargin=-3pt]\begin{rthm7}}
  {\end{rthm7}\end{mdframed}}

  \newenvironment{frthm9}
  {\begin{mdframed}[hidealllines=true,backgroundcolor=ballblue!10	,innerleftmargin=3pt,innerrightmargin=3pt,leftmargin=-3pt,rightmargin=-3pt]\begin{rthm9}}
  {\end{rthm9}\end{mdframed}}

  \newenvironment{frthm4}
  {\begin{mdframed}[hidealllines=true,backgroundcolor=ballblue!10	,innerleftmargin=3pt,innerrightmargin=3pt,leftmargin=-3pt,rightmargin=-3pt]\begin{rthm4}}
  {\end{rthm4}\end{mdframed}}

  \newenvironment{frlem1}
  {\begin{mdframed}[hidealllines=true,backgroundcolor=ballblue!10	,innerleftmargin=3pt,innerrightmargin=3pt,leftmargin=-3pt,rightmargin=-3pt]\begin{rlem1}}
  {\end{rlem1}\end{mdframed}}

\def \gf{g_f}

\def \fa {f _{\mathfrak{A}}}

\def \fa {t_{1-\alpha}(f)}
\def \fb {t_{\beta}(f)}

\def \egf{E(g\mathcal{F})}
\def \gfi{g\mathcal{F}|_{\x_+ = \x_i}}

\def \x {\bm{x}}

\def \phil{\phi_{\ell}}

\ifCLASSOPTIONcompsoc
  \usepackage[nocompress]{cite}
\else
\fi
\ifCLASSINFOpdf
\else
\fi

\begin{document}
%
\title{Optimizing Two-way Partial AUC with an End-to-end Framework}

\author{Zhiyong~Yang,
        Qianqian~Xu*,~\IEEEmembership{Senior Member, IEEE },
        Shilong Bao, Yuan He, \\
        Xiaochun~Cao,~\IEEEmembership{Senior Member, IEEE},
        Qingming~Huang*,~\IEEEmembership{Fellow, IEEE}

\IEEEcompsocitemizethanks{

\IEEEcompsocthanksitem Zhiyong Yang is with School of Computer Science and Technology, University of Chinese Academy of Sciences, Beijing 100049, China (email: \texttt{yangzhiyong21@ucas.ac.cn}).\protect\\
\IEEEcompsocthanksitem Qianqian Xu is with the Key Laboratory of
Intelligent Information Processing, Institute of Computing Technology, Chinese
Academy of Sciences, Beijing 100190, China, (email: \texttt{xuqianqian@ict.ac.cn}).\protect\\
\IEEEcompsocthanksitem Shilong Bao is with State Key Laboratory of Information Security (SKLOIS), Institute of Information Engineering, Chinese Academy of Sciences, Beijing 100093, China, and also with School of Cyber Security, University of Chinese Academy of Sciences, Beijing 100049, China (email: \texttt{baoshilong@iie.ac.cn}).\protect \\
\IEEEcompsocthanksitem  Yuan He is with the Security Department of Alibaba
Group, Hangzhou 311121, China (e-mail \texttt{:heyuan.hy@alibaba-inc.com}). \protect\\
\IEEEcompsocthanksitem Xiaochun Cao is with School of Cyber Science and Technology, Shenzhen Campus, Sun Yat-sen University, Shenzhen 518107, China (caoxiaochun@mail.sysu.edu.cn).\protect\\
\IEEEcompsocthanksitem Q. Huang is with the School of Computer Science and Technology,
    University of Chinese Academy of Sciences, Beijing 101408, China, also
    with the Key Laboratory of Big Data Mining and Knowledge Management (BDKM),
    University of Chinese Academy of Sciences, Beijing 101408, China,  also
    with the Key Laboratory of Intelligent Information Processing, Institute of
    Computing Technology, Chinese Academy of Sciences, Beijing 100190, China, and also with Peng Cheng Laboratory, Shenzhen 518055, China 
    (e-mail: \texttt{qmhuang@ucas.ac.cn}).\protect\\
    \IEEEcompsocthanksitem *: Corresponding authors.
   }
}
%
%

\markboth{IEEE Transactions on Pattern Analysis and Machine Intelligence}%
{Shell \MakeLowercase{\textit{et al.}}: Bare Demo of IEEEtran.cls for Computer Society Journals}
%



\maketitle

\begin{abstract}
    The Area Under the ROC Curve (AUC) is a crucial metric for machine learning, which evaluates the average performance over all possible True Positive Rates (TPRs) and False Positive Rates (FPRs). Based on the knowledge that a skillful classifier should simultaneously embrace a high TPR and a low FPR, we turn to study a more general variant called Two-way Partial AUC (TPAUC), where only the region with $\mathsf{TPR} \ge \alpha, \mathsf{FPR} \le \beta$ is included in the area. Moreover, a recent work shows that the TPAUC is essentially inconsistent with the existing Partial AUC metrics where only the FPR range is restricted, opening a new problem to seek solutions to leverage high TPAUC. Motivated by this, we present the first trial in this paper to optimize this new metric.  The critical challenge along this course lies in the difficulty of performing gradient-based optimization with end-to-end stochastic training, even with a proper choice of surrogate loss. To address this issue, we propose a generic framework to construct surrogate optimization problems, which supports efficient end-to-end training with deep learning. Moreover, our theoretical analyses show that: 1) the objective function of the surrogate problems will achieve an upper bound of the original problem under mild conditions, and 2) optimizing the surrogate problems leads to good generalization performance in terms of TPAUC with a high probability. Finally, empirical studies over several benchmark datasets speak to the efficacy of our framework.
  
\end{abstract}

%
\IEEEpeerreviewmaketitle

\section{Introduction}

\label{sec:intro}
ROC (Receiver Operating Characteristics) curve is a well-known tool to evaluate classification performance at varying threshold levels. More precisely, as shown in Fig.\ref{fig:roc_compare}-(a), it captures the relationship between True Positive Rate (TPR) and False Positive Rate (FPR) as a function of the classification thresholds. {A}{UC} (Area Under the ROC Curve), summarizes the average performance of a given classifier by calculating its area. More intuitively, as shown in \cite{ROCmean}, AUC is equivalent to the possibility that a positive instance has a higher predicted score to be positive than a negative instance. Any skillful classifier that can produce well-separated scores for positive and negative instances will enjoy a high AUC value, no matter how skewed the class distribution is. As a natural result, AUC is more appropriate than accuracy for long-tail classification problems such as disease prediction \cite{Diagnosis1,Diagnosis2} and rare event detection \cite{Anomaly1,Anomaly2,transactions1,fin}, due to this appealing property \cite{rocsur,AUCM}.
\begin{figure}[th]  
  \centering
    \includegraphics[width=0.98\columnwidth]{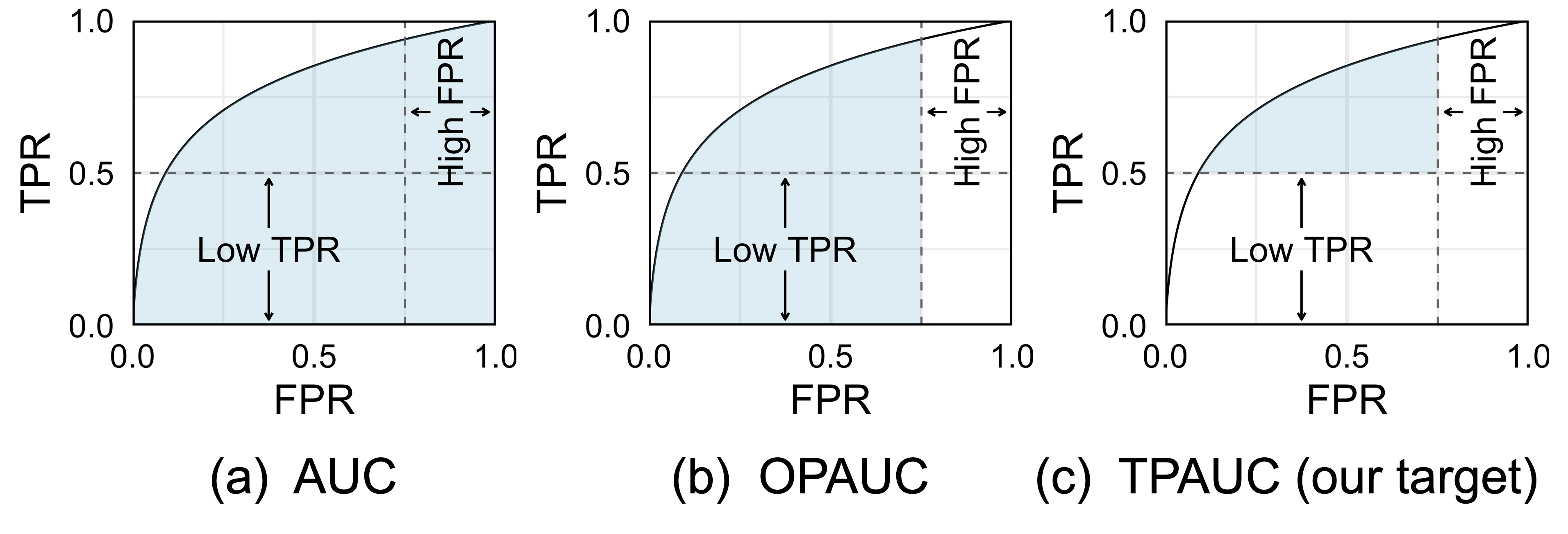} 
  \caption{\label{fig:roc_compare} Comparisons of different AUC variants: \textbf{(a)} The entire area of ROC curve; \textbf{(b)} The One-way Partial AUC (OPAUC) which measures the area of a local region of ROC within an FPR range; \textbf{(c)} The Two-way Partial AUC (TPAUC). }
\end{figure}

Over the past two decades, the importance of AUC has raised a new wave to directly optimize AUC, which has achieved tremendous success. A partial list of the related studies includes \cite{logitauc,svmauc1,svmauc2,logitauc2,partial2,auconepass,partial3}. However, the vast majority of such studies only consider the area over the entire ROC curve. As argued by \cite{partial1}, for some applications,  only the performance within a given range of FPR is of interest, as shown in Fig.\ref{fig:roc_compare}-(b). In this sense, the standard AUC tends to provide a biased estimation of the performance by including unrelated  regions. This key investigation has motivated a series of successful studies to optimize the One-way Partial AUC (OPAUC) with an FPR range $[\alpha, \beta]$  \cite{partial1, partial2, partial3}. Here we note that the choice to truncate FPR on the ROC curve is based on domain-specific prior knowledge for some specific fields such as biometric screening, and medical diagnosis \cite{partial1}.  

\emph{\textbf{Taking a step further, what should be a \underline{general rule} to select the target region under the ROC curve for classification problems?}}

Since TPR and FPR evaluate complementary aspects of the model performance, we argue that a practical classifier in most applications must simultaneously have a high TPR and a low FPR. In other words, a high TPR is meaningless if the FPR is lower than a tolerance threshold, while a low FPR cannot compensate for a low TPR (\emph{say, one can hardly consider a model with FPR higher than 0.8 even if its TPR is as high as 0.99, and vice versa for a low TPR model}).  In this sense, we only need to pay attention to the upper-left head region under the ROC curve, as shown in Fig.\ref{fig:roc_compare}-(c).

A recent work \cite{tpauc} exactly realizes this idea, where a new metric called Two-Way Partial AUC (TPAUC) is proposed to measure the area of a partial region of the ROC curve with $\mathsf{TPR} \ge p, \mathsf{FPR} \le q$. Furthermore, \cite{tpauc} shows that the TPAUC is essentially inconsistent with one-way partial AUC. In other words, a higher OPAUC does not necessarily imply a higher TPAUC, posing a demand to seek new solutions to leverage high TPAUC. 

\emph{\textbf{Inspired by this fact, we present the \underline{first trial} to optimize the TPAUC metric with an end-to-end framework.}}

The \emph{\textbf{major challenge}} of this task is that the objective function is not differentiable even with a proper surrogate loss function, suggesting that there is no easy way to perform end-to-end training. Facing this challenge, we propose a generic framework to approximately optimize the TPAUC with the help of deep learning. Generally speaking, our contributions are as follows.

First, we reformulate the original optimization problem as a bi-level optimization problem, where the inner-level problem provides a sparse sample selection process and the outer-level problem minimizes the loss over the selected instances.
On top of the reformulation, we propose a generic framework to construct surrogate optimization problems for the original problem. In the core of this framework lies the interplay of the surrogate penalty functions and surrogate weighting functions defined in this paper. Moreover, we construct a dual correspondence between these two classes of functions, such that we can easily find a standard single-level surrogate optimization problem whenever a surrogate penalty or a surrogate weighting function is obtained. We also provide a minimax reformulation of the objective function to induce an instance-wise objective function to improve efficiency. 

We then proceed to explore theoretical guarantees for the framework. First, by comparing the surrogate problem and the original problem, we provide a mild sufficient condition under which the surrogate objective function  becomes an upper bound of the original one. Meanwhile, we also show that the concave weighting function tends to be a better choice than its convex counterparts. Second, we present a theoretical analysis of the Bayes-optimal scoring functions for TPAUC optimization. The result shows that the Bayes scoring function for TPAUC is much harder to obtain than the plain AUC. Finally, we show that optimizing the surrogate problems could leverage good generalization performance in TPAUC, with two upper bounds of the excess risk. One of our results attains much sharper order than previous results on AUC generalization. 

This paper extends our ICML 2021 long-talk paper \cite{ours}, where we proposed a generic framework to implement the end-to-end TPAUC optimization. In this version, we have included systematic improvements in theory, methodology, and experiments. The novelty of the extended version is summarized as follows:

\begin{enumerate}
  \item \textbf{New Theoretical Results.} We derive the Bayes-optimal solution for TPAUC optimization in Thm.\ref{prop:bayes}, Thm.\ref{thm:simplebayes}. The result shows that the optimal scoring functions are required not only to be aligned with $\eta(\bm{x})$ in a local region, but also to choose the optimal hard examples. Moreover, under Asum.\ref{asum:bound}, we derive a sharp upper bound for the excess risk upper bound for a general form of hypothesis classes in Thm.\ref{thm:abs}.  
  \item \textbf{New Methodology.} The original objective function is formed as a sum of loss functions of instance pairs, which might lead to a heavy computational burden and the difficulty of generating a mini-batch. To address this issue, we propose an instance-wise minimax reformulation of the original objective function in Thm.\ref{thm:reform}.
  \item \textbf{New Experiments.} In the current version, we include a series of new experiments, including new competitors, new finer-grained analysis, and the validation of the upper bounds in Prop.\ref{prop:concon}. 
  \item \textbf{Miscellaneous Contents.} Besides, we also refine the details of some existing contents, including the review of the related works (throughout Sec.\ref{sec:art}), the analysis of the inconsistency between TPAUC and OPAUC (in Sec.\ref{sec:incon}), and practical illustration of Prop.\ref{prop:concon} (in Sec.\ref{sec:surro}). 
\end{enumerate}

\section{Prior Art}\label{sec:art}
\subsection{General AUC Optimization}
Back in 1954, ROC  made its first appearance in the literature of signal processing theory \cite{aucbegin} as a tool to study the detectability of signals in the presence of noise, which has thereafter become a fundamental tool for signal detection together with AUC\cite{bestauc,auctest}. Two decades ago, it started to gain attention from machine learning community due to its elegant statistical properties\cite{mlroc1,mlroc2,mlroc3,mlroc4,AUCM,aucdis}.  For example, \cite{aucdis} argues that ROC or AUC is a better metric than accuracy in terms of the discriminancy to differentiate model performance.  At almost the same time, a motivating study \cite{AUCvErr} further points out that maximizing AUC  should not be replaced with minimizing the error rate, which shows the necessity to study direct AUC optimization methods. 
After that, a series of algorithms are designed for optimizing AUC. At the early stage, the majority of studies focus on a full-batch off-line setting. \cite{logitauc} optimizes AUC based on a logistic surrogate loss function and ordinary gradient descent method. \cite{logitauc2} approximates the original $0\text{-}1$ loss with a  polynomial function and obtains an efficient optimization method.  RankBoost \cite{boost} provides an efficient ensemble-based AUC learning method based on a ranking extension of the AdaBoost algorithm. The work of \cite{svmauc2,svmauc3} constructs $SVM^{struct}$-based frameworks that optimize a direct upper bound of the $0-1$ loss version AUC metric instead of its surrogates. \cite{reg1,reg2} also study the joint effort of AUC optimization and regression.

Later on, to accommodate big data analysis, researchers start to explore online extensions of the AUC optimization methods. \cite{online1} provides an early trial for this direction based on the reservoir sampling technique. \cite{auconepass} provides a completely one-pass AUC optimization method for streaming data based on the squared surrogate loss. Moreover,  \cite{minimaxauc1} reformulates the squared-loss-based stochastic AUC maximization problem as a stochastic saddle point problem. The new saddle point problem's objective function only involves the summations of instance-wise loss terms, which significantly reduces the burden from the pairwise formulation. \cite{minimaxauc2,minimaxauc3} further accelerate this framework with tighter convergence rates. Most recently, it also has been applied to structured models \cite{aucstructdata}, deep learning \cite{minimaxdeep} and federated learning \cite{fedauc}.

In this paper, we take a further step to optimize the TPAUC. From the optimization perspective, the difficulty comes from the much more complicated formulation of the surrogate risk of TPAUC. To address this issue, we derive an approximated surrogate risk to support end-to-end training. Moreover, we also employ the saddle point reformulation technique \cite{minimaxauc1} to derive an instance-wise objective function.  

\subsection{Generalization Analysis for AUC}
It is well-known that the recipe of machine learning comes from generalization, namely, the ability of a well-trained model to predict unseen data points. Generally speaking, its theoretical foundation could be traced back at least to the PAC theory of Valiant \cite{paclearn} and the statistical learning theory of Vapnik \cite{nature}. Nowadays, understanding the mathematical foundation of generalization has emerged as one of the most long-standing goals for the machine learning community\cite{achieve2,mlfun, mathfun,deeptheory}. It is noteworthy that the standard techniques for generalization analysis \cite{mlfun,achieve2,mathfun} require the empirical risk function to be expressed as a sum of independent terms. Unfortunately, the AUC-based empirical risks do not satisfy this assumption, which poses a big challenge toward understanding its generalization guarantee.   To solve this problem, many efforts have been made against the interdependency of AUC loss terms. The first study on the generalization property of the AUC optimization problem could be traced back to \cite{genal}, where a specific notion of VC-dimension is proposed for the AUC risks.  Later, \cite{aucrade} derived a novel symmetrization technique for the pairwise AUC risk function, where a novel notion of Rademacher complexity is proposed for the AUC risks. However, both of the studies require significant revision of the standard technique. Thereafter, a new wave of ideas emerges where the AUC risks are reformulated into a (series of) standard risk function(s).  For example, \cite{geninde,part} employ the graph coloring technique to divide interdependent sample sets into a sequence of independent subsets. Within each subset, the samples satisfy the standard assumption so that traditional techniques such as the McDiarmid's inequality \cite{concen} can be extended to interdependent datasets. Meanwhile, \cite{cle1} proposes a U-statistics-based framework to derive generalization bounds for bipartite ranking, which is closely related to the AUC optimization problem. The key idea is to reformulate the U-statistics as a sum of independent terms. Moreover, it also proposes sharp results under certain assumptions. Later on, \cite{entropy} provides a novel Talagrand concentration inequality for sub-additive and fractionally self-bounding functions, making it possible to introduce the notion of local Rademacher complexity \cite{local} into the AUC optimization problems. 

Besides the binary setting of the AUC optimization, there are also new explorations toward the generalization ability of AUC-like metrics under more complicated scenarios such as multipartite ranking \cite{multipart1,multipart2}, multi-class classification \cite{multiauc} and fairness learning \cite{AUCfair}.  

In this work, we will present two generalization results under two different assumptions for TPAUC optimization problem.  We could find an upper bound for the excess risk in order of $O(\frac{\mathsf{polylog(\max\{\np, \nn\})}}{\min\{\np  \nn\}})$ which is sharper than the state-of-art result $O(\sqrt{\frac{1}{(\min\{\np,\nn\})}})$ for AUC \cite{entropy, geninde} and pretty close to the $O(1/n)$ result for U-statistics \cite{cle1} but with much weaker assumptions. Here, $\np,\nn,n$ are the number of positive, negative and total instances for a binary classification dataset, respectively.

\subsection{Partial AUC Optimization} 
Generally speaking, directly optimizing the OPAUC requires techniques of combinatorial optimization. As a classical work, \cite{partial2} proposes a cutting-plane-based framework to perform the optimization. Later on \cite{partial2, partial3} provide refinement of \cite{partial2} by formulating a tight convex upper bound of the original objective function. Recently, \cite{AUCtime,semipauc} also extend this framework to time-series prediction and semi-supervised learning. \cite{aucnewonline} proposes an online learning method for OPAUC optimization with sublinear regret and online to batch conversion bounds. However, it requires the convexity of the loss function and a memory buffer.  Consequently, though scaling well for non-deep learning models, it still cannot  be employed to perform deep learning.

To make end-to-end training of deep models available, \cite{samplepauc1, samplepauc2} present meaningful attempts to apply OPAUC to the speaker verification task. Generally speaking, such algorithms adopt the following steps to perform the optimization: a) construct the pairwise training set by random sampling; b) pick the negative imposters over the FPR range $[\alpha,\beta]$; c) optimize the one-way partial AUC with the picked negatives and all positive samples. In this way, one can perform end-to-end training on top of the generated dataset. However, their work is specifically designed for metric learning settings, which is not available for generic scenarios. Besides, there is no theoretical guarantee that the training on the generated dataset could generalize well on unseen data.

Compared with existing studies to optimize partial AUC, the key difference is two-fold. First, the previous studies only focus on  one-way partial AUC, where only the FPR is restricted within $[\alpha, \beta]$; while we are the first to study TPAUC optimization, a new AUC metric where both TPR and FPR are truncated. Moreover, most related studies either adopt the cutting plane framework that is not scalable toward the training samples or adopt the generated dataset without a theoretical guarantee.  In our work, getting rid of complicated combinatorial optimization techniques, we propose a general framework to construct much simpler surrogate optimization problems for TPAUC that supports end-to-end training based on systematic theoretical guarantees.

\subsection{Learning with Constrained Optimization}
According to the work \cite{aucconst1}, we can as well formulate the AUC maximization as a constrained optimization problem by
binning the FPR range into a finite number of thresholds and constraining them to satisfy
specific FPR/TPR values. In this sense, an alternative way to optimize partial AUC is to consider the constrained optimization framework proposed in \cite{aucconst3,aucconst4}. For example, based on the most recent work in this direction \cite{aucconst2}, it is possible to perform OPAUC optimization by considering a selected range of FPRs in the constraints.  

By contrast, our paper aims to formulate a partial AUC optimization problem on top of a simple optimization framework, which could be implemented at ease with the deep learning framework. Hence, our work is based on the unconstrained form of AUC risk. The constrained form of TPAUC optimization is considered an opening problem worth broad future effort.
\subsection{Top Instance Ranking}
Last but not least, we discuss the major difference between our work with another similar direction of top instance ranking, where only the ranking of the first few instances in the ranking list is optimized \cite{top1,top2}. First of all, in our setting, we are not optimizing the rank of the top instances but rather the relative ranking of the bottom positive instances and the top negative instances, which are often much harder to be ranked in the correct place. Naturally, our theoretical analysis also differs from what is carried out in these works. Secondly, our proposed framework directly supports end-to-end training of deep neural networks, which is not studied in the literature.

\section{Preliminaries}
\subsection{Standard AUC metric}
Before showing the formal definition of the two-way partial AUC, we first provide a quick review of the standard AUC metric. Under the context of binary classification problems, an instance is denoted as $(\x,y)$, where $\x \in \mathcal{X}$ is the input raw features and $y \in \left\{0,1 \right\}$ is the label. Taking a step further, given a dataset $\mathcal{D}$, denote $\XP$ as the set of positive instances in our dataset, and $\XN$ as the set of the negative ones, then the sampling process could be expressed as:
\begin{equation*}
    \begin{split}
      &\XP = \{\xpi\}_{i=1}^{\np}~ \sub \mathcal{P}: ~\prob\left[\xp | y =1 \right], \\
      &\XN = \{\xnj\}_{j=1}^{\nn}~ \sub \mathcal{N}:~ \prob\left[\xn | y =0 \right], \\
    \end{split}
  \end{equation*}
where $\np, \nn$ are the numbers of positive/negative instances, respectively; and $\mathcal{P},  \mathcal{N}$ are the corresponding conditional distributions.
For binary class problems, our goal is to learn a score function $\f: \mathcal{X} \rightarrow [0,1]$, such that $\f(\x)$ is proportional to the possibility that $\x$ belongs to the positive class. Based on the score function, we can further predict  the label of an instance $\x$ as $\ind{\f(\x) > t}$, where $t$ is the decision threshold, \textbf{$\ind{\cdot}$ is the indicator function}. Given a threshold $t$, we can define two elementary metrics known as True Positive Rate (TPR) and False Positive Rate (FPR), which are the probabilities that a positive/negative instance is predicted as a positive instance, \textit{{i.e.}}:
\begin{equation}
  \begin{split}
    &\mathsf{TPR}_{\f}(t) = \prob_{\xp \in \mathcal{P}}\left[ \f(\xp) >t \right],\\ 
    &\mathsf{FPR}_{\f}(t) = \prob_{\xn \in \mathcal{N}}\left[ \f(\xn) >t \right].
  \end{split}
\end{equation}
Based on the label predictions,  AUC is defined as the Area under the Receiver Operating Characteristic (ROC) curve plotted by True Positive Rate (TPR) against False Positive Rate (FPR) with varying thresholds, which could be expressed mathematically as follows:
\begin{equation}
  \begin{split}
    \mathsf{AUC}(\f) = \int_{0}^1 \mathsf{TPR}_{\f}\left(\mathsf{FPR}^{-1}_{\f}(t)\right)dt.\\ 
  \end{split}
\end{equation}
 When the possibility to observe a tied comparison is null, i.e. 
 \[\prob_{\xp \in \mathcal{P}, \xn \in \mathcal{N}}\left[\f(\xp) = \f(\xn)  \right]= 0,\] 
 AUC is known \cite{ROCmean} to enjoy a much simpler formulation as the possibility that correct ranking takes place between a positive and negative instance: 
\begin{equation*}
    \mathsf{AUC}(\f) = 1- \eP\left[ \eN\left[\lt\left( \f(\xp) -\f(\xn) \right)  \right] \right],
  \end{equation*}
where $\lt$ denotes the $0-1$ loss with $\lt(\x) = 1$ if $x<0$, and  $\lt(\x) = 0$, otherwise. Given a finite dataset $\mathcal{S} = \XP\cup\XN$, the unbiased estimation of $\mathsf{AUC}(\f)$ could be expressed as:
\begin{equation*}
  \widehat{\mathsf{{AUC}}}(\f) = 1-  \sumpn \frac{\lt\left( \f(\xpi) -\f(\xnj) \right)}{\np\nn}.
\end{equation*}

\subsection{Two-Way Partial AUC Metrics}\label{subsec:tpauc}
\textbf{Definitions}. As presented in the introduction, instead of the complete area of ROC, we focus on the area of ROC in a partial region with $\mathsf{TPR}_{\f}(t) \ge 1 - \alpha, \mathsf{FPR}_{\f}(t) \le \beta$, which is called two-way partial AUC in \cite{tpauc}. Here we define it as $\mathsf{AUC}_{\alpha}^\beta$:

\begin{equation*}
  \begin{split}
  \mathsf{AUC}_{\alpha}^\beta(\f) =& \int_{\mathsf{FPR}_{\f}\left(\mathsf{TPR}^{-1}_{\f}(1-\alpha)\right)}^\beta \mathsf{TPR}_{\f}\left(\mathsf{FPR}^{-1}_{\f}(t)\right)dt \\
  & - \big(1-\alpha\big) \cdot \left(\beta - \mathsf{FPR}_{\f}\left(\mathsf{TPR}^{-1}_{\f}(1-\alpha)\right) \right).
  \end{split}
\end{equation*}
Moreover, TPAUC could also be reformulated as a bipartite ranking accuracy across positive and negative classes:
\begin{equation*}
  \mathsf{TPAUC}(\f) = 1- \eP\left[ \eN\left[\delta_{+,-}\lt\left( \f(\xp) -\f(\xn) \right)  \right] \right],
\end{equation*}
where
\begin{align*}
  \delta_{+,-} = \ind{\f(\xp) \le t_{1-\alpha}(\f)} \cdot \ind{\f(\xn) \ge t_{\beta}(\f)},
\end{align*}
and $ t_{1-\alpha}(\f)$ and $ t_{\beta}(\f)$ are two quantiles:
\begin{align*}
  t_{1-\alpha}(f)  &= \argmin_{\delta \in \mathbb{R}}\left[\delta \in \mathbb{R}:~ \eP\left[ \ind{\fxp \le \delta} \right] = \alpha  \right]
  \\
    t_{\beta}(f) &=\argmin_{\delta \in \mathbb{R}}\left[\delta \in \mathbb{R}:~ \eN\left[ \ind{\fxn \ge \delta} \right] = \beta  \right].
 \end{align*}
Different from AUC, the extra term $\delta_{+,-}$ allows TPAUC to attend to the hard positive (with a score below a quantile) and hard negative (with a score above a quantile). 

Since the data distributions $\mathcal{P}, \mathcal{N}$ are often unknown, it is necessary to study the empirical estimation of TPAUC based on an observed dataset $\mathcal{S}$.  \cite{tpauc} derives an empirical version of $\mathsf{AUC}_{\alpha}^\beta(\f)$, where both the expectation and the quantiles are replaced directly with their empirical version. We denote it as $\hat{\mathsf{AUC}}_\alpha^\beta(\f,\mathcal{S})$ in our paper:
\begin{equation*}
  \hat{\mathsf{AUC}}_\alpha^\beta(\f, \mathcal{S})  = 1- \psumpn \frac{\lt\left(\f(\xpio) - f(\xnjo)  \right)}{n_+n_-}
\end{equation*}
where $\xpio$ denotes the hard positive instance that achieves \emph{\textbf{bottom-$i$}} score among all positive instances, and $\xnjo$ denotes the hard negative instance achieves\emph{ \textbf{top-j} }score among all negative instances, $\npa = \floor{\np \cdot \alpha}$, and $\nnb = \floor{\nn \cdot \beta}$, are the numbers of the chosen hard positive and negative examples.  
\subsection{Inconsistency between OPAUC and TPAUC}\label{sec:incon}
\begin{figure}[h]  
  \centering
    \includegraphics[width=0.8\columnwidth]{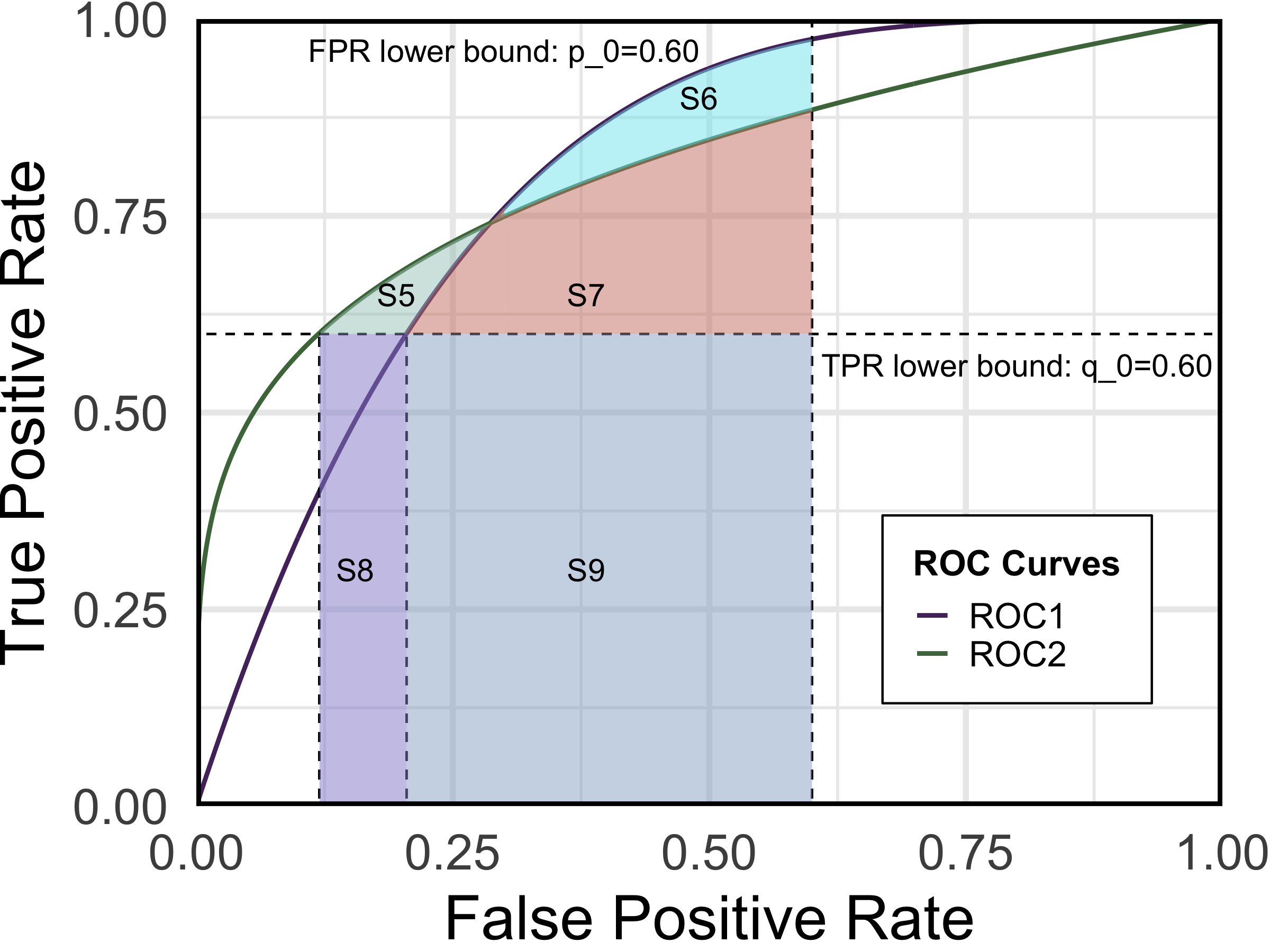} 
  \caption{\label{fig:incons}\textbf{TPR-dependent OPAUC \textit{vs.} TPAUC}: Practically, OPAUC and TPAUC tend to exhibit different comparison results when two ROC curves intersect. For the two ROC curves in the figure, it is easy to see that $TPAUC_1 = S_6 + S_7$ and  $TPAUC_2 = S_5 + S_7$. By calculating their tightest OPAUC upper bounds, we have $OPAUC_1 = S_6 + S_7 + S_9$ and $OPAUC_2 = S_5 + S_7 + S_8 + S_9$. In this case, we have $OPAUC_1 < OPAUC_2$ but $TPAUC_2 < TPAUC_1$ whenever $S_6<S_5 +S_8$ and $S_6>S_5$. Moreover, $ROC_1$ tends to have better higher TPR than $ROC_2$ in the most points within the restricted ROC region. TPAUC is more consistent with a better performance curve than TPR-dependent OPAUC.}
\end{figure}

In this section, we show the inconsistency between TPAUC metric and the OPAUC metric. Mathematically, OPAUC calculates the partial AUC within FPR range $[\alpha,\beta]$, which could be defined as:
  \begin{equation*}
    \begin{split}
    \mathsf{AUC}_{\alpha}^{\beta^{ \mathsf{OP}}}(\f) =& \int_{\alpha}^\beta \mathsf{TPR}_{\f}\left(\mathsf{FPR}^{-1}_{\f}(t)\right)dt .
    \end{split}
\end{equation*}
Recall that TPAUC could be defined as:
\begin{align*}
  \mathsf{AUC}_{\alpha}^{\beta^{ \mathsf{TP}}}(\f) &= \int_{\mathsf{FPR}_{\f}\left(\mathsf{TPR}^{-1}_{\f}(1-\alpha)\right)}^\beta \mathsf{TPR}_{\f}\left(\mathsf{FPR}^{-1}_{\f}(t)\right)dt \\ 
  &- (1-\alpha) \cdot \left(\beta - \mathsf{FPR}_{\f}\left(\mathsf{TPR}^{-1}_{\f}(1-\alpha)\right) \right).
\end{align*}
\noindent From the definitions, we can find that the TPAUC is intrinsically inconsistent with OPAUC. The source of the inconsistency is that both $\mathsf{FTR}_{\f}$ and  $\mathsf{TPR}_{\f}$ are functions of $\f$. It is thus impossible to regard $\mathsf{FPR}^{-1}_{\f}(1-\alpha)$ and $\mathsf{TPR}_{\f}(\mathsf{FPR}^{-1}_{\f}(1-\alpha))$ as constants. Thus one cannot simply replace the  FPR lower bound $\mathsf{FPR}^{-1}_{\f}(1-\alpha)$ with any constant $c$. Consequently, $\mathsf{AUC}_\alpha^\beta(\f)$ is in general not consistent with any OPAUC with  FPR range $[c, \beta]$. 
But can we estimate TPAUC with an OPAUC metric? A natural choice of such OPAUC could be defined as follows:
\begin{align*}
 \int_{\mathsf{FPR}_{\f}\left(\mathsf{TPR}^{-1}_{\f}(1-\alpha)\right)}^\beta \mathsf{TPR}_{\f}\left(\mathsf{FPR}^{-1}_{\f}(t)\right)dt 
\end{align*}
which is called TPR-dependent OPAUC in our paper since its lower limit is defined by a TPR lower bound.  As suggested by \cite{tpauc}, TPR-dependent OPAUC and TPAUC provide inconsistent results when the two ROC curves compared intersect. As shown in Fig.\ref{fig:incons}, we calculate TPAUC with $\alpha = 0.4$ and $\beta = 0.6$. It is easy to see that $TPAUC_1 = S_6 + S_7$ and  $TPAUC_2 = S_5 + S_7$. Next, we calculate the TPR-dependent OPAUC. As shown in the figure, we have $OPAUC_1 = S_6 + S_7 + S_9$ and $OPAUC_2 = S_5 + S_7 + S_8 + S_9$. In this case, we have $OPAUC_1 < OPAUC_2$ but $TPAUC_2 < TPAUC_1$ whenever $S_6<S_5 +S_8$ and $S_6>S_5$. 

Since TPR-dependent OPAUC and TPAUC provide different comparison results, we continue to ask why TPAUC is better than TPR-dependent OPAUC. Recall that our primal goal is to optimize the performance curve with TPR lower bound and FPR upper bound. Back to our example in Fig.\ref{fig:incons}, it is easy to see that, in an average sense, $ROC_1$  
has higher TPR values than $ROC_2$ in the restricted region. (Note that the restricted regions are different for $ROC_1$ and  $ROC_2$). So $ROC_1$ is obviously better than $ROC_2$. In this sense, TPAUC can provide a consistent result with the performance while TPR-dependent OPAUC fails to do so. Such failure is because the OPAUC of $ROC_2$ comes from its wide lower rectangle, which only has a weak connection with the TPR values up in the curve.      

\subsection{Covering Number as A Measure of Complexity}
To obtain an upper bound of the Rademacher complexity, we need to recall the notion of covering number and its chaining bound. First, we provide the definition of $\epsilon$-covering and the covering number, which are elaborated in Def.\ref{def:ecover} and Def.\ref{def:covernum}.
\begin{defi}[$\epsilon$-covering]\label{def:ecover}\cite{tala} Let $(\mathcal{H}, d)$ be a (pseudo)metric space, and $\Theta \in \mathcal{H}$. $\{h_1,\cdots, h_K\}$ is said to be an  $\epsilon$-covering of $\Theta$ if $\Theta \in \bigcup_{i=1}^K\mathcal{B}(h_i,\epsilon)$, \emph{i.e.}, $\forall \theta \in \Theta$, $\exists i ~s.t.~ d(\theta, h_i) \le \epsilon$.
\end{defi}
\begin{defi}[Covering Number]\label{def:covernum}\cite{tala} Based on the notations in Def.\ref{def:ecover}, the covering number of $\Theta$ with radius $\epsilon$ is defined as:
\begin{equation*}
\mathcal{N}(\Theta,\epsilon, d) = \min\left\{n: \exists \epsilon-\text{covering over $\Theta$ with size n}\right\}.
\end{equation*}
\end{defi}
To calculate the covering number, one has to first clarify the metric $d$. In our proof, we will employ the following norms to induce a metric. Given an input set $\mathcal{X}$, for  functions ${f}(\cdot),\tilde{{f}}(\cdot): \mathbb{X} \rightarrow \mathbb{R}_+$  we have:
\begin{align*}
  & ||{f} -\tilde{f}||_{2,\np} = \sqrt{\frac{1}{\np}  \sump (f(\xpi) - \tilde{f}(\xpi))^2},\\ 
  & ||f -\tilde{f}||_{2,\nn} = \sqrt{\frac{1}{\nn}  \sumn (f(\xnj) - \tilde{f}(\xnj))^2},
\end{align*}

\section{The Proposed Framework}

\subsection{A Generic Framework to Construct Surrogate Optimization Problems}\label{sec:surro}
Based on the empirical estimation shown in Sec.\ref{subsec:tpauc}, it is clear that optimizing TPAUC over a finite dataset $\mathcal{S}$ requires minimizing the following quantity:
\begin{equation*}
  1- \hat{\mathsf{AUC}}_\alpha^\beta(\f, \mathcal{S})  = \psumpn \frac{\lt\left(\f(\xpio) - f(\xnjo)  \right)}{n_+n_-}.
\end{equation*}
Following the framework of surrogate loss \cite{mlfun}, we replace the non-differential 0-1 loss with a convex loss function $\ell$, such that $\ell(t)$ is an upper bound of  $\lt(t)$. Note that if the scores live in $[0,1]$, standard loss functions such as $\ell_{\exp}(t) = \exp(-t)$, $\ell_{sq}(t) = (1-t)^2$  often satisfy this constraint. Hence given a feasible surrogate loss $\ell$, our goal is then to solve the following problem:
\begin{equation*}
  \begin{split}
    \opz  \min_{\bm{\theta}}  \rorg(S,\f) =  \psumpn \frac{\ell\left(\f(\xpio) - f(\xnjo)  \right)}{n_+n_-}.
  \end{split}
\end{equation*}
\textbf{where $f$ is parameterized by $\bm{\theta}$.}
Unfortunately, even with the choice of differentiable surrogate losses, the objective function $\rorg(S,\f)$  is still not differentiable. This is because calculating $\xpio, \xnjo$ requires sorting the scores of positive and negative instances. Nonetheless, the objective function is essentially a composition of a sparse sample selection operation and the original loss. This is shown in the following proposition, where $\opz$ is reformulated  as a so-called bi-level optimization problem \cite{bilevel-1,bi-level-2}. The inner-level problems provide a sparse sample selection process, and the outer-level problem performs the optimization based on the chosen instances.  {{Please see Appendix.\ref{app:prop1} for the proof.}}

  \begin{prop}\label{prop:reform}
  For any $\alpha, \beta \in (0,1)$, if scores $\f(\x) \in [0,1]$,  and there are no ties in the scores, the original optimization problem is equivalent to the following problem:
     \begin{equation*}
       \begin{split}
          &\min_{\bm{\theta}}  \frac{1}{\np\nn}\sumpn \vpi \cdot \vnj \cdot \ell(\f, \xpi, \xnj)\\ 
          s.t.~~&\vp = \argmax_{\vpi \in [0,1], \sum_{i=1}^{\np} \vpi \le \npa} \sum_{i=1}^{\np} \left(\vpi \cdot (1- \f(\xpi))\right)\\ 
          &\vn = \argmax_{\vnj \in [0,1], \sum_{j=1}^{\nn} \vnj \le \nnb} \sum_{j=1}^{\nn} \left(\vnj \cdot \f(\xnj) \right)
       \end{split}
     \end{equation*}
     where 
     \begin{equation*}
      \ell(\f, \xpi, \xnj) =   \ell(\f(\xpi) - \f(\xnj)).
     \end{equation*}
    \end{prop}

    Based on the proposition, the source of the intractability of $\opz$ comes from the $\ell_1$ ball constraints   $\sum_{i=1}^{\np} \vpi \le \npa,~ \sum_{j=1}^{\nn} \vnj \le \nnb $ in the inner-level problem. Moreover, we can find out  $\lambda_+, \lambda_-$ such that the $\ell_1$ constraints can be equivalently expressed as a penalty function:
    \begin{equation*}
      \begin{split}
        &\vp = \argmax_{\vpi \in [0,1]} \sum_{i=1}^{n_+} \left(\vpi \cdot (1- \f(\xpi))  - {\lambda^+ \cdot \vpi}\right)\\ 
          &\vn = \argmax_{\vnj \in [0,1]} \sum_{j=1}^{\nn} \left(\vnj \cdot \f(\xnj)  - {\lambda^- \cdot \vnj} \right)
      \end{split}
    \end{equation*}

    To establish an efficient approximation of the original problem,  we leave $\lambda^+,\lambda_-$ as hyperparameters to be tuned. Furthermore, to avoid sparsity, we replace the sparsity-inducing  $\ell_1$ penalty with a smooth surrogate $\varphi_\gamma$. This naturally leads to a smooth problem:  
    \begin{equation*}
  \begin{split}
    \opo & \min_{\bm{\theta}}  \frac{1}{n^\alpha_+n^\beta_-}\sumpn \vpi \cdot \vnj \cdot \ell(\f, \xpi, \xnj)\\ 
     s.t~~&\vp = \argmax_{\vpi \in [0,1]} \sum_{i=1}^{n^+} \left(\vpi \cdot (1- \f(\xpi))  -  {{\varphi_{\gamma}(\vpi)}}\right)\\ 
     &\vn = \argmax_{\vnj \in [0,1]} \sum_{j=1}^{\nn} \left(\vnj \cdot \f(\xnj)  - {{\varphi_\gamma(\vnj)}} \right)
  \end{split}
\end{equation*}
To ensure that the chosen $\varphi_\gamma$ provides an effective approximation of the $\ell_1$ penalty, we pose several regularities on such functions. In the following, we define this class of functions as the calibrated smooth penalty function.
  
  \begin{defi}\label{def:pen}
A penalty function $\varphi_\gamma(x): \mathbb{R}_+ \rightarrow \mathbb{R}$  is called a \emph{\textbf{calibrated smooth penalty function}}, if it satisfies the following regularities:
\begin{enumerate}
  \item[(A)] $\varphi_\gamma$ has continuous third-order derivatives.
  \item[(B)] $\varphi_\gamma$ is strictly increasing in the sense that $\varphi_\gamma'(x) >0$.
  \item[(C)] $\varphi_\gamma$ is strictly convex in the sense that $\varphi_\gamma''(x) >0$.
  \item[(D)] $\varphi_\gamma$ has positive third-order derivatives in the sense that $\varphi_\gamma'''(x) >0$.
\end{enumerate}
    \end{defi}
    Note that since we restrict the output of $\f$ to $[0,1]$, condition (B) is inherited from the $\ell_1$ norm for non-negative inputs. While the other conditions improve the smoothness of the function. Moreover, the last condition is to ensure that the weighting function is strictly concave (see the arguments about the weights).

    Given the penalty functions, we turn to explore a corresponding factor in the framework. 
    According to the inner level problem, the sample weights $\vpi, \vnj$ have a dual correspondence with the penalty functions. More precisely, given a fixed $\varphi_\gamma$, one can derive the corresponding weighting function $\psi_\gamma$ as a function of $\f(\x)$ such that:
\begin{equation*}
  \vpi = \psi_\gamma(1-\f(\xpi)),~ \vnj= \psi_\gamma(\f(\xnj)),~ \vpi,\vnj \in [0,1].
\end{equation*}
Moreover, if $\psi_\gamma$ has a closed-form expression, then we can cancel the inner optimization problem and instead minimize the following weighted empirical risk $\rpsi$:
\begin{equation}\label{eq:surr}
  \begin{split}
    \rpsi(\mathcal{S},\f) =&  \frac{1}{\np\nn}\sumpn \psi_\gamma(1-\f(\xpi))  \cdot  \\ &\psi_\gamma(\f(\xnj)) \cdot \ell(\f, \xpi, \xnj).
  \end{split}
\end{equation}

In this sense, adopting a smooth penalty function ends up with a dual soft weighting strategy over the hard instances. Again, to reach a proper weight function, we also require it to satisfy some necessary regularities. In the following, we define this class of functions as the calibrated weighting function. 

    \begin{defi}
       A weighting function $\psi_\gamma(x): [0,1] \rightarrow \mathsf{Rng}$, where $\mathsf{Rng} \subseteq [0,1]$, is called a \emph{\textbf{calibrated weighting function}}, if it satisfies the following regularities:
      \begin{enumerate}
        \item[(A)] $\psi_\gamma$ has continuous second-order derivatives.
        \item[(B)] $\psi_\gamma$ is strictly increasing in the sense that $\psi_\gamma'(x) >0$.
         \item[(C)] $\psi_\gamma$ is strictly concave in the sense that $\psi_\gamma''(x) <0$.
      \end{enumerate}
          \end{defi}
          In this definition, $(B)$ is a natural requirement to make the weight proportional to the target instance's difficulty. Condition $(C)$ is an interesting trait in our framework. To see why this is necessary, let us note that the weight functions $\vpi$, $\vnj$ are continuous surrogates residing in $[0,1]$ for threshold function \begin{equation*}
  \begin{split}
    & \ind{ 1 - \f(\xpi) > 1- \f(\fxpo{\npa})},\\ 
    &\ind{\f(\xnj) > \f(\fxno{\nnb})},
  \end{split}
\end{equation*}
respectively. To be simple, we continue our discussion with a general form $\ind{x > 0}$. Obviously, weight decay for large $x$ should be smooth such that the loss could attend at the top $\f(\xp)$  and $(1 -\f(\xn))$ scores.  Moreover, to avoid over-fitting, the model should as well have sufficient memory of the easy examples. Hence the weights for such examples should not be too close to zero. These observations are exactly typical traits for a concave function. As shown in Fig.\ref{fig:concon}, we visualize the difference between a convex function $y = x^2$ and $y = \ind{x >0.5}$, and the difference between concave function $y = x^{0.05}$ and $y = \ind{x >0.5}$. 
\begin{figure}[t]  
  \centering
    \includegraphics[width=0.98\columnwidth]{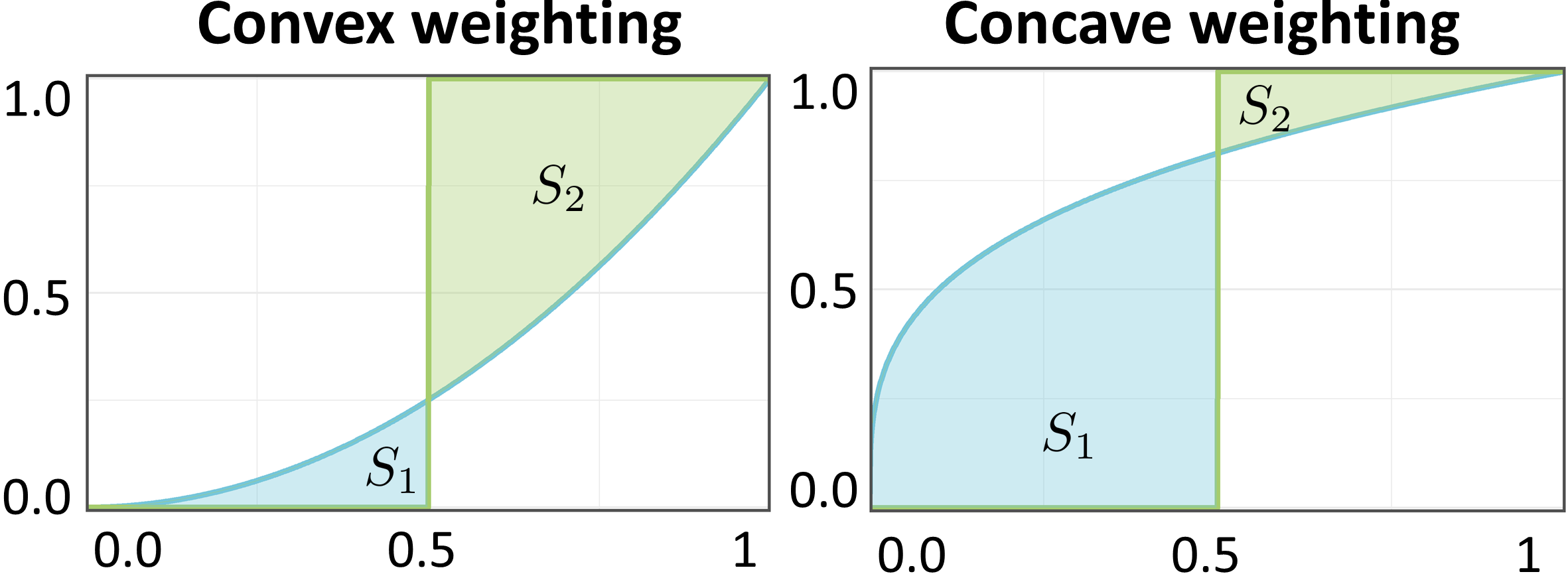} 
  \caption{\label{fig:concon}\textbf{convex vs. concave weighting functions.} It is easy to achieve the upper bound of the original hard threshold weighting when the area of $S_1$ is large and the area of $S_2$ is small. In this sense, concave functions are intuitively more suitable choices for the weighting functions. }
\end{figure}

\begin{figure}[t]  
  \centering
  \subfigure[Sufficient Condition]{
    \includegraphics[width=0.45\columnwidth]{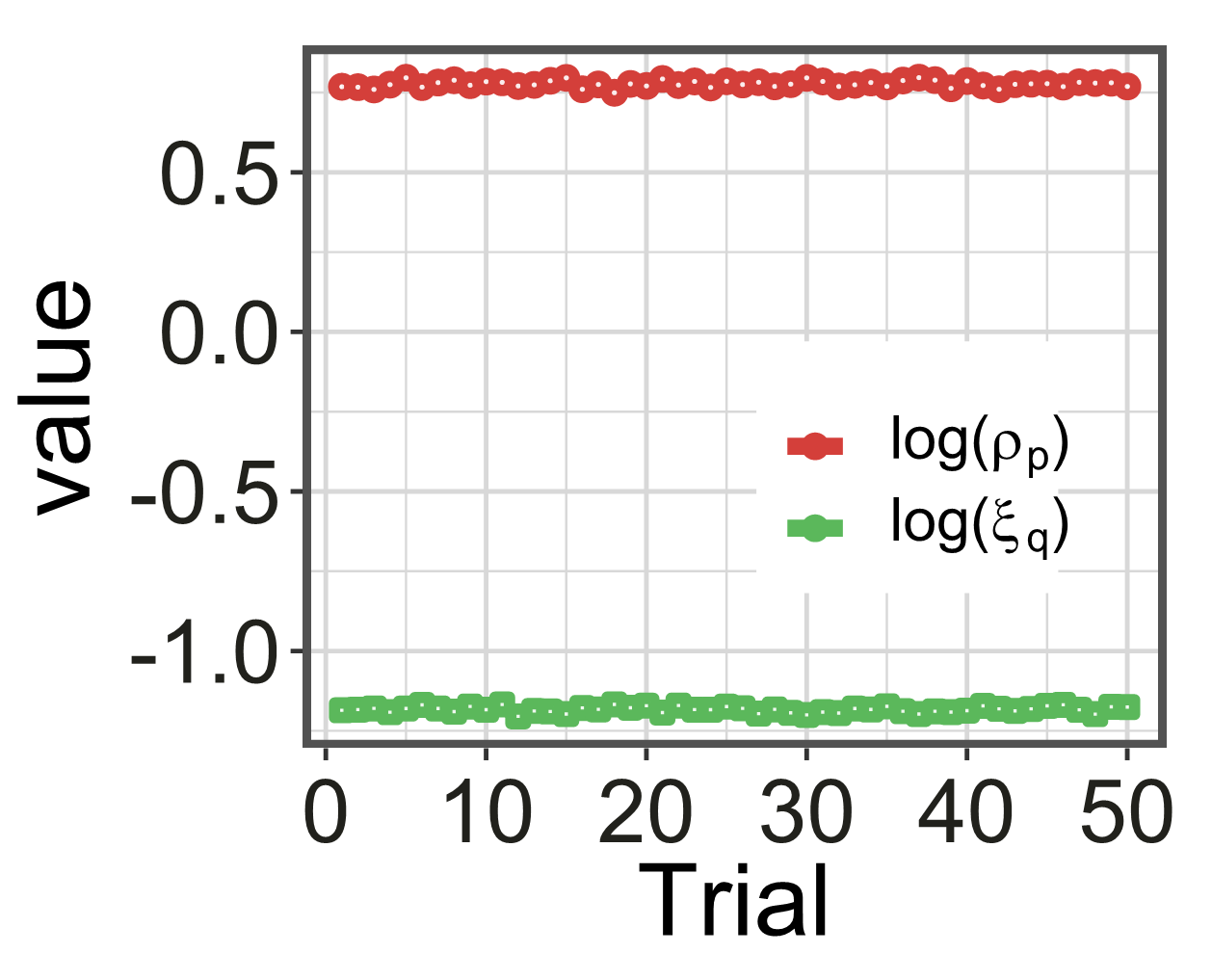} 
  }
  \subfigure[Upper Bound]{
    \includegraphics[width=0.45\columnwidth]{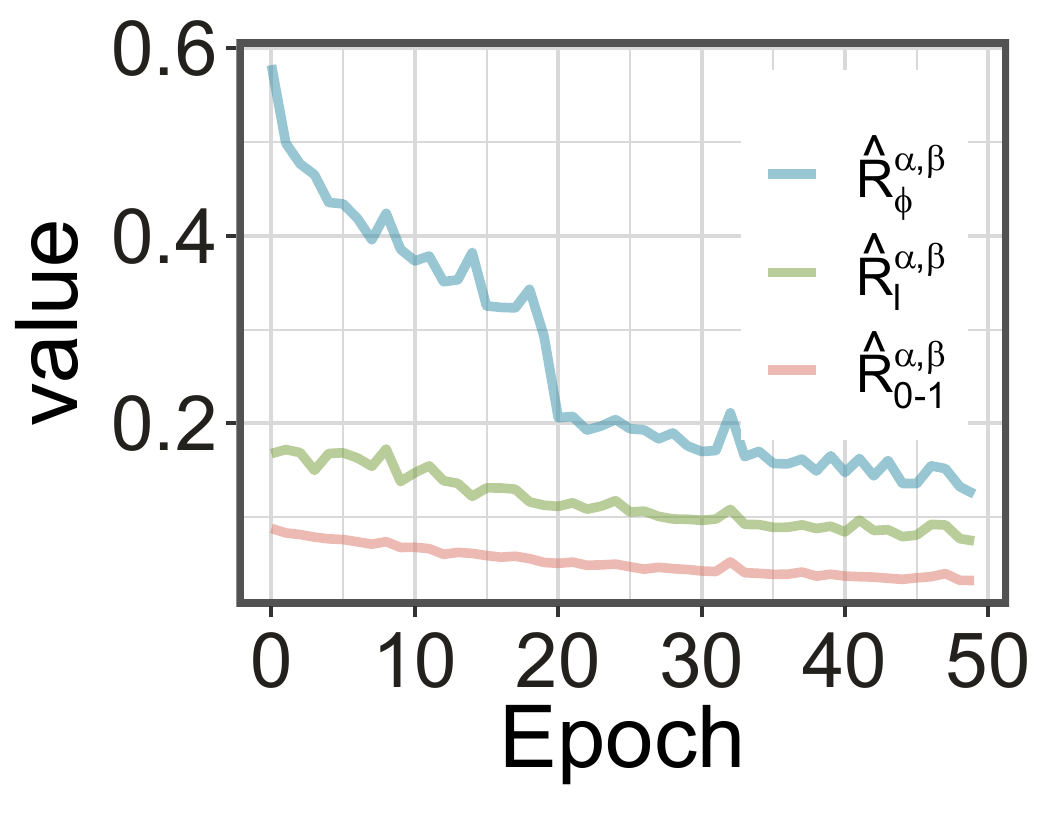} 
  }
\caption{\label{fig:inter}\textbf{Practical Validation of Proposition \ref{prop:reform}}. \textbf{(a)}: We construct a distribution that $f_{\theta}(\bm{x}^+)\sim N(0.5, 0.08)$, $f_{\theta}(\bm{x}^-)\sim N(0.3, 0.08)$, and sample 100 points for each class, and repeat it for 50 trials. For all these trials, we can find $p,q$ that satisfies the ineq. in prop.2 (a). Thus the sufficient condition holds for these trials. \textbf{(b)}: We plot the training curve of the relaxed surrogate empirical risk  $\hat{\mathcal{R}}^{\alpha, \beta}_\psi$, the surrogate empirical risk $\hat{{R}}^{\alpha, \beta}_\ell$, and the $0-1$ empirical risk $\hat{\mathcal{R}}^{\alpha, \beta}_{0-1}$(1-TPAUC) for CIFAR-10-LT-Subset-1, with a concave polynomial weighting function.  As shown in the figure, the relaxed risk function $\hat{\mathcal{R}}^{\alpha, \beta}_\psi$ is always an upper bound of the other two risk functions.}
\end{figure}

Another reason to choose concave functions is that they can benefit the optimization process. More precisely, we expect  the loss function $\rpsi$ in \eqref{eq:surr} to be an upper bound of $\rorg$, such that minimizing $\rpsi$ could also minimize the original loss. Back to Fig.\ref{fig:concon}, this is more likely to happen if $S_1/S_2$ is large. In fact, this is a condition which is much easier for concave functions to satisfy. This could be proved theoretically from the following proposition. From a quantitative perspective, we can provide a sufficient condition under which $\rpsi \ge \rorg$. Moreover, we show that a similar sufficient condition also holds in the population-level. {Please see Appendix.\ref{app:prop2} for the proof.}
          \begin{prop}\label{prop:concon}
Given a strictly increasing weighting function $\psi_\gamma: [0,1] \rightarrow [0,1]$, such that $\vpi = \psi_\gamma(1-\f(\xpi))$, $\vnj = \psi_\gamma(\f(\xnj))$, denote:
\begin{equation*}
  \begin{split}
    &\Ionep = \left\{x_+: x_+ \in \XP, f(x_+) \ge f(x^{(\npa)}_+) \right\}, \\ 
    &\Ionen = \left\{x_-: x_- \in \XN, f(x_-) \le f(x^{(\nnb)}_-) \right\},
  \end{split}
\end{equation*}
denote $\Itwo$ as $(\XP \times  \XN) \backslash (\Ionep \times  \Ionen)$; denote $\expiop[x]$ as the empirical expectation of $x$ over the set $\Ionep$, and  $\expion[x],~ \expiopn,$ $~\expit$ are defined similarly; define
$l_{i,j} = \ell(\f, \xpi,\xnj)$. We assume that
\[\npa \in \mathbb{N}, ~\nnb \in \mathbb{N}, ~\fxp, \fxn \in (0,1), \] 
then:
\begin{enumerate}
  \item[(a)] A sufficient condition for $\rorg(\mathcal{S},\f) \le \rpsi(\mathcal{S},\f)$ is that:
  \begin{equation*}
    \begin{split}
      \sup_{ p \in (0,1), q = -\frac{p}{1-p} } \left[\rho_p -\xi_q \right] \ge 0,
    \end{split}
  \end{equation*} 
  where
  \begin{equation*}
  	\begin{split}
    \rho_p &= \frac{\left(\expit\left[v^{p}_+ \cdot v^{p}_-\right]\right)^{1/p}  }{\left(\expiopn\left[(1 - v_+v_-)^2\right]\right)^{1/2}}, \\
      \xi_q &=  \frac{\alpha\beta}{1-\alpha\beta} \cdot \frac{\left( \expit(\ell^2_{i,j}) \right)^{1/2}}{\left( \expiopn(\ell^q_{i,j}) \right)^{1/q}}.
    \end{split}
  \end{equation*}
  \item[(b)] A sufficient condition for $ \mathcal{R}^\ell_{\alpha,\beta}(f) \le  \mathcal{R}^\ell_{\psi}(f) $ is that:
  \begin{equation*}
    \begin{split}
      \sup_{ p \in (0,1), q = -\frac{p}{1-p} } \left[\rho_p -\xi_q \right] \ge 0,
    \end{split}
  \end{equation*} 
  where
  \begin{equation*}
  	\begin{split}
    \rho_p &= \frac{\left(\expite\left[v^{p}_+ \cdot v^{p}_-\right]\right)^{1/p}  }{\left(\expiopne\left[(1 - v_+v_-)^2\right]\right)^{1/2}}, \\
      \xi_q &=  \frac{\alpha\beta}{1-\alpha\beta} \cdot \frac{\left( \expite(\ell^2_{i,j}) \right)^{1/2}}{\left( \expiopne(\ell^q_{i,j}) \right)^{1/q}}.
    \end{split}
  \end{equation*}
  \item[(c)] If there exists at least one strictly concave $\psi_\gamma$  such that the  $\rorg(\mathcal{S},\f) > \rpsi(\mathcal{S},\f)$, then  $\rorg(\mathcal{S},\f) > \rpsi(\mathcal{S},\f)$ holds for all convex $\psi_\gamma$.
\end{enumerate}
\end{prop}
\begin{figure*}[th]  

  \subfigure[$(\gamma -1)^{-1} = 0.3$]{
  
    \includegraphics[width=0.23\textwidth]{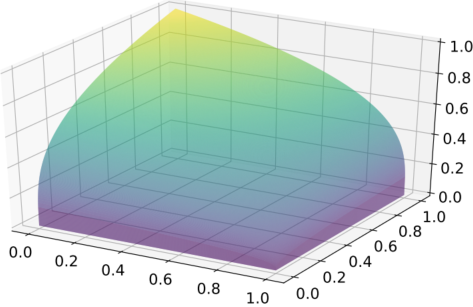} 
  }
  \subfigure[$(\gamma -1)^{-1}  =0.5$]{
    \includegraphics[width=0.23\textwidth]{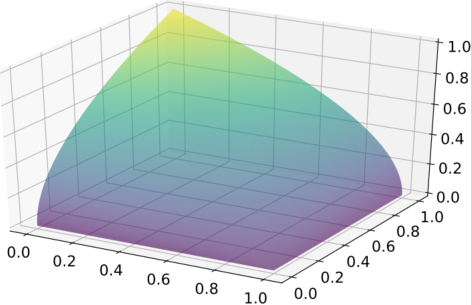} 
  }
  \subfigure[$\gamma =5$]{
    \includegraphics[width=0.23\textwidth]{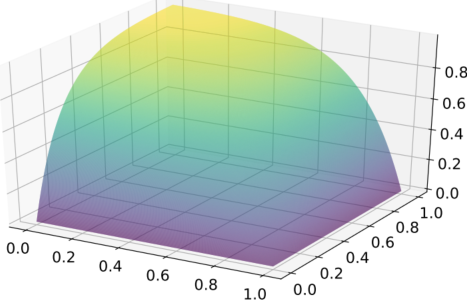} 
  }
  \subfigure[$\gamma = 10$]{
    \includegraphics[width=0.23\textwidth]{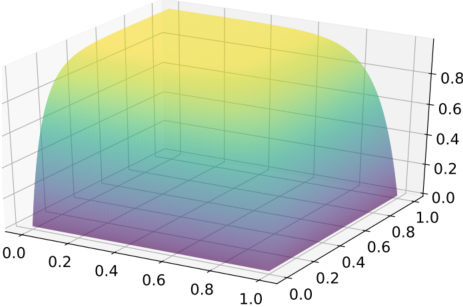} 
  
  }
  \caption{\label{fig:weight} Visualization of the Landscape of the pairwise weights $\psi_\gamma(x) \cdot \psi_\gamma(y)$. Here, (a) and (b) plot $\psi^{\mathsf{poly}}_{\gamma}$, while (c) and (d) plot $\psi^{\mathsf{exp}}_{\gamma}$. }
\end{figure*}

\noindent According to Prop.\ref{prop:concon}, $\rpsi$ becomes an upper bound of $\rorg$ if $\alpha, \beta$ are small, and the empirical distribution has significant masses at instances with moderate difficulty. To validate the result empirically, we conduct the following two sets of experiments. 
\begin{enumerate}
  \item[1)] First, we validate the effectiveness of the sufficient condition. To do this, we construct a distribution that $f_{\theta}(\bm{x}^+)\sim N(0.5, 0.08)$, $f_{\theta}(\bm{x}^-)\sim N(0.3, 0.08)$, and sample 100 points for each class, and repeat it for 50 trials, the result of which is shown in Fig.\ref{fig:inter}-(a). For all these trials, we can find $p,q$ that satisfy the inequality: $\rho_p -\xi_q > 0$. 
   Thus the sufficient condition holds for these trials.
   \item [2)] Second, in Fig.\ref{fig:inter}-(b), we plot the training curve of the relaxed surrogate empirical risk  $\hat{\mathcal{R}}^{\alpha, \beta}_\psi$, the surrogate empirical risk $\hat{{R}}^{\alpha, \beta}_\ell$, and the $0-1$ empirical risk $\hat{\mathcal{R}}^{\alpha, \beta}_{0-1}$(1-TPAUC) for CIFAR-10-LT-Subset-1 (please refer to Sec.\ref{sec:exp} for more details), with $\ell(t) = (1-t)^2$  and a concave polynomial weighting function.  As shown in the figure, the relaxed risk function $\hat{\mathcal{R}}^{\alpha, \beta}_\psi$ is always an upper bound of the surrogate risk. Moreover, we can observe that all three curves show a convergent trend. This again validates the effectiveness of our proposed relaxed weighting scheme.
\end{enumerate}

\textbf{Dual Correspondence Theory.} Now with the penalty function and weighting function clarified, we establish their dual correspondence with the following proposition. {{Please see Appendix.\ref{app:prop3} for the proof.}}
\begin{prop}\label{prop:dual}
  Given a strictly convex function $\varphi_\gamma$, and define $\psi_\gamma(t)$ as:
  \begin{equation*}
  \psi_\gamma(t) = \argmax_{v \in [0,1]} ~v \cdot t - \varphi_\gamma(v),
  \end{equation*} 
  then we can draw the following conclusions:
  \begin{enumerate}
    \item[(a)] If $\varphi_\gamma$ is a calibrated smooth penalty function, we have $\psi_\gamma(t) = \varphi_{\gamma}^{'-1}(t)$,which  is a calibrated weighting function.
    \item[(b)] If $\psi_\gamma$ is a calibrated weighting function such that $v = \psi_\gamma(t)$, we have $\varphi_\gamma(v) = \int \psi^{-1}_\gamma(v)dv +const.$, which is a  calibrated smooth penalty function.
  \end{enumerate}
  \end{prop}

  According to Prop.\ref{prop:dual}, given a calibrated smooth penalty function, one can obtain an implicit soft weighting strategy via Prop.\ref{prop:dual}-(a). Likewise, given a calibrated weighting function, one can find an implicit regularizer over the sample weights via Prop.\ref{prop:dual}-(b). Based on the regularities of the two components, both $\psi$ and $\varphi$ have closed-form formulations if the other one is known. This means that we can solve the bi-level optimization framework in $\opo$  by simply minimizing the resulting  $\rpsi(\mathcal{S}, \f)$, leading to a much simpler optimization that can be solved directly by an end-to-end training framework. \emph{Consequently, the dual correspondence theory provides a simple way to establish a surrogate optimization problem of TPAUC, once a weighting function or a penalty function is at hand.} 

  \subsection{An Instance-wise Minimax Reformulation of the Relaxed Objective Function}
  It is easy to see that the proposed objective function is formed as a sum of pairwise loss functions, which brings extra computational burden and the difficulty to generate a mini-batch. To address this issue, inspired by \cite{minimaxauc1}, we propose a minimax reformulation of the original objective function, where the objective function turns out to be an instance-wise loss sum. Note that, for the sake of convenience, we adopt the shorthand notation:
  \begin{align*}
    \vpi =\psi_\gamma(1-\fxpi),~ \vnj = \psi_\gamma(\fxnj),
  \end{align*}
  in this subsection. Moreover, denote:
  \begin{align*}
    &\cp = \frac{1}{\np} \sump \vpi, ~~\cn = \frac{1}{\nn} \sumn \vnj,\\
    &\f^+ = \frac{1}{\np} \sump \vpi \cdot \f(\xpi),
    ~~\f^- = \frac{1}{\nn} \sumn \vnj \cdot \f(\xnj),\\
    &\f^{+,2} = \frac{1}{\np} \sump \vpi \cdot \f^2(\xpi),
    ~~\f^{-,2} = \frac{1}{\nn} \sumn \vnj \cdot \f^2(\xnj),\\
  \end{align*}
we come to the following theorem, please see Sec.\ref{sec:1} for the proof:

\begin{thm}\label{thm:reform} Denote $v^\gamma_\infty = \sup_{{x}} |\psi_\gamma(x)|$, $f_\infty = \sup_{{x}} |\f(x)|$. assume that $v^\gamma_\infty < \infty, f_\infty < \infty$, $\ell(t) = (1-t)^2$, $(OP1)$ could be reformulated as \footnote{the inequalities $\bm{0} \le \bm{a} \le \bm{c}_a$ and $\bm{0} \le \bm{b} \le  \bm{c}_b$ should be understood  elementwise.}:
 \begin{equation*}
  \min_{\bm{\theta},  \bm{0}_{10} \le \bm{a} \le \bm{c}_a } \max_{ \bm{0}_{8} \le \bm{b} \le  \bm{c}_b }~\bm{a}^\top \zeta_1 + \bm{b}^\top \zeta_2 -||\tilde{\bm{b}}||^2+||\tilde{\bm{a}}||^2,
\end{equation*}
where
\begin{align*}
   \zeta_1  &= -[\cp,~ \cn,~ 2(\f^+ + \cn),~ 2\cp,~ 2\f^-,~ \cn,~ \f^{+,2},\cp,~ \f^{-,2},\\
    &~~~~~~~~~~~~ 2(\f^+ + \f^-)] ,\\
  \zeta_2 &={\color{white}{-}} [\cp + \cn,~ 2\cn,~ 2\f^+,~ 2(\f^- +\cp) ,~ \cn + \f^{+,2},\\  
           &~~~~~~~~~~~~ \cp + \f^{-,2}, 2 \f^+,~ 2 \f^-], \\
   \tilde{\bm{a}} &= (1/\sqrt{2}) \cdot [a_1, a_2, \sqrt{2} \cdot a_3, \sqrt{2}a_4, \sqrt{2}a_5, a_6, a_7, a_8, \\ 
   &~~~~~~~~~~~~~~~~~~~~~a_9, \sqrt{2} \cdot  a_{10}], \\
  \tilde{\bm{b}} &= (1/\sqrt{2}) \cdot [b_1, \sqrt{2} \cdot b_2,\sqrt{2} \cdot b_3, \sqrt{2} b_4, b_5, b_6,\\ 
  & ~~~~~~~~~~~~~~~~~~~~~ \sqrt{2} \cdot b_7,\sqrt{2} \cdot b_8],\\
  \bm{c}_a &= [v^\gamma_\infty,v^\gamma_\infty,v^\gamma_\infty + f_\infty,v^\gamma_\infty,f_\infty,v^\gamma_\infty,f^2_\infty,v^\gamma_\infty,f^2_\infty,2  f_\infty], \\
  \bm{c}_b &= [2v^\gamma_\infty,v^\gamma_\infty,f_\infty,v^\gamma_\infty + f_\infty,v^\gamma_\infty + f^2_\infty,v^\gamma_\infty + f^2_\infty,f_\infty,f_\infty]. 
\end{align*}

\end{thm}
We adopt the stochastic gradient descent-ascent algorithm \cite{sgda} to perform the minimax optimization, which is summarized  in Alg.\ref{algo1}. Moreover, it is easy to see that the new objective function is strongly-concave with respect to the max player $\bm{b}$, if in addition it is a smooth function, we can attain an iteration complexity of $O(1/\epsilon^4)$ to obtain an $\epsilon$-stationary point with a proper choice of $\eta_1,\eta_2,\eta_3$.  \cite{sgda}.
\begin{algorithm}[htb]
	\caption{The Minimax optimization algorithm of the TPAUC framework.}
	\label{algo1}
	\hspace*{0.02in} {\bf Input:} 
	 $\mathcal{D}$, $\gamma$ \\
	\hspace*{0.02in} {\bf Output:} 
	best parameters $\bm{\theta}$
	\begin{algorithmic}[1]
	\State Initialize $\bm{a} = \bm{b} = \bm{0}$.
  \State $\bm{C}_a = (1/\sqrt{2}) \cdot [1, 1, \sqrt{2}, \sqrt{2}, \sqrt{2}, 1, 1, 1, 1, \sqrt{2}]$;
  \State $\bm{C}_b = (1/\sqrt{2}) \cdot [1, \sqrt{2}, \sqrt{2}, \sqrt{2}, 1, 1, \sqrt{2},\sqrt{2}]$;
  \State Setup the stepsizes $\eta_1,\eta_2,\eta_3$.
	\While{\textit{Not Converged}} 
	\State Random sample a mini-batch $\mathcal{B}$ from $\mathcal{D}$;
	\State Calculate $\zeta_1$, $\zeta_2$, $\tilde{\bm{a}}$, $\tilde{\bm{b}}$,  according to $\mathcal{B}$;
	\State $\bm{\theta} = \bm{\theta} - \eta_1 \cdot \nabla_{\bm{\theta}}r(\bm{\theta})$;
	\State $\bm{a} = \bm{a} - \eta_2 \cdot (\zeta_1 + 2 \cdot \bm{C}_a \odot \tilde{\bm{a}})$;
	\State $\bm{b} = \bm{b} + \eta_3 \cdot (\zeta_2 - 2 \cdot \bm{C}_b \odot \tilde{\bm{b}})$;
	\EndWhile
	\State \Return $\bm{\theta}$;
	\end{algorithmic}
\end{algorithm}

  \subsection{Two Instantiations of the Generic Framework}
  Based on the generic framework, in this subsection, we provide two practical instantiations.  

  \noindent \textbf{Polynomial Surrogate Model.} From the penalty function perspective, the original $\ell_1$ penalty  realizes  $\psi_\gamma = \gamma \cdot t$. In this way, it is a natural choice to adopt a polynomial  penalty $ \varphi^{\mathsf{poly}}_\gamma(t) = C \cdot t^\gamma$ as a dense surrogate for $\ell_1$. Inspired by this, we propose a polynomial  surrogate model as example \ref{exa:poly}.
  \begin{exa}[Polynomial Surrogate Model]\label{exa:poly}
  In the polynomial surrogate model, we set:
  \begin{equation*}
    \varphi^{\mathsf{poly}}_\gamma(t) = \frac{1}{\gamma} \cdot t^\gamma, ~
    \psi^{\mathsf{poly}}_\gamma(t) =  {t}^{\frac{1}{\gamma-1}}, ~\gamma > 2
  \end{equation*}
  \end{exa}
  The visualizations of the weights are shown in Fig.\ref{fig:weight}.
  \begin{exa}[Exponential Surrogate Model]\label{exp:exp}
    In the exponential surrogate model, we set:
    \begin{equation*}
      \varphi^{\mathsf{exp}}_\gamma(t) = \frac{(1-t)( \log(1-t) - 1) + 1}{\gamma},  ~
      \psi^{\mathsf{exp}}_\gamma(t) = 1-e^{-\gamma t}
    \end{equation*}
    \end{exa}
\noindent  \textbf{Exponential Surrogate Model.} Considering the properties of the weighting functions, we expect that $\psi_\gamma$ will have a flat landscape for large $t$. Motivated by this, we adopt an exponential weighting function $ \psi^{\mathsf{exp}}_\gamma(t) = 1-e^{-\gamma t}$ (the landscape is shown in Fig.\ref{fig:weight} (c)-(d)).
   The resulting model is then shown as Exp.\ref{exp:exp}. The visualizations of the weights are shown in Fig.\ref{fig:weight}.
   \section{Theoretical Analysis}
   In this section, we will present an investigation into the theoretical guarantees of the proposed framework. Specifically, our study includes two aspects. In the first subsection, we will derive the Bayes optimal scoring function to see what are the best solutions under TPAUC optimization. In the second subsection, we turn to analyze the generalization ability of the proposed framework, where we provide excess risk upper bounds under two different assumptions.  

   Note that all the proofs for the results are shown in the appendix: please see App.\ref{sec:2and3} for the details associated with Thm.\ref{prop:bayes} and Thm.\ref{thm:simplebayes}; please see  App.\ref{sec:4} for the details associated with Thm.\ref{thm:gen}; please see App.\ref{sec:5} for the details associated with Thm.\ref{thm:abs}.
   \subsection{Bayes scoring functions for TPAUC}
   In this subsection, we derive the set of Bayes learners realizing the optimized TPAUC in expectation.

   First, given a scoring function $f$, we define $\cfab,\cfabup, \cfabdown$ as:
   \begin{align*}
   &\cfab = \{\bm{x} \in \mathcal{X}: t_{\beta}(f) \le f(\bm{x}) \le t_{1-\alpha}(f) \},\\
   &\cfabup = \{(\bm{x}_1,\bm{\x}_2) \in \mathcal{X}^2: f(\x_1) > \fa, f(\x_2) \le \fa \},\\
   &\cfabdown = \{ (\bm{x}_1,\bm{\x}_2) \in \mathcal{X}^2 : f(\x_1) < \fb, f(\x_2) \ge \fb \},
   \end{align*}
   where
   \begin{align*}
    &t_{1-\alpha}(f)  = \argmin\left[\eP\left[ \ind{\fxp \ge \delta} \right] = 1-\alpha  \right]
    \\
      &t_{\beta}(f) =\argmin\left[ \eN\left[ \ind{\fxn \ge \delta} \right] = \beta  \right].
   \end{align*}
and $\mathcal{X}$ is the domain of the input feature.

   The following theorem reveals the set of all Bayes optimal scoring functions:
   
   \begin{thm}\label{prop:bayes}
     Assume that $t_{\beta}(f) \le t_{1-\alpha}(f)$, and that there are no tied comparisons, $f$ is a Bayes  scoring function for $\mathsf{TPAUC}^\alpha_\beta$ if it is a solution to the  following problem:
     \begin{align*}
      \min_f &\int_{\x_1,\x_2 \in \cfab \otimes \cfab} p(\x_1) \cdot p(\x_2) \cdot \min\left\{\eta_{1,2},\eta_{2,1}\right\} d\x_1d\x_2 \\ 
    +& 2 \int_{\cfabup}  p(\x_1) \cdot p(\x_2) \cdot\eta_{2,1}\cdot d\x_1d\x_2\\ 
    +&  2 \int_{\cfabdown}  p(\x_1) \cdot p(\x_2) \cdot\eta_{1,2} \cdot d\x_1d\x_2  \\
       \end{align*}
Specifically, for all $\x_1,\x_2 \in \cfab \otimes \cfab$, we have:
   \begin{equation*}
     \begin{split}
      \left(\eta(\bm{x}_1) - \eta(\bm{x}_2)\right) \cdot \left(f(\bm{x}_1) - f(\bm{x}_2)\right)  >0,\\ 
     \end{split}
   \end{equation*}
   where $\eta(\x) =  \mathbb{P}\left[y =1 | \x\right]$, $\eta_{1,2} = \eta(\x_1)(1-\eta(\x_2)),~ \eta_{2,1} = \eta(\x_2)(1-\eta(\x_1))$, $p(\x)$ is the p.d.f of the marginal distribution of $\x$.
   \end{thm}

   \noindent To optimize TPAUC, it suffices to make sure that, within the set $\cfab$, the scoring function $f$ preserves the ranking result of $\eta$ and that  $\cfab, \cfabup, \cfabdown$ are carefully chosen to minimize the Bayes error.  This is a  much tougher condition than that for AUC (One does not need to choose any instance at all) and for OPAUC (Only FPR has a constraint). Unfortunately, the optimization problem in Thm.\ref{prop:bayes} does not have a closed-form solution in general. Moreover, we can see that $\eta$ may not be a solution of the Bayes classifier since the choice of $\cfab, \cfabup, \cfabdown$ with $f=\eta$ might not be optimal. This matches with a previous conjecture \cite{top1} that $\eta$ might be suboptimal for metric like OPAUC. However, if we restrict the choice of $f$ in a simple subclass of the hypothesis, it is possible to find a much simpler solution. To do this, we first define the $(\alpha,\beta)$-weakly consistent $f$ functions.
  
   \begin{defi}[$(\alpha,\beta)$-Weakly Consistent $f$] $f$ is said to be \textbf{$(\alpha,\beta)$-weakly consistent} w.r.t. $\eta(\x) = \prob\left[ y = 1|\x\right]$, if:
     \begin{equation*}
      \mathcal{C}_f^{\alpha,\beta} =  \mathcal{C}_\eta^{\alpha,\beta},~ \tilde{\mathcal{C}}_f^{\alpha,\beta,\uparrow} =  \tilde{\mathcal{C}}_\eta^{\alpha,\beta,\uparrow},~ \tilde{\mathcal{C}}_f^{\alpha,\beta,\downarrow} =  \tilde{\mathcal{C}}_\eta^{\alpha,\beta,\downarrow}
     \end{equation*}
     where
     \begin{align*}
      &\tilde{\mathcal{C}}_f^{\alpha,\beta,\uparrow} = \{\x \in \mathcal{X}: f(\x) > \fa \},\\ 
      &\tilde{\mathcal{C}}_f^{\alpha,\beta,\downarrow} = \{\x \in \mathcal{X}: f(\x) < \fb \}.
     \end{align*}
   \end{defi}
\noindent We call such $f$ functions weakly consistent since they can at least provide a consistent partition of the three intervals: $ \mathcal{C}_f^{\alpha,\beta}, \tilde{\mathcal{C}}_f^{\alpha,\beta,\uparrow}, \tilde{\mathcal{C}}_f^{\alpha,\beta,\downarrow}$ w.r.t $\eta$.
   Then the following theorem shows that if we focus on the subclass $\mathcal{F}_a$, then the optimal scoring functions have a simpler expression.
   \begin{thm}\label{thm:simplebayes}
    Under the same setting as Thm.\ref{prop:bayes},  if we restrict our choice of $f$ in $\mathcal{F}_{a}$, where 
    \begin{align*}
    \mathcal{F}_{a} = \{&f: f ~\text{is}~ (\alpha,\beta)\text{-weakly consistent w.r.t}~\eta \},
    \end{align*}
 Then $f$ minimizes the expected risk in $\mathcal{F}_a$, if $\forall (\x_1,\x_2) \in \mathcal{C}^{\alpha,\beta}_{\eta} \otimes  \mathcal{C}^{\alpha,\beta}_{\eta}$:
  \begin{align*}
    (f(\x_1) -f(\x_2)) \cdot (\eta(\x_1) -\eta(\x_2)) >0.
  \end{align*} 
  \end{thm}
   
  \subsection{Generalization Analysis}

  In this subsection, we explore how generalization error behaves away from the training error in terms of the TPAUC metric. In other words, we will show when a well-trained model will lead to a reasonable generalization performance. Our analysis is based on a standard assumption that the classifiers are chosen from a hypothesis class $\mathcal{F}$ (e.g. the class of a specific type of deep neural network). To do this, we provide uniform upper bounds for the excess risk $ \mathcal{R}^\ell_{\alpha, \beta} - \hat{\mathcal{R}}^\ell_{\psi}$. Such that 
  \[ \sup_{f \in \mathcal{F}}[\mathcal{R}^\ell_{\alpha, \beta} - \hat{\mathcal{R}}^\ell_{\psi}] \le \epsilon \] 
 holds with high probability. Specifically, we present two results using two different assumptions.
 \begin{asm}\label{asum:bound}
 We focus on scoring functions satisfying either of the following assumptions:
 \begin{enumerate}
 
   \item[(a)] The sufficient condition in Prop.\ref{prop:concon}-(a) holds, such that $\rorg(\mathcal{S},\f) \le \rpsi(\mathcal{S},\f)$.
   \item[(b)] The sufficient condition in Prop.\ref{prop:concon}-(c) holds, such that $\mathcal{R}^\ell_{\alpha, \beta}(f) \le \mathcal{R}^{\ell}_{\psi}(f)$.
 
 \end{enumerate}
 \end{asm}
 
 \noindent Moreover, since we are dealing with imbalanced datasets, we adopt the following assumption without loss of generality.
 \begin{asm}\label{asm:imb}
  We assume that the training data satisfies: $\np \ll \nn.$
 \end{asm}
 
 \noindent \textbf{Note that the assumption above will be employed throughout Sec.\ref{subsubsec:b}; we will omit in the claim of the lemmas and theorems therein.}
 
 Finally, since only the order of the bounds is of interest, we will use the asymptotic notation $\lesssim$ such that constants and undominated terms could be omitted from the expressions. Specifically, we have:
 \[f(x) \lesssim g(x) ~ \Longleftrightarrow ~ \exists ~\text{a constant}~C~ s.t.~  f(x) \le C \cdot g(x). \]

\subsubsection{\textbf{Results under Assumption (a)}}

When assumption (a) is available, we have:
\begin{equation*}
  \mathcal{R}^\ell_{\alpha, \beta} -  \hat{\mathcal{R}}^\ell_{\psi}  \le   \mathcal{R}^\ell_{\alpha, \beta} -  \hat{\mathcal{R}}^\ell_{\alpha,\beta}. 
\end{equation*}
\noindent This allows us to instead focus on  the excess risk $\mathcal{R}^\ell_{\alpha, \beta} -  \hat{\mathcal{R}}^\ell_{\alpha,\beta}$. The key challenge here is that $\rorg(\mathcal{S},\f)$ is not an unbiased estimation of ${\mathsf{AUC}}_\alpha^\beta(\f, \mathcal{S})$, making standard generalization analysis \cite{mlfun} unavailable.  Here we extend the error decomposition technique for OPAUC \cite{partial3} and employ the result in Prop.\ref{prop:concon} to reach the following theorem. More precisely, we have the following decomposition lemma which allows one to decompose the upper bound of the  excess risk into the sum of the excess risk of the threshold functions. 

\begin{lem}[Excess Risk Decomposition]\label{lem:decomp}
  For $\forall f \in \mathcal{F}$, we have:
  \begin{equation*}
    \begin{split}
   \aucf - \aucs \le  2 (\Delta_+ + \Delta_-)
    \end{split}
  \end{equation*}
where
\begin{align*}
  \Delta_+  &= \sup_{\delta  \in \mathbb{R}}  \left| \frac{1}{\np} \cdot \sump   \ind{\fxpi \le \delta} -  \eP \left[ \ind{\fxp \le \delta}  \right]
  \right|\\
  \Delta_- & = \sup_{\delta  \in \mathbb{R}}  \left| \sumtermp \cdot \sumn   \ind{\fxnj \ge \delta} -  \eN \left[ \ind{\fxn \ge \delta}  \right]
  \right|
\end{align*}
\end{lem}

    \begin{thm}[Excess Risk Bound]\label{thm:gen}
      Assume that there are no ties in the datasets, and the surrogate loss function $\ell$ with range $[0,1]$, is an upper bound of the $0\text{-}1$ loss, then, for all $\f \in \mathcal{F}$, and all $(\alpha, \beta) \in \mathcal{I}^1_{suff}(\mathcal{S})$, the following inequality holds with probability at least $1 - \delta$ over the choice of $\mathcal{S}$:
\begin{align*}
    \aucf - ~ \rpsi(\f, \mathcal{S}) \lesssim&  \sqrt{\frac{\VC \cdot \log(\np) + \log(1/\delta)}{\np}} \\ 
    & ~ + \sqrt{\frac{\VC \cdot \log(\nn) + \log(1/\delta)}{\nn}} ,
  \end{align*}

  where $\VC$ is the VC dimension of the hypothesis class: 
  \begin{align*}
    \mathcal{T}(\mathcal{F}) \triangleq \{\mathsf{sign}(\f(\cdot) - \delta): \f \in \mathcal{F}, ~ \delta \in \mathbb{R}\}
  \end{align*}
  and

  \begin{align*}
    \mathcal{I}^1_{suff}(\mathcal{S}) =  \bigg\{ &(\alpha,\beta): \alpha \in (0,1),~ \beta \in(0,1),~ \npa \in \mathbb{N}_+,~ \\  
    &\nnb \in \mathbb{N}_+, \text{condition (a) in Prop.\ref{prop:concon} holds} \bigg\},
   \end{align*}

    \end{thm}
    According to the theorem, for all $\alpha, \beta$ satisfying condition (a) of Prop.\ref{prop:concon} and any model in $\mathcal{F}$, the generalization error represented by the loss version of TPAUC $ \aucf  =   1- {\mathsf{AUC}}_\alpha^\beta(\f, \mathcal{S})$ is no larger than the empirical loss $\rpsi(\f, \mathcal{S})$ plus a complexity term. The complexity term is affected by two factors. On one hand, it vanishes with large enough training datasets. On the other hand, it remains moderate if the model hypothesis class's VC dimension is not too large. Moreover, moderate upper bounds for the VC dimension are now available for typical models ranging from linear models to deep neural networks. Finally, for a well-trained model,  the empirical loss $\rpsi(\f, \mathcal{S})$ is restricted to be small in our framework; one can then reach reasonable generalization results with high probability.

    \subsubsection{\textbf{Results under Assumption (b)}}\label{subsubsec:b}

    When assumption (b) is available, we have:
    \begin{equation*}
      \mathcal{R}^\ell_{\alpha, \beta} -  \hat{\mathcal{R}}^\ell_{\psi}  \le   \mathcal{R}^\ell_{\psi} -  \hat{\mathcal{R}}^\ell_{\psi}. 
    \end{equation*}
    \noindent This allows us to instead focus on  the excess risk $\mathcal{R}^\ell_{\psi} -  \hat{\mathcal{R}}^\ell_{\psi}$. Before elaborating on the details, we first provide the preliminaries.

    Since $\hat{\mathcal{R}}^\ell_{\psi}$ is an unbiased estimation of $\mathcal{R}^\ell_{\psi}$, we can now get rid of the error decomposition. Typically, the standard technique for deriving uniform upper bounds of the excess risks requires the objective function to be a finite sum of independent terms. However, in our case, the terms in $\mathcal{R}^\ell_{\psi}$ exert certain degrees of interdependency.  For instance, the terms associated with $(\xp_1, \xn_1) $ and $(\xp_2, \xn_2)$ are interdependent as long as $\xp_1 = \xp_2$ or $\xn_1 = \xn_2$. To address this issue, previous work applies coloring based Rademacher complexity  \cite{part,entropy}, making the calculation of the (local) Rademacher complexity much more complicated. To avoid this trick, we instead derive the following error decomposition lemma:
    \begin{lem}\label{lem:errde}
      The following error decomposition inequality holds:
    \begin{align*}
      &\sup_{f \in \mathcal{F}} \bigg[ \expe_{\xp,\xn} \left[ \gf(\xp,\xn)   \right]] \\ 
       &~~~~- \frac{1}{\np\nn}\sumpn  \frac{K^2}{(K-1)^2}\gf(\xpi,\xnj) \bigg]\\  
      \le& 
       \sup_{f \in \mathcal{F}} \bigg[ \expe_{\xp,\xn} \left[ \gf(\xp,\xn)   \right] \\ 
       &~~~~~- \frac{K}{(K-1)} \frac{1}{\np} \sump \expe_{\xn} \gf(\xpi,\xn) \bigg] \\
       +&   \sup_{f \in \mathcal{F}} 
      \bigg[ \frac{K}{(K-1)} \frac{1}{\np} \sump \expe_{\xn} \gf(\xpi,\xn) \\ 
      &~~~~~ - \frac{1}{\np\nn} \sumpn  \frac{K^2}{(K-1)^2}\gf(\xpi,\xnj)  \bigg] 
    \end{align*}
    where $\gf(\xp,\xn) = \psi_\gamma(1-\f(\xp)) \psi_\gamma(\f(\xn)) \ell(\f(\xp) - \f(\xn))$.
    \end{lem}
With the lemma above, the generalization bound could be divided into two parts. Both of them could be expressed as a  sum over i.i.d terms. This allows us to employ the standard local Rademacher complexity with proper extensions.

 With all the preliminaries mentioned above, we first provide a generic result under \textbf{Asm.\ref{asm:imb}} as the following theorem:
 \begin{thm}\label{thm:abs}
  Assume that the weighting function $\psi$ is $L_v$ Lipschitz continuous, the loss function $\ell$ is $L_\ell$ Lipschitz continuous with  $||\psi||_\infty  = v_\infty$, $||\ell||_\infty = L_\infty$, and the covering number of $\mathcal{F}$ has the following form with respect to $||\cdot||_\infty$ norm:
  \begin{align*}
    \log\left(\mathcal{N}\left(\mathcal{F}, \epsilon, ||\cdot||_\infty  \right)\right) \le A log(R/\epsilon),
  \end{align*}
  then for all $f \in \mathcal{F}$ and $(\alpha,\beta) \in \mathcal{I}^2_{suff}$,  the following inequality holds with probability at least $1-\delta$:
  \begin{align*}
    &\mathcal{R}^{\alpha,\beta}(f, \mathcal{S}) - \frac{K^2}{(K-1)^2}\hat{\mathcal{R}}^{\ell}_\psi(f,\mathcal{S})\\ 
     \le&  C \cdot\frac{A\log\left(\Gamma\nn\right) + \log(2/\delta)}{\np},~~ \forall K >1
  \end{align*}
where $C, \Gamma$ are universal constants depending on $v_\infty, \ell_\infty, L_\ell, L_v, K, R$, and
  \begin{align*}
  \mathcal{I}^2_{suff}(\mathcal{S}) =  \bigg\{ &(\alpha,\beta): \alpha \in (0,1),~ \beta \in(0,1),~ \npa \in \mathbb{N}_+,~ \\  
  &\nnb \in \mathbb{N}_+, \text{condition (b) in Prop.\ref{prop:concon} holds} \bigg\}.
\end{align*}

\end{thm}
\noindent Let $K \rightarrow \infty$, we reach the desired result:
\begin{equation*}
  \sup\left[ \mathcal{R}^{\alpha,\beta}(f, \mathcal{S}) -  \hat{\mathcal{R}}^{\ell}_\psi(f,\mathcal{S}) \right] = O\left(\frac{A\log\left(\Gamma\nn\right)}{\np}\right).
\end{equation*}
with high probability. Moreover, we have the following crucial remarks on the theorem above:
\begin{rem}\label{rem:2}
It is easy to show that when $\psi(\cdot) \equiv 1$, TPAUC objective function degenerates to a normal AUC risk function. Moreover, if we set the weight of the positive examples to be 1, TPAUC objective function degenerates to an OPAUC approximated risk function. In this sense, our result could be extended generally to the case of AUC and OPAUC optimization settings.
\end{rem}

\begin{rem}\label{rem:3}
Note that our upper bound scales as $O(\frac{\mathsf{Polylog}(\nn)}{\np})$ with respect to the sample size, that is much sharper than the existing results such as \cite{genal,aucrade,geninde,genpac}, if we degenerate our result to ordinary AUC. Taking a step further, if the dataset is roughly balanced in the sense that $\np = \theta(\np + \nn)$, we can even recover the result of $O(1/n)$ for U-statistics-based analysis \cite{cle1} with much weaker assumptions. 
  \end{rem}

Next, we present a concrete upper bound for deep convolution neural networks.  First, we introduce the setting of the deep neural networks employed in the forthcoming theoretical analysis, which is adopted from \cite{convbound}.\\ 
 \textbf{Settings.} We focus on the deep neural networks with $N_{conn}$ fully-connected layers and $N_{conv}$ convolutional layers. The $i$-th convolutional layer has a kernel $\bm{K}^{(i)} \in \mathbb{R}^{k_i \times k_i \times c_{i-1} \times c_i}$. Recall that convolution is a linear operator. For a given kernel $\bm{K}$, we denote its associated matrix as $op(\bm{K})$, such that $\bm{K}(\bm{x}) = op(\bm{K}) \bm{x}$. Moreover, we assume that, each time, the convolution layer is followed by a componentwise non-linear activation function and an optional pooling operation. We assume that the activation functions and the pooling operations are all $1$-Lipschitz. For the $i$-th fully-connected layer,  we denote its weight as $\bm{V}^{(i)}$. Above all, the complete parameter set of a given deep neural network could be represented as $\bm{P}= \{\bm{K}^{(1)},\cdots, \bm{K}^{(N_{conv})}, \bm{V}^{(1)}, \cdots, \bm{V}^{(N_{conn})}\}$. Again, we also assume that the loss function is $\phil$-Lipschitz and $Range\{\ell\} \subseteq [0,B]$. Finally, given two deep neural networks with parameters $\bm{P}$ and $\bm{\tilde{P}}$, we adopt a metric $d_{NN}(\cdot,\cdot)$ to measure their distance:
\begin{equation*}
\begin{split}
&d_{NN}(\bm{P},\bm{\tilde{P}}) = \sum_{i=1}^{N_{conv}} ||op(\bm{K}^{(i)}) -op(\bm{\tilde{K}}^{(i)})||_2 \\
&+ \sum_{i=1}^{N_{conn}} ||\bm{V}^{(i)} -\bm{\tilde{V}}^{(i)}||_2.
\end{split}
\end{equation*} \\
\textbf{Constraints Over the Parameters}. First, we define  $\mathcal{P}^{(0)}_{\nu}$ as the class for \underline{initialization of the parameters}:
\begin{equation*}
\begin{split}
\mathcal{P}^{(0)}_{\nu} &= \bigg\{\bm{P}: \left(\max_{i\in \{1,\cdots, N_{conv}\}}||op(\bm{K}^{(i)})||_2 \right)\\
&\le 1+\nu, ~ \left(\max_{j \in \{1,\cdots, N_{conn}\}} ||\bm{V}^{(j)}||_2\right) \le 1+\nu \bigg\}.
\end{split}
\end{equation*}
Now we further assume that the learned parameters should be chosen from a class denoted by $\mathcal{P}_{\beta, \nu}$, where the distance between the learned parameter and the fixed initialization residing in $\mathcal{P}^{(0)}_{\nu}$ is no bigger than $d$:
\begin{equation*}
\mathcal{P}_{\beta, \nu} = \bigg\{\bm{P}: d_{NN}(\bm{P},\bm{\tilde{P}}_0) \le d, ~\bm{\tilde{P}}_0 \in \mathcal{P}^{(0)}_\nu \bigg\}.
\end{equation*}

 \begin{table*}[ht]
  \centering
  \setlength{\abovecaptionskip}{0pt}    
  \setlength{\belowcaptionskip}{15pt}    
  \caption{Data Description.}
  \setlength{\tabcolsep}{5pt}
  \resizebox{0.95\textwidth}{!}{%
   \begin{tabular}{lllll}
    \toprule
    Dataset  & Pos. Class ID & Pos. Class Name & \# Pos. Examples & \# Neg. Examples \\
    \toprule
    CIFAR-10-LT-1  & $2$   & birds         &$2,155$  & $12,729$ \\
    CIFAR-10-LT-2  & $1$   & automobiles   &$3,596$  & $11,288$ \\
    CIFAR-10-LT-3  & $3$   & cats          &$1,292$  & $13,592$ \\
    \midrule
    CIFAR-100-LT-1 & $6,7,14,18,24$   & insects              &$2,756$  & $18,938$ \\
    CIFAR-100-LT-2 & $0,51,53,57,83$  & fruits and vegetables &$1,266$ & $20,428$ \\
    CIFAR-100-LT-3 & $15,19,21,32,38$ & large omnivores and herbivores &$1,678$ & $20,016$ \\
    \midrule
    Tiny-ImageNet-200-LT-1  & $24,25,26,27,28,29$   & dogs   &$3,000$  & $97,000$ \\
    Tiny-ImageNet-200-LT-2  & $11, 20, 21, 22$   & birds   &$2,000$  & $98,000$ \\
    Tiny-ImageNet-200-LT-3  & $70, 81, 94, 107, 111, 116, 121, 133, 145, 153, 164, 166$   & vehicles   &$6,000$  & $94,000$ \\
    \bottomrule
   \end{tabular}
  }
  \label{tab:dataset}
 \end{table*}

\noindent Now we can apply Thm.\ref{thm:abs} to obtain the following corollary based on the result in \cite{convbound}. 
  \begin{col} \label{thm:cnn} Under the assumption of Thm.\ref{thm:abs},
 assume that for all $\x \in \mathcal{X}, ||\bm{x}||_F \le R_\mathcal{X}$. Define $\mathcal{F}_{\beta,\nu}$ as the hypothesis class of the scoring function induced by the above-mentioned CNN:
  \begin{equation*}
    \begin{split}
    \mathcal{F}_{d,\nu} = \{&s_{\bm{P}}:\mathbb{R}^{N_{N_L-1}} \rightarrow [0,1]| ~ \bm{P} \in \mathcal{P}_{\beta, \nu} \}.
    \end{split}
    \end{equation*}
For all $f \in \mathcal{F}_{d, \nu}$, and $(\alpha,\beta) \in \mathcal{I}^2_{suff}$, with the same assumption in Thm.\ref{thm:abs}, the following inequality holds with probability at least $1-\delta$:
  \begin{equation*}
    \mathcal{R}^{\alpha,\beta}(f, \mathcal{S}) \lesssim \hat{\mathcal{R}}^{\ell}_\psi(f,\mathcal{S}) +  \frac{N_{par}\log(C_L \cdot n_-)+\log(1/\delta)}{n_+}.
  \end{equation*}
\end{col}

  \begin{table*}[htbp]
    \centering
    \small
    \caption{Performance Comparisons on CIFAR-10-LT with different metrics, where $(x,y)$ stands for $\mathsf{TPAUC}(x,y)$ in short and the first and second best results are highlighted with \textbf{bold text} and \underline{underline}, respectively.}
    \resizebox{0.95\textwidth}{!}{%
      \begin{tabular}{c|c|c|ccc|ccc|ccc}
        \toprule
        \multirow{2}[4]{*}{dataset} & \multirow{2}[4]{*}{type} & \multirow{2}[4]{*}{methods} & \multicolumn{3}{c|}{Subset1} & \multicolumn{3}{c|}{Subset2} & \multicolumn{3}{c}{Subset3} \\
        \cmidrule{4-12}         &      &      & (0.3,0.3)  & (0.4,0.4)  & (0.5,0.5)  & (0.3,0.3)  & (0.4,0.4)  & (0.5,0.5)  & (0.3,0.3)  & (0.4,0.4)  & (0.5,0.5) \\
        \midrule
        \multirow{13}[6]{*}{CIFAR-10-LT} & \multirow{9}[2]{*}{Competitors} & CE-RW & \cellcolor[rgb]{ 1,  .988,  .984} 9.09 & \cellcolor[rgb]{ 1,  .988,  .984} 30.86 & \cellcolor[rgb]{ .996,  .973,  .961} 47.99 & 72.83 & \cellcolor[rgb]{ 1,  .988,  .984} 83.33 & \cellcolor[rgb]{ 1,  .988,  .984} 88.71 & \cellcolor[rgb]{ .996,  .965,  .949} 23.47 & \cellcolor[rgb]{ .996,  .973,  .961} 44.44 & \cellcolor[rgb]{ .996,  .973,  .961} 59.69 \\
        &      & Focal & \cellcolor[rgb]{ 1,  .988,  .984} 9.84 & \cellcolor[rgb]{ 1,  .988,  .984} 30.89 & \cellcolor[rgb]{ .996,  .965,  .949} 50.83 & \cellcolor[rgb]{ 1,  .988,  .984} 75.72 & \cellcolor[rgb]{ .996,  .965,  .949} 85.10 & \cellcolor[rgb]{ .996,  .965,  .949} 90.06 & \cellcolor[rgb]{ 1,  .988,  .984} 21.47 & \cellcolor[rgb]{ .996,  .973,  .961} 45.88 & \cellcolor[rgb]{ .996,  .973,  .961} 59.09 \\
        &      & CBCE & 3.29 & 27.30 & 43.95 & 69.48 & 80.80 & 86.87 & 12.94 & 34.06 & 51.09 \\
        &      & CBFocal & \cellcolor[rgb]{ 1,  .988,  .984} 9.04 & \cellcolor[rgb]{ .996,  .973,  .961} 31.73 & \cellcolor[rgb]{ .996,  .973,  .961} 48.13 & \cellcolor[rgb]{ .996,  .973,  .961} 77.99 & \cellcolor[rgb]{ .996,  .949,  .925} 86.75 & \cellcolor[rgb]{ .996,  .949,  .925} 91.13 & \cellcolor[rgb]{ 1,  .988,  .984} 21.32 & \cellcolor[rgb]{ 1,  .988,  .984} 43.03 & \cellcolor[rgb]{ .996,  .973,  .961} 59.11 \\
        &      & SqAUC & \cellcolor[rgb]{ .996,  .965,  .949} 18.05 & \cellcolor[rgb]{ .996,  .949,  .925} 40.74 & \cellcolor[rgb]{ .996,  .949,  .925} 57.94 & \cellcolor[rgb]{ .996,  .949,  .925} 80.09 & \cellcolor[rgb]{ .996,  .949,  .925} 87.78 & \cellcolor[rgb]{ .996,  .949,  .925} 91.87 & \cellcolor[rgb]{ .988,  .925,  .882} 31.52 & \cellcolor[rgb]{ .996,  .949,  .925} 50.00 & \cellcolor[rgb]{ .996,  .949,  .925} 64.42 \\
        &      & TruncOPAUC & \cellcolor[rgb]{ .996,  .973,  .961} 15.00 & \cellcolor[rgb]{ .996,  .949,  .925} 42.05 & \cellcolor[rgb]{ .996,  .949,  .925} 58.27 & \cellcolor[rgb]{ .98,  .863,  .788} \underline{84.71} & \cellcolor[rgb]{ .98,  .863,  .788} 90.11 & \cellcolor[rgb]{ .98,  .863,  .788} 93.30 & \cellcolor[rgb]{ .988,  .925,  .882} 31.11 & \cellcolor[rgb]{ .988,  .925,  .882} 51.47 & \cellcolor[rgb]{ .98,  .863,  .788} 66.07 \\
        &      & Op-Poly & \cellcolor[rgb]{ .996,  .949,  .925} 20.69 & \cellcolor[rgb]{ .996,  .949,  .925} 41.85 & \cellcolor[rgb]{ .996,  .949,  .925} 58.97 & \cellcolor[rgb]{ .996,  .949,  .925} 80.54 & \cellcolor[rgb]{ .996,  .949,  .925} 87.97 & \cellcolor[rgb]{ .988,  .925,  .882} 92.09 & \cellcolor[rgb]{ .98,  .863,  .788} 32.04 & \cellcolor[rgb]{ .98,  .863,  .788} \underline{53.41} & \cellcolor[rgb]{ .98,  .863,  .788} 66.30 \\
        &      & OP-Exp & \cellcolor[rgb]{ .996,  .949,  .925} 20.86 & \cellcolor[rgb]{ .996,  .965,  .949} 36.70 & \cellcolor[rgb]{ .996,  .949,  .925} 58.27 & \cellcolor[rgb]{ .996,  .965,  .949} 79.42 & \cellcolor[rgb]{ .996,  .949,  .925} 86.95 & \cellcolor[rgb]{ .996,  .965,  .949} 90.93 & \cellcolor[rgb]{ .996,  .949,  .925} 28.94 & \cellcolor[rgb]{ .988,  .925,  .882} 51.23 & \cellcolor[rgb]{ .98,  .863,  .788} 66.34 \\
        &      & TruncTPAUC & \cellcolor[rgb]{ 1,  .988,  .984} 11.90 & \cellcolor[rgb]{ 1,  .988,  .984} 29.13 & \cellcolor[rgb]{ 1,  .988,  .984} 46.64 & \cellcolor[rgb]{ .996,  .949,  .925} 81.26 & \cellcolor[rgb]{ .98,  .863,  .788} 89.12 & \cellcolor[rgb]{ .988,  .925,  .882} 92.86 & \cellcolor[rgb]{ .996,  .965,  .949} 23.70 & \cellcolor[rgb]{ .996,  .965,  .949} 47.75 & \cellcolor[rgb]{ .996,  .949,  .925} 64.92 \\
        \cmidrule{2-12}         & \multirow{2}[2]{*}{Ours-TPAUC} & Poly & \cellcolor[rgb]{ .988,  .925,  .882} 21.43 & \cellcolor[rgb]{ .988,  .925,  .882} 44.41 & \cellcolor[rgb]{ .988,  .925,  .882} 59.10 & \cellcolor[rgb]{ .996,  .949,  .925} 80.66 & \cellcolor[rgb]{ .988,  .925,  .882} 88.07 & \cellcolor[rgb]{ .988,  .925,  .882} 92.15 & \cellcolor[rgb]{ .973,  .796,  .678} \textbf{36.54} & \cellcolor[rgb]{ .973,  .796,  .678} \textbf{54.48} & \cellcolor[rgb]{ .973,  .796,  .678} \underline{67.19} \\
        &      & Exp  & \cellcolor[rgb]{ .996,  .949,  .925} 20.86 & \cellcolor[rgb]{ .996,  .949,  .925} 41.78 & \cellcolor[rgb]{ .996,  .949,  .925} 58.38 & \cellcolor[rgb]{ .996,  .949,  .925} 81.22 & \cellcolor[rgb]{ .996,  .949,  .925} 87.88 & \cellcolor[rgb]{ .996,  .949,  .925} 91.93 & \cellcolor[rgb]{ .98,  .863,  .788} 32.47 & \cellcolor[rgb]{ .98,  .863,  .788} 53.86 & \cellcolor[rgb]{ .973,  .796,  .678} \textbf{67.32} \\
        \cmidrule{2-12}         & \multirow{2}[2]{*}{Ours-minmax} & Poly & \cellcolor[rgb]{ .973,  .796,  .678} \textbf{31.77} & \cellcolor[rgb]{ .973,  .796,  .678} \textbf{46.34} & \cellcolor[rgb]{ .973,  .796,  .678} \textbf{62.38} & \cellcolor[rgb]{ .973,  .796,  .678} \textbf{86.75} & \cellcolor[rgb]{ .973,  .796,  .678} \textbf{91.90} & \cellcolor[rgb]{ .973,  .796,  .678} \textbf{94.65} & \cellcolor[rgb]{ .98,  .863,  .788} \underline{33.24} & \cellcolor[rgb]{ .98,  .863,  .788} 53.19 & \cellcolor[rgb]{ .98,  .863,  .788} 66.76 \\
        &      & Exp  & \cellcolor[rgb]{ .98,  .863,  .788} \underline{25.53} & \cellcolor[rgb]{ .98,  .863,  .788} \underline{45.49} & \cellcolor[rgb]{ .98,  .863,  .788} \underline{60.82} & \cellcolor[rgb]{ .988,  .925,  .882} 83.11 & \cellcolor[rgb]{ .98,  .863,  .788} \underline{89.63} & \cellcolor[rgb]{ .98,  .863,  .788} \underline{93.32} & \cellcolor[rgb]{ .988,  .925,  .882} 31.91 & \cellcolor[rgb]{ .98,  .863,  .788} 53.27 & \cellcolor[rgb]{ .988,  .925,  .882} 65.80 \\
        \bottomrule
      \end{tabular}%
      
    }
    \label{tab:perf-cifar-10}%
  \end{table*}%

  \section{Experiments}\label{sec:exp}

  In this section, we present our empirical results and some of the details of the experiments. \emph{\textbf{Please see Appendix.\ref{sec:app_exp} for more details on the settings and results.}}



  
    
    

\begin{table*}[htbp]
	\centering
	\small
	\caption{Performance Comparisons on CIFAR-100-LT with different metrics, where $(x,y)$ stands for $\mathsf{TPAUC}(x,y)$ in short and the first and second best results are highlighted with \textbf{bold text} and \underline{underline}, respectively.}
	\resizebox{0.95\textwidth}{!}{%
		\begin{tabular}{c|c|c|ccc|ccc|ccc}
			\toprule
			\multirow{2}[4]{*}{dataset} & \multirow{2}[4]{*}{type} & \multicolumn{1}{c|}{\multirow{2}[4]{*}{methods}} & \multicolumn{3}{c|}{Subset1} & \multicolumn{3}{c|}{Subset2} & \multicolumn{3}{c}{Subset3} \\
      \cmidrule{4-12}         &      &      & (0.3,0.3)  & (0.4,0.4)  & (0.5,0.5)  & (0.3,0.3)  & (0.4,0.4)  & (0.5,0.5)  & (0.3,0.3)  & (0.4,0.4)  & (0.5,0.5) \\

			\midrule
			\multirow{13}[4]{*}{CIFAR-100-LT} & \multirow{9}[1]{*}{Competitors} & CE-RW & \cellcolor[rgb]{ .984,  .992,  .976} 31.43 & \cellcolor[rgb]{ .984,  .992,  .976} 52.60 & \cellcolor[rgb]{ .961,  .976,  .953} 66.21 & \cellcolor[rgb]{ .984,  .992,  .976} 79.70 & \cellcolor[rgb]{ .961,  .976,  .953} 88.06 & \cellcolor[rgb]{ .961,  .976,  .953} 92.64 & 3.09 & 21.32 & \cellcolor[rgb]{ .984,  .992,  .976} 40.75 \\
			&      & Focal & \cellcolor[rgb]{ .984,  .992,  .976} 36.51 & \cellcolor[rgb]{ .961,  .976,  .953} 61.71 & \cellcolor[rgb]{ .945,  .969,  .929} 73.25 & \cellcolor[rgb]{ .945,  .969,  .929} 83.08 & \cellcolor[rgb]{ .914,  .953,  .89} 90.35 & \cellcolor[rgb]{ .945,  .969,  .929} 93.76 & \cellcolor[rgb]{ .984,  .992,  .976} 8.09 & \cellcolor[rgb]{ .984,  .992,  .976} 28.88 & \cellcolor[rgb]{ .961,  .976,  .953} 49.89 \\
			&      & CBCE & 17.53 & 38.79 & 55.19 & 67.91 & 79.32 & 85.82 & 1.84 & 18.46 & 37.04 \\
			&      & CBFocal & \cellcolor[rgb]{ .961,  .976,  .953} 41.85 & \cellcolor[rgb]{ .945,  .969,  .929} 62.41 & \cellcolor[rgb]{ .945,  .969,  .929} 73.13 & \cellcolor[rgb]{ .961,  .976,  .953} 82.75 & \cellcolor[rgb]{ .945,  .969,  .929} 89.57 & \cellcolor[rgb]{ .961,  .976,  .953} 92.89 & \cellcolor[rgb]{ .984,  .992,  .976} 7.10 & \cellcolor[rgb]{ .984,  .992,  .976} 29.12 & \cellcolor[rgb]{ .984,  .992,  .976} 44.84 \\
			&      & SqAUC & \cellcolor[rgb]{ .886,  .937,  .851} 63.24 & \cellcolor[rgb]{ .835,  .91,  .784} 76.62 & \cellcolor[rgb]{ .835,  .91,  .784} 84.68 & \cellcolor[rgb]{ .776,  .878,  .706} 91.02 & \cellcolor[rgb]{ .835,  .91,  .784} 93.69 & \cellcolor[rgb]{ .886,  .937,  .851} 94.73 & \cellcolor[rgb]{ .886,  .937,  .851} 41.60 & \cellcolor[rgb]{ .886,  .937,  .851} 60.36 & \cellcolor[rgb]{ .886,  .937,  .851} 70.86 \\
			&      & TruncOPAUC & \cellcolor[rgb]{ .945,  .969,  .929} 56.51 & \cellcolor[rgb]{ .914,  .953,  .89} 70.56 & \cellcolor[rgb]{ .886,  .937,  .851} 81.03 & \cellcolor[rgb]{ .886,  .937,  .851} 87.72 & \cellcolor[rgb]{ .835,  .91,  .784} 93.26 & \cellcolor[rgb]{ .914,  .953,  .89} 94.23 & \cellcolor[rgb]{ .961,  .976,  .953} 22.75 & \cellcolor[rgb]{ .961,  .976,  .953} 51.30 & \cellcolor[rgb]{ .945,  .969,  .929} 66.78 \\
			&      & Op-Poly & \cellcolor[rgb]{ .914,  .953,  .89} 58.40 & \cellcolor[rgb]{ .914,  .953,  .89} 70.07 & \cellcolor[rgb]{ .914,  .953,  .89} 79.15 & \cellcolor[rgb]{ .886,  .937,  .851} 88.37 & \cellcolor[rgb]{ .886,  .937,  .851} 92.70 & \cellcolor[rgb]{ .886,  .937,  .851} 94.67 & \cellcolor[rgb]{ .914,  .953,  .89} 37.73 & \cellcolor[rgb]{ .914,  .953,  .89} 57.94 & \cellcolor[rgb]{ .914,  .953,  .89} 69.75 \\
			&      & OP-Exp & \cellcolor[rgb]{ .914,  .953,  .89} 59.01 & \cellcolor[rgb]{ .886,  .937,  .851} 73.00 & \cellcolor[rgb]{ .886,  .937,  .851} 81.36 & \cellcolor[rgb]{ .835,  .91,  .784} 89.18 & \cellcolor[rgb]{ .835,  .91,  .784} 93.06 & \cellcolor[rgb]{ .835,  .91,  .784} 95.08 & \cellcolor[rgb]{ .945,  .969,  .929} 30.83 & \cellcolor[rgb]{ .961,  .976,  .953} 52.12 & \cellcolor[rgb]{ .945,  .969,  .929} 65.34 \\
			&      & TruncTPAUC & \cellcolor[rgb]{ .961,  .976,  .953} 46.42 & \cellcolor[rgb]{ .961,  .976,  .953} 60.23 & \cellcolor[rgb]{ .914,  .953,  .89} 78.66 & \cellcolor[rgb]{ .914,  .953,  .89} 85.99 & \cellcolor[rgb]{ .886,  .937,  .851} 92.31 & \cellcolor[rgb]{ .886,  .937,  .851} 94.90 & \cellcolor[rgb]{ .961,  .976,  .953} 26.34 & \cellcolor[rgb]{ .945,  .969,  .929} 54.69 & \cellcolor[rgb]{ .945,  .969,  .929} 66.77 \\
			\cmidrule{2-12}         & \multirow{2}[2]{*}{Ours-TPAUC} & Poly & \cellcolor[rgb]{ .776,  .878,  .706} \textbf{68.02} & \cellcolor[rgb]{ .776,  .878,  .706} \textbf{79.11} & \cellcolor[rgb]{ .776,  .878,  .706} \underline{85.17} & \cellcolor[rgb]{ .776,  .878,  .706} \underline{91.13} & \cellcolor[rgb]{ .835,  .91,  .784} \underline{93.78} & \cellcolor[rgb]{ .776,  .878,  .706} \underline{95.69} & \cellcolor[rgb]{ .776,  .878,  .706} \textbf{47.07} & \cellcolor[rgb]{ .776,  .878,  .706} \textbf{65.89} & \cellcolor[rgb]{ .776,  .878,  .706} \textbf{75.08} \\
			&      & Exp  & \cellcolor[rgb]{ .886,  .937,  .851} 63.24 & \cellcolor[rgb]{ .835,  .91,  .784} 77.94 & \cellcolor[rgb]{ .835,  .91,  .784} 84.62 & \cellcolor[rgb]{ .835,  .91,  .784} 90.69 & \cellcolor[rgb]{ .835,  .91,  .784} 93.74 & \cellcolor[rgb]{ .776,  .878,  .706} 95.41 & \cellcolor[rgb]{ .835,  .91,  .784} \underline{44.54} & \cellcolor[rgb]{ .835,  .91,  .784} 64.58 & \cellcolor[rgb]{ .835,  .91,  .784} 73.02 \\
			\cmidrule{2-12}         & \multirow{2}[1]{*}{Ours-minmax} & Poly & \cellcolor[rgb]{ .835,  .91,  .784} 65.74 & \cellcolor[rgb]{ .835,  .91,  .784} \underline{78.35} & \cellcolor[rgb]{ .776,  .878,  .706} \textbf{85.24} & \cellcolor[rgb]{ .776,  .878,  .706} \textbf{91.40} & \cellcolor[rgb]{ .776,  .878,  .706} \textbf{94.05} & \cellcolor[rgb]{ .776,  .878,  .706} \textbf{95.81} & \cellcolor[rgb]{ .835,  .91,  .784} 44.24 & \cellcolor[rgb]{ .835,  .91,  .784} \underline{64.68} & \cellcolor[rgb]{ .835,  .91,  .784} \underline{73.60} \\
			&      & Exp  & \cellcolor[rgb]{ .835,  .91,  .784} \underline{66.79} & \cellcolor[rgb]{ .835,  .91,  .784} 77.72 & \cellcolor[rgb]{ .835,  .91,  .784} 84.87 & \cellcolor[rgb]{ .835,  .91,  .784} 90.27 & \cellcolor[rgb]{ .835,  .91,  .784} 93.25 & \cellcolor[rgb]{ .776,  .878,  .706} 95.41 & \cellcolor[rgb]{ .835,  .91,  .784} 43.99 & \cellcolor[rgb]{ .835,  .91,  .784} 64.62 & \cellcolor[rgb]{ .886,  .937,  .851} 71.76 \\
			
			\bottomrule
		\end{tabular}%
	}
	\label{tab:perf-cifar-100}%
\end{table*}%

\subsection{Evaluation Metrics}
Aiming at optimizing the TPAUC metrics, we consider TPAUC with $\alpha = 0.3, \beta = 0.3$, $\alpha = 0.4, \beta = 0.4$, $\alpha = 0.5, \beta = 0.5$, respectively. Moreover, to normalize the range of their magnitude to $[0,1]$, we adopt the following variant of the TPAUC metric:
\begin{equation*}
  \mathsf{TPAUC}(\alpha,\beta)  = 1- \psumpn \frac{\lt\left(\f(\xpio) - f(\xnjo)  \right)}{\npa\nnb}.
\end{equation*}


  \begin{figure}[th]  
    \centering
  
    \subfigure[Subset1, $\alpha = 0.3, \beta = 0.3$]{
      \includegraphics[width=0.46\columnwidth]{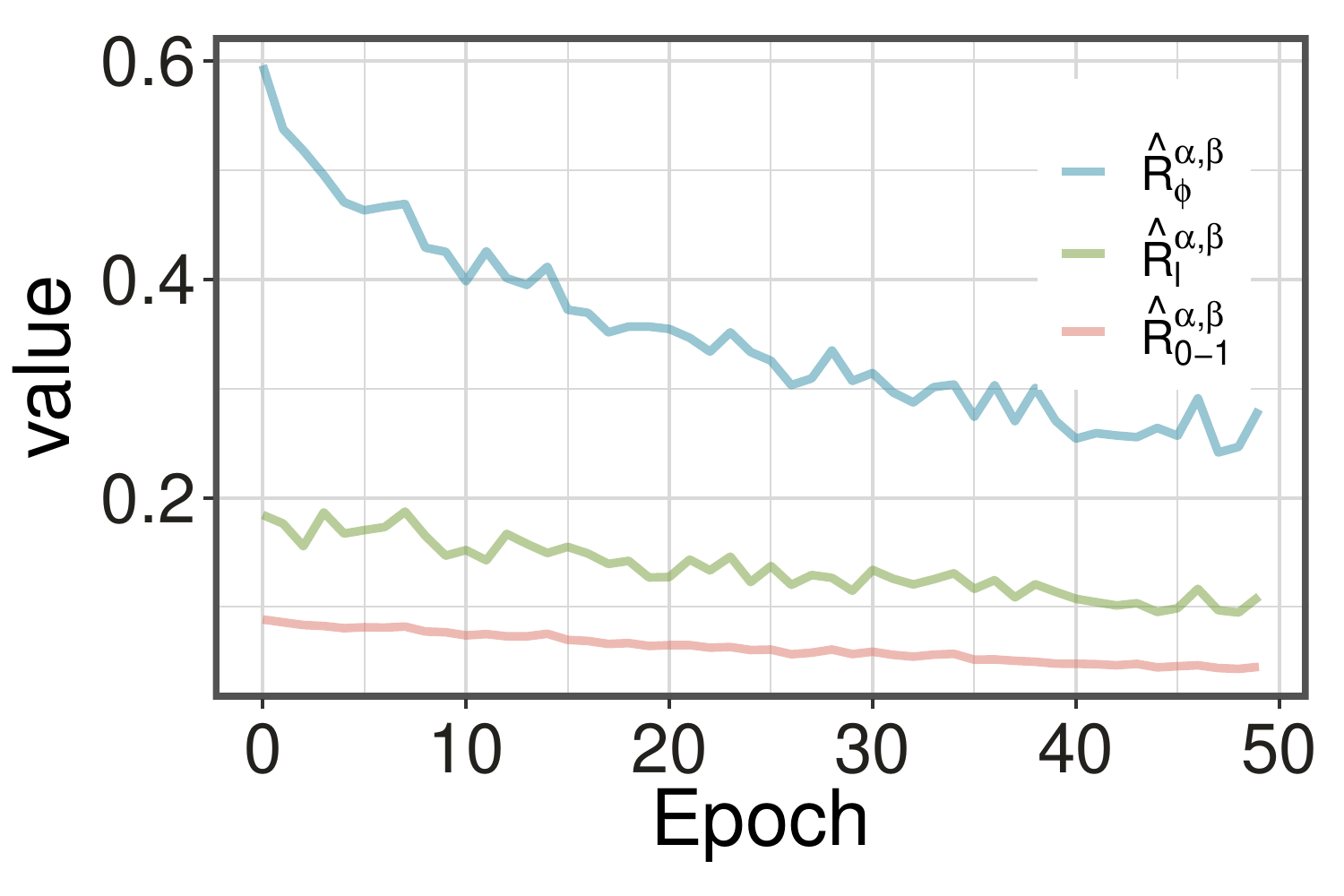} 
    }
    \subfigure[Subset1, $\alpha = 0.4, \beta = 0.4$]{
      \includegraphics[width=0.46\columnwidth]{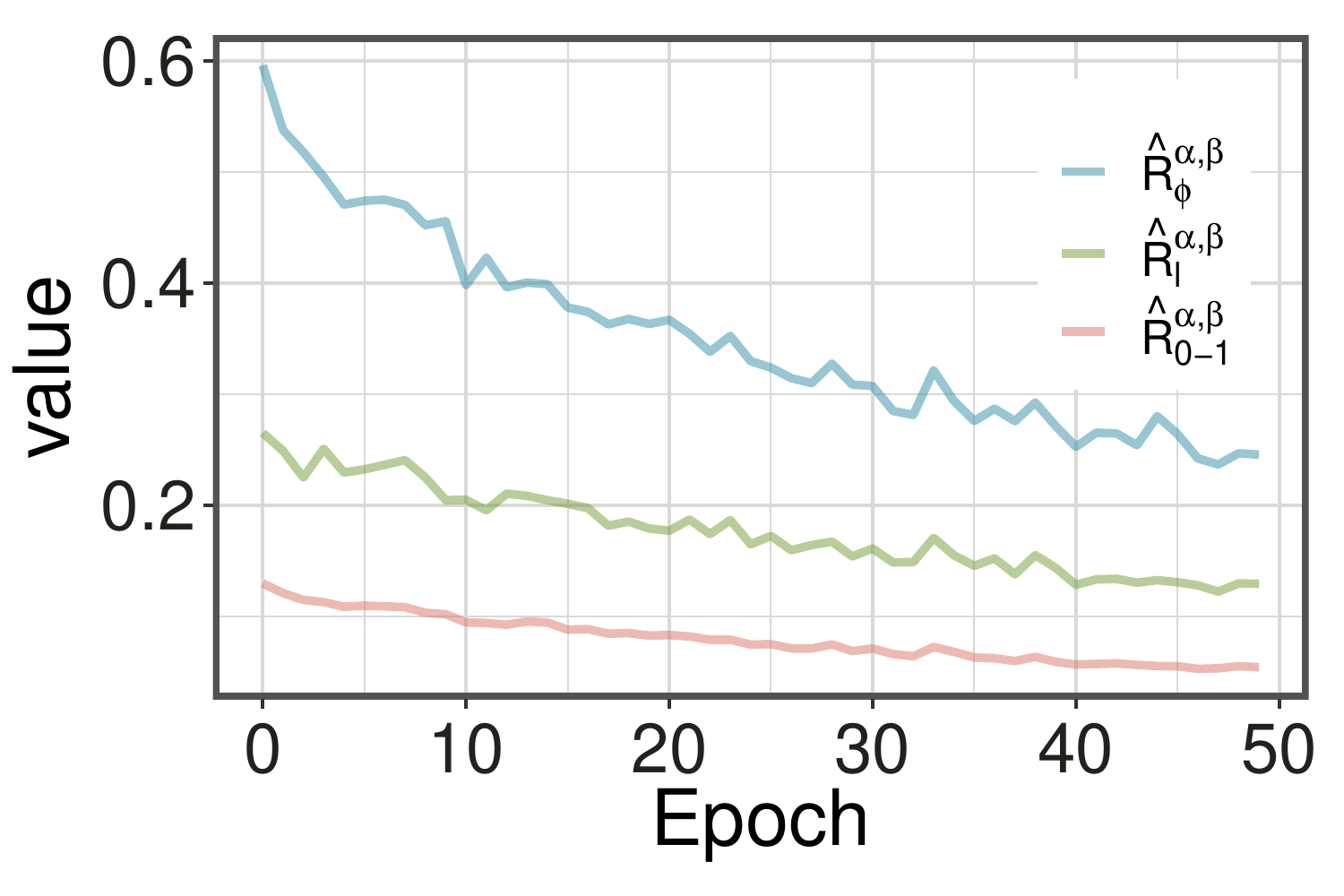} 
    }
  
    \subfigure[Subset1, $\alpha = 0.5, \beta = 0.5$]{
      \includegraphics[width=0.46\columnwidth]{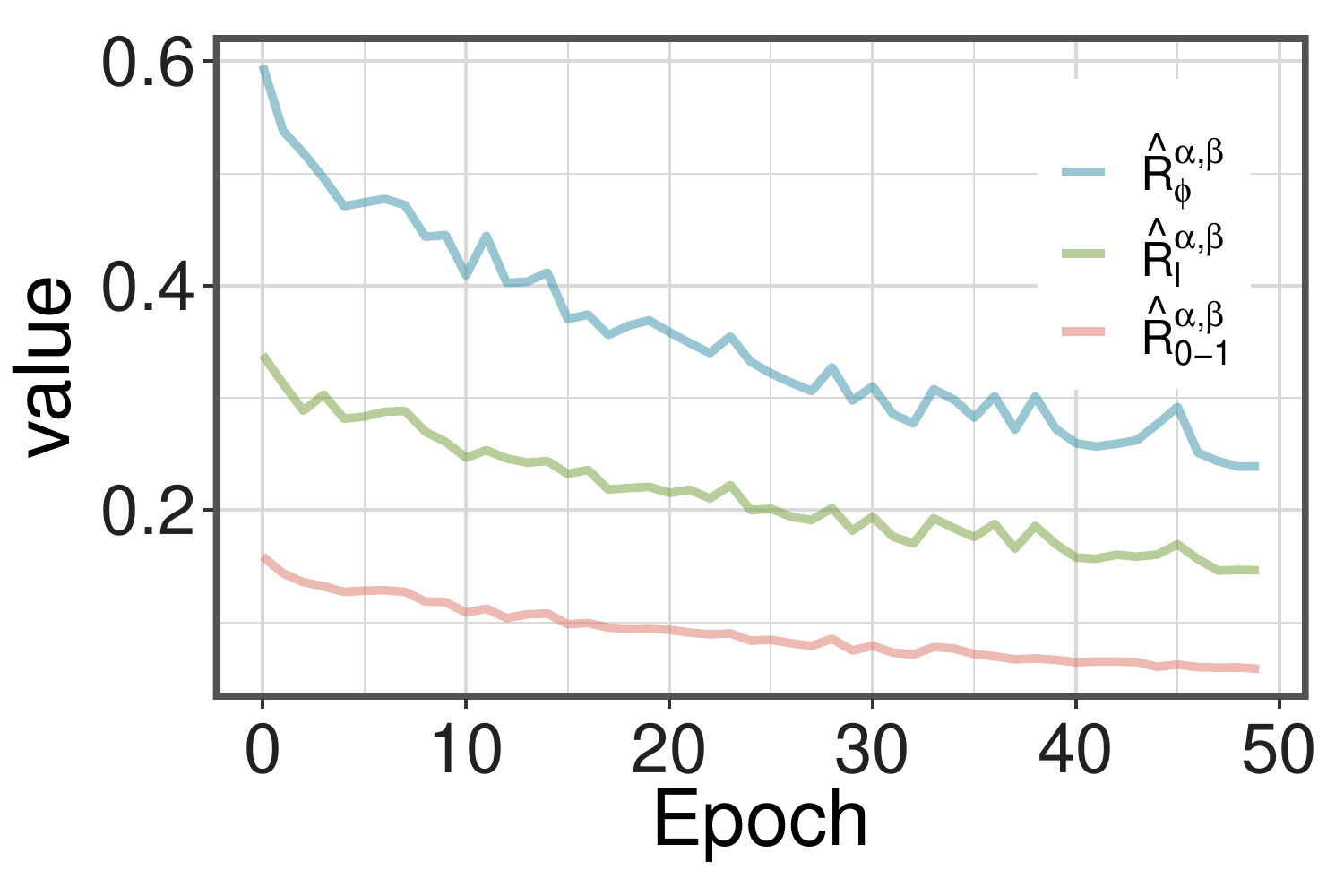} 
    }
    \subfigure[Subset2, $\alpha = 0.3, \beta = 0.3$]{
      \includegraphics[width=0.46\columnwidth]{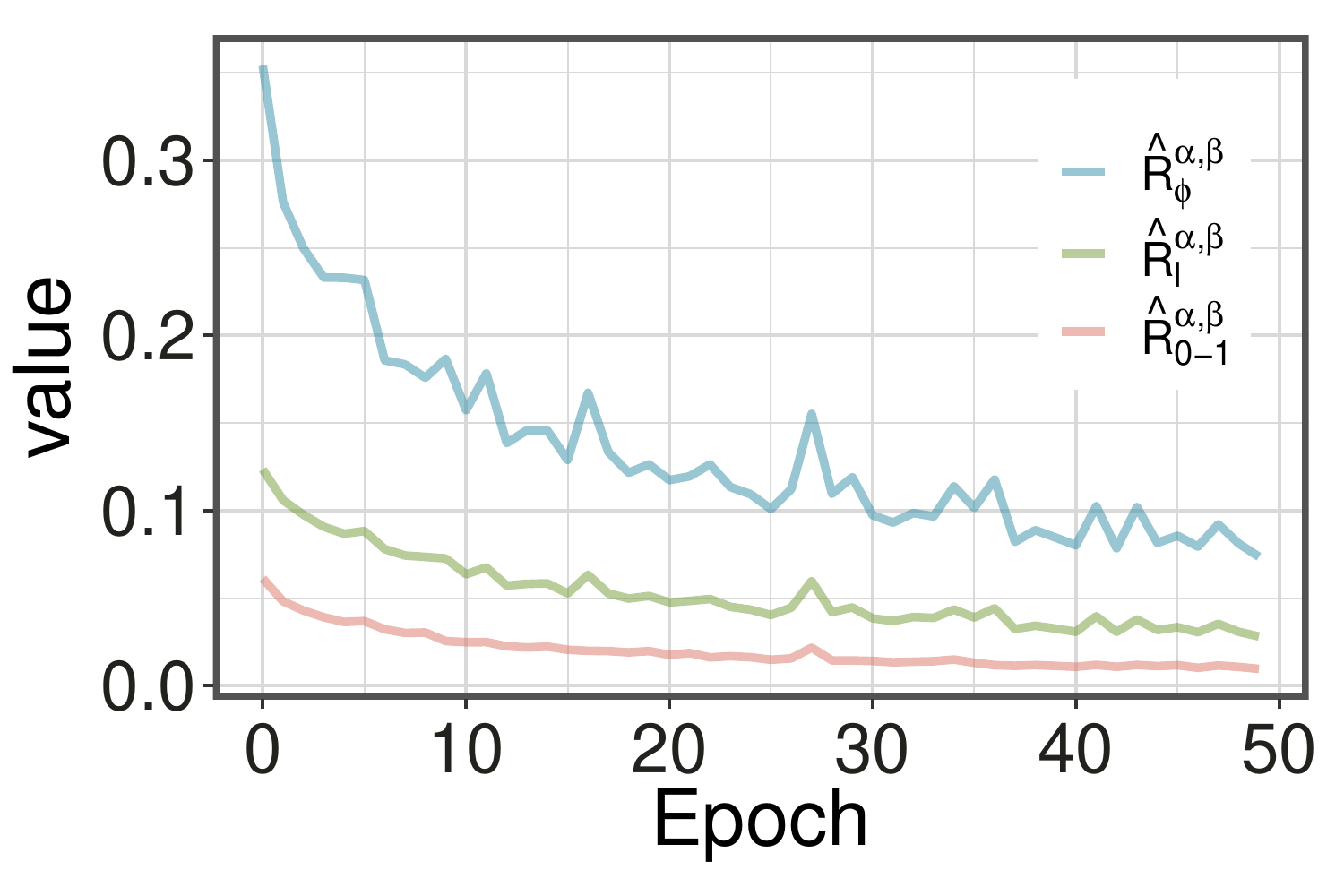} 
    }
  
    \subfigure[Subset2, $\alpha = 0.4, \beta = 0.4$]{
      \includegraphics[width=0.46\columnwidth]{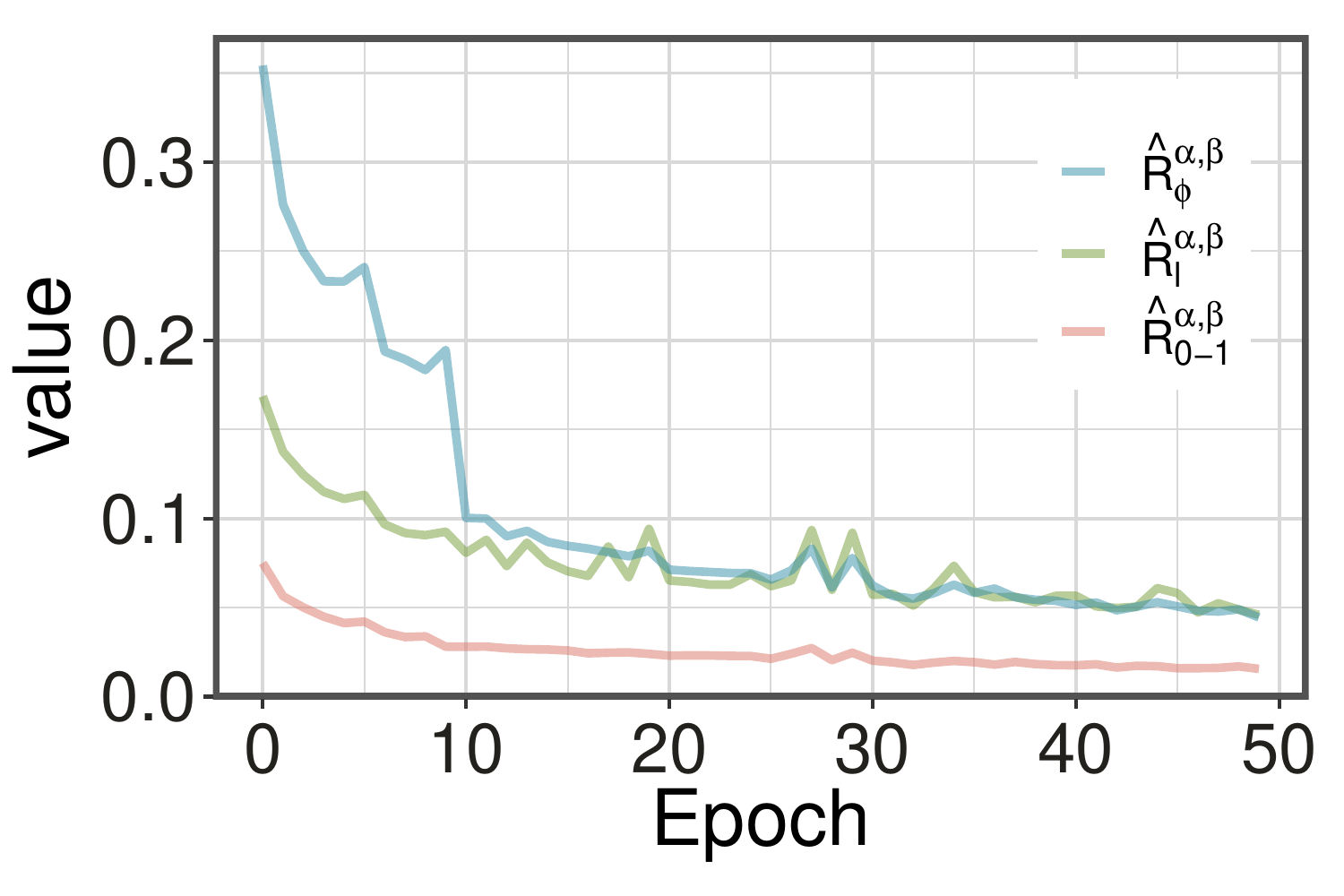} 
    }
    \subfigure[Subset2, $\alpha = 0.5, \beta = 0.5$]{
      \includegraphics[width=0.46\columnwidth]{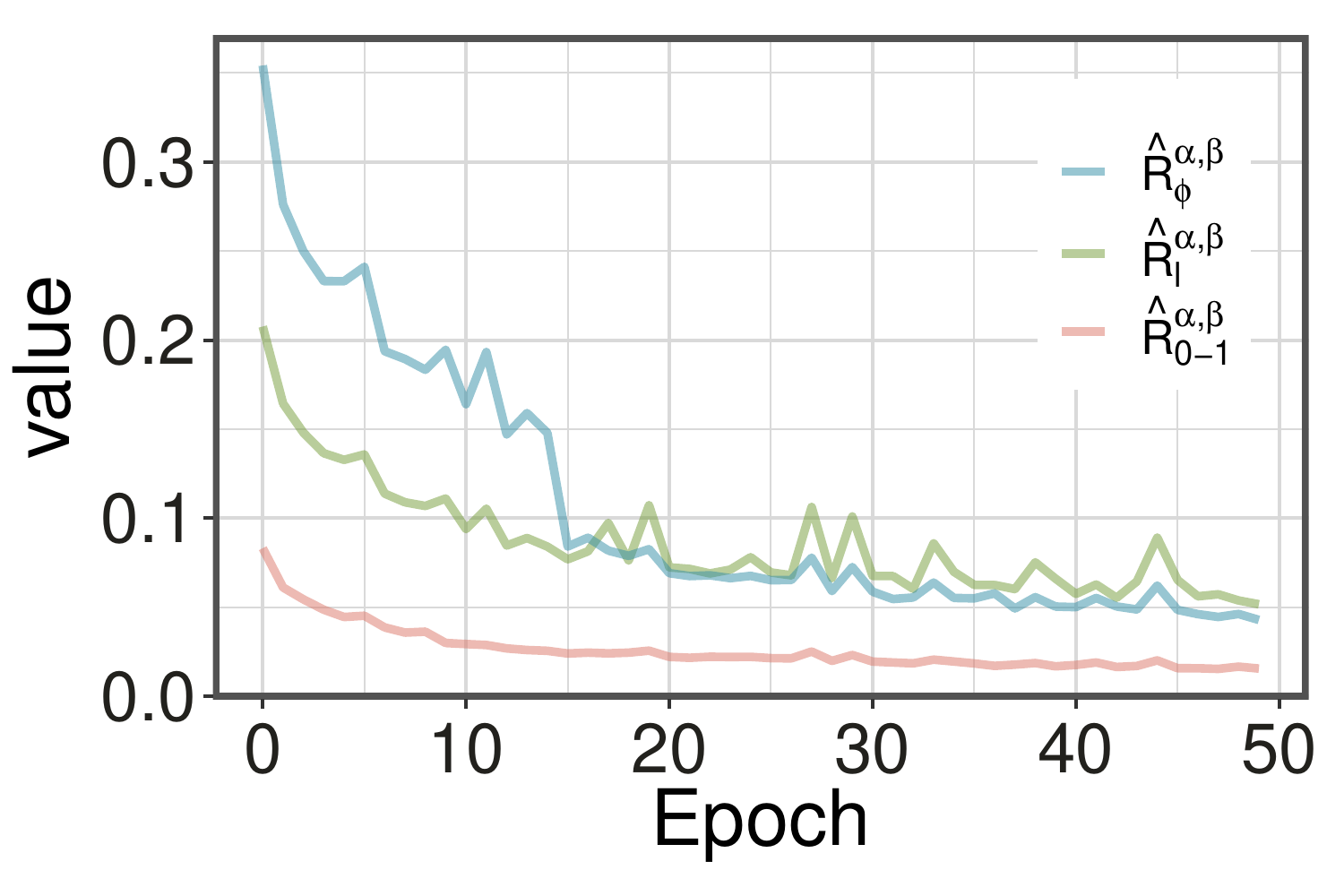} 
    }
  
    \subfigure[Subset3, $\alpha = 0.3, \beta = 0.3$]{
      \includegraphics[width=0.46\columnwidth]{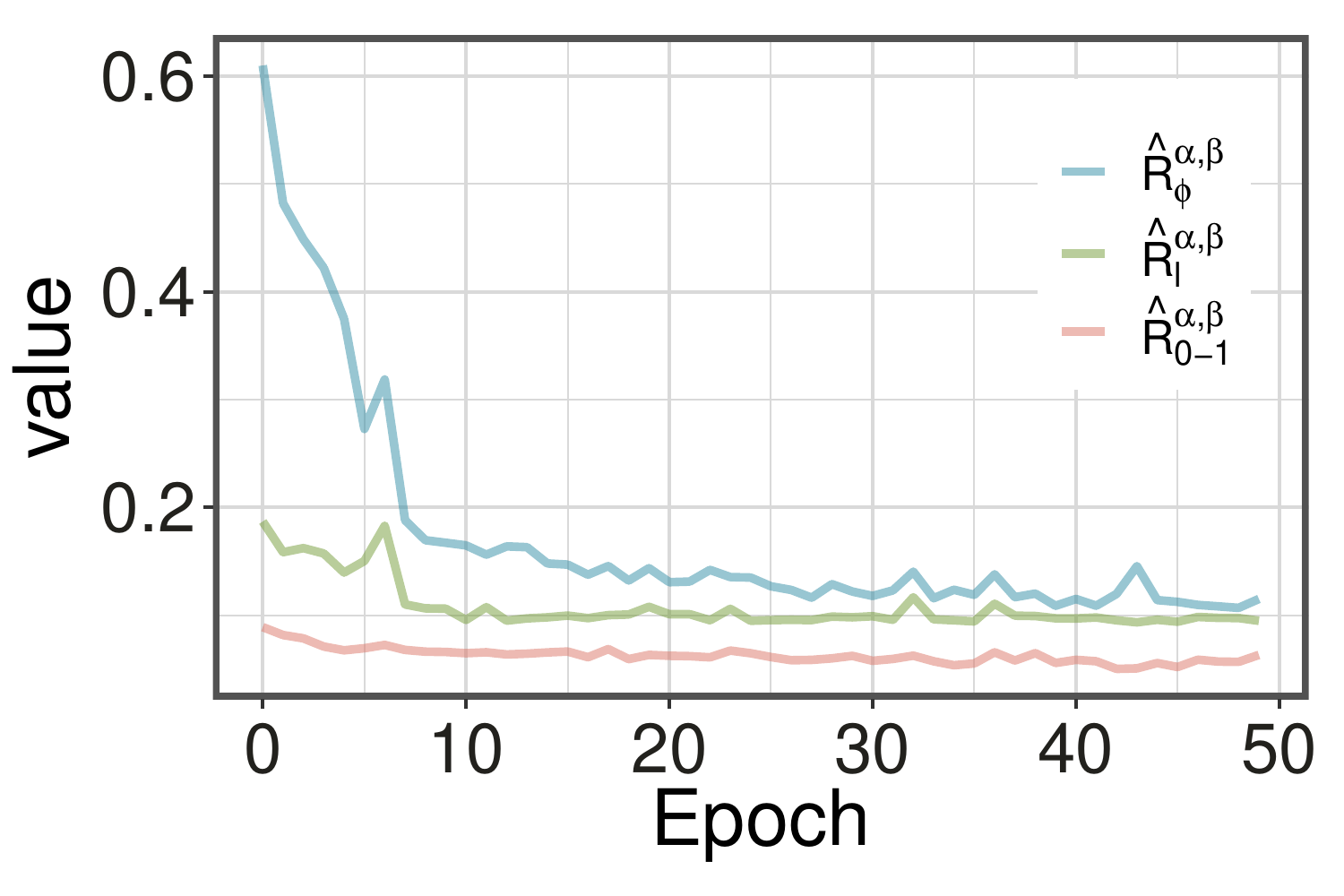} 
    }
    \subfigure[Subset3, $\alpha = 0.4, \beta = 0.4$]{
      \includegraphics[width=0.46\columnwidth]{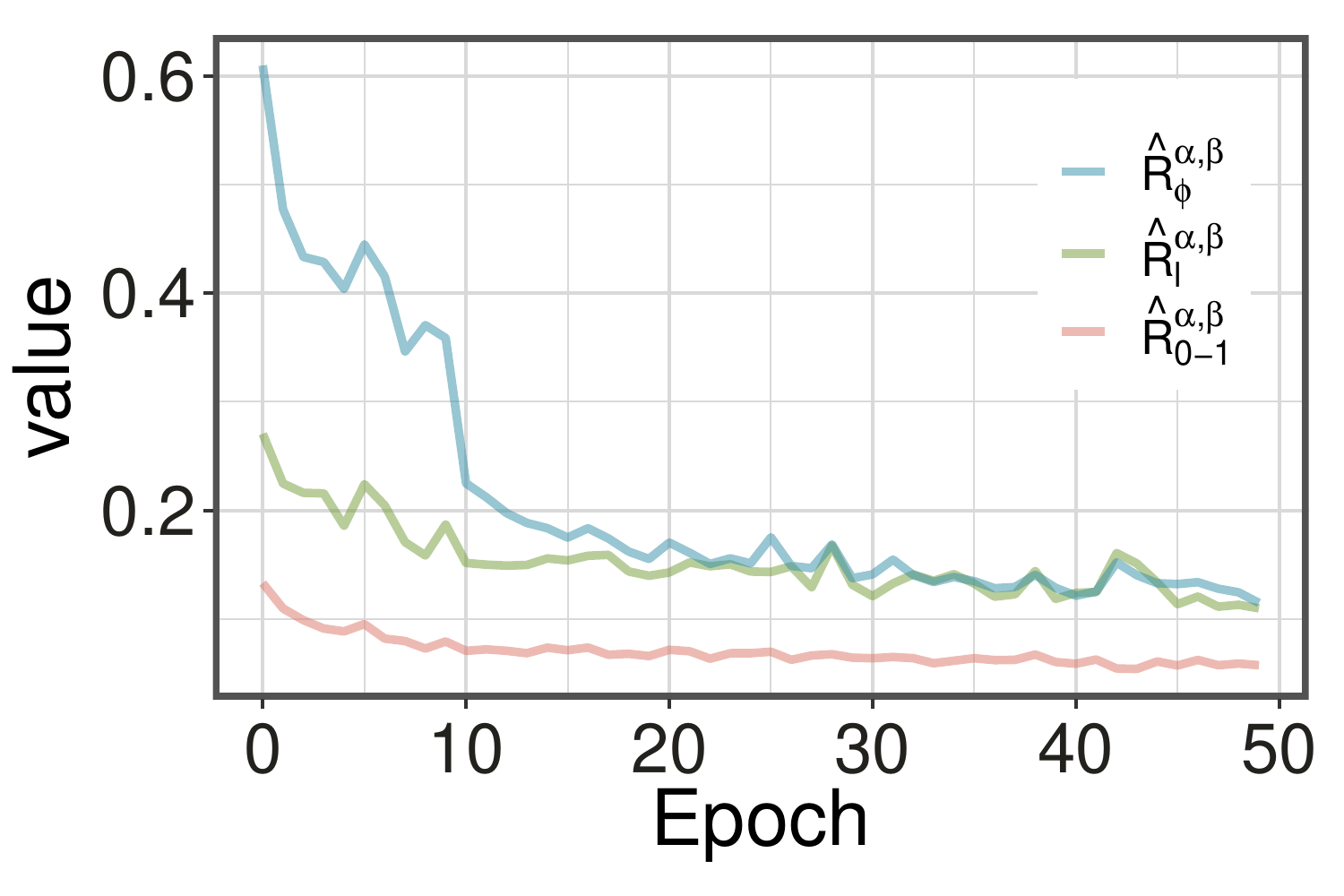} 
    }
  
    \subfigure[Subset3, $\alpha = 0.5, \beta = 0.5$]{
      \includegraphics[width=0.46\columnwidth]{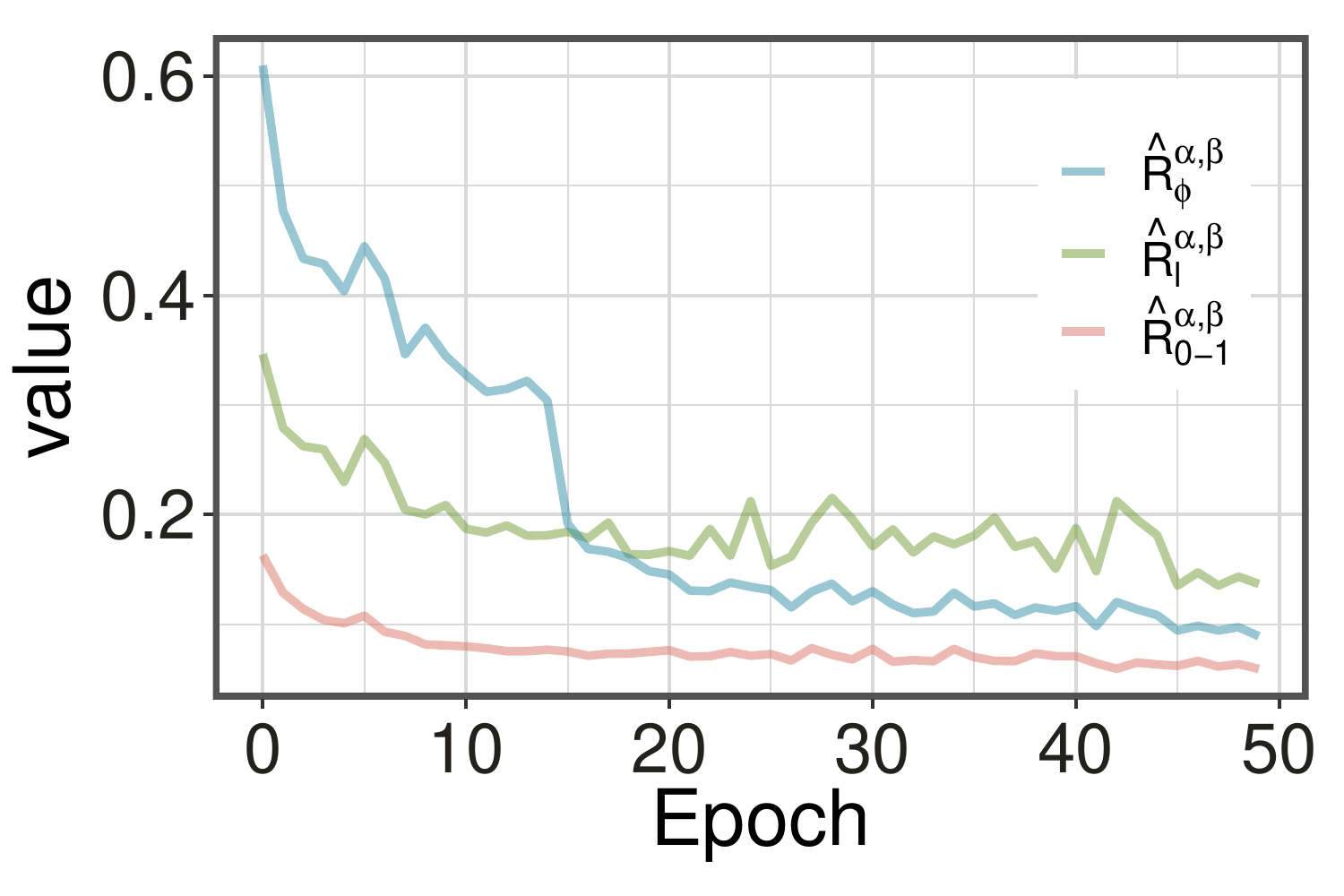} 
    }
  
    \caption{Empirical validations for Proposition 2(a). $\hat{\mathcal{R}}_\Phi^{\alpha,\beta}$, $\hat{\mathcal{R}}_\ell^{\alpha,\beta}$, $\hat{\mathcal{R}}_{0-1}^{\alpha,\beta}$ in the training process with\texttt{Poly} weighiting are plotted in each subfigure. Our loss $\hat{\mathcal{R}}_\Phi^{\alpha,\beta}$ is always an upper bound of the original loss.}\label{fig:val}
  \end{figure}
  \subsection{Dataset Description}
  Note that AUC is aimed at dealing with binary classification problems, hence we construct  long-tail binary datasets as follows. 
  
  \noindent \textbf{Binary CIFAR-10-LT Dataset}. The original CIFAR-10 dataset consists of 60,000 $32\times 32$ color images in 10 classes, with 6,000 images per class. There are 50,000 and 10,000 images in the training set and the test set, respectively. We create a long-tailed CIFAR-10 where the sample sizes across different classes decay exponentially, and the ratio of sample sizes of the least frequent to the most frequent class $\rho$ is set to 0.01. We then create binary long-tailed datasets based on CIFAR-10-LT by selecting one category as positive examples and the others as negative examples. We construct three binary subsets, in which the positive categories are  \textbf{1) }birds, \textbf{2) } automobiles, and \textbf{3) }cats, respectively. The datasets are split into training, validation and test sets with ratios $70\%, 15\%, 15\%$, respectively. More details are provided in Tab. \ref{tab:dataset}.

\noindent \textbf{Binary CIFAR-100-LT Dataset}. The original CIFAR-100 dataset is similar to CIFAR-10, except it has 100 classes with each containing 600 images. The 100 classes in the CIFAR-100 are grouped into 20 superclasses. We create CIFAR-100-LT in the same way as CIFAR-10-LT, and transform it into three binary long-tailed datasets by selecting a superclass as positive class examples each time. Specifically, the positive superclasses are \textbf{1) }fruits and vegetables, \textbf{2) }insects and \textbf{3) }large omnivores and herbivores, respectively. More details are provided in Tab. \ref{tab:dataset}.

\noindent \textbf{Implementation details on CIFAR Datasets}. We utilize the ResNet-20 \cite{resnet} as the backbone, which takes images with size $32\times 32 \times 3$ as input and outputs 64-d features. Then the features are mapped into $[0, 1]$ with an FC layer and Sigmoid function. During the training phase, we apply data augmentation including random horizontal flipping (50\%), random rotation (from $-15^\circ$ to $15^\circ$) and random cropping ($32\times 32$).

 \noindent  \textbf{Binary Tiny-ImageNet-200-LT Dataset}. The Tiny-ImageNet-200 dataset contains 100,000 $256\times 256$ color images  from 200 different categories, with 500 images per category. Similar to the CIFAR-100-LT dataset, we choose three positive superclasses to construct binary subsets: \textbf{1) }dogs, \textbf{2) }birds and \textbf{3) }vehicles. The datasets are further split into training, validation and test sets according to the ratio of $0.7:0.15:0.15$. See Tab. \ref{tab:dataset} for more details.

 \noindent \textbf{Implementation details on Tiny-ImageNet-200 }. The implementation details are basically the same with CIFAR-10-LT and CIFAR-100-LT datasets, except the backbone network is implemented with ResNet-18 \cite{resnet}, which takes images with size $224\times 224 \times 3$  as input and outputs 512-d features.

 \begin{figure}[th]  
  \centering
    
    \subfigure[Effect of $\gamma$ on subset1, $\alpha=0.3, \beta=0.3$]{
      \includegraphics[width=0.145\textwidth]{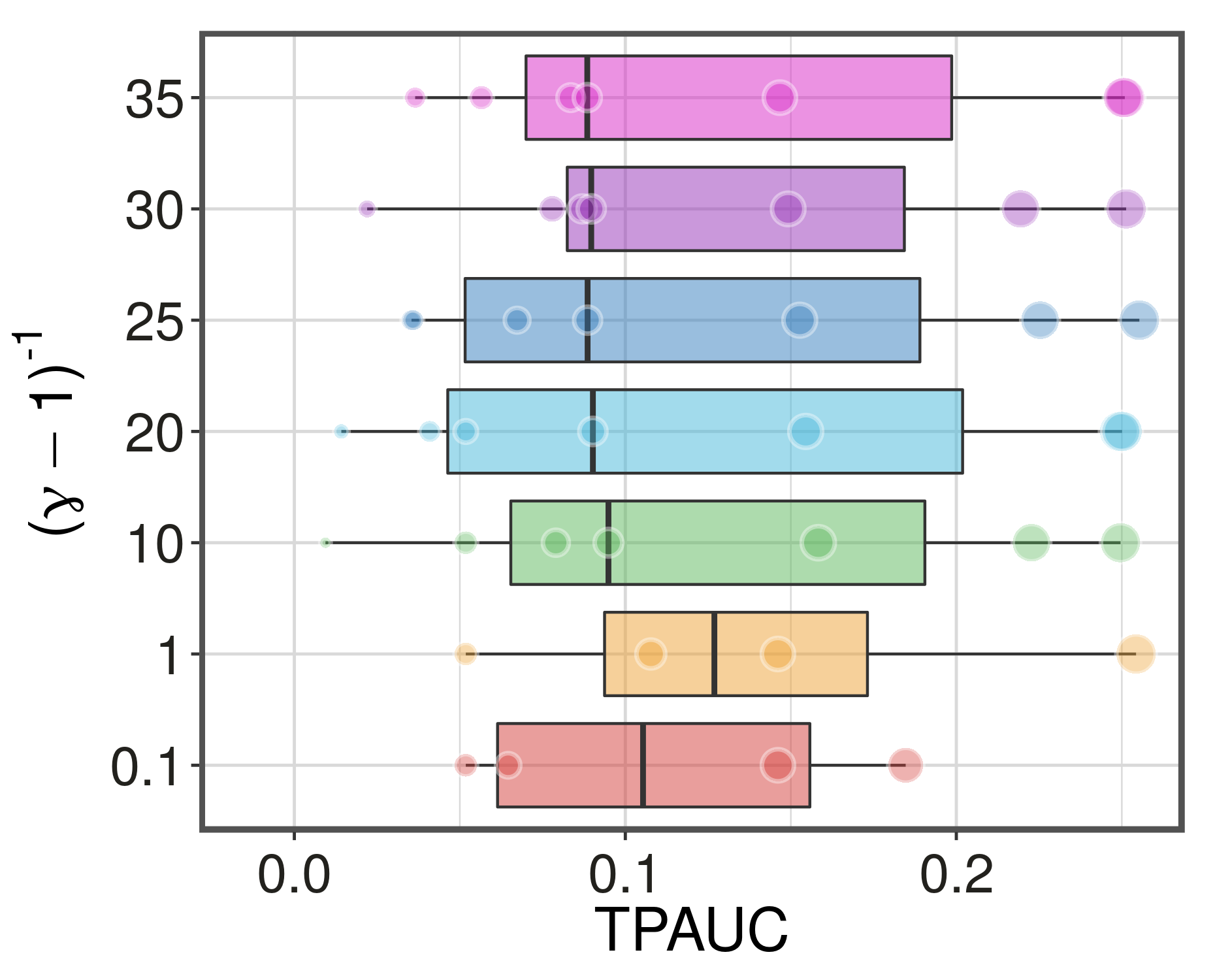} 
    }~~
    \subfigure[Effect of $\gamma$ on subset1, $\alpha=0.4, \beta=0.4$]{
      \includegraphics[width=0.145\textwidth]{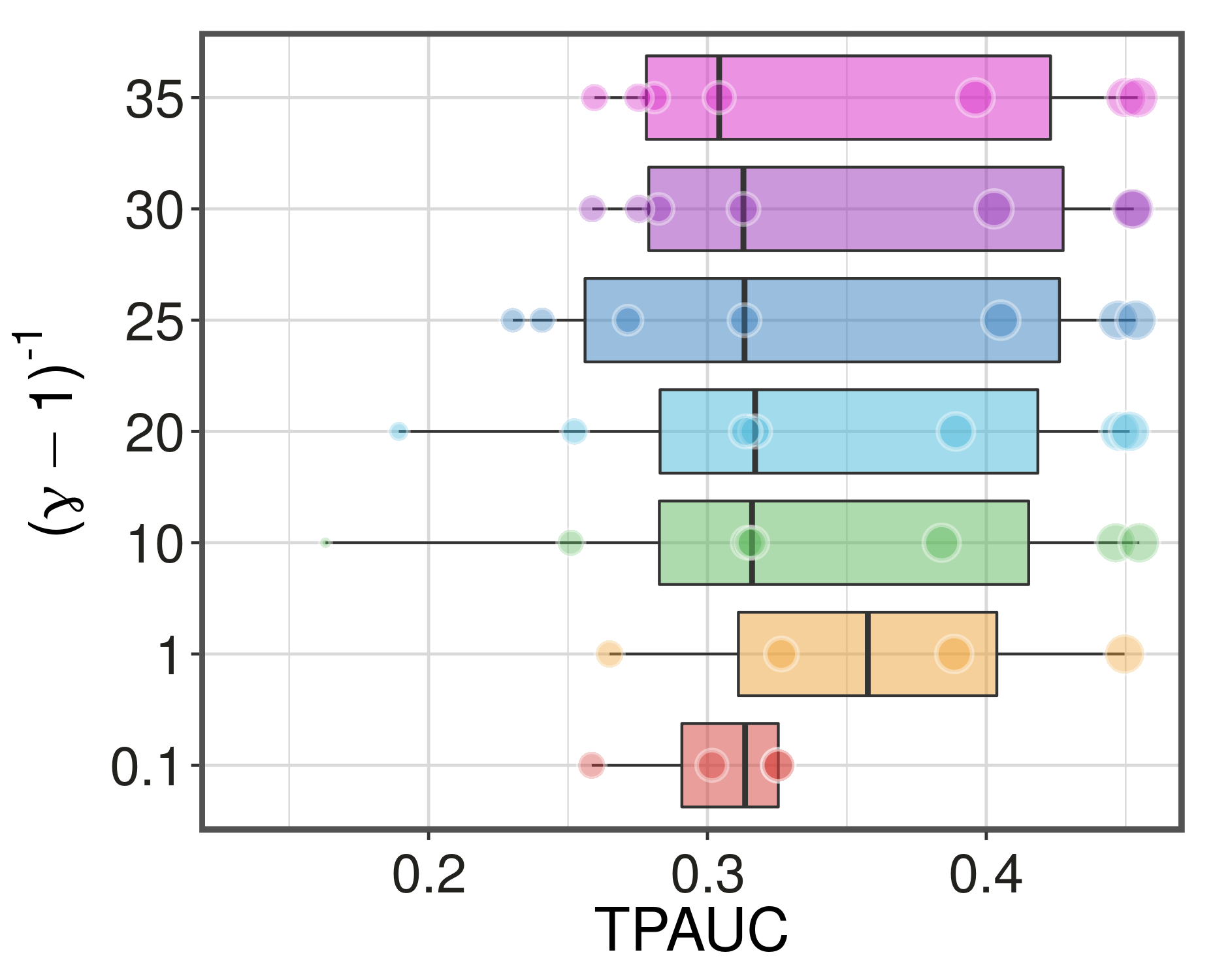} 
    }~~
    \subfigure[Effect of $\gamma$ on subset1, $\alpha=0.5, \beta=0.5$]{
      \includegraphics[width=0.145\textwidth]{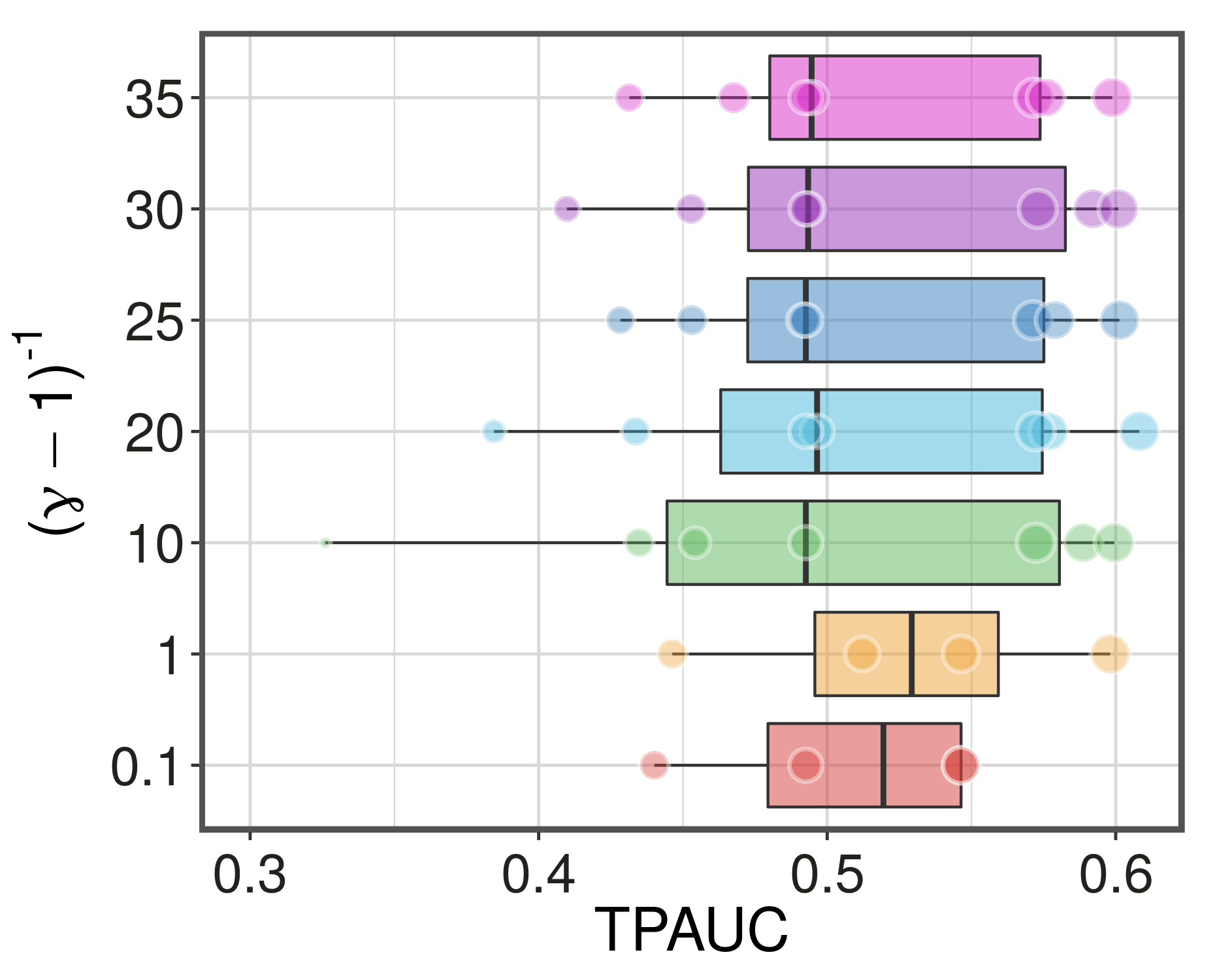} 
    }

    \subfigure[Effect of $\gamma$ on subset2, $\alpha=0.3, \beta=0.3$]{
      \includegraphics[width=0.145\textwidth]{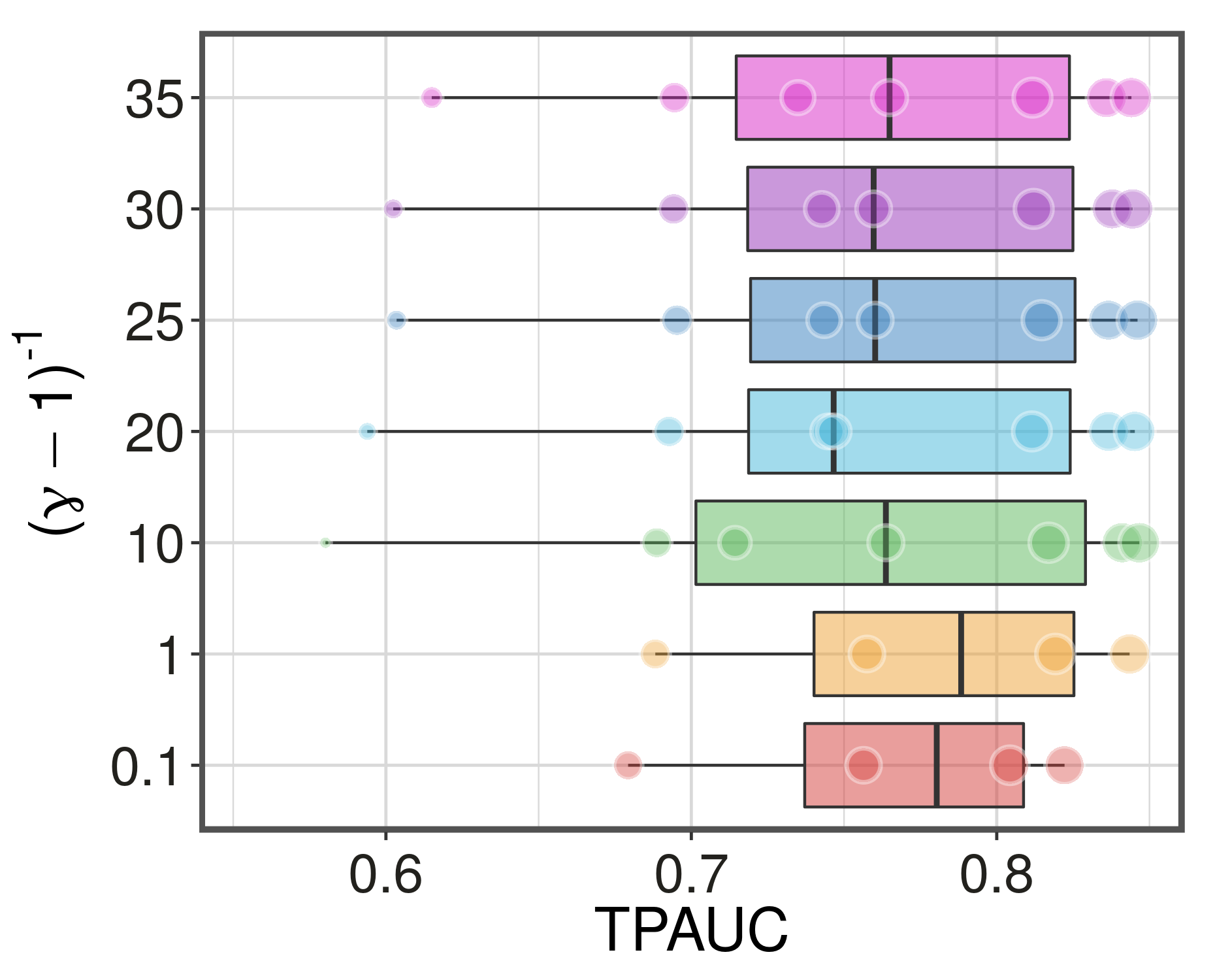} 
    }~~
    \subfigure[Effect of $\gamma$ on subset2, $\alpha=0.4, \beta=0.4$]{
      \includegraphics[width=0.145\textwidth]{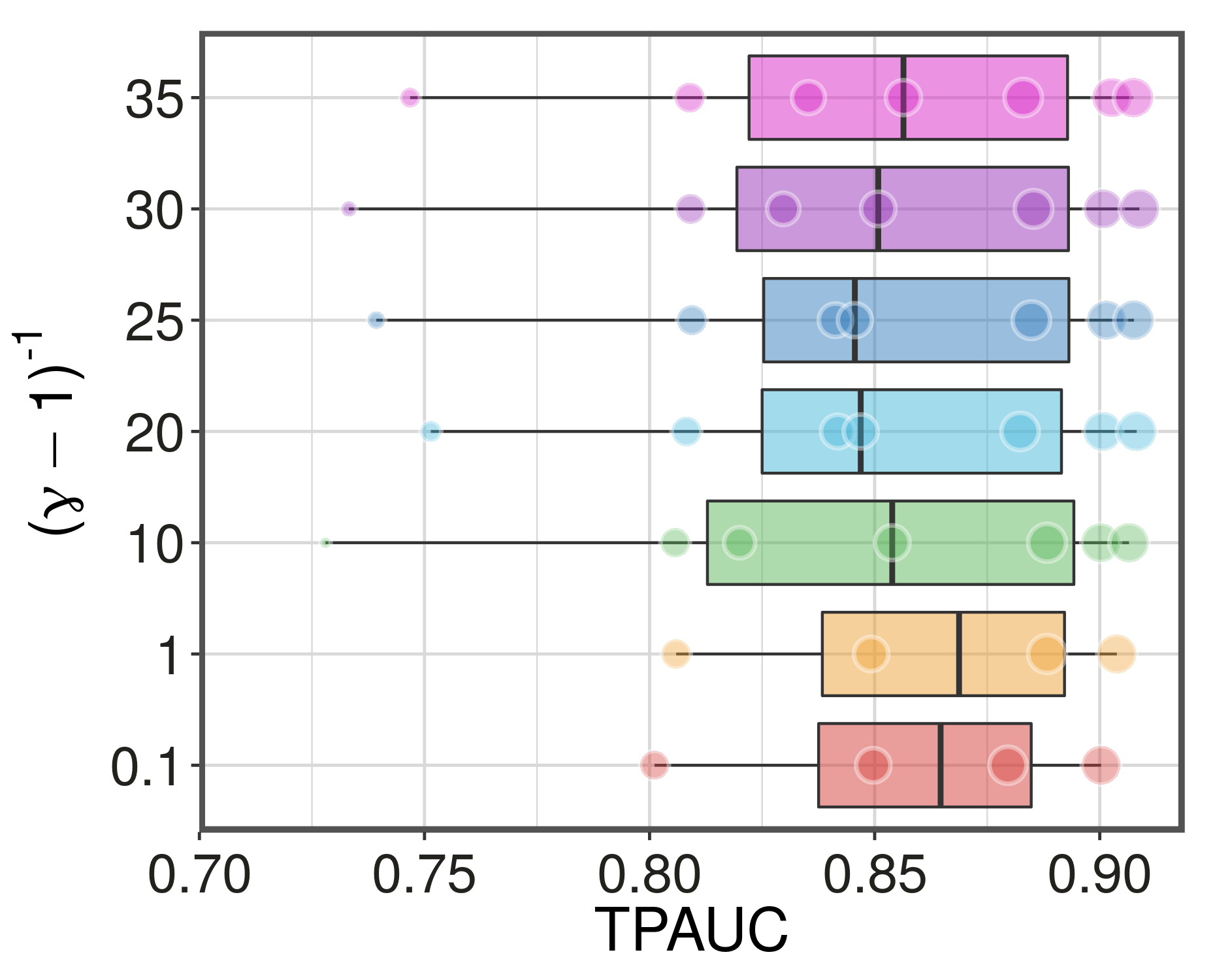} 
    }~~
    \subfigure[Effect of $\gamma$ on subset2, $\alpha=0.5, \beta=0.5$]{
      \includegraphics[width=0.145\textwidth]{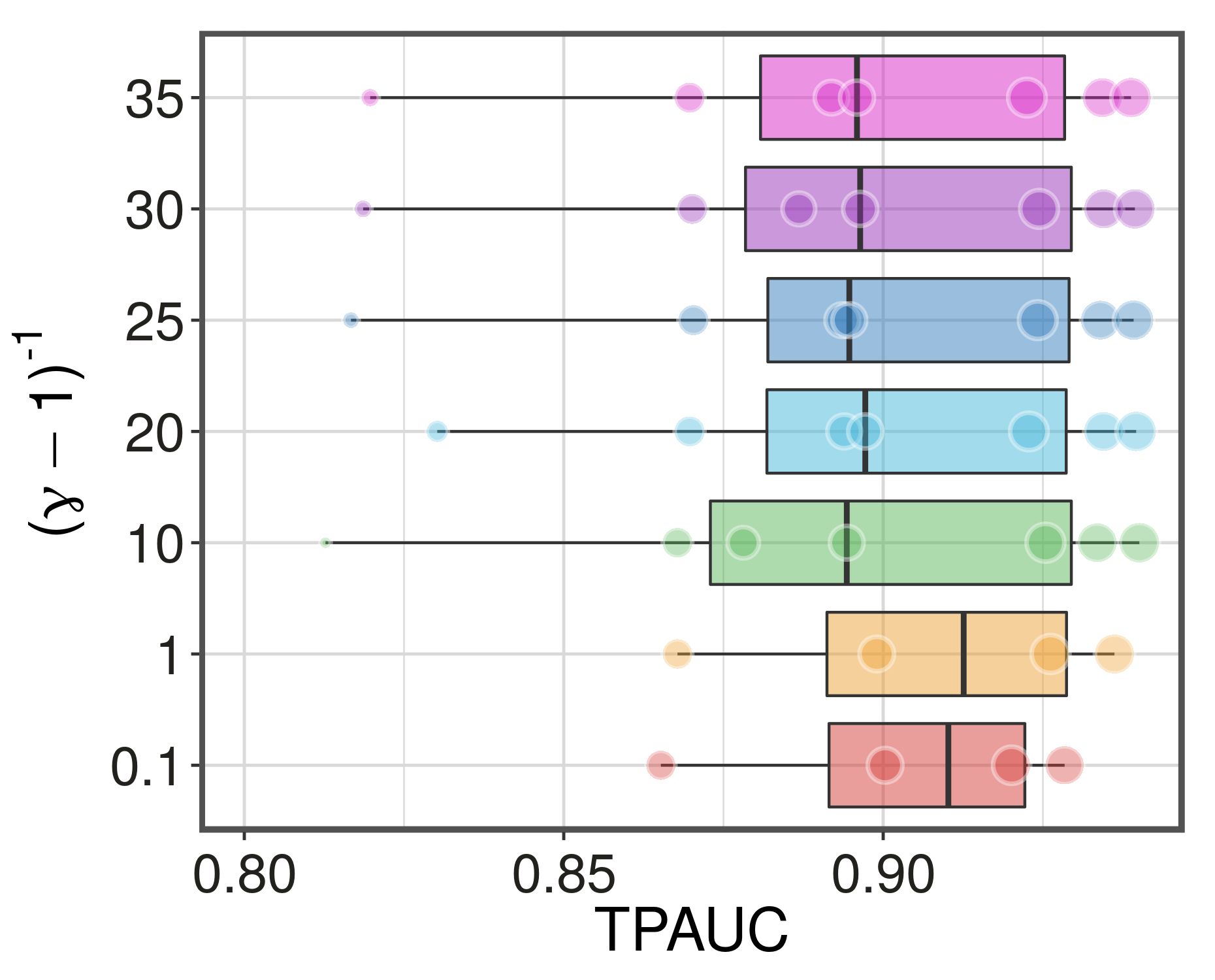} 
    }

    \subfigure[Effect of $\gamma$ on subset3, $\alpha=0.3, \beta=0.3$]{
      \includegraphics[width=0.145\textwidth]{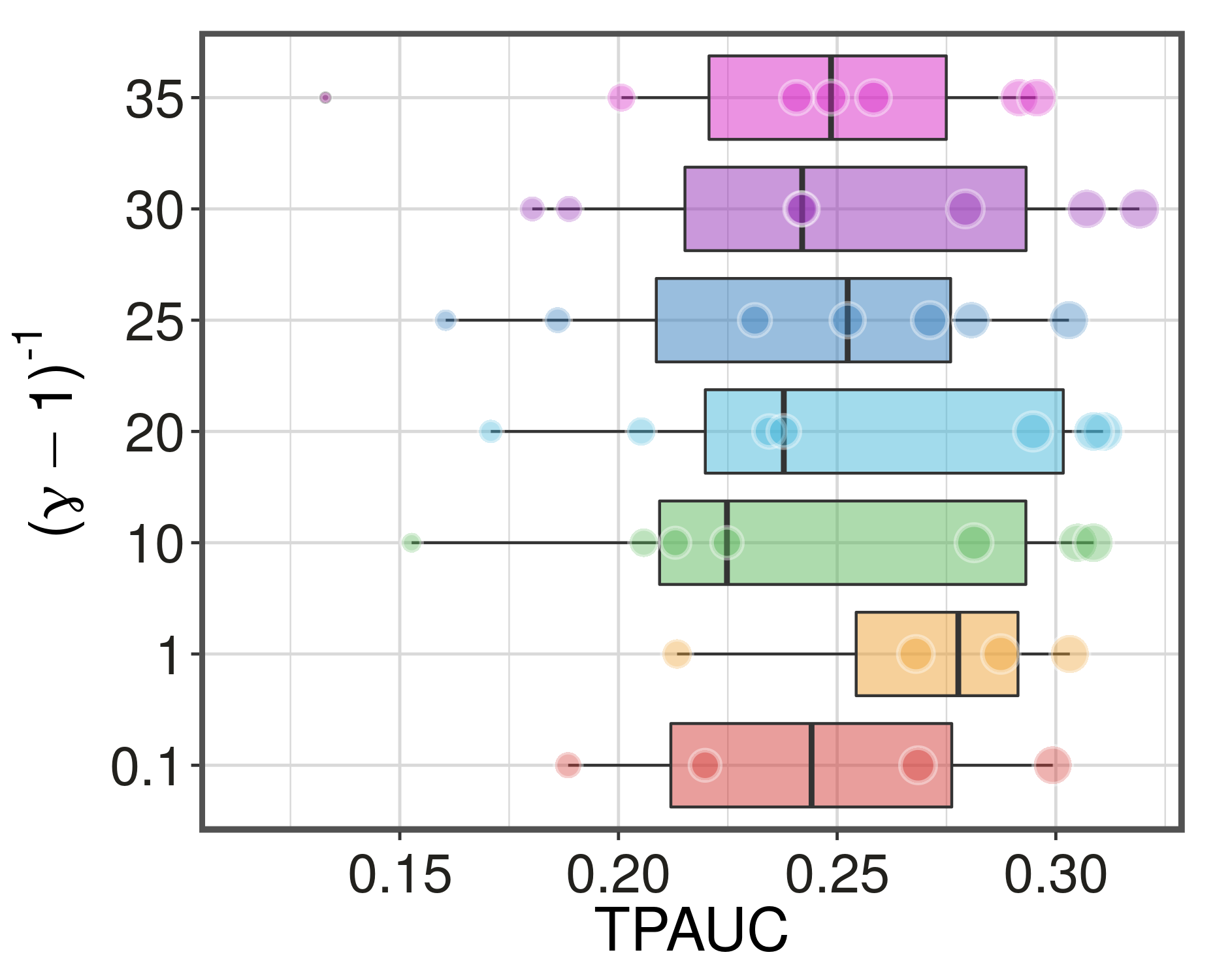} 
    }~~
    \subfigure[Effect of $\gamma$ on subset3, $\alpha=0.4, \beta=0.4$]{
      \includegraphics[width=0.145\textwidth]{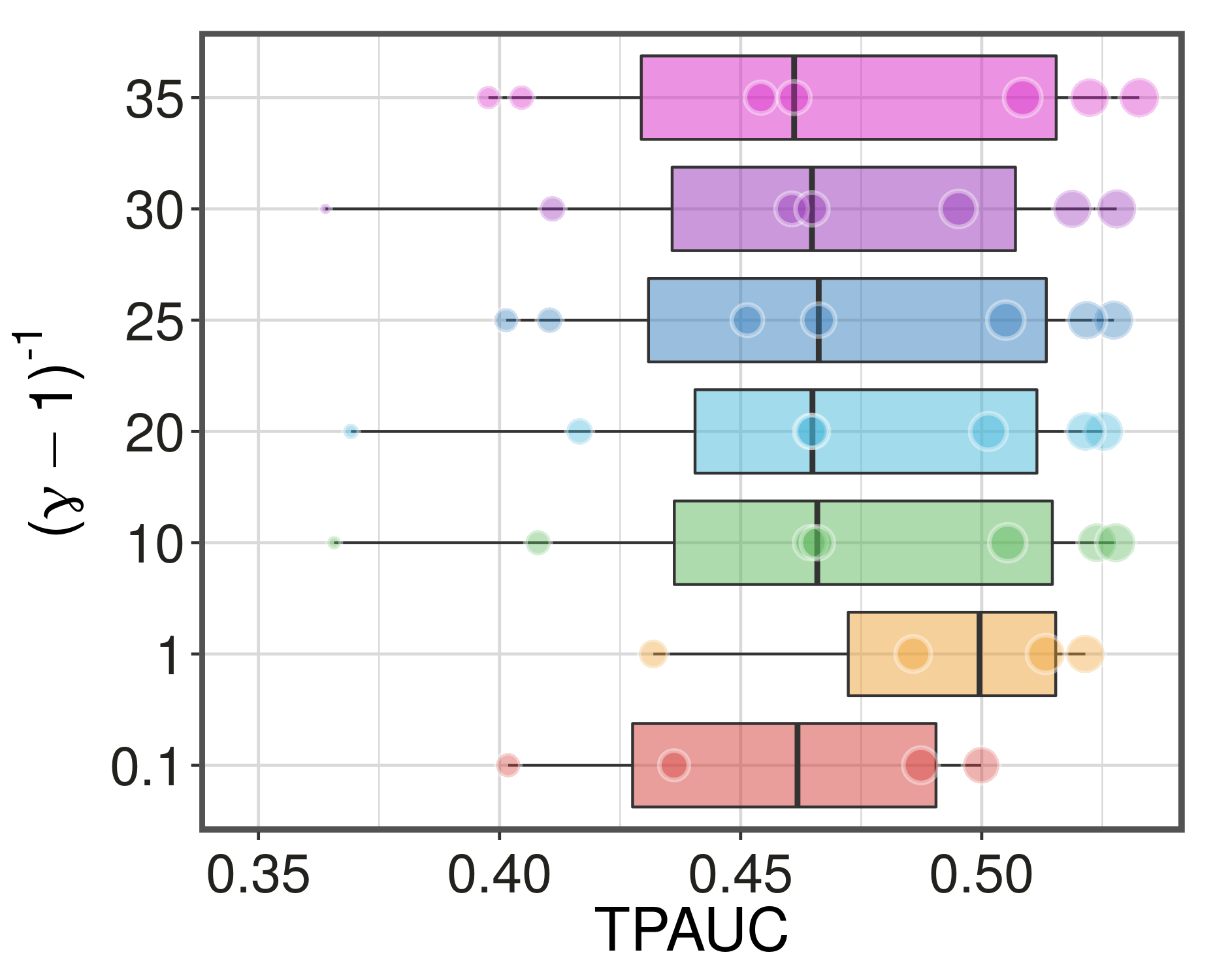} 
    }~~
    \subfigure[Effect of $\gamma$ on subset3, $\alpha=0.5, \beta=0.5$]{
      \includegraphics[width=0.145\textwidth]{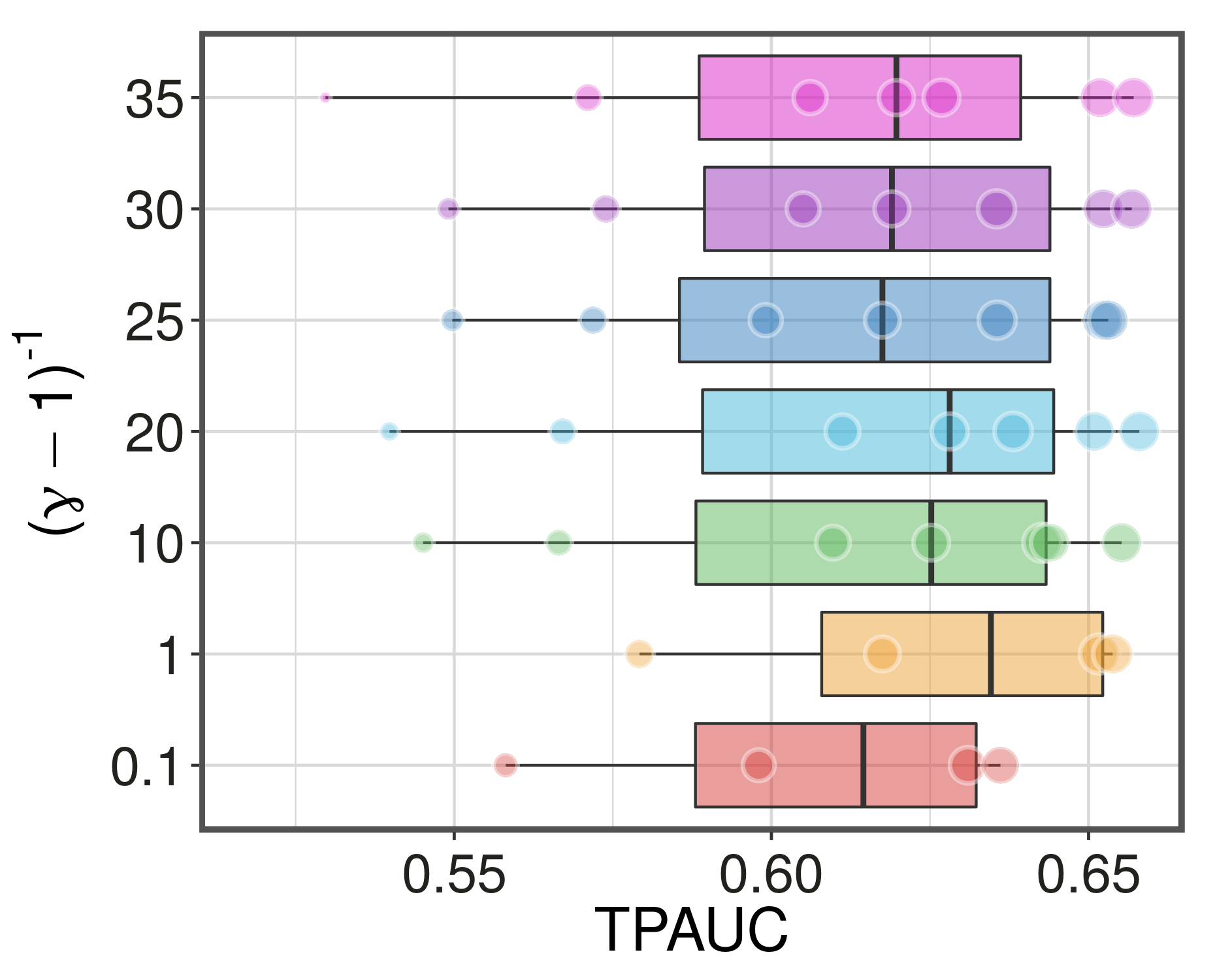} 
    }
    \caption{\label{fig:gammaexp}Sensitivity analysis of \texttt{Exp} on $\gamma$. Experiments are conducted on CIFAR-10-LT. For each box, $\gamma$ is fixed as the y-axis value, and the scattered points along the box show the variation of $\gamma$.}
\end{figure}

  \subsection{Competitors}
  
  To validate the effectiveness of our proposed methods, we consider two types of competitors in our experiments. On one hand, we compare our proposed methods with other methods dealing with imbalanced data:
  \begin{enumerate}[ itemindent=0pt, leftmargin =13pt]
    \item \textbf{CE}: Here use a class-wise reweighted version of the CE loss as one of our competitors, the sample weight is set to $1/n_y$, where $n_y$ the frequency of the class the sample belongs to.
    \item \textbf{Focal}: \cite{FOCAL} It tackles the imbalance problem by adding a modulating factor to the cross-entropy loss to highlight the hard and minority samples during the training process. 
  
    \item \textbf{CB-CE}:  It refers to the loss function that applies the reweighting scheme proposed in \cite{CB-CE-FOCAL} on the cross-entropy loss.
    
    \item \textbf{CB-Focal}:  It refers to the loss function that applies the reweighting scheme proposed in \cite{CB-CE-FOCAL} on the Focal loss.
  \end{enumerate}
  On the other hand, we also include standard AUC optimization methods as our baseline.
  \begin{enumerate}[ itemindent=0pt, leftmargin =13pt]
    \item \textbf{SqAUC}:  Perform a standard AUC optimization with the surrogate loss function $\ell_{sq}(t) = (1-t)^2$.
  \end{enumerate}
  To show the effect of OPAUC optimizing on TPAUC, we implement the following algorithms:
  \begin{enumerate}[ itemindent=0pt, leftmargin =13pt]
    \item \textbf{TruncOPAUC} (inspired by \cite{samplepauc1, samplepauc2}): We implement a Truncated-base method, which performs OPAUC optimization with the following objective function:
  	\begin{equation*}
  		\sum_{i=1}^{n_+}\sum_{j=1}^{n^\beta_-} \frac{\ell_{sq}\left(\f(\bm{x}_i^+) - f(\xnjo)  \right)}{n_+n^\beta_-}
  	\end{equation*}
  where $\bm{x}_i^+$ denotes the $i$-th positive instance. In each epoch, we first rank the negative instances among $\XN$ in ascending order based on their prediction scores. Then, the hard negative samples are generated from the top $1$-th to $\nnb$-th positions and all positive samples are involved in the optimization process.
    \item  \textbf{OPAUC-Poly}:  Perform OPAUC optimization with the objective function:
    \begin{equation*}
      \begin{split}
        \frac{1}{\np\nn}\sumpn & \psi^{\mathsf{poly}}_\gamma(\f(\vnj)) \cdot\ell(\f, \xpi, \xnj)
      \end{split}
    \end{equation*} 
  \item \textbf{OPAUC-Exp}:  Perform TPAUC optimization with the objective function:
  \begin{equation*}
    \begin{split}
      \frac{1}{\np\nn}\sumpn & \psi^{\mathsf{Exp}}_\gamma(\f(\vnj))\cdot \ell(\f, \xpi, \xnj)
    \end{split}
  \end{equation*} 
\end{enumerate}
Moreover, we include the truncation-based TPAUC method as a baseline of TPAUC optimization:
\begin{enumerate}[ itemindent=0pt, leftmargin =13pt]
 \item \textbf{TruncTPAUC}: We implement a  {Trunc}ation-based \textbf{TPAUC} method, which performs TPAUC optimization with the following objective function:
  \begin{equation*}
    \psumpn \frac{\ell_{sq}\left(\f(\xpio) - f(\xnjo)  \right)}{n_+^\alpha n_-^\beta}
  \end{equation*}
  and the other settings follow from TruncOPAUC.
\end{enumerate}
  Finally, we implement our proposed methods on top of SqAUC:
  \begin{enumerate}[ itemindent=0pt, leftmargin =13pt]
  \item \textbf{TPAUC-Poly}:  Perform TPAUC optimization with the objective function:
  \begin{equation*}
    \begin{split}
      \frac{1}{\np\nn}\sumpn &\psi^{\mathsf{poly}}_\gamma(1-\f(\xpi))  \cdot \psi^{\mathsf{poly}}_\gamma(\f(\vnj)) \\
       &\cdot\ell(\f, \xpi, \xnj)
    \end{split}
  \end{equation*} 
  \item \textbf{TPAUC-Exp}:  Perform TPAUC optimization with the objective function:
  \begin{equation*}
    \begin{split}
      \frac{1}{\np\nn}\sumpn &\psi^{\mathsf{Exp}}_\gamma(1-\f(\xpi))  \cdot \psi^{\mathsf{Exp}}_\gamma(\f(\vnj))\\
       &\cdot \ell(\f, \xpi, \xnj)
    \end{split}
  \end{equation*} 
  \item \textbf{TPAUC-Minmax-Poly}:  Perform the \textbf{TPAUC-Poly} algorithm with the minimax  reformulation in Thm.\ref{thm:reform}. 
  \item \textbf{TPAUC-Minmax-Exp}:  Perform the \textbf{TPAUC-Exp} algorithm with the minimax  reformulation in Thm.\ref{thm:reform}. 
  
  \end{enumerate}

  \begin{table*}[htbp]
    \centering
    \small
    \caption{Performance Comparisons on Tiny-ImageNet-200-LT with different metrics, where $(x,y)$ stands for $\mathsf{TPAUC}(x,y)$ in short and the first and second best results are highlighted with \textbf{bold text} and \underline{underline}, respectively.}
    \resizebox{0.95\textwidth}{!}{%
      \begin{tabular}{c|c|c|ccc|ccc|ccc}
        \toprule
        \multirow{2}[4]{*}{dataset} & \multicolumn{1}{c|}{\multirow{2}[4]{*}{type}} & \multicolumn{1}{c|}{\multirow{2}[4]{*}{methods}} & \multicolumn{3}{c|}{Subset1} & \multicolumn{3}{c|}{Subset2} & \multicolumn{3}{c}{Subset3} \\
        \cmidrule{4-12}          & \multicolumn{1}{c|}{} &       & (0.3,0.3) & (0.4,0.4) & \multicolumn{1}{c|}{(0.5,0.5)} & (0.3,0.3) & (0.4,0.4) & \multicolumn{1}{c|}{(0.5,0.5)} & (0.3,0.3) & (0.4,0.4) & (0.5,0.5) \\
        \midrule
        \multicolumn{1}{c|}{\multirow{11}[3]{*}{Tiny-200-LT}} & \multicolumn{1}{c|}{\multirow{9}[1]{*}{Competitors}} & CE-RW & \cellcolor[rgb]{ .906,  .941,  .973} 80.90 & \cellcolor[rgb]{ .965,  .976,  .988} 87.76 & \cellcolor[rgb]{ .988,  .992,  .996} 91.54 & \cellcolor[rgb]{ .906,  .941,  .973} 93.30 & \cellcolor[rgb]{ .933,  .957,  .98} 96.15 & \cellcolor[rgb]{ .906,  .941,  .973} 97.53 & \cellcolor[rgb]{ .965,  .976,  .988} 90.37 & 94.34 & \cellcolor[rgb]{ .863,  .918,  .965} 96.75 \\
        & \multicolumn{1}{c|}{} & Focal & \cellcolor[rgb]{ .863,  .918,  .965} 81.18 & \cellcolor[rgb]{ .906,  .941,  .973} 88.06 & \cellcolor[rgb]{ .933,  .957,  .98} 91.72 & \cellcolor[rgb]{ .933,  .957,  .98} 93.23 & \cellcolor[rgb]{ .965,  .976,  .988} 96.08 & \cellcolor[rgb]{ .906,  .941,  .973} 97.59 & \cellcolor[rgb]{ .933,  .957,  .98} 91.35 & \cellcolor[rgb]{ .933,  .957,  .98} 94.87 & \cellcolor[rgb]{ .933,  .957,  .98} 96.63 \\
        & \multicolumn{1}{c|}{} & CBCE & \cellcolor[rgb]{ .933,  .957,  .98} 80.64 & \cellcolor[rgb]{ .988,  .992,  .996} 87.58 & 91.17 & \cellcolor[rgb]{ .808,  .882,  .949} 93.77 & \cellcolor[rgb]{ .808,  .882,  .949} \underline{96.52} & \cellcolor[rgb]{ .808,  .882,  .949} \underline{97.77} & \cellcolor[rgb]{ .906,  .941,  .973} 91.66 & \cellcolor[rgb]{ .863,  .918,  .965} 95.19 & \cellcolor[rgb]{ .808,  .882,  .949} 96.79 \\
        & \multicolumn{1}{c|}{} & CBFocal & \cellcolor[rgb]{ .965,  .976,  .988} 80.44 & \cellcolor[rgb]{ .933,  .957,  .98} 87.95 & \cellcolor[rgb]{ .863,  .918,  .965} 91.91 & \cellcolor[rgb]{ .863,  .918,  .965} 93.46 & \cellcolor[rgb]{ .808,  .882,  .949} 96.43 & \cellcolor[rgb]{ .863,  .918,  .965} 97.64 & 91.06 & \cellcolor[rgb]{ .965,  .976,  .988} 94.82 & \cellcolor[rgb]{ .965,  .976,  .988} 96.62 \\
        & \multicolumn{1}{c|}{} & SqAUC & 80.16 & \cellcolor[rgb]{ .933,  .957,  .98} 87.99 & \cellcolor[rgb]{ .965,  .976,  .988} 91.67 & \cellcolor[rgb]{ .965,  .976,  .988} 93.10 & \cellcolor[rgb]{ .965,  .976,  .988} 96.07 & 97.32 & \cellcolor[rgb]{ .808,  .882,  .949} 92.15 & \cellcolor[rgb]{ .863,  .918,  .965} 95.16 & \cellcolor[rgb]{ .863,  .918,  .965} 96.75 \\
        & \multicolumn{1}{c|}{} & {TruncOPAUC} & \cellcolor[rgb]{ .965,  .976,  .988} 80.45 & \cellcolor[rgb]{ .863,  .918,  .965} 88.23 & \cellcolor[rgb]{ .933,  .957,  .98} 91.71 & \cellcolor[rgb]{ .863,  .918,  .965} 93.44 & \cellcolor[rgb]{ .863,  .918,  .965} 96.33 & \cellcolor[rgb]{ .863,  .918,  .965} 97.63 & \cellcolor[rgb]{ .863,  .918,  .965} 91.70 & \cellcolor[rgb]{ .906,  .941,  .973} 95.04 & \cellcolor[rgb]{ .906,  .941,  .973} 96.71 \\
        & \multicolumn{1}{c|}{} & {Op-Poly} & \cellcolor[rgb]{ .863,  .918,  .965} 81.15 & \cellcolor[rgb]{ .808,  .882,  .949} 88.41 & \cellcolor[rgb]{ .933,  .957,  .98} 91.73 & \cellcolor[rgb]{ .863,  .918,  .965} 93.53 & \cellcolor[rgb]{ .863,  .918,  .965} 96.30 & \cellcolor[rgb]{ .808,  .882,  .949} 97.74 & \cellcolor[rgb]{ .808,  .882,  .949} 92.22 & \cellcolor[rgb]{ .808,  .882,  .949} \underline{95.29} & \cellcolor[rgb]{ .808,  .882,  .949} 96.82 \\
        & \multicolumn{1}{c|}{} & {OP-Exp} & \cellcolor[rgb]{ .863,  .918,  .965} 81.02 & \cellcolor[rgb]{ .933,  .957,  .98} 87.99 & \cellcolor[rgb]{ .906,  .941,  .973} 91.83 & \cellcolor[rgb]{ .965,  .976,  .988} 93.10 & \cellcolor[rgb]{ .863,  .918,  .965} 96.36 & \cellcolor[rgb]{ .863,  .918,  .965} 97.67 & \cellcolor[rgb]{ .808,  .882,  .949} 92.15 & \cellcolor[rgb]{ .863,  .918,  .965} 95.16 & \cellcolor[rgb]{ .863,  .918,  .965} 96.75 \\
        & \multicolumn{1}{c|}{} & {TruncTPAUC} & \cellcolor[rgb]{ .933,  .957,  .98} 80.73 & 87.41 & \cellcolor[rgb]{ .965,  .976,  .988} 91.67 & \cellcolor[rgb]{ .988,  .992,  .996} 93.09 & \cellcolor[rgb]{ .965,  .976,  .988} 96.03 & \cellcolor[rgb]{ .906,  .941,  .973} 97.58 & \cellcolor[rgb]{ .906,  .941,  .973} 91.55 & \cellcolor[rgb]{ .863,  .918,  .965} 95.12 & \cellcolor[rgb]{ .808,  .882,  .949} 96.81 \\
        \cmidrule{2-12}         & \multicolumn{1}{c|}{\multirow{2}[2]{*}{Ours-TPAUC}} & Poly & \cellcolor[rgb]{ .965,  .976,  .988} 80.44 & \cellcolor[rgb]{ .863,  .918,  .965} 88.21 & \cellcolor[rgb]{ .863,  .918,  .965} 91.98 & 93.00 & 95.61 & \cellcolor[rgb]{ .933,  .957,  .98} 97.47 & \cellcolor[rgb]{ .808,  .882,  .949} 92.02 & \cellcolor[rgb]{ .808,  .882,  .949} 95.25 & \cellcolor[rgb]{ .808,  .882,  .949} \underline{96.84} \\
        & \multicolumn{1}{c|}{} & Exp  & \cellcolor[rgb]{ .741,  .843,  .933} \textbf{82.61} & \cellcolor[rgb]{ .741,  .843,  .933} \textbf{89.13} & \cellcolor[rgb]{ .808,  .882,  .949} \underline{92.62} & \cellcolor[rgb]{ .808,  .882,  .949} \underline{93.82} & \cellcolor[rgb]{ .933,  .957,  .98} 96.12 & \cellcolor[rgb]{ .965,  .976,  .988} 97.38 & \cellcolor[rgb]{ .933,  .957,  .98} 91.25 & \cellcolor[rgb]{ .933,  .957,  .98} 94.78 & \cellcolor[rgb]{ .988,  .992,  .996} 96.57 \\
        \cmidrule{2-12}         & \multirow{2}[2]{*}{Ours-minmax} & {Exp} & \cellcolor[rgb]{ .808,  .882,  .949} \underline{82.35} & \cellcolor[rgb]{ .808,  .882,  .949} 88.70 & \cellcolor[rgb]{ .808,  .882,  .949} 92.51 & \cellcolor[rgb]{ .808,  .882,  .949} 93.77 & \cellcolor[rgb]{ .988,  .992,  .996} 95.85 & \cellcolor[rgb]{ .988,  .992,  .996} 97.34 & \cellcolor[rgb]{ .741,  .843,  .933} \textbf{92.57} & \cellcolor[rgb]{ .988,  .992,  .996} 94.43 & 96.25 \\
        &      & {Poly} & \cellcolor[rgb]{ .808,  .882,  .949} 82.24 & \cellcolor[rgb]{ .808,  .882,  .949} \underline{88.78} & \cellcolor[rgb]{ .741,  .843,  .933} \textbf{92.79} & \cellcolor[rgb]{ .741,  .843,  .933} \textbf{94.55} & \cellcolor[rgb]{ .741,  .843,  .933} \textbf{96.76} & \cellcolor[rgb]{ .741,  .843,  .933} \textbf{97.92} & \cellcolor[rgb]{ .808,  .882,  .949} \underline{92.34} & \cellcolor[rgb]{ .741,  .843,  .933} \textbf{95.63} & \cellcolor[rgb]{ .741,  .843,  .933} \textbf{97.05} \\
        \bottomrule
      \end{tabular}%
    }
    \label{tab:perf-tiny}%
  \end{table*}%
  \subsection{General Implementation Details}
  \label{exp_detail}
  All the experiments are carried out on a ubuntu 16.04.1 server equipped with Intel(R) Xeon(R) Gold 5218 CPU@2.30GHz and an RTX 3090 GPU, and all codes are implemented with \texttt{PyTorch} (v-1.8.1) \cite{DBLP:conf/nips/PaszkeGMLBCKLGA19} under \texttt{python 3.7} environment. \texttt{Stochastic Gradient Descent} (SGD) \cite{DBLP:conf/icml/SutskeverMDH13} with Nesterov momentum is adopted to optimize the objective function. Empirically, for all datasets, the learning rate is $10^{-3}$; the $l_2$ regularization term is set as $10^{-5}$, and the Nesterov momentum is $0.9$. We also employ an exponential learning rate decay scheduler to adjust the learning rate after each training epoch, where the learning rate decay rate is set as $0.99$ for all methods. The training batch size is 128, and we restrict the ratio of positive and negative samples by $1:10$ in each batch. The batch size of validation/test examples is 256. Specifically, $E_k$ is searched in $\{0, 3, 5, 8, 10, 15, 20, 30, 35\}$. For\texttt{Poly}, $\gamma$ is searched in $\{0.01, 0.03, 0.05, 0.08, 0.1, 1, 5\}$. For\texttt{Exp}, $\gamma$ is searched in $\{0.1, 1, 10, 20, 25, 30, 35\}$. Finally, we select the model based on the best validation performance and report the test set results.

  \begin{figure*}[h]  
    \centering
      
      \subfigure[Effect of $\gamma$ on subset1, $\alpha=0.3, \beta=0.3$]{
        \includegraphics[width=0.3\textwidth]{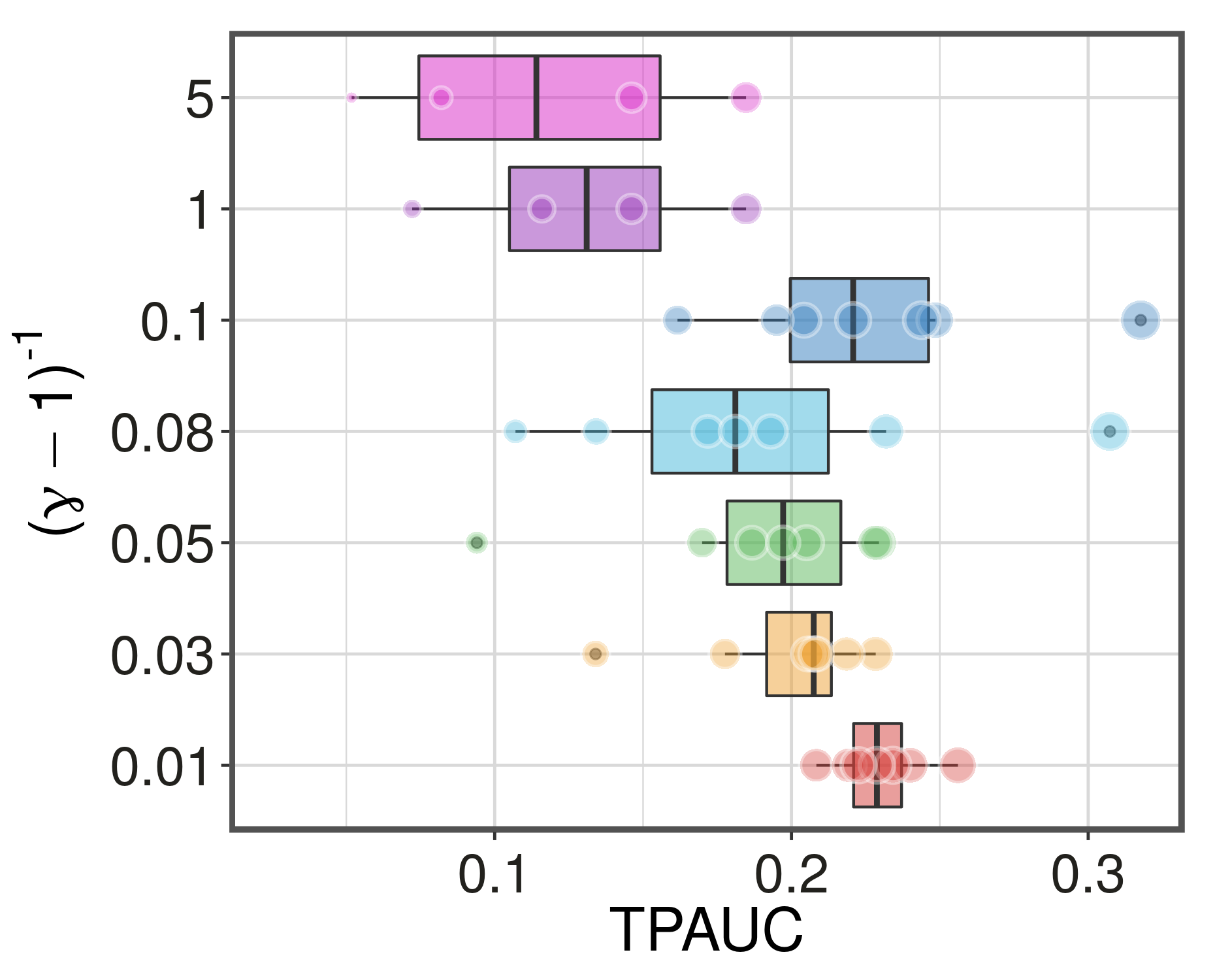} 
      }~~
      \subfigure[Effect of $\gamma$ on subset1, $\alpha=0.4, \beta=0.4$]{
        \includegraphics[width=0.3\textwidth]{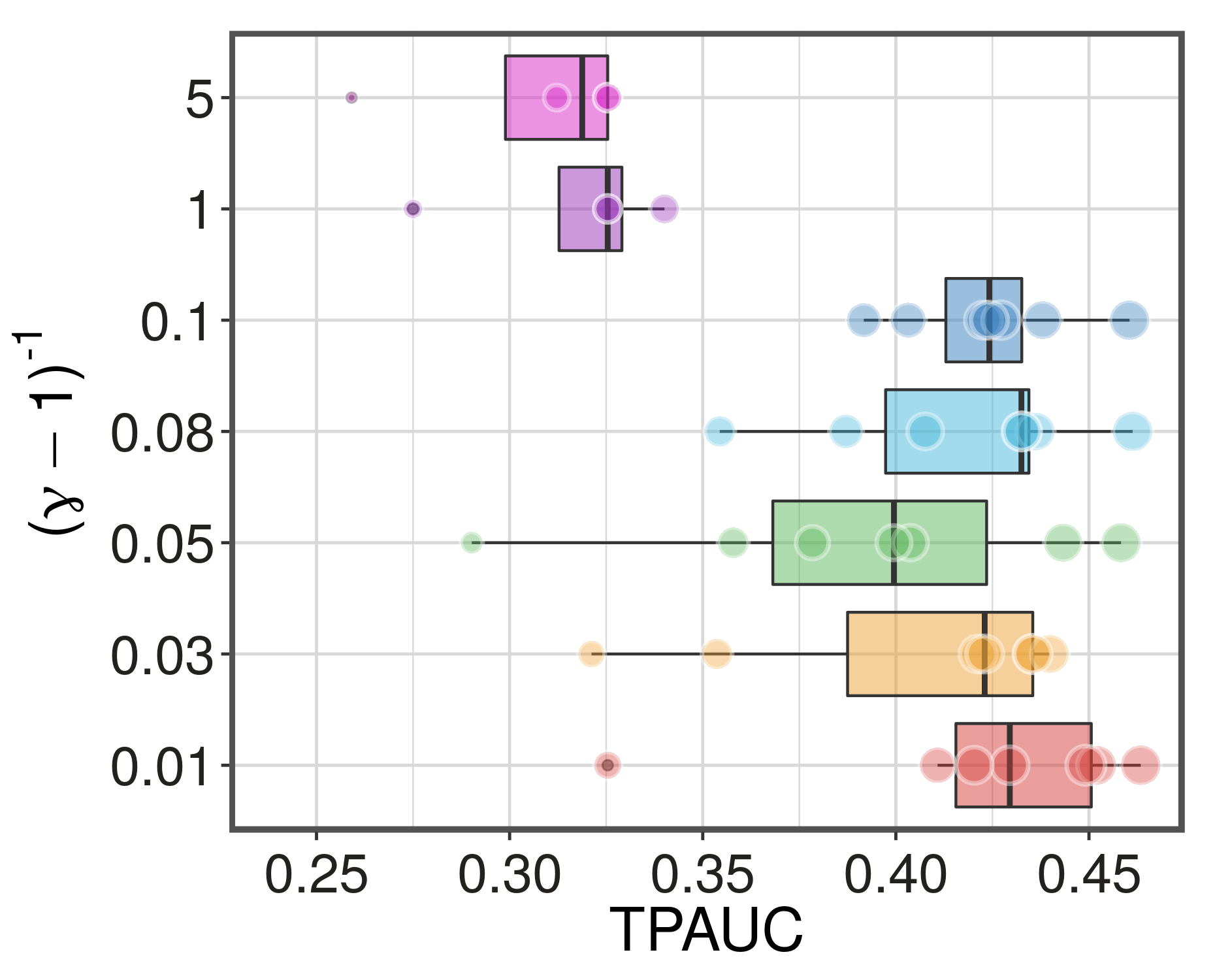} 
      }~~
      \subfigure[Effect of $\gamma$ on subset1, $\alpha=0.5, \beta=0.5$]{
        \includegraphics[width=0.3\textwidth]{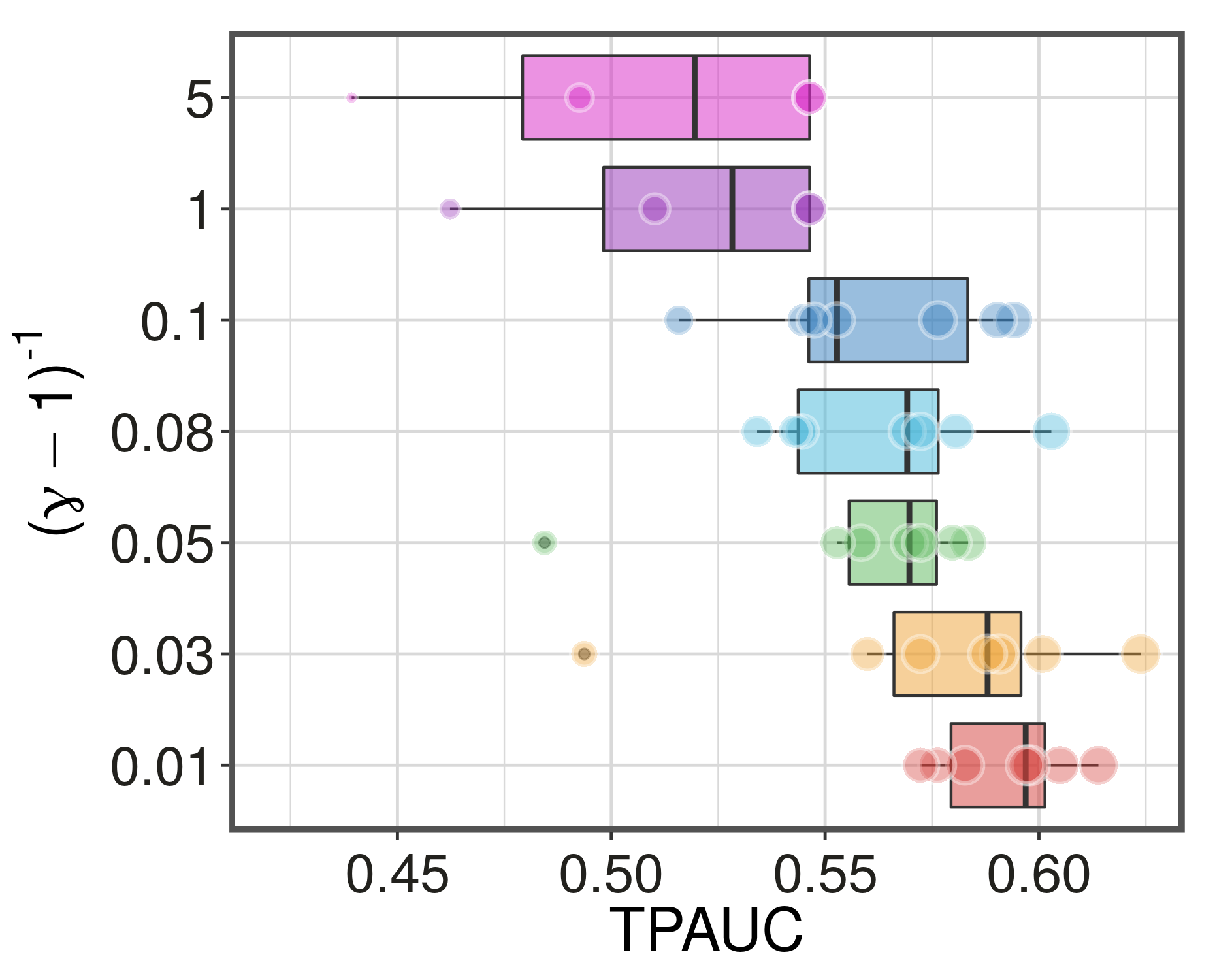} 
      }
  
      \subfigure[Effect of $\gamma$ on subset2, $\alpha=0.3, \beta=0.3$]{
        \includegraphics[width=0.3\textwidth]{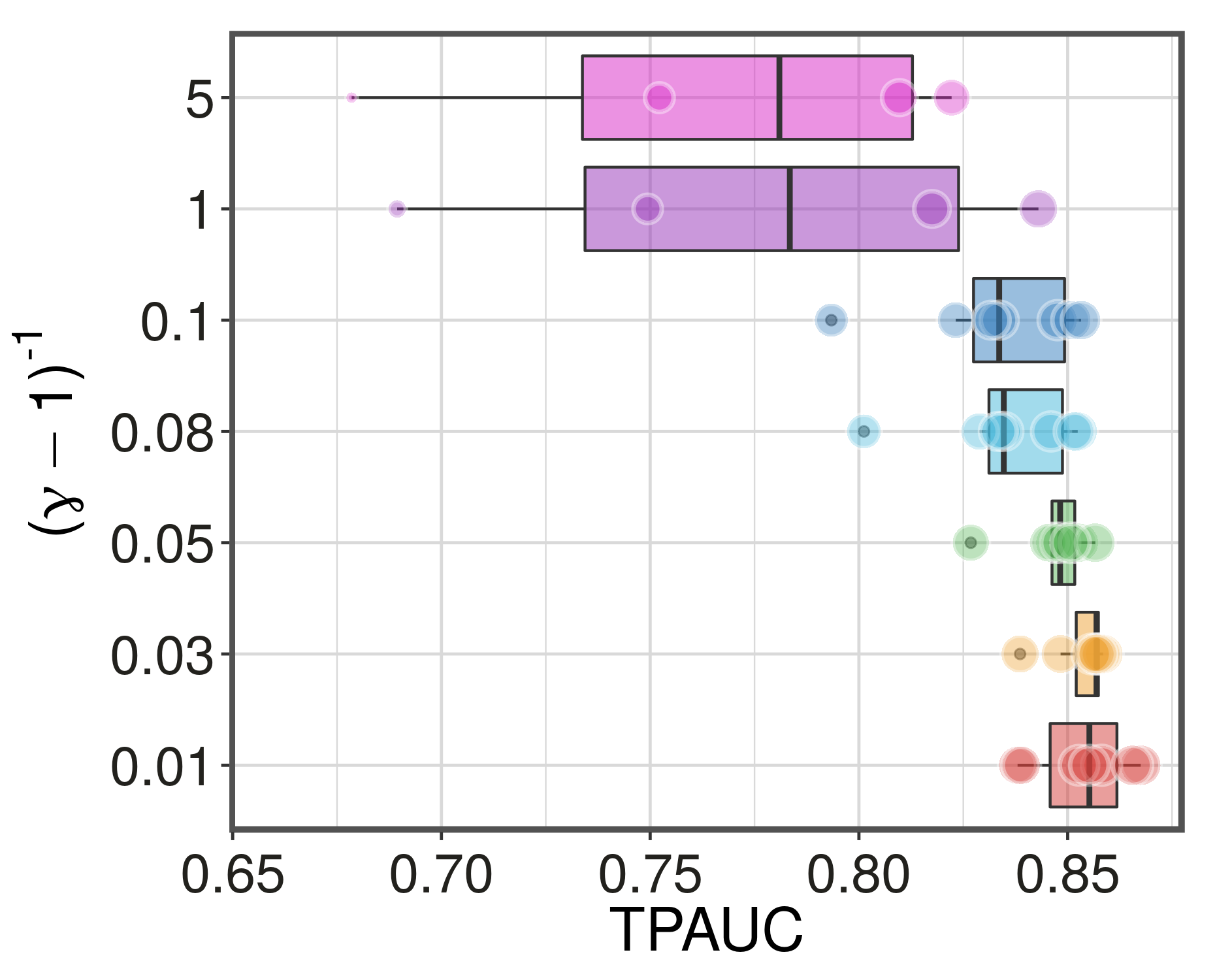} 
      }~~
      \subfigure[Effect of $\gamma$ on subset2, $\alpha=0.4, \beta=0.4$]{
        \includegraphics[width=0.3\textwidth]{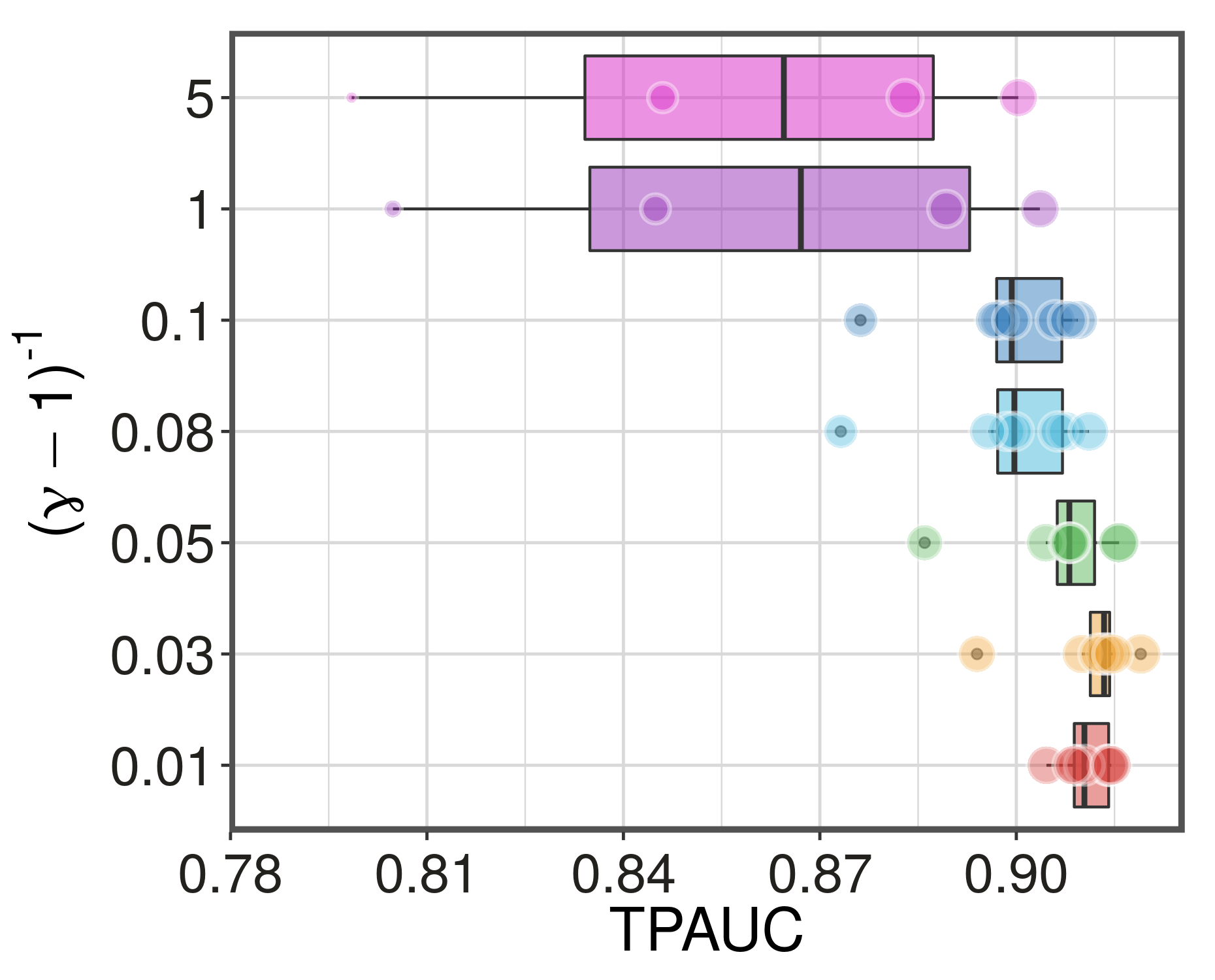} 
      }~~
      \subfigure[Effect of $\gamma$ on subset2, $\alpha=0.5, \beta=0.5$]{
        \includegraphics[width=0.3\textwidth]{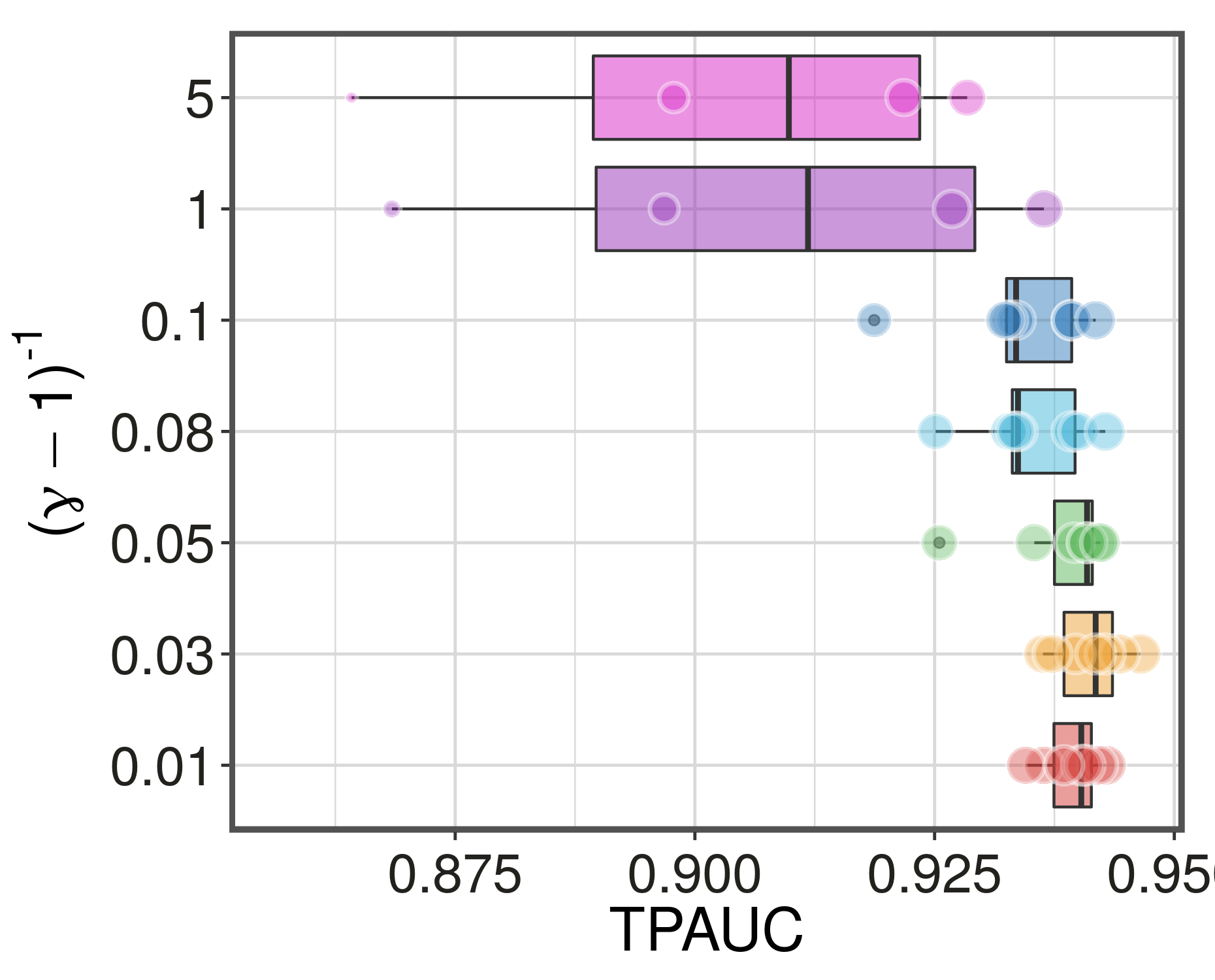} 
      }
  
      \subfigure[Effect of $\gamma$ on subset3, $\alpha=0.3, \beta=0.3$]{
        \includegraphics[width=0.3\textwidth]{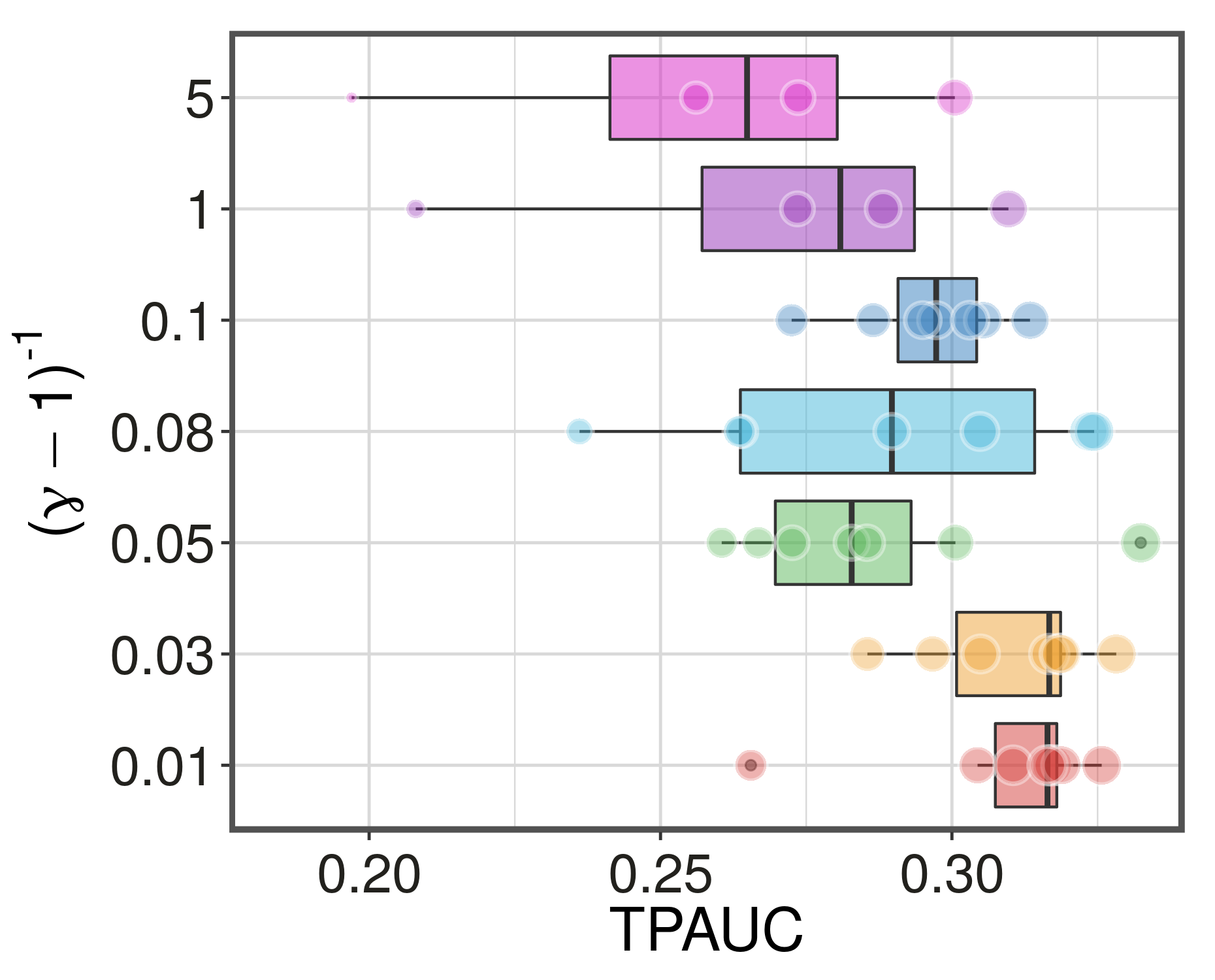} 
      }~~
      \subfigure[Effect of $\gamma$ on subset3, $\alpha=0.4, \beta=0.4$]{
        \includegraphics[width=0.3\textwidth]{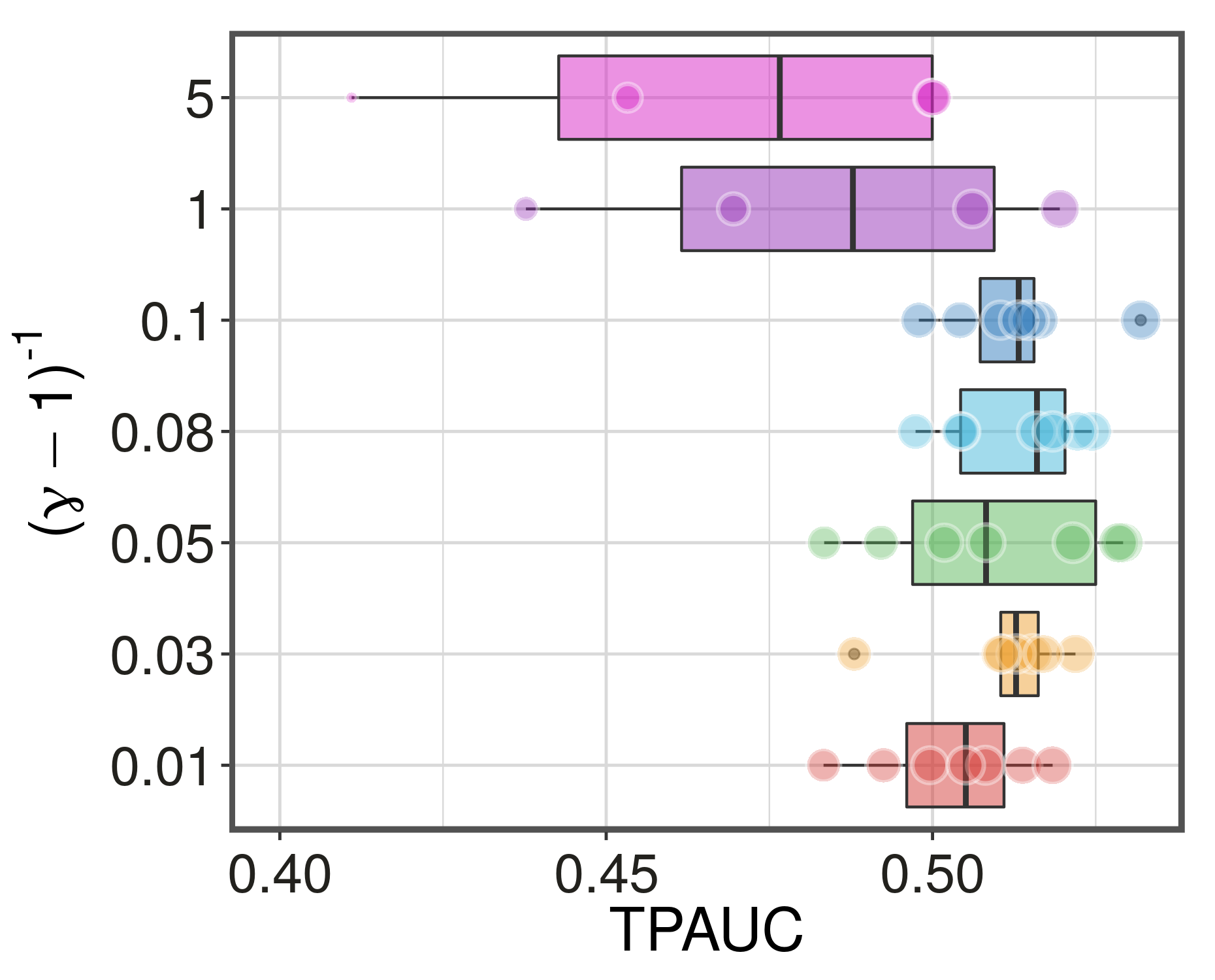} 
      }~~
      \subfigure[Effect of $\gamma$ on subset3, $\alpha=0.5, \beta=0.5$]{
        \includegraphics[width=0.3\textwidth]{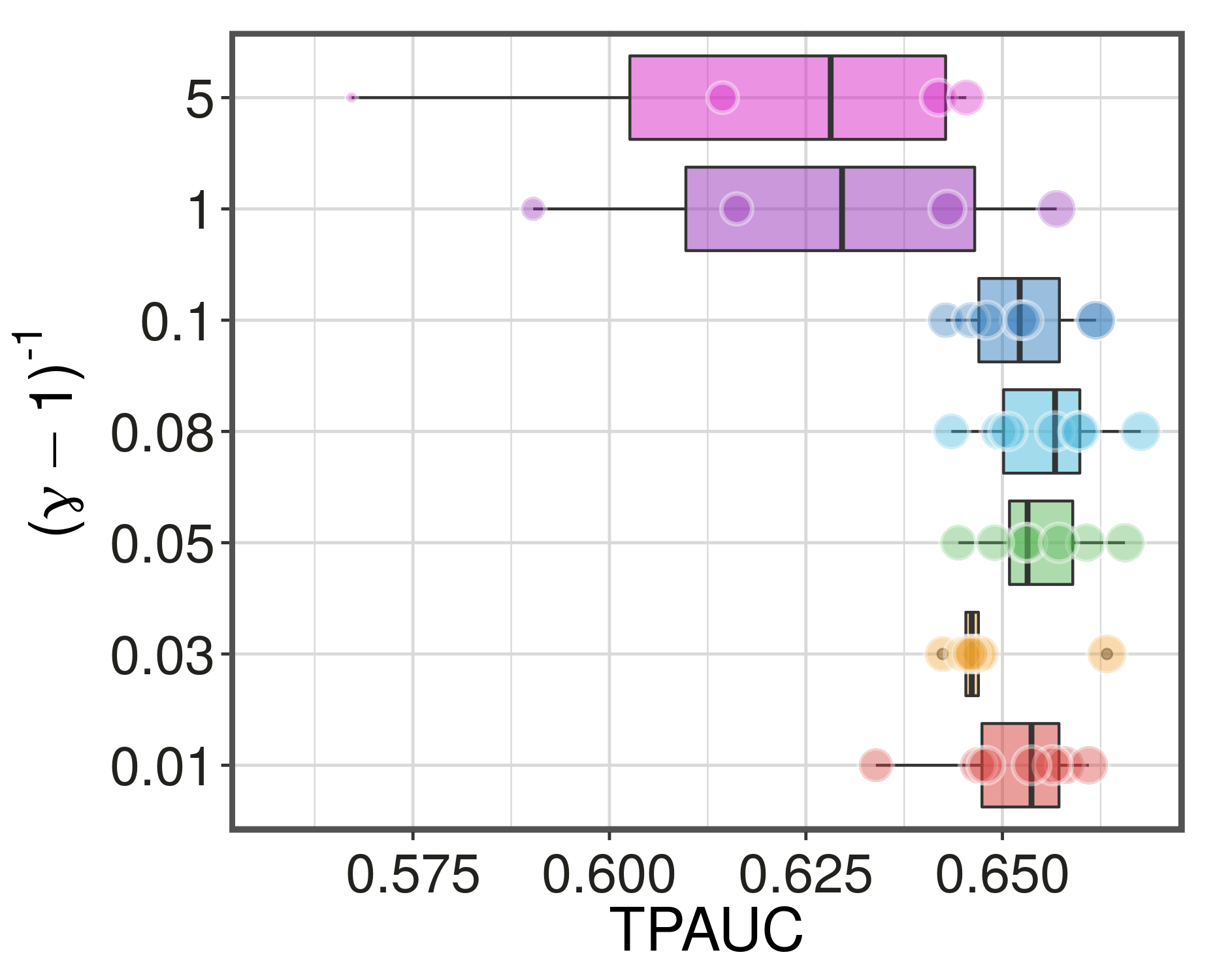} 
      }
      \caption{\label{fig:gamma}Sensitivity analysis of \texttt{Poly} on $\gamma$. Experiments are conducted on CIFAR-10-LT. For each box, $(\gamma-1)^{-1}$ is fixed as the y-axis value, and the scattered points along the box show the variation of $\gamma$.}
  \end{figure*}

\subsection{Warm-Up Training Phase With Delay Epochs}
Focusing on the hard examples at the beginning of the training process brings a high risk of over-fitting. It is thus necessary to focus on the entire dataset to capture the global information. Inspired by this investigation, we adopt a warm-up training strategy. Specifically, the model will go through a warm-up phase with \underline{\textbf{$E_k$ { Epochs}}} of ordinary AUC optimization training. Afterward, we start the TPAUC training phase by optimizing our proposed surrogate problems. We will show its effect in the next subsection.

\subsection{Validation of the Upper Bound Result in Prop.\ref{prop:concon} }
In this subsection, we validate the sufficient condition in Prop.\ref{prop:concon}-(a). To do this, we plot the training curve for the three subsets of CIFAR-10-LT under three TPAUC metrics with polynomial weighting. Here $ \hat{\mathcal{R}}^{\alpha,\beta}_{\phi}, \hat{\mathcal{R}}^{\alpha,\beta}_{\ell}, \hat{\mathcal{R}}^{\alpha,\beta}_{0-1}$ represent the approximated, surrogate, and 0-1 empirical risk function over the training data. We can observe that our proposed risk $\hat{\mathcal{R}}^{\alpha,\beta}_{\phi}$ is an upper bound of the other two quantities in most cases. Even  for the failure case Fig.\ref{fig:val}-(e),(f), we can also observe that the three curves at least exert a consistent decreasing trend. Hence the overall results show the effectiveness of our theoretical result.

\subsection{Overall Performance}
The performance comparisons for CIFAR-10-LT, CIFAR-100-LT and tiny-ImageNet-200-LT are shown in Tab.\ref{tab:perf-cifar-10}-\ref{tab:perf-tiny}, respectively. We have the following observations from the results:

\begin{enumerate}
  \item The best performance of our proposed methods consistently surpasses all the competitors significantly on all metrics. Moreover, the improvements are significant in most cases. 
  \item In most cases, the AUC-based algorithms often tend to outperform other algorithms. This shows that optimizing the AUC-like metric is more effective in improving the TPAUC metric. This observation could also be supported by Thm.\ref{prop:bayes}. Thm.\ref{prop:bayes} states optimizing TPAUC is equivalent to partially generate a consistent ranking with $\eta(\x)$. Meanwhile, optimizing AUC or OPAUC could also generate a consistent ranking $\eta(\x)$ globally or in a different partial region from TPAUC. In this sense, optimizing AUC or OPAUC shares similar (but different) behavior to optimizing TPAUC. 
  \item Sometimes, we see that the OPAUC-based algorithms could not outperform the sqAUC, specifically on the CIFAR-100-LT dataset. This suggests that OPAUC and TPAUC tend to provide inconsistent comparison results, which could also be supported by the analysis in Sec.\ref{sec:incon}.
  \item Comparing different ways to implement the partial AUC optimization, we observe that the weighting-based algorithms, \textit{i.e.}, OPAUC-Poly, OPAUC-Exp, TPAUC-Poly, TPAUC-Exp, outperform the truncated version, \textit{i.e.}, TruncOPAUC, TruncTPAUC. This suggests that our proposed weighting function is more effective than sample manipulation. 
  \item For our algorithms, the ordinary and minimax implementations tend to show different performances. This is not surprising since the reformulated objective function has a different landscape from the original one. Hence when they have to be solved with different optimization algorithms, they often end up with different local optimal solutions practically, though their global optimal solutions are proved to be the same in Thm.\ref{thm:reform}. 
\end{enumerate}

\subsection{Effect of $\gamma$}

 In this setting, we will use one same $\gamma$ for both the positive example weight $v_+$ and the negative weight $v_-$. In Fig.\ref{fig:gammaexp} and Fig.\ref{fig:gamma}, we show the sensitivity in terms of $\gamma$ on  CIFAR-10-LT for \texttt{Exp} and \texttt{Poly}, respectively. One can observe very different trends in these two methods. This is because that \texttt{Exp} and \texttt{Poly} have different characteristics in terms of the landscape of the weight function. As shown in Fig.\ref{fig:weight}-(c), (d), the weight  landscape of \texttt{Exp} is flat within a large subset of the domain. In this sense, it does not have a strong dependency on $\gamma$. As shown in Fig.\ref{fig:weight}-(a), (b), the weight  landscape of \texttt{Poly} is more sensitive toward $\gamma$. Moreover, the weighting function of \texttt{Poly} changes from a concave function ($(\gamma - 1)^{-1} \le 1$) to a concave function $(\gamma - 1)^{-1} \ge 1$ with increasing $\gamma$. It leads to a clear trend with various $\gamma$. Moreover, we see that concave weight functions own significantly better performance. \textbf{This validates our theoretical analysis in Prop.\ref{prop:concon}}. In appendix \ref{sec:app_exp}, we will present more analysis on the effect of $\gamma$.


\section{Conclusion}
In this paper, we initiate the study on TPAUC optimization. Since the original optimization problem could not be solved directly with an end-to-end framework, we propose a general framework to construct surrogate optimization problems for TPAUC. Following our dual correspondence theory, we can establish a surrogate problem once a calibrated penalty function or a calibrated weighting function is found. To see how and when our framework could provide efficient approximations of the original problem, we show that the surrogate objective function could reach the upper bound of the original one and that concave weighting functions are better choices than their convex counterparts. Moreover, we also provide high probability uniform upper bounds for the generalization error. The experiments on three datasets consistently show the advantage of our framework.

\section{Acknowledgement}
This work was supported in part by the National Key R\&D Program of China under Grant 2018AAA0102000, in part by National Natural Science Foundation of China: U21B2038, U1936208, 61931008, 62025604, 6212200758 and 61976202, in part by the Fundamental Research Funds for the Central Universities, in part by Youth Innovation Promotion Association CAS, in part by the Strategic Priority Research Program of Chinese Academy of Sciences, Grant No. XDB28000000, and in part by the National Postdoctoral Program for Innovative Talents under Grant BX2021298.

\ifCLASSOPTIONcaptionsoff
  \newpage
\fi

\bibliographystyle{abbrv}
\bibliography{example_paper}

\begin{IEEEbiography}[{\includegraphics[width=1in,height=1.25in,clip,keepaspectratio]{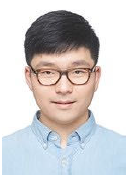}}]{Zhiyong Yang} received the M.Sc. degree in computer science and technology from University of Science and Technology Beijing (USTB) in 2017, and the Ph.D. degree from University of Chinese Academy of Sciences (UCAS) in 2021. He is currently a postdoctoral research fellow with the University of Chinese Academy of Sciences. His research interests lie in machine learning and learning theory, with special focus on AUC optimization, meta-learning/multi-task learning, and learning theory for recommender systems. He has authored or coauthored about 32 academic papers in top-tier international conferences and journals including T-PAMI/ICML/NeurIPS/CVPR. He served as a TPC member for IJCAI 2021 and a reviewer for several top-tier journals and conferences such as T-PAMI, TMLR, ICML, NeurIPS and ICLR.
\end{IEEEbiography}

\begin{IEEEbiography}
	[{\includegraphics[width=1in,height=1.25in,clip,keepaspectratio]{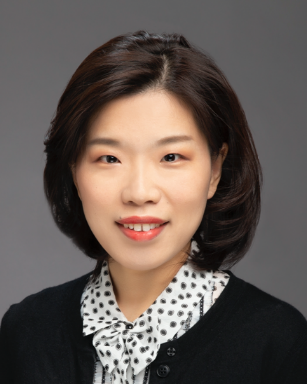}}]{Qianqian Xu}  received the B.S. degree in computer science from China University of Mining and Technology in 2007 and the Ph.D. degree in computer science from University of Chinese Academy of Sciences in 2013. She is currently an Associate Professor with the Institute of Computing Technology, Chinese Academy of Sciences, Beijing, China. Her research interests include statistical machine learning, with applications in multimedia and computer vision. She has authored or coauthored 50+ academic papers in prestigious international journals and conferences (including T-PAMI, IJCV, T-IP, NeurIPS, ICML, CVPR, AAAI, etc). Moreover, she serves as an associate editor of IEEE Transactions on Circuits and Systems for Video Technology, ACM Transactions on Multimedia Computing, Communications, and Applications, and Multimedia Systems.
\end{IEEEbiography}

\begin{IEEEbiography}[{\includegraphics[width=1in,height=1.25in,clip,keepaspectratio]{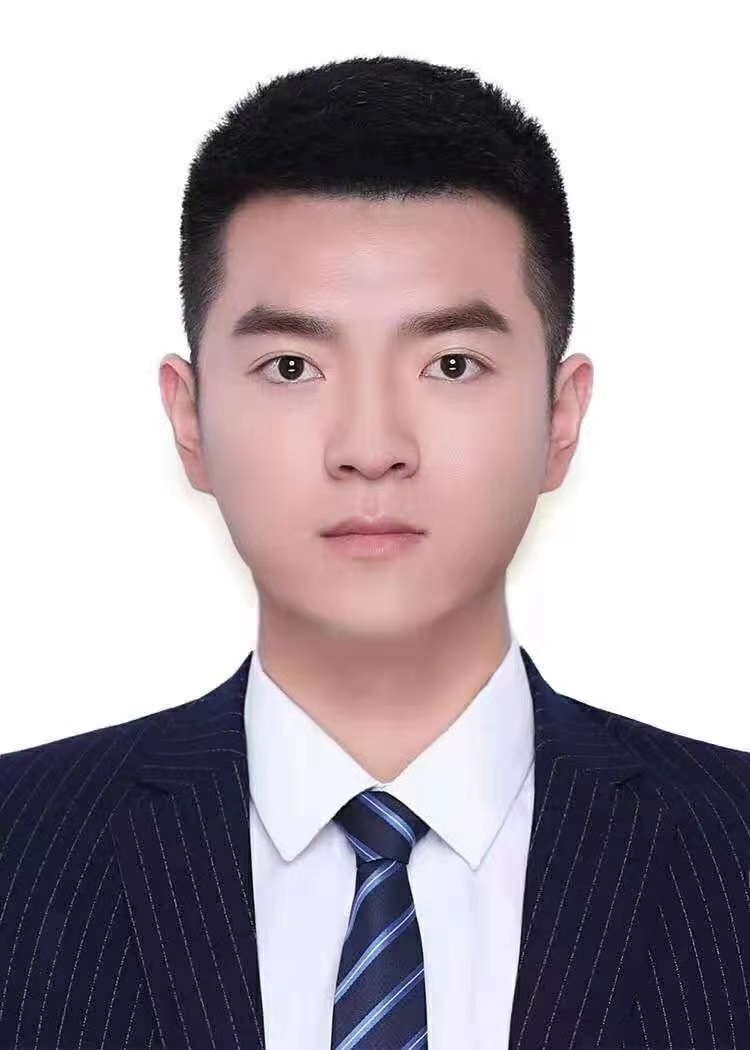}}]{\textbf{Shilong Bao}} 
	received the B.S. degree in College of Computer Science and Technology from Qingdao University in 2019. He is currently pursuing the Ph.D. degree with University of Chinese Academy of Sciences. His research interest is machine learning and data mining.
\end{IEEEbiography}
	
\begin{IEEEbiography}[{\includegraphics[width=1in,height=1.25in,clip,keepaspectratio]{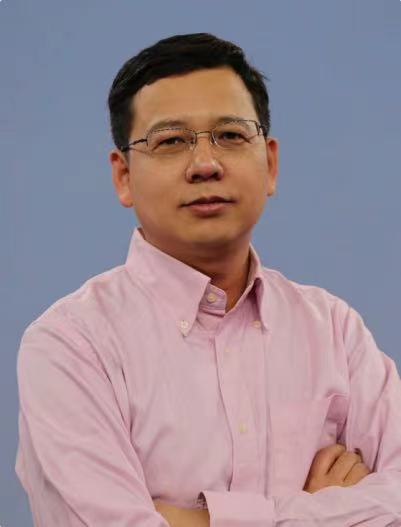}}]{\textbf{Yuan He}} 
  Yuan He received his B.S. degree and Ph.D. degree from Tsinghua University, P.R. China. He is a Senior Staff Engineer in the Security Department of Alibaba Group, and working on artificial intelligence-based content moderation and intellectual property protection systems. Before joining Alibaba, he was a research manager at Fujitsu working on document analysis system. He has published more than 30 papers in computer vision and machine learning related conferences and journals including CVPR, ICCV, ICML, NeurIPS, AAAI and ACM MM. His research interests include computer vision, machine learning, and AI security.
\end{IEEEbiography}

\begin{IEEEbiography}
	[{\includegraphics[width=1in,height=1.25in,clip,keepaspectratio]{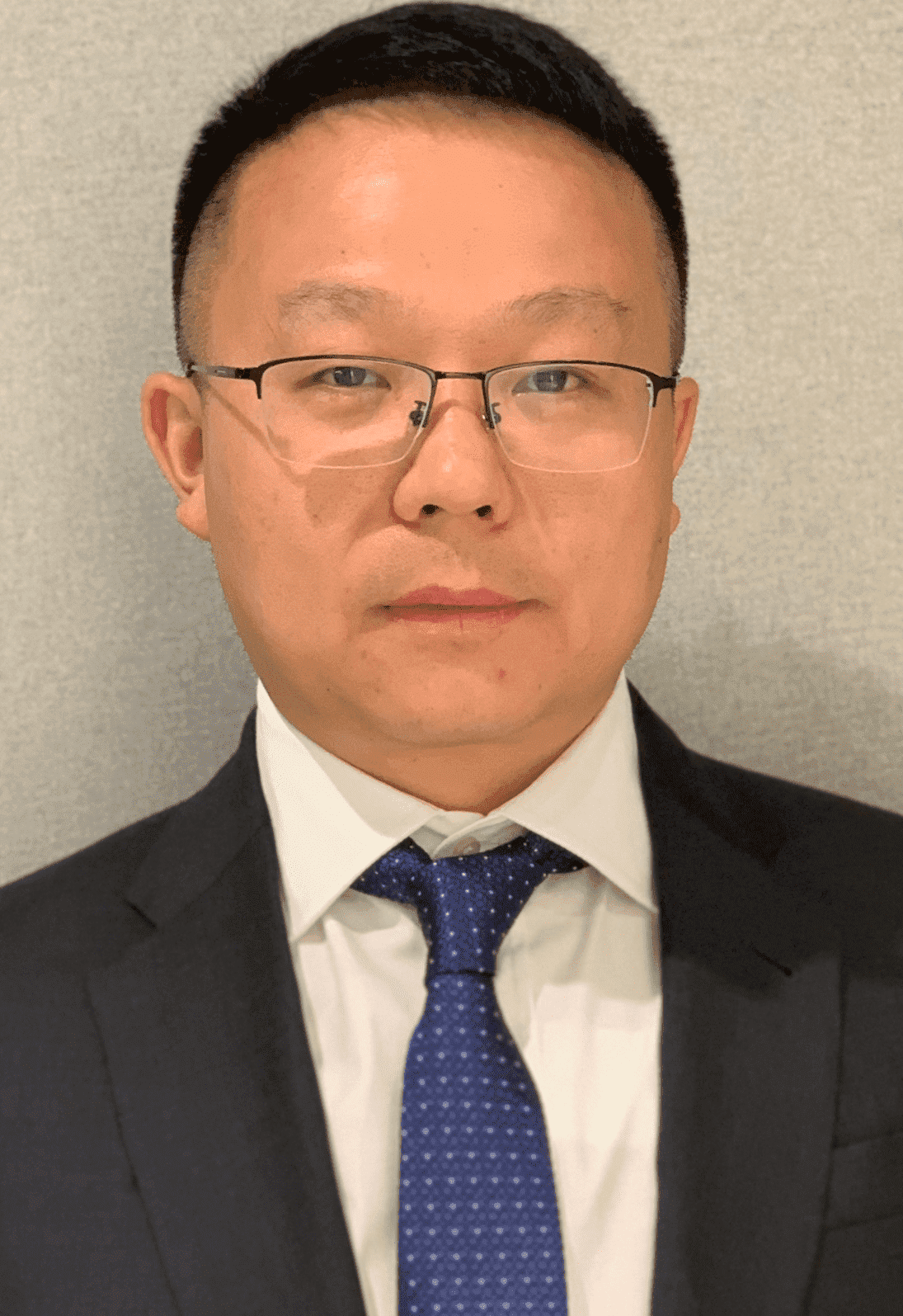}}]{Xiaochun Cao}, is a Professor of  School of Cyber Science and Technology, Shenzhen Campus, Sun Yat-sen University. He received the B.E. and M.E. degrees both in computer science from Beihang University (BUAA), China, and the Ph.D. degree in computer science from the University of Central Florida, USA, with his dissertation nominated for the university level Outstanding Dissertation Award. After graduation, he spent about three years at ObjectVideo Inc. as a Research Scientist. From 2008 to 2012, he was a professor at Tianjin University. From 2012 to 2022, he was a professor at Institute of Information Engineering, Chinese Academy of Sciences.  He has authored and coauthored over 200 journal and conference papers.  In 2004 and 2010, he was the recipients of the Piero Zamperoni best student paper award at the International Conference on Pattern Recognition. He is a fellow of IET and a Senior Member of IEEE. He is an associate editor of IEEE Transactions on Image Processing, IEEE Transactions on Circuits and Systems for Video Technology and IEEE Transactions on Multimedia.
\end{IEEEbiography}

\begin{IEEEbiography}
	[{\includegraphics[width=1in,height=1.25in,clip,keepaspectratio]{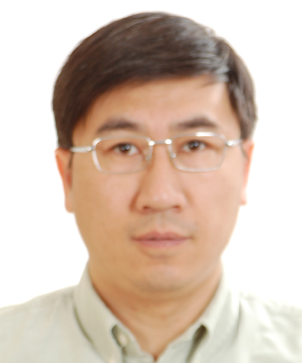}}]{Qingming Huang} is a chair professor in the University of Chinese Academy of Sciences and an adjunct research professor in the Institute of Computing Technology, Chinese Academy of Sciences. He graduated with a Bachelor degree in Computer Science in 1988 and Ph.D. degree in Computer Engineering in 1994, both from Harbin Institute of Technology, China. His research areas include multimedia computing, image processing, computer vision and pattern recognition. He has authored or coauthored more than 400 academic papers in prestigious international journals and top-level international conferences. He was the associate editor of IEEE Trans. on CSVT and Acta Automatica Sinica, and the reviewer of various international journals including IEEE Trans. on PAMI, IEEE Trans. on Image Processing, IEEE Trans. on Multimedia, etc. He is a Fellow of IEEE and has served as general chair, program chair, track chair and TPC member for various conferences, including ACM Multimedia, CVPR, ICCV, ICME, ICMR, PCM, BigMM, PSIVT, etc.
\end{IEEEbiography}

\clearpage
\onecolumn
\appendices


\section*{\textcolor{blue}{\Large{Contents}}}

\startcontents[sections]
\printcontents[sections]{l}{1}{\setcounter{tocdepth}{3}}
\newpage

\section{Proof of Proposition \ref{prop:reform}} \label{app:prop1}
First, we need the following lemma to finish the proof:
\begin{flem}\label{lem:equi}
   For $\{t_i\}_{i=1}^n$ with $t_i \ge 0$, assume that $\min_{i \neq j} |t_i - t_j| >0$. Then for the problem:
   \begin{equation*}
     \begin{split}
       &\max_{v_i \in [0,1], ~\sum_{i=1}^{n} v_i \le k} \sum_{i=1}^{n_+} v_i \cdot t_i,\\
     \end{split}
   \end{equation*}
    the unique solution  is $v^\star_i = \ind{t_i \ge t^\downarrow_{(k)}}$, where $k < n,~  k \in \mathbb{N}_+$, $t^\downarrow_{(k)}$ is top $k\text{-}th$ element in $\{t_i\}_{i=1}^n$.
 
\end{flem}

 \noindent\rule[0.15\baselineskip]{\columnwidth}{1pt}
 \begin{proof}
  For a set of weights $\{v_i\}_{i=1}^n$, let us denote $v^\downarrow_{(i)}$ as the weight for $t^\downarrow_i$. For any $\{v'_i\}_{i=1}^n \neq \{v^\star_i\}_{i=1}^n$. We can write down the difference between the objective functions as:

   \begin{equation*}
     \begin{split}
       \sum_{i=1}^{n} (v^\star_i - v'_i )\cdot t_i &  \\ 
       &= \sum_{i \le k} (1 - v^{'\downarrow}_{(i)} )\cdot t^\downarrow_{(i)}  - \sum_{j > k}  v^{'\downarrow}_{(j)} \cdot  t^\downarrow_{(j)} \\ 
       &\overset{(*)}{>}  (k  - \sum_{i \le k} v^{'\downarrow}_{(i)} )\cdot t^\downarrow_{(k)}   - \sum_{j > k}  v^{'\downarrow}_{(j)} \cdot  t^\downarrow_{(j)} \\ 
       &\overset{(**)}{>} (k  - \sum_{i \le k} v^{'\downarrow}_{(i)} )\cdot t^\downarrow_{(k)} -   \sum_{j > k}  v^{'\downarrow}_{(j)} \cdot  t^\downarrow_{(k)} \\ 
       & = (k  - \sum_{i=1}^n v^{'\downarrow}_{(i)} ) \cdot t^\downarrow_{(k)}\\
       & \ge 0,
      \end{split}
   \end{equation*}
   where $(*), (**)$, follow the assumption that  $\min_{i \neq j} |t_i - t_j| >0$. Note that since the $\{v'_i\}_{i=1}^n$ is arbitrarily chosen, the proof is thus completed.
 \end{proof}
 \noindent\rule[0.15\baselineskip]{\columnwidth}{1pt}

 \begin{frthm1}
   For any $\alpha, \beta \in (0,1)$, if scores $\f(\x) \in [0,1]$,  and there are no ties in the scores, the original optimization problem is equivalent to the following problem:
   \begin{equation*}
     \begin{split}
        &\min_{\bm{\theta}}  \frac{1}{\np\nn}\sumpn \vpi \cdot \vnj \cdot \ell(\f, \xpi, \xnj)\\ 
        s.t.~~&\vp = \argmax_{\vpi \in [0,1], \sum_{i=1}^{\np} \vpi \le \npa} \sum_{i=1}^{n^+} \left(\vpi \cdot (1- \f(\xpi))\right)\\ 
        &\vn = \argmax_{\vnj \in [0,1], \sum_{j=1}^{\nn} \vnj \le \nnb} \sum_{j=1}^{\nn} \left(\vnj \cdot \f(\xnj) \right)
     \end{split}
   \end{equation*}
   \end{frthm1}
   \noindent\rule[0.15\baselineskip]{\columnwidth}{1pt}
   \begin{proof}

     First, it is easy to see that $\opz$ could be formulated as follows:
     \begin{equation*}
       \begin{split}
          &\min_{\bm{\theta}}  \frac{1}{\np\nn}\sumpn \vpi \cdot \vnj \cdot \ell(\f, \xpi, \xnj)\\ 
          s.t~~&\vpi = \begin{cases}
            1, & 1 - \f(\xpi) \ge 1 - \f(\fxpo{\npa})\\ 
            0, & \text{otherwise}
          \end{cases}  \\ 
          &\vnj = \begin{cases}
           1, & \f(\xnj) \ge \f(\fxpo{\nnb})\\ 
           0, & \text{otherwise}
          \end{cases}
       \end{split}
     \end{equation*}
 Then the rest of the proof follows Lem.\ref{lem:equi} directly.
   \end{proof}
   \noindent\rule[0.15\baselineskip]{\columnwidth}{1pt}

 \section{Proof of Proposition \ref{prop:concon}}\label{app:prop2}
 \begin{flem}[H\"older's Inequality]\label{lem:hol}
   $\forall p>1, ~q>1$ such that $1/p +1/q = 1$, we have:
   \begin{equation*}
     \expeb\left[|XY|\right] \le (\expeb\left[|X|^p\right])^{1/p} \cdot (\expeb\left[|Y|^q\right])^{1/q}  
   \end{equation*}
 \end{flem}

 \begin{flem}\label{lem:invhol}
   $\forall 0<p<1, ~q = -p/(1-p)$, we have:
   \begin{equation*}
     \expeb\left[|X'Y'|\right] \ge (\expeb\left[|X'|^p\right])^{1/p} \cdot (\expeb\left[|Y'|^q\right])^{1/q}  
   \end{equation*}
 \end{flem}
 \noindent\rule[0.15\baselineskip]{\columnwidth}{1pt}
 \begin{proof}
   It could be proved by applying Lem.\ref{lem:hol} to $X = |X'Y'|^p$ and $Y = |Y'|^{-p}$.
 \end{proof}
 \noindent\rule[0.15\baselineskip]{\columnwidth}{1pt}

 \begin{frthm2}
   Given a strictly increasing weighting function $\psi_\gamma: [0,1] \rightarrow [0,1]$, such that $\vpi = \psi_\gamma(1-\f(\xpi))$, $\vnj = \psi_\gamma(\f(\xnj))$, denote:
   \begin{equation*}
     \begin{split}
       &\Ionep = \left\{x_+: x_+ \in \XP, f(x_+) \ge f(x^{(\npa)}_+) \right\}, \\ 
       &\Ionen = \left\{x_-: x_- \in \XN, f(x_-) \le f(x^{(\nnb)}_-) \right\},
     \end{split}
   \end{equation*}
   denote $\Itwo$ as $(\XP \times  \XN) \backslash (\Ionep \times  \Ionen)$; denote $\expiop[x]$ as the empirical expectation of $x$ over the set $\Ionep$, and  $\expion[x],~ \expiopn,$ $~\expit$ are defined similarly; define
   $l_{i,j} = \ell(\f, \xpi,\xnj)$. We assume that
   \[\npa \in \mathbb{N}, ~\nnb \in \mathbb{N}, ~\fxp, \fxn \in (0,1), \] 
   then:
   \begin{enumerate}
     \item[(a)] A sufficient condition for $\rorg(\mathcal{S},\f) \le \rpsi(\mathcal{S},\f)$ is that:
     \begin{equation*}
       \begin{split}
         \sup_{ p \in (0,1), q = -\frac{p}{1-p} } \left[\rho_p -\xi_q \right] \ge 0,
       \end{split}
     \end{equation*} 
     where
     \begin{equation*}
       \begin{split}
       \rho_p &= \frac{\left(\expit\left[v^{p}_+ \cdot v^{p}_-\right]\right)^{1/p}  }{\left(\expiopn\left[(1 - v_+v_-)^2\right]\right)^{1/2}}, \\
         \xi_q &=  \frac{\alpha\beta}{1-\alpha\beta} \cdot \frac{\left( \expit(\ell^2_{i,j}) \right)^{1/2}}{\left( \expiopn(\ell^q_{i,j}) \right)^{1/q}}.
       \end{split}
     \end{equation*}
     \item[(b)] A sufficient condition for $ \mathcal{R}^\ell_{\alpha,\beta}(f) \le  \mathcal{R}^\ell_{\psi}(f) $ is that:
     \begin{equation*}
       \begin{split}
         \sup_{ p \in (0,1), q = -\frac{p}{1-p} } \left[\rho_p -\xi_q \right] \ge 0,
       \end{split}
     \end{equation*} 
     where
     \begin{equation*}
       \begin{split}
       \rho_p &= \frac{\left(\expite\left[v^{p}_+ \cdot v^{p}_-\right]\right)^{1/p}  }{\left(\expiopne\left[(1 - v_+v_-)^2\right]\right)^{1/2}}, \\
         \xi_q &=  \frac{\alpha\beta}{1-\alpha\beta} \cdot \frac{\left( \expite(\ell^2_{i,j}) \right)^{1/2}}{\left( \expiopne(\ell^q_{i,j}) \right)^{1/q}}.
       \end{split}
     \end{equation*}
     \item[(c)] If there exists at least one strictly concave $\psi_\gamma$  such that the  $\rorg(\mathcal{S},\f) > \rpsi(\mathcal{S},\f)$, then  $\rorg(\mathcal{S},\f) > \rpsi(\mathcal{S},\f)$ holds for all convex $\psi_\gamma$.
   \end{enumerate}
 \end{frthm2}
 \noindent\rule[0.15\baselineskip]{\columnwidth}{1pt}
 \begin{proof} First, $l_{i,j} = \ell(\f, \xpi,\xnj)$, we can reformulate $\rpsi - \rorg$ as follows.
   \begin{equation}\label{eq:diffr}
     \begin{split}
       \rpsi - \rorg &= \frac{1}{\np\nn}\sumpn \vpi  \cdot \vnj \cdot \ell_{i,j} - \frac{1}{\np\nn}\psumpn {\ell_{i,j}}\\
       &= \frac{1}{\np\cdot \nn} \cdot \sum_{\xpi, \xnj \in \mathcal{I}_2} \vpi\cdot\vnj\cdot \ell_{i,j} - \frac{1}{\np\cdot \nn} \cdot \expiopn  (1 -\vpi\vnj) \ell_{i,j} \\ 
     & = (1 - \alpha \beta) \cdot \frac{1}{|\mathcal{I}_2|} \cdot \sum_{\xpi, \xnj \in \mathcal{I}_2} \vpi\cdot\vnj\cdot \ell_{i,j} - (\alpha\cdot \beta) \cdot \frac{1}{|\mathcal{I}_1|} \cdot \expiopn  (1 -\vpi\vnj) \ell_{i,j}\\
     & =  (1 - \alpha \beta) \cdot \expit\left[ \vp \cdot \vn \cdot \ell  \right]  - \alpha \cdot \beta \expiopn\left[ (1- \vp \cdot \vn) \cdot \ell  \right]
     \end{split}
   \end{equation}
 Now we prove (a)-(b) based on this result.
   \begin{enumerate}
     \item[(a)]According to Lem.\ref{lem:invhol}, $\forall~ 1>p>0,~ q = -p/(1-p)$, we have:
     \begin{equation*}
       \begin{split}
         (1 - \alpha \beta) \cdot \expit\left[ \vp \cdot \vn \cdot \ell  \right] \ge 
         \underbrace{(1 - \alpha \beta) \cdot \left(\expit \left[ \vp^p \cdot \vn^p\right]\right)^{1/p} \cdot  \left(\expit \left[ \ell^q \right]\right)^{1/q} }_{(a)}
       \end{split}
     \end{equation*}
     Meanwhile, we have:
     \begin{equation*}
       \begin{split}
         \alpha \cdot \beta \expiopn\left[ (1- \vp \cdot \vn) \cdot \ell  \right] \le 
         \underbrace{\alpha \cdot \beta \expiopn\left[ (1- \vp \cdot \vn)^2\right]^{1/2} \cdot \expiopn\left[ \ell^2\right]^{1/2}}_{(b)}
       \end{split}
     \end{equation*}
     This shows that $(a) - (b) \ge 0$ implies $\rpsi \ge \rorg$. Moreover,  $(a) - (b) \ge 0$ is equivalent to $\rho_p - \xi_q \ge 0$. The proof of (a) is completed since $p$ and $q$ are arbitrarily chosen within their domain.
     \item[(b)] This is a natural extension of (a) in the population level, which is thus omitted.
     \item[(c)] Given a strictly concave function $\psi_\gamma: [0,1] \rightarrow [0,1]$ and a convex function $\widetilde{\psi}_\gamma: [0,1] \rightarrow [0,1]$. We have that \[\forall y \in [0,1],~ \psi_\gamma\left(y\right) = \psi_\gamma\left(0\cdot (1-y) + y\cdot 1\right)  > y \cdot \psi_\gamma(1) = y  \]
     \[\forall y \in (0,1),~ \widetilde{\psi}_\gamma\left(y\right) =  \widetilde{\psi}_\gamma\left(0\cdot (1-y) + y\cdot 1\right)  \le y \cdot \widetilde{\psi}_\gamma(1) = y  \]   
     This implies that $\psi_\gamma\left(y\right) >  \widetilde{\psi}_\gamma\left(y\right) ,~\forall y \in (0,1)$. The proof then follows that \[ \rpsi - \rorg \propto \min_{i,j} \left[v^+_i\cdot v^-_j\right] = \min_{i,j} \left[\psi(\f(\xp)) \cdot \psi(1 -\f(\xn))\right] \]
   and 
   $\fxp, \fxn \in (0,1)$.
 \end{enumerate}
 \end{proof}
 \noindent\rule[0.15\baselineskip]{\columnwidth}{1pt}

 \section{Proof of Proposition \ref{prop:dual}} \label{app:prop3}
 \begin{frthm3}
   Given a strictly convex function $\varphi_\gamma$, and define $\psi_\gamma(t)$ as:
   \begin{equation*}
   \psi_\gamma(t) = \argmax_{v \in [0,1]} ~v \cdot t - \varphi_\gamma(v)
   \end{equation*} 
   then we can draw the following conclusions:
   \begin{enumerate}
     \item[(a)] If $\varphi_\gamma$ is a calibrated smooth penalty function, we have $\psi_\gamma(t) = \varphi_{\gamma}^{'-1}(t)$,which  is a calibrated weighting function.
     \item[(b)] If $\psi_\gamma$ is a calibrated weighting function such that $v = \psi_\gamma(t)$, we have $\varphi_\gamma(v) = \int \psi^{-1}_\gamma(v)dv +const.$, which is a  calibrated smooth penalty function.
   \end{enumerate}
   \end{frthm3}
   \noindent\rule[0.15\baselineskip]{\columnwidth}{1pt}
   \begin{proof}
     {\color{white}{dsadsa}}\\
     \begin{enumerate}
       \item[(a)] Since $\varphi_\gamma$ is strictly convex, $v \cdot t - \varphi_\gamma(v)$ is strictly concave, then  $\psi_\gamma$ has a unique global optimal solution. To reach the optimal solution, we have:
       \begin{equation*}
         \begin{split}
            \left(v \cdot t - \varphi_\gamma(v)\right)'  = t- \varphi_\gamma'(v) = 0 
         \end{split}
       \end{equation*} 
     Note that $v  =\varphi_\gamma(t)$, we have:
     \begin{equation*}
       \begin{split}
          t- \varphi'_\gamma(\psi_\gamma(t)) = 0 
       \end{split}
     \end{equation*}
     Equivalently, note that $\varphi^{'}_\gamma(t)$ is invertible since it is strictly increasing ($\varphi^{''}_\gamma(t) >0$),  
     we have:
     \begin{equation*}
       \begin{split}
          \psi_\gamma(t)= \varphi^{'-1}_\gamma(t) 
       \end{split}
     \end{equation*}
     Moreover, we have:
      \[ \psi'_\gamma(t)= \frac{1}{\varphi^{''}(\varphi^{'-1}_\gamma(t))} ~~, \psi''_\gamma(t)=- \frac{\varphi^{'''}(\varphi^{'-1}_\gamma(t))}{(\varphi^{''}(\varphi^{'-1}_\gamma(t)))^3 }.\]
      Since $\varphi^{''}_\gamma(x) > 0,   \varphi^{'''}_\gamma(x) > 0$, we know that $\psi'_\gamma(t)$ is a calibrated weighting function according to the definition.
 \item[(b)] Assume that $\psi_\gamma(t)$ is the solution of the optimization problem, recall the optimal condition:  
 \begin{equation*}
   \begin{split}
      t- \varphi'_\gamma(v) = 0 
   \end{split}
 \end{equation*} 
 Since $t = \psi^{-1}(v)$, we have:
 \begin{equation*}
   \begin{split}
     \psi^{-1}(v) = \varphi'_\gamma(v)  
   \end{split}
 \end{equation*} 
 leading to the fact that 
 \begin{equation*}
   \begin{split}
    \int \psi^{-1}(v) dv = \varphi_\gamma(v)  
   \end{split}
 \end{equation*} 
 Moreover, we have:
 \[ \varphi'_\gamma(v)= {\psi_\gamma^{-1}\left(v\right)} ~~, \varphi''_\gamma(v)= \frac{1}{\psi_\gamma^{'}\left(\psi^{-1}(v)\right) }, ~~   \varphi'''_\gamma(v)=  - \frac{\psi_\gamma^{''}(\psi^{-1}(v))}{\left(\psi_\gamma^{'}(\psi^{-1}(v))\right)^3} \]
 Since $\psi^{-1}_\gamma(x) > 0,   \psi_\gamma^{'}(x) > 0$ and $\psi^{''}_\gamma(x) <0$, $\varphi_\gamma$  is then a calibrated weighting function according to the definition.

     \end{enumerate}
   \end{proof}
   \noindent\rule[0.15\baselineskip]{\columnwidth}{1pt}

 \section{Proof of Theorem \ref{thm:reform}}\label{sec:1}
 First, we will need the following lemma.

 \begin{lem}\label{lem:optsq} Define $\mathcal{C}_1 = [0, B_1]$, $\mathcal{C}_2 = [0, B_2]$, $\forall p \in \mathcal{C}_1, q \in \mathcal{C}_2$, we have:
   \begin{align*}
     p \cdot q &= \min_{ a_1 \in \mathcal{C}_1, a_2 \in \mathcal{C}_2} \max_{b \in \mathcal{C}_3} \frac{1}{2} \cdot \nu(b, a_1, a_2,p,q) \\   
     & = \max_{b \in \mathcal{C}_3} \min_{ a_1 \in \mathcal{C}_1, a_2 \in \mathcal{C}_2}  \frac{1}{2} \cdot \nu(b, a_1, a_2,p,q)
     \end{align*}
 where:
 \begin{align*}
   \nu(b,a_1,a_2,e,f) &=  w(b, e+f) - w(a_1,e) - w(a_2,f)\\ 
   w(x, y) &= 2 x\cdot y - x^2,\\
   \mathcal{C}_3 &= [0, B_1 + B_2].
 \end{align*}
 \end{lem}
 \begin{proof}
 Since 
 \begin{align*}
   p \cdot q = \frac{1}{2} \cdot \left((p+q)^2 - p^2 -  q^2)\right)
 \end{align*}
 and that 
 \begin{align*}
   & y^2 = \max_{x} 2x\cdot y - x^2.
  \end{align*}
  with the maximum attained at $x=y$.
 Since $p+q \in \mathcal{C}_3, p \in \mathcal{C}_1, q \in \mathcal{C}_2$, we then reach the conclusion that:
 \begin{align*}
   p \cdot q &= \min_{ a_1 \in \mathcal{C}_1, a_2 \in \mathcal{C}_2} \max_{b \in \mathcal{C}_3} \frac{1}{2} \cdot \nu(b, a_1, a_2,p,q). 
   \end{align*}
 Moreover, since $\mathcal{C}_1, \mathcal{C}_2$ are compact and $ \nu(b,\cdot,\cdot,p,q)$ is convex, $ \nu(\cdot, a_1,a_2,p,q)$ is concave, we have:
 \begin{align*}
  &\min_{ a_1 \in \mathcal{C}_1, a_2 \in \mathcal{C}_2} \max_{b \in \mathcal{C}_3} \frac{1}{2} \cdot \nu(b, a_1, a_2,p,q) \\   
   & = \max_{b \in \mathcal{C}_3} \min_{ a_1 \in \mathcal{C}_1, a_2 \in \mathcal{C}_2}  \frac{1}{2} \cdot \nu(b, a_1, a_2,p,q)
   \end{align*}
 according to von Neumann's minimax theorem.

 \noindent The proof is then completed.
 \end{proof}
 \noindent\rule[0.15\baselineskip]{\columnwidth}{1pt}

 \begin{frthm4}
   Denote $v^\gamma_\infty = \sup_{{x}} |\psi_\gamma(x)|$, $f_\infty = \sup_{{x}} |\f(x)|$. assume that $v^\gamma_\infty < \infty, f_\infty < \infty$, $\ell(t) = (1-t)^2$, $(OP1)$ could be reformulated as \footnote{the inequalities $\bm{0} \le \bm{a} \le \bm{c}_a$ and $\bm{0} \le \bm{b} \le  \bm{c}_b$ should be understood  elementwise.}:
   \begin{equation*}
    \min_{\bm{\theta},  \bm{0}_{10} \le \bm{a} \le \bm{c}_a } \max_{ \bm{0}_{8} \le \bm{b} \le  \bm{c}_b }~\bm{a}^\top \zeta_1 + \bm{b}^\top \zeta_2 -||\tilde{\bm{b}}||^2+||\tilde{\bm{a}}||^2,
  \end{equation*}
  where
  \begin{align*}
     \zeta_1  &= -[\cp,~ \cn,~ 2(\f^+ + \cn),~ 2\cp,~ 2\f^-,~ \cn,~ \f^{+,2},\cp,~ \f^{-,2},\\
      &~~~~~~~~~~~~ 2(\f^+ + \f^-)] ,\\
    \zeta_2 &={\color{white}{-}} [\cp + \cn,~ 2\cn,~ 2\f^+,~ 2(\f^- +\cp) ,~ \cn + \f^{+,2},\\  
             &~~~~~~~~~~~~ \cp + \f^{-,2}, 2 \f^+,~ 2 \f^-], \\
     \tilde{\bm{a}} &= (1/\sqrt{2}) \cdot [a_1, a_2, \sqrt{2} \cdot a_3, \sqrt{2}a_4, \sqrt{2}a_5, a_6, a_7, a_8, \\ 
     &~~~~~~~~~~~~~~~~~~~~~a_9, \sqrt{2} \cdot  a_{10}], \\
    \tilde{\bm{b}} &= (1/\sqrt{2}) \cdot [b_1, \sqrt{2} \cdot b_2,\sqrt{2} \cdot b_3, \sqrt{2} b_4, b_5, b_6,\\ 
    & ~~~~~~~~~~~~~~~~~~~~~ \sqrt{2} \cdot b_7,\sqrt{2} \cdot b_8],\\
    \bm{c}_a &= [v^\gamma_\infty,v^\gamma_\infty,v^\gamma_\infty + f_\infty,v^\gamma_\infty,f_\infty,v^\gamma_\infty,f^2_\infty,v^\gamma_\infty,f^2_\infty,2  f_\infty], \\
    \bm{c}_b &= [2v^\gamma_\infty,v^\gamma_\infty,f_\infty,v^\gamma_\infty + f_\infty,v^\gamma_\infty + f^2_\infty,v^\gamma_\infty + f^2_\infty,f_\infty,f_\infty]. 
  \end{align*}
 \end{frthm4}

   \begin{pf}
     First, we have:
     \begin{align*}
      &~(1 - (\fxpi-
         \fxnj))^2\\  
      =&~ 1 - 2(\fxpi - \fxnj) + \fxpi^2 + \fxnj^2 \\  
      &~-2 \fxpi \cdot \fxnj
     \end{align*}
     Denote 
     \begin{equation*}
      r(\bm{\theta}) =  \frac{1}{\np}\frac{1}{\nn} \sump \sumn \vpi \cdot  \vnj (1 - (\f(\xpi) -  \f(\xnj)))^2.
    \end{equation*}
    
     By substituting the equation above, we then have:
     \begin{align}\label{eq:init}
      r(\bm{\theta}) =& \cp \cdot \cn - 2\cn\f^+ + 2\cp \f^-  + \cn \f^{+,2} + \cp \f^{-,2} - 2 \f^+\cdot \f^-
      \end{align}

     Accodring to Lem.\ref{lem:optsq}, we have the following reformulations:
    
     \begin{align*}
      \cp \cdot \cn &=  \min_{0 \le a_1 \le v^\gamma_\infty, 0 \le a_2 \le v^\gamma_\infty} \max_{ 0 \le b_1 \le 2  v^\gamma_\infty } \frac{1}{2} \cdot \nu(b_1, a_1, a_2,\cp,\cn) \\ 
       -2\cn\f^+ &= \min_{ 0 \le  a_3 \le f_\infty + v^\gamma_\infty } \max_{ 0 \le b_2 \le v^\gamma_\infty, 0 \le b_3 \le  f_\infty}  -\nu(a_3, b_2,b_3,\cn,\f^+)\\ 
       2\cp\f^- &= \min_{0 \le a_4 \le v^\gamma_\infty ,0 \le a_5 \le f_\infty} \max_{ 0 \le b_4 \le v^\gamma_\infty  + f_\infty}  \nu(b_4, a_4,a_5,\cp,\f^-)\\ 
       \cn \cdot \f^{+,2}  &= \min_{0\le a_6 \le v^\gamma_\infty, 0\le a_7 \le f^2_\infty } \max_{0\le b_5 \le v^\gamma_\infty + f^2_\infty} \frac{1}{2} \cdot \nu(b_5, a_6, a_7,\cn, \f^{+,2}) \\ 
       \cp \cdot \f^{-,2} &= \min_{0 \le a_8 \le v^\gamma_\infty, 0 \le a_9 \le f^2_\infty} \max_{0 \le b_6 \le v^\gamma_\infty + f^2_\infty} \frac{1}{2} \cdot  \nu(b_6, a_8, a_9,\cp, \f^{-,2}) \\ 
       -2 \f^{+} \cdot \f^{-} & = \min_{0 \le a_{10} \le 2 f_\infty} \max_{ 0\le b_7 \le f_\infty, 0 \le b_8 \le f_\infty}  - \nu(a_{10}, b_7, b_{8},\f^+, \f^-) \\
     \end{align*}
   Then the proof is completed by substituting each term into the expression of $r(\bm{\theta})$.
     \qed 
     \end{pf}
     \noindent\rule[0.15\baselineskip]{\columnwidth}{1pt}

 \section{Proof of the Bayes Error Analysis}\label{sec:2and3}
 \begin{frthm5}  Assume that $t_{\beta}(f) \le t_{1-\alpha}(f)$, and that there are no tied comparisons, $f$ is a Bayes  scoring function for $\mathsf{TPAUC}^\alpha_\beta$ if it is a solution to the  following problem:
   \begin{align*}
    \min_f &\int_{\x_1,\x_2 \in \cfab \otimes \cfab} p(\x_1) \cdot p(\x_2) \cdot \min\left\{\eta_{1,2},\eta_{2,1}\right\} d\x_1d\x_2 \\ 
  +& 2 \int_{\cfabup}  p(\x_1) \cdot p(\x_2) \cdot\eta_{2,1}\cdot d\x_1d\x_2\\ 
  +&  2 \int_{\cfabdown}  p(\x_1) \cdot p(\x_2) \cdot\eta_{1,2} \cdot d\x_1d\x_2  \\
     \end{align*}
 Specifically, for all $\x_1,\x_2 \in \cfab \otimes \cfab $, we have:
 \begin{equation*}
   \begin{split}
    \left(\eta(\bm{x}_1) - \eta(\bm{x}_2)\right) \cdot \left(f(\bm{x}_1) - f(\bm{x}_2)\right)  >0,\\ 
   \end{split}
 \end{equation*}
 where $\eta(\x) =  \mathbb{P}\left[y =1 | \x\right]$, $\eta_{1,2} = \eta(\x_1)(1-\eta(\x_2)),~ \eta_{2,1} = \eta(\x_2)(1-\eta(\x_1))$, $p(\x)$ is the p.d.f function of the marginal distribution of $\x$.
 \end{frthm5}
    \begin{proof}
    The expected risk of TPAUC could be written as:
    \begin{align*}
      \rfab  \propto \expe_{\x_1,\x_2}\bigg[& \eta(\x_1) \cdot (1-\eta(\x_2)) \cdot \ind{\mathcal{A}_{1,2}} + \eta(\x_2) \cdot (1-\eta(\x_1)) \cdot \ind{\mathcal{A}_{2,1}} \bigg] 
    \end{align*}
   
    \noindent where :
    \begin{align*}
      &\mathcal{A}_{1,2} ~\text{is the event}~  f(\x_1) < f(\x_2), f(\x_1) \le t_{1-\alpha}(f), f(\x_2) \ge t_{\beta}(f),  \\ 
      &\mathcal{A}_{2,1} ~\text{is the event}~  f(\x_2) < f(\x_1), f(\x_1) \ge t_{\beta}(f), f(\x_2) \le t_{1-\alpha}(f) \\
    \end{align*}
 We first minimize the risk for any fixed pair $\x_1, \x_2 \in \mathcal{X}, \x_1 \neq \x_2$. In other words, we need to minimize:
    \begin{align*}
     r(\x_1,\x_2) =&~ \eta(\x_1) \cdot (1-\eta(\x_2)) \cdot \ind{\mathcal{A}_{1,2}} \\ 
     & + \eta(\x_2) \cdot (1-\eta(\x_1)) \cdot \ind{\mathcal{A}_{2,1}}
    \end{align*}
   
    \noindent To do this, we consider the following cases:\\
    \noindent \textbf{Case 1}: $\x_1, \x_2 \in \cfab$. In this case the optimal $f$ should satisfy:
    \begin{equation*}
      \begin{cases}
        f(\x_2) > f(x_1), & \eta(\x_1) \cdot (1-\eta(\x_2)) < \eta(\x_2) \cdot (1-\eta(\x_1) ) \\
        f(\x_1) > f(x_2), & \eta(\x_2) \cdot (1-\eta(\x_1)) < \eta(\x_1) \cdot (1-\eta(\x_2) ), 
      \end{cases}
    \end{equation*}
    which implies that $\left(\eta(\bm{x}_1) - \eta(\bm{x}_2)\right) \cdot \left(f(\bm{x}_1) - f(\bm{x}_2)\right)  >0$.\\
    \noindent\textbf{Case 2}:  $f(\x_1) > \fa, f(\x_2) \le \fa$, in this sense we can only have:
    \begin{align*}
      r(\x_1,\x_2) = \eta(\x_2)(1-\eta(\x_1))
    \end{align*}
    \noindent\textbf{Case 3}: $f(\x_2) > \fa, f(\x_1) \le \fa$. Similarly, we have:
    \begin{align*}
     r(\x_1,\x_2) = \eta(\x_1)(1-\eta(\x_2))
   \end{align*}
   
   \noindent\textbf{Case 4} $f(\x_1) < \fb, f(\x_2) \ge \fb$.  In this sense we can only have:
   \begin{align*}
     r(\x_1,\x_2) = \eta(\x_1)(1-\eta(\x_2))
   \end{align*}
   \noindent\textbf{Case 5} $f(\x_2) < \fb, f(\x_1) \ge \fb$. We have:
   \begin{align*}
    r(\x_1,\x_2) =  \eta(\x_2)(1-\eta(\x_1))
  \end{align*}

  \noindent\textbf{Case 6} $f(\x_1) > \fa, f(\x_2) > \fa$ or  $f(\x_1) < \fb, f(\x_2) < \fb$. In either case, both  $\mathcal{A}_{1,2}$ and  $\mathcal{A}_{2,1}$ are inactivated. Hence we have: 
  \begin{align*}
    r(\x_1,\x_2) = 0.
  \end{align*} 
   
  Hence for any fixed choice of $f$, the population risk becomes

    \begin{align*}
   &\int_{\x_1,\x_2 \in \cfab \otimes \cfab} p(\x_1) \cdot p(\x_2) \cdot \min\left\{\eta(\x_1)(1-\eta(\x_2)),\eta(\x_2)(1-\eta(\x_1))\right\} d\x_1d\x_2 \\ 
 +& 2 \int_{\cfabup}  p(\x_1) \cdot p(\x_2) \cdot\eta(\x_2)(1-\eta(\x_1)) d\x_1d\x_2\\ 
 +&  2 \int_{\cfabdown}  p(\x_1) \cdot p(\x_2) \cdot\eta(\x_1)(1-\eta(\x_2)) d\x_1d\x_2  \\
    \end{align*}
 Here the factor $2$ comes from the symmetry of $\x_1$ and $\x_2$.

 Hence, to reach the Bayes solution, we have to minimize the choice of such $f$. Mathematically, to Bayes optimality, $f$ has to be a solution to:

 \begin{align*}
   \min_f &\int_{\x_1,\x_2 \in \cfab \otimes \cfab} p(\x_1) \cdot p(\x_2) \cdot \min\left\{\eta(\x_1)(1-\eta(\x_2)),\eta(\x_2)(1-\eta(\x_1))\right\} d\x_1d\x_2 \\ 
 +& 2 \int_{\cfabup}  p(\x_1) \cdot p(\x_2) \cdot\eta(\x_2)(1-\eta(\x_1)) d\x_1d\x_2\\ 
 +&  2 \int_{\cfabdown}  p(\x_1) \cdot p(\x_2) \cdot\eta(\x_1)(1-\eta(\x_2)) d\x_1d\x_2  \\
    \end{align*}
   \end{proof} 

 \begin{frthm9}
   Under the same setting as Thm.\ref{prop:bayes},  if we restrict our choice of $f$ in $\mathcal{F}_{a}$, where 
   \begin{align*}
   \mathcal{F}_{a} = \{&f: f ~\text{is}~ (\alpha,\beta)\text{-weakly consistent w.r.t}~\eta \},
   \end{align*}
 Then $f$ minimizes the expected risk in $\mathcal{F}_a$, if:
 \begin{align*}
   (f(\x_1) -f(\x_2)) \cdot (\eta(\x_1) -\eta(\x_2)) >0, (\x_1,\x_2) \in \mathcal{C}^{\alpha,\beta}_{\eta} \otimes  \mathcal{C}^{\alpha,\beta}_{\eta}.
 \end{align*} 
 \end{frthm9}

 \begin{proof}
   When:
   \begin{equation*}
    \mathcal{C}_f^{\alpha,\beta} =  \mathcal{C}_\eta^{\alpha,\beta},~ \tilde{\mathcal{C}}_f^{\alpha,\beta,\uparrow} =  \tilde{\mathcal{C}}_\eta^{\alpha,\beta,\uparrow},~ \tilde{\mathcal{C}}_f^{\alpha,\beta,\downarrow} =  \tilde{\mathcal{C}}_\eta^{\alpha,\beta,\downarrow}
   \end{equation*}
 we have: ${\mathcal{C}}_f^{\alpha,\beta,\uparrow} \equiv  {\mathcal{C}}_\eta^{\alpha,\beta,\uparrow}, {\mathcal{C}}_f^{\alpha,\beta,\downarrow} \equiv  {\mathcal{C}}_\eta^{\alpha,\beta,\downarrow}$. This way:
 \begin{align}
 2 \int_{\cfabup}  p(\x_1) \cdot p(\x_2) \cdot\eta(\x_2)(1-\eta(\x_1)) d\x_1d\x_2
 + 2 \int_{\cfabdown}  p(\x_1) \cdot p(\x_2) \cdot\eta(\x_1)(1-\eta(\x_2)) d\x_1d\x_2  \\
 \equiv   2 \int_{\mathcal{C}_\eta^{\alpha,\beta,\uparrow}}  p(\x_1) \cdot p(\x_2) \cdot\eta(\x_2)(1-\eta(\x_1)) d\x_1d\x_2
 + 2 \int_{\mathcal{C}_\eta^{\alpha,\beta,\downarrow}}  p(\x_1) \cdot p(\x_2) \cdot\eta(\x_1)(1-\eta(\x_2)) d\x_1d\x_2 
 \end{align}
 \noindent becomes a constant irrelevant with $f$.

 To minimize the Bayes error, it suffices to reach the minimal error on $\mathcal{C}^{\alpha,\beta}_{\eta} \otimes  \mathcal{C}^{\alpha,\beta}_{\eta}$:
 \begin{align*}
 \int_{\x_1,\x_2 \in \mathcal{C}^{\alpha,\beta}_{\eta} \otimes  \mathcal{C}^{\alpha,\beta}_{\eta}} p(\x_1) \cdot p(\x_2) \cdot \min\left\{\eta(\x_1)(1-\eta(\x_2)),\eta(\x_2)(1-\eta(\x_1))\right\}d\x_1d\x_2,
 \end{align*}
 with the optimal solution satisfying:
 \begin{align*}
 (f(\x_1) -f(\x_2)) \cdot (\eta(\x_1) -\eta(\x_2)) >0, (\x_1,\x_2) \in \cfab \otimes \cfab.
 \end{align*}
 \end{proof}

    \noindent\rule[0.15\baselineskip]{\columnwidth}{1pt}

 \section{Proof of Theorem \ref{thm:gen}} \label{sec:4}
 First, we need the following definitions about the population and empirical quantile of the scores:
 \begin{align*}
   &\da = \argmin_{\delta \in \mathbb{R}}\left[\delta \in \mathbb{R}:~ \eP\left[ \ind{\fxp \le \delta} \right] = \alpha  \right], ~\dah = \argmin_{\delta \in \mathbb{R}} \left[\delta \in \mathbb{R}:~ \frac{1}{\np} \sump \left[ \ind{\fxp \le \delta} \right] = \alpha  \right]
 \\
   &\db =\argmin_{\delta \in \mathbb{R}}\left[\delta \in \mathbb{R}:~ \eN\left[ \ind{\fxn \ge \delta} \right] = \beta  \right],
 ~
   \dbh =\argmin_{\delta \in \mathbb{R}} \left[\delta \in \mathbb{R}:~ \frac{1}{\nn} \sumn \left[ \ind{\fxn \ge \delta} \right] = \beta  \right] 
 \end{align*}
 Furthermore, we denote the population and empirical TPAUC risk function as:
 \begin{align*}
   &\aucf = \eN\eP\left[ \ind{\f(\xp) > \f(\xn)} \cdot \ind{\f(\xp) < \da} \cdot \ind{\f(\xn) > \db} \right] \\ 
   &\aucs = \sumterm \cdot  \sumn \sump \ind{\fxnj \ge \fxpi} \cdot \ind{\fxpi \le  \dah} \cdot \ind{\fxnj \ge \dbh}
 \end{align*}

 \begin{frlem1}
   For $\forall f \in \mathcal{F}$, we have:
   \begin{equation*}
     \begin{split}
    \aucf - \aucs \le  2 (\Delta_+ + \Delta_-)
     \end{split}
   \end{equation*}
 where
 \begin{align*}
   \Delta_+  &= \sup_{\delta  \in \mathbb{R}}  \left| \frac{1}{\np} \cdot \sump   \ind{\fxpi \le \delta} -  \eP \left[ \ind{\fxp \le \delta}  \right]
   \right|\\
   \Delta_- & = \sup_{\delta  \in \mathbb{R}}  \left| \sumtermp \cdot \sumn   \ind{\fxnj \ge \delta} -  \eN \left[ \ind{\fxn \ge \delta}  \right]
   \right|
 \end{align*}
 \end{frlem1}
 \noindent\rule[0.15\baselineskip]{\columnwidth}{1pt}
 \begin{proof}
   First, we define some intermediate variables:
   \begin{align*}
     \lpj &= \eP\left[ \ind{\fxp \le \da} \cdot \ind{\fxnj \ge \fxp}  \right]\\
    R_1  &=  \aucf = \eN\eP\left[ \ind{\f(\xn) \ge \f(\xp)} \cdot \ind{\f(\xp) \le \da} \cdot \ind{\f(\xn) \ge \db} \right]\\
    R_2  &= \sumtermp \cdot \sumn \lpj \cdot \ind{\fxnj \ge \db}\\
    R_3  & =\sumtermp \cdot  \sumn \lpj \cdot \ind{\fxnj \ge \dbh}\\
    R_4  & = \sumterm \cdot \sumn \sump \ind{\fxnj \ge \fxpi} \cdot \ind{\fxpi \le  \da} \cdot \ind{\fxnj \ge \dbh}\\ 
    R_5  &  = \aucs = \sumterm \cdot  \sumn \sump \ind{\fxnj \ge \fxpi} \cdot \ind{\fxpi \le  \dah} \cdot \ind{\fxnj \ge \dbh}
   \end{align*}
 In this sense, we can decompose $R_1 -R_5$ as:
 \begin{equation*}
   |R_1 -R_5| \le |R_1 - R_2| + |R_2 - R_3| + |R_3 -R_4| + |R_4 - R_5|
 \end{equation*}
 Now, we bound each term in the equation above successively.
 For $|R_1 - R_2|$, we have:
 \begin{align*}
   |R_1 - R_2| =&~ \bigg|\eP\bigg[\eN \left[ \ind{\f(\xn) \ge \f(\xp)} \cdot \ind{\f(\xp) \le \da} \cdot \ind{\f(\xn) \ge \db} \right]\\
    &-  \sumtermp \sumn  \ind{\fxnj \ge \fxp} \cdot \ind{\fxp \le \da}  \cdot \ind{\fxnj \ge \db} \bigg]\bigg|\\
    \le &~ \bigg|\sup_{\xp}\bigg[\eN \left[ \ind{\f(\xn) \ge \max\{\f(\xp), \db\} } \right]- \sumtermp \cdot \sumn  \ind{\fxnj \ge \max\{\fxpi, \db \}} \bigg]\bigg|
    \\
    \le &~ \sup_{\delta \in \mathbb{R}} \left| \eN \left[ \ind{\f(\xn) \ge \delta } \right] - \sumtermp \cdot \sumn  \ind{\fxnj \ge \delta} \right|
 \end{align*}
 For $|R_2 - R_3|$, we have:
 \begin{align*}
   |R_2 - R_3| =&~ \left| \sumtermp \cdot \sumn \lpj \cdot \ind{\fxnj \ge \db} - \sumtermp \cdot  \sumn \lpj \cdot \ind{\fxnj \ge \dbh} 
   \right|\\
   \overset{(a_1)}{\le} &~ \left| \sumtermp \cdot \sumn   \ind{\fxnj \ge \db} - \sumtermp \cdot  \sumn   \ind{\fxnj \ge \dbh} 
 \right|\\
  \overset{(a_2)}{=} &~  \left| \sumtermp \cdot \sumn   \ind{\fxnj \ge \db} - \beta
  \right| \\
  \overset{(a_3)}{=} &~  \left| \sumtermp \cdot \sumn   \ind{\fxnj \ge \db} -  \eN \left[ \ind{\fxnj \ge \db}  \right]
  \right|\\
 \le&~ \sup_{\delta  \in \mathbb{R}}  \left| \sumtermp \cdot \sumn  \ind{\fxnj \ge \delta} -  \eN \left[ \ind{\fxn \ge \delta}  \right]
 \right|
 \end{align*}
 Here, $(a_1)$ follows from the fact that $\ind{\fxnj \ge \db} - \ind{\fxnj \ge \dbh} $ must be simultaneously $\ge 0$ or $\le 0$; $(a_2)$ and $(a_3)$ are based on the definition of $\db$ and $\dbh$ and the assumption that no tie occurs in the dataset.
  
 For $|R_3 - R_4|$, we have:
 \begin{align*}
   |R_3 - R_4| =&~ \bigg|  \sumtermp \cdot  \sumn \lpj \cdot \ind{\fxnj \ge \dbh}\\
    &- \sumterm \cdot \sumn \sump \ind{\fxnj \ge \fxpi} \cdot \ind{\fxpi \le  \da} \cdot \ind{\fxnj \ge \dbh}
   \bigg|\\
   \le &~ \frac{1}{\nn} \cdot \sumn \left|  \lpj  - \frac{1}{\np} \cdot  \sump \ind{\fxpi \le \min\{\fxnj, \da \}} 
   \right|\\
   \le &~ \sup_{\delta  \in \mathbb{R}}  \left| \frac{1}{\np} \cdot \sump   \ind{\fxpi \le \delta} -  \eP \left[ \ind{\fxp \le \delta}  \right]
   \right|
 \end{align*}

 For $|R_4 - R_5|$, we have:
 \begin{align*}
   |R_4 - R_5| =&~ \bigg|  \sumterm \cdot \sumn \sump \ind{\fxnj \ge \fxpi} \cdot \ind{\fxpi \le  \da} \cdot \ind{\fxnj \ge \dbh} \\
   &- \sumterm \cdot  \sumn \sump \ind{\fxnj \ge \fxpi} \cdot \ind{\fxpi \le  \dah} \cdot \ind{\fxnj \ge \dbh}
   \bigg|\\
   \le&~  \frac{1}{\nn} \cdot \left(\sumn \left|\ind{\fxnj \ge \dbh}\right| \cdot \left|\frac{1}{\np} \cdot \sump \ind{\fxnj \ge \fxpi}  \cdot   \left( \ind{\fxpi \le  \da} -  \ind{\fxpi \le  \dah}   \right)\right|\right)\\
   \le &~  \frac{1}{\np}\sup_{\xn} \left[\left|\sump \ind{\fxn \ge \fxpi}  \cdot   \left( \ind{\fxpi \le  \da} -  \ind{\fxpi \le  \dah}   \right)\right|\right]\\
   \overset{(b_1)}{\le} &~ \frac{1}{\np} \left|\sump  \left( \ind{\fxpi \le  \da} -  \ind{\fxpi \le  \dah}   \right)\right|\\
   \overset{(b_2)}{\le} &~ \sup_{\delta  \in \mathbb{R}}  \left| \frac{1}{\np} \cdot \sump   \ind{\fxpi \le \delta} -  \eP \left[ \ind{\fxp \le \delta}  \right]
   \right|
 \end{align*}
 Here $(b_1)$ and $(b_2)$ follow a similar argument to $(a_1) \text{-} (a_3)$.
 \end{proof}
 \noindent\rule[0.15\baselineskip]{\columnwidth}{1pt}

 \begin{frthm6}
   Assume that there are no ties in the datasets, and the surrogate loss function $\ell$ with range $[0,1]$, is an upper bound of the $0\text{-}1$ loss, then, for all $\f \in \mathcal{F}$, and all $(\alpha, \beta) \in \mathcal{I}_{suff}(\mathcal{S})$, the following inequality holds with probability at least $1 - \delta$ over the choice of $\mathcal{S}$:
   \begin{align*}
     \aucf \le \rpsi(\f, \mathcal{S}) + C\left( \sqrt{\frac{\VC \cdot \log(\np) + \log(1/\delta)}{\np}} + \sqrt{\frac{\VC \cdot \log(\nn) + \log(1/\delta)}{\nn}}   \right),
   \end{align*}
   where $\VC$ is the VC dimension of the hypothesis class: 
   \begin{align*}
     \mathcal{T}(\mathcal{F}) \triangleq \{\mathsf{sign}(\f(\cdot) - \delta): \f \in \mathcal{F}, ~ \delta \in \mathbb{R}\}
   \end{align*}
   and
   \begin{align*}
     \mathcal{I}_{suff}(\mathcal{S}) =  \left\{(\alpha,\beta): \alpha \in (0,1),~ \beta \in(0,1),~ \npa \in \mathbb{N}_+,~ \nnb \in \mathbb{N}_+, \text{condition (a) in Prop.\ref{prop:concon} holds} \right\},
    \end{align*}
  
 \end{frthm6}
 \noindent\rule[0.15\baselineskip]{\columnwidth}{1pt}
 \begin{proof}

   First, we have:
   \begin{align*}
     &\prob \left[\sup_{f\in \mathcal{F}, (\alpha, \beta) \in  \mathcal{I}_{suff}(\mathcal{S}) }\left[|\aucf - \aucs |\right] > \epsilon \right]  \\
     \le&~  \prob\left[\sup_{f\in \mathcal{F}, (\alpha, \beta) \in  \mathcal{I}_{suff}(\mathcal{S}), \delta \in \mathbb{R}} \left[\Delta_+\right] > \epsilon/4\right] +  \prob \left[\sup_{f\in \mathcal{F},(\alpha, \beta) \in  \mathcal{I}_{suff}(\mathcal{S}), \delta \in \mathbb{R}} \left[\Delta_-\right] > \epsilon/4\right]\\
     =&~ \prob\left[\sup_{f\in \mathcal{F}, \delta \in \mathbb{R}} \left[\Delta_+\right] > \epsilon/4\right]+  \prob\left[\sup_{f\in \mathcal{F}, \delta \in \mathbb{R}} \left[\Delta_-\right] > \epsilon/4 \right]
   \end{align*}
   Following Lem.1 in \cite{partial3}, we have that, for all $\f \in \mathcal{T}(\mathcal{F})$, and all $\alpha, \beta \in (0,1)$ s.t. $\npa \in \mathbb{N}_+,~ \nnb \in \mathbb{N}_+$, the following inequality holds with probability at least $1 - \delta$:
   \begin{equation*}
     \aucf \le \aucs + C\left( \sqrt{\frac{\VC \cdot \log(\np) + \log(1/\delta)}{\np}} + \sqrt{\frac{\VC \cdot \log(\nn) + \log(1/\delta)}{\nn}}   \right).
   \end{equation*}

   Since $\alpha,~\beta \in  \mathcal{I}_{suff}(\mathcal{S})$, $\aucs \le \rorg(\f, \mathcal{S})\le \rpsi(\f, \mathcal{S})$, we have the following inequality holds with probability at least $1 - \delta$ under the same condition:
 \begin{equation*}
   \aucf \le \rpsi(\f, \mathcal{S}) + C\left( \sqrt{\frac{\VC \cdot \log(\np) + \log(1/\delta)}{\np}} + \sqrt{\frac{\VC \cdot \log(\nn) + \log(1/\delta)}{\nn}}   \right).
 \end{equation*}

 \end{proof}
 \noindent\rule[0.15\baselineskip]{\columnwidth}{1pt}
 \section{Proof of Theorem \ref{thm:abs}} \label{sec:5}

 To avoid using the graph coloring technique adopted in \cite{entropy}, we in turn adopt an error decomposition scheme for the  proof of Thm.\ref{thm:abs}. To do this, we need the following two hypothesis classes:
 \begin{align*}
   &\egf = \left\{g: g(\x) = \expe_{\xp}[g_f(\xp,\x)], f \in \mathcal{F}  \right\}\\  
   &\gfi = \left\{g: g(\x) = g_f(\xp,\x), f \in \mathcal{F}  \right\}, ~~ \text{with} ~\xpi ~\text{fixed}
 \end{align*}
 where
 \begin{align*}
   \gf(\xp,\xn) = \psi_\gamma(1-\f(\xp)) \psi_\gamma(\f(\xn)) \ell(\f(\xp) - \f(\xn))).
 \end{align*}
 Then we have the following error decomposition:

 \begin{flem}
   The following error decomposition inequality holds:
 \begin{align*}
   &\sup_{f \in \mathcal{F}} \left[ \expe_{\xp,\xn} \left[ \gf(\xp,\xn)   \right]] - \frac{1}{\np\nn}\sumpn  \frac{K^2}{(K-1)^2}\gf(\xpi,\xnj) \right]\\  
   \le& 
    \sup_{f \in \mathcal{F}} \left[ \expe_{\xp,\xn} \left[ \gf(\xp,\xn)   \right]- \frac{K}{(K-1)} \frac{1}{\np} \sump \expe_{\xn} \gf(\xpi,\xn) \right] \\
    +&   \sup_{f \in \mathcal{F}} 
   \left[ \frac{K}{(K-1)} \frac{1}{\np} \sump \expe_{\xn} \gf(\xpi,\xn) - \frac{1}{\np\nn} \sumpn  \frac{K^2}{(K-1)^2}\gf(\xpi,\xnj)  \right] 
 \end{align*}
 \end{flem}
 \begin{proof}
   The lemma directly follows from the following inequality:
   \begin{align*}
     & \expe_{\xp,\xn} \left[ \gf(\xp,\xn)   \right] -\frac{1}{\np\nn} \sumpn  \frac{K^2}{(K-1)^2}\gf(\xpi,\xnj)\\ 
 \le&  \expe_{\xp,\xn} \left[ \gf(\xp,\xn)   \right]- \frac{K}{(K-1)} \frac{1}{\np} \sump \expe_{\xn} \gf(\xpi,\xn) \\ 
  &+    \frac{K}{(K-1)} \frac{1}{\np} \sump \expe_{\xn} \gf(\xpi,\xn) - \frac{1}{\np\nn}\sumpn  \frac{K^2}{(K-1)^2}\gf(\xpi,\xnj)   
   \end{align*}
 \end{proof}

 \begin{flem}\label{lem:cover_1}
   Suppose that the weighting function $v(\cdot)$ is $L_v$-Lipschitz continuous, and the loss function $\ell$ is $L_\ell$-Lipschitz continuous, then the following inequality holds given the definition of $\vinf, \linf$:
   \begin{equation*}
     \begin{split}
       \mathcal{N}\bigg(&\egf, \epsilon, ||\cdot||_{2,\np}  \bigg) \le  \mathcal{N}\left(\mathcal{F}, \frac{\epsilon}{2(\vinf\cdot \linf \cdot L_v + \vinf^2 \cdot L_\ell )}, ||\cdot||_{\infty}  \right).
     \end{split}
   \end{equation*}
 \end{flem}

 \begin{pf}
   \begin{align*}
     \ell_{i} &= \ell(\f(\xpi) - \f(\x_-)), &\tilde{\ell}_{i} &= \ell(\tlf(\xpi) - \tlf(\x_-))\\ 
     v_i &= \psi_\gamma\left(1-\f(\xpi)\right), &v_- &= \psi_\gamma\left(\f(\x)\right),\\
     \tilde{v}_i &= \psi_\gamma\left(1-\tlf(\xpi)\right), &\tilde{v}_- &= \psi_\gamma\left(\tlf(\x_-)\right)
   \end{align*}
   \noindent Given any 
   \[g  =   \gf \in  \egf, \tilde{g} =  g_{\tilde{f}}  \in \egf :\]

 
   \begin{align*}
  &\left\|\expe_{\x_-}g-\expe_{\x_-}\tilde{g}\right\|_{2,n_+} =  \sqrt{\frac{1}{\np}\sump \left(\expe_{\xn}\left[  v_i \cdot v_- \cdot \ell_{i}\right] - \expe_{\xn}\left[\tilde{v}_i \cdot \tilde{v}_- \cdot \tilde{\ell}_{i} \right] \right)^2 }\\ 
  &\le~ \max_{(\xpi,\x_-)} |  v_i \cdot v_- \cdot \ell_{i} - \tilde{v}_i \cdot \tilde{v}_- \cdot \tilde{\ell}_{i} |\\ 
  &\le~ \max_{(\xpi,\x_-)}  | {v}_i \cdot  v_- \cdot  \ell_{i} -  \tilde{v}_i \cdot  {v}_- \cdot  \ell_{i}|  + | \tilde{v}_i \cdot  {v}_- \cdot  \ell_{i} -  \tilde{v}_i \cdot  \tilde{v}_- \cdot  {\ell}_{i}|+~ | \tilde{v}_i \cdot  \tilde{v}_- \cdot  \ell_{i} -  \tilde{v}_i \cdot  \tilde{v}_- \cdot  \tilde{\ell}_{i}|\\ 
  &\le 2\cdot (\vinf^2 \cdot L_\ell + \vinf \cdot \linf \cdot L_v) \cdot ||f-\tilde{f}||_\infty.
 \end{align*}

 \noindent Define a $\frac{\epsilon}{2(\vinf\cdot \linf \cdot L_v + \vinf^2 \cdot L_\ell )}$-covering of the class $\mathcal{F}$ with $ ||\cdot||_{\infty}$ norm: 
 \[\left\{\mathcal{C}_1,\cdots, \mathcal{C}_{N}\right\}.\]

 with 

 \[N =  \mathcal{N}\left(\mathcal{F}, \frac{\epsilon}{2(\vinf\cdot \linf \cdot L_v + \vinf^2 \cdot L_\ell )}, ||\cdot||_{\infty}  \right)\]

 There exists a ${f}_{\mathcal{C}_i} \in \mathcal{F}$, such that for any $\f \in \mathcal{C}_i \cap \mathcal{F}$: 
 \begin{align*}
   ||\f - f_{\mathcal{C}_i}||_\infty \le \frac{\epsilon}{2(\vinf\cdot \linf \cdot L_v + \vinf^2 \cdot L_\ell )}.
 \end{align*}
 which implies that 
 \begin{align*}
   ||\expe_{x-}\f - \expe_{x-}g_{f_{\mathcal{C}_i}}||_{2,\np} \le 2(\vinf\cdot \linf \cdot L_v + \vinf^2 \cdot L_\ell ) \cdot \frac{\epsilon}{2(\vinf\cdot \linf \cdot L_v + \vinf^2 \cdot L_\ell )} = \epsilon. 
 \end{align*}
 Denote 
 \[\mathcal{C}_{g,i} = \left\{g_f: ||g_f(\x,\xn) - g_{f_{\mathcal{C}_i}}||_{\infty} \le \frac{\epsilon}{2(\vinf\cdot \linf \cdot L_v + \vinf^2 \cdot L_\ell )}, f \in \mathcal{F} \right\}, \] 
 then $\left\{\mathcal{C}_{g,1},\cdots,\mathcal{C}_{g,N}\right\}$ realizes an $\epsilon$-covering of $E(g\mathcal{F})$. Hence, the minimum size of the $\epsilon$-covering is at most $N$. Mathematically, we then have:
 \begin{align*}
       \mathcal{N}\bigg(\egf, \epsilon, ||\cdot||_{2,\np}  \bigg) \le  \mathcal{N}\left(\mathcal{F}, \frac{\epsilon}{2(\vinf\cdot \linf \cdot L_v + \vinf^2 \cdot L_\ell )}, ||\cdot||_{\infty}  \right).
 \end{align*}
 \qed
 \end{pf}

 \begin{flem}\label{lem:cover_2}
   Suppose that the weighting function $v(\cdot)$ is $L_v$-Lipschitz continuous, and the loss function $\ell$ is $L_\ell$-Lipschitz continuous, then the following inequality holds given the definition of $\vinf, \linf$:
   \begin{equation*}
     \begin{split}
       \mathcal{N}\bigg(&\gfi, \epsilon, ||\cdot||_{2,\nn}  \bigg) \le  \mathcal{N}\left(\mathcal{F}, \frac{\epsilon}{2(\vinf\cdot \linf \cdot L_v + \vinf^2 \cdot L_\ell )}, ||\cdot||_{\infty}  \right).
     \end{split}
   \end{equation*}
 \end{flem}
 \begin{pf}
   The proof follows a similar argument as Lem.\ref{lem:cover_1}, which is thus omitted.
 \end{pf}

 \begin{flem}\label{lem:prob} For $\epsilon_i \in \mathbb{R}, \forall i$, if  $\sum_i\epsilon_i = \epsilon$, then:
   \begin{align*}
     \prob\left[ \sum_i x_i \ge \epsilon  \right] \le \sum_i  \prob\left[ x_i \ge \epsilon_i  \right]. 
   \end{align*}
 \end{flem}
 \begin{proof}
   Denote event $ \mathcal{A}_i(\epsilon_i)$ as:
   \begin{align*}
     \mathcal{A}_i(\epsilon_i) = \{x_i: x_i \ge \epsilon_i\}. 
   \end{align*}
 and denote event $\mathcal{A}$ as:
 \begin{align*}
   \mathcal{A}(\epsilon) = \{(x_1,x_2,\cdots x_n): \sum_i x_i \ge \epsilon\}. 
 \end{align*}
  We first show that
   \[\mathcal{A}(\epsilon) \subseteq \bigcup_i \mathcal{A}_i(\epsilon_i), ~\sum_i \epsilon_i  =\epsilon.\]
 Assume that:
  \[\mathcal{A}(\epsilon) \nsubseteq  \bigcup_i \mathcal{A}_i(\epsilon_i),\]
 then there exist $x_1,x_2,\cdots x_n$ such that:
 \begin{align*}
   \sum_i x_i \ge \epsilon, ~\text{and}~, x_i < \epsilon_i, \forall i,
 \end{align*}
 which contradicts the fact that $\sum_i\epsilon_i = \epsilon$. Hence, we have:
 \begin{align*}
   \prob\bigg[ \sum_i x_i \ge \epsilon  \bigg]   = \prob[\mathcal{A}] \le  \prob\bigg[\bigcup_i \mathcal{A}_i(\epsilon_i)\bigg] \le \sum_i \prob\bigg[ \mathcal{A}_i(\epsilon_i)\bigg] = \sum_i  \prob\bigg[ x_i \ge \epsilon_i  \bigg] .
 \end{align*}
 \end{proof}

 \begin{frthm7}
   Assume that the weighting function $\psi$ is $L_v$ Lipschitz continuous, the loss function $\ell$ is $L_\ell$ Lipschitz continuous with  $||\psi||_\infty  = v_\infty$, $||\ell||_\infty = L_\infty$, and the covering number of $\mathcal{F}$ has the following form with respect to $|\cdot|_\infty$ norm:
   \begin{align*}
     \log(\mathcal{N}\bigg(\mathcal{F}, \epsilon, ||\cdot||_\infty  \bigg)) \le A \log\bigg(R/\epsilon\bigg),
   \end{align*}
   then for all $f \in \mathcal{F}$ and $(\alpha,\beta) \in \mathcal{I}^2_{suff}$,  the following inequality holds with probability at least $1-\delta$:
   \begin{equation*}
     \mathcal{R}^{\alpha,\beta}(f, \mathcal{S}) - \frac{K^2}{(K-1)^2}\hat{\mathcal{R}}^{\ell}_\psi(f,\mathcal{S}) \le  C \cdot\frac{A\log\left(\Gamma\nn\right) + \log(2/\delta)}{\nn},~~ \forall K >1
   \end{equation*}
 where $C, \Gamma$ are universal constants depending on $v_\infty, \ell_\infty, L_\ell, L_v, K, R$, and
   \begin{align*}
   \mathcal{I}^2_{suff}(\mathcal{S}) =  \bigg\{ &(\alpha,\beta): \alpha \in (0,1),~ \beta \in(0,1),~ \npa \in \mathbb{N}_+,~ \\  
   &\nnb \in \mathbb{N}_+, \text{condition (b) in Prop.\ref{prop:concon} holds} \bigg\}.
 \end{align*}
 \end{frthm7}
 \begin{proof}
 For the sake of convenience, we first define some intermediate terms:
 \begin{align*}
  & \Delta_1 =  \sup_{f \in \mathcal{F}} \left[ \expe_{\xp,\xn} \left[ \gf(\xp,\xn)   \right]- \frac{K}{(K-1)} \frac{1}{\np} \sump \expe_{\xn} \gf(\xpi,\xn) \right] \\
  & \Delta_2 =  \sup_{f \in \mathcal{F}} 
  \left[ \frac{K}{(K-1)} \frac{1}{\np} \sump \expe_{\xn} \gf(\xpi,\xn) - \frac{1}{\np\nn} \sumpn  \frac{K^2}{(K-1)^2}\gf(\xpi,\xnj)  \right] \\
  & \Delta_2(\xpi) =  \sup_{f \in \mathcal{F}} 
  \left[ \frac{K}{(K-1)} \expe_{\xn} \gf(\xpi,\xn) - \frac{1}{\nn} \sumn  \frac{K^2}{(K-1)^2}\gf(\xpi,\xnj)  \right]
 \end{align*}
 It is easy to check that: 
 \begin{align*}
   \Delta_2  \le \frac{1}{\np} \sump \Delta_2(\xpi).
 \end{align*}
 Moreover, we have the following smoothness property:
 \begin{align*}
   &\sup_{f \in \egf } (\expe_{\xp}(f)/Var_{\xp}(f)) \le v^2_\infty L_\infty \\ 
   & \sup_{f \in \gfi } (\expe_{\xn}(f)/Var_{\xn}(f)) \le v^2_\infty L_\infty, \forall i=1,2,\cdots, \np \\ 
 \end{align*}
 For $\Delta_1$, by combining lem.\ref{lem:cover_1},  a similar proof of Thm.2 in \cite{localchain} (by replacing the $||\cdot||_2$ for covering number with $||\cdot||_\infty$ norm), and Thm.4 in \cite{localchain}, the following inequality holds with probability at least $1-\delta/2$:
 \begin{align*}
   \Delta_1 \le c_1\left[\frac{A\log \Gamma_1 \np + \log(2/\delta)}{\np}  \right]
 \end{align*}
 Fix $\xpi$, similarly, the following inequality holds with probability at least $1-\delta/2$:
 \begin{align*}
   \Delta_2(\xpi) \le c_{2,i}\left[\frac{A\log (\Gamma_{2,i}\nn) + \log(2/\delta)}{\nn}  \right]
 \end{align*}
 or equivalently: 
 \begin{align*}
    \prob\left[ \Delta_2(\xpi) \ge \epsilon_i \right] \le  \exp\left( A\log(\Gamma_{2,i} \nn) - \frac{\nn \epsilon_i}{c_{2,i}}   \right)
 \end{align*}

 \noindent According to Lem.\ref{lem:prob}, we have:
 \begin{align*}
   \prob\left[ \Delta_2 \ge \epsilon \right]  \le  \sump \prob\left[ \Delta_2(\xpi) \ge \epsilon_i \right] \le \sump \exp\left(A\log(\Gamma_{2,i} \nn) - \frac{\nn \epsilon_i}{c_{2,i}}  \right)
 \end{align*}
 as long as $\frac{1}{\np}\sump \epsilon_i = \epsilon$. 
 By choosing:
 \begin{align*}
   \epsilon_i = c_{2,i} \frac{A\log(\Gamma_{2,i}\nn) + \log(2/\delta) + \log(\np)}{\nn} 
 \end{align*}
 we have: 
 \begin{align*}
   \prob\left[ \Delta_2 \ge \epsilon \right]  \le  \sump \frac{\delta/2}{\np}  = \delta/2
 \end{align*}
 Hence, with probability at least $1-\delta/2$ over the random drawn of $\mathcal{S}$, the following inequality holds:
 \begin{align*}
   \Delta_2 \le c_2 \cdot \frac{A\log(\Gamma_2\nn) + \log(2/\delta) + \log(\np)}{\nn} 
 \end{align*}
 where $c_2 = \max_i\{c_{2,i}\}, \Gamma_2 = \max_i\{\Gamma_{2,i}\}$.
 Finally, by Lem.\ref{lem:errde} and the union bound,  the following inequality holds for all $f \in \mathcal{F}$ with probability at least $1-\delta$ over the random drawn of $\mathcal{S}$: 
 \begin{align*}
   &\expe_{\xp, \xn}\left[ \gf(\xp,\xn)   \right]- \frac{1}{\np\nn} \sumpn  \frac{K^2}{(K-1)^2}\gf(\xpi,\xnj) \le C \cdot \frac{A\log\left(\Gamma\max\{\np,\nn\}\right) + \log(2/\delta)}{\min\{\np,\nn\}} \\ 
   \le& C \cdot \frac{A\log\left(\Gamma\nn\right) + \log(2/\delta)}{\np}
 \end{align*}
 where $C = 2 \max\{c_1,c_2\}, \Gamma =\max\{\Gamma_1,\Gamma_2\}$.
 The claim of the theorem then follows that 
 \begin{align*}
   & \hat{\mathcal{R}}^{\ell}_\psi(f,\mathcal{S}) = \frac{1}{\np\nn} \sumpn  \gf(\xpi,\xnj)\\ 
   & \mathcal{R}^{\alpha,\beta}(f, \mathcal{S}) \le \expe_{\xp,\xn}[\hat{\mathcal{R}}^{\ell}_\psi(f,\mathcal{S})] =   \expe_{\xp, \xn}\left[ \gf(\xp,\xn)   \right]
 \end{align*}
 \end{proof}

    \noindent\rule[0.15\baselineskip]{\columnwidth}{1pt}

 \section{Experiments}\label{sec:app_exp}

   \subsection{Dataset Description}

   \begin{table*}[ht]
     \centering
     \setlength{\abovecaptionskip}{0pt}    
     \setlength{\belowcaptionskip}{15pt}    
     \caption{Details on the datasets.}
     \setlength{\tabcolsep}{5pt}
     \resizebox{0.95\textwidth}{!}{%
      \begin{tabular}{lllll}
       \toprule
       Dataset  & Pos. Class ID & Pos. Class Name & \# Pos. Examples & \# Neg. Examples \\
       \toprule
       CIFAR-10-LT-1  & $2$   & birds         &$1,508$  & $8,907$ \\
       CIFAR-10-LT-2  & $1$   & automobiles   &$2,517$  & $7,898$ \\
       CIFAR-10-LT-3  & $3$   & cats          &$904$  & $9,511$ \\
       \midrule
       CIFAR-100-LT-1 & $6,7,14,18,24$   & insects              &$1,928$  & $13,218$ \\
       CIFAR-100-LT-2 & $0,51,53,57,83$  & fruits and vegetables &$885 $ & $14,261$ \\
       CIFAR-100-LT-3 & $15,19,21,32,38$ & large omnivores and herbivores &$1,172$ & $13,974$ \\
       \midrule
       Tiny-ImageNet-200-LT-1  & $24,25,26,27,28,29$   & dogs   &$2,100$  & $67,900$ \\
       Tiny-ImageNet-200-LT-2  & $11, 20, 21, 22$   & birds   &$1,400$  & $68,600$ \\
       Tiny-ImageNet-200-LT-3  & $70, 81, 94, 107, 111, 116, 121, 133, 145, 153, 164, 166$   & vehicles   &$4,200$  & $65,800$ \\
       \bottomrule
      \end{tabular}
     }
     \label{tab:dataset1}
    \end{table*}

   \textbf{Binary CIFAR-10-LT Dataset}. The original CIFAR-10 dataset consists of 60,000 $32\times 32$ color images in 10 classes, with 6,000 images per class. There are 50,000 and 10,000 images in the training set and the test set, respectively. We create a long-tailed CIFAR-10 where the sample sizes across different classes decay exponentially, and the ratio of sample sizes of the least frequent to the most frequent class $\rho$ is set to 0.01. We then create binary long-tailed datasets based on CIFAR-10-LT by selecting one category as positive examples and the others as negative examples. We construct three binary subsets, in which the positive categories are  \textbf{1) }birds, \textbf{2) } automobiles, and \textbf{3) }cats, respectively. The datasets are split into training, validation and test sets according to the ratio of $0.7:0.15:0.15$. More details are provided in Tab. \ref{tab:dataset}.

   \textbf{Binary CIFAR-100-LT Dataset}. The original CIFAR-100 dataset is similar to CIFAR-10, except it has 100 classes with each containing 600 images. The 100 classes in the CIFAR-100 are grouped into 20 superclasses. We create CIFAR-100-LT in the same way as CIFAR-10-LT, and transform it into three binary long-tailed datasets by selecting a superclass as positive class examples each time. Specifically, the positive superclasses are \textbf{1) }fruits and vegetables, \textbf{2) }insects and \textbf{3) }large omnivores and herbivores, respectively. More details are provided in Tab. \ref{tab:dataset}.

   \textbf{Implementation details On CIFAR Datasets}. We utilize the ResNet-20 \cite{resnet} as the backbone, which takes images with size $32\times 32 \times 3$ as input and outputs 64-d features. Then the features are mapped into $[0, 1]$ with an FC layer and Sigmoid function. During the training phase, we apply data augmentation including random horizontal flipping (50\%), random rotation (from $-15^\circ$ to $15^\circ$) and random cropping ($32\times 32$).

   \textbf{Binary Tiny-ImageNet-200-LT Dataset}. The Tiny-ImageNet-200 dataset contains 100,000 $256\times 256$ color images  from 200 different categories, with 500 images per category. Similar to the CIFAR-100-LT dataset, we choose three positive superclasses to construct binary subsets: \textbf{1) }dogs, \textbf{2) }birds and \textbf{3) }vehicles. The datasets are further split into training, validation and test sets according to the ratio of $0.7:0.15:0.15$. See Tab. \ref{tab:dataset} for more details.
  
   \textbf{Implementation details on Tiny-ImageNet-200 }. The implementation details are basically the same with CIFAR-10-LT and CIFAR-100-LT datasets, except the backbone network is implemented with ResNet-18 \cite{resnet}, which takes images with size $224\times 224 \times 3$  as input and outputs 512-d features.

   \subsection{Best Parameters}

   \begin{table}[h]
     \centering
     \caption{Best Parameters on CIFAR-10-LT}
  
     \begin{tabular}{ccccccccc}
       \toprule
       Loss  & subset & \multicolumn{1}{c}{$\eta_1$} & \multicolumn{1}{c}{$\eta_2$} & \multicolumn{1}{c}{$\eta_3$} & \multicolumn{1}{c}{$\alpha$} & \multicolumn{1}{c}{$\beta$} & \multicolumn{1}{c}{$(\gamma-1)^{-1}$} & \multicolumn{1}{c}{$E_k$} \\
       \midrule
      \texttt{Exp}   & $subset_1$ & $0.001$ & $0.001$ & $0.001$ & $0.3$   & $0.3$   & $25$    & $35$ \\
      \texttt{Exp}   & $subset_1$ & $0.001$ & $0.001$ & $0.001$ & $0.4$   & $0.4$   & $10$    & $35$ \\
      \texttt{Exp}   & $subset_1$ & $0.001$ & $0.001$ & $0.001$ & $0.5$   & $0.5$   & $20$    & $30$ \\
      \texttt{Exp}   & $subset_2$ & $0.001$ & $0.001$ & $0.001$ & $0.3$   & $0.3$   & $8$    & $5$ \\
      \texttt{Exp}   & $subset_2$ & $0.001$ & $0.001$ & $0.001$ & $0.4$   & $0.4$   & $35$    & $5$ \\
      \texttt{Exp}   & $subset_2$ & $0.001$ & $0.001$ & $0.001$ & $0.5$   & $0.5$   & $30$    & $5$ \\
      \texttt{Exp}   & $subset_3$ & $0.001$ & $0.001$ & $0.001$ & $0.3$   & $0.3$   & $30$    & $35$ \\
      \texttt{Exp}   & $subset_3$ & $0.001$ & $0.001$ & $0.001$ & $0.4$   & $0.4$   & $35$    & $30$ \\
      \texttt{Exp}   & $subset_3$ & $0.001$ & $0.001$ & $0.001$ & $0.5$   & $0.5$   & $20$    & $35$ \\
      \texttt{Poly} & $subset_1$ & $0.001$ & $0.001$ & $0.001$ & $0.3$   & $0.3$   & $0.1$       & $0 $\\
      \texttt{Poly} & $subset_1$ & $0.001$ & $0.001$ & $0.001$ & $0.4$   & $0.4$   & $0.01$ & $0$ \\
      \texttt{Poly} & $subset_1$ & $0.001$ & $0.001$ & $0.001$ & $0.5$   & $0.5$   & $0.03$ & $8$ \\
      \texttt{Poly} & $subset_2$ & $0.001$ & $0.001$ & $0.001$ & $0.3$   & $0.3$   & $0.01$ & $5$ \\
      \texttt{Poly} & $subset_2$ & $0.001$ & $0.001$ & $0.001$ & $0.4$   & $0.4$   & $0.03$ & $5$ \\
      \texttt{Poly} & $subset_2$ & $0.001$ & $0.001$ & $0.001$ & $0.5$   & $0.5$   & $0.03$ & $5$ \\
      \texttt{Poly} & $subset_3$ & $0.001$ & $0.001$ & $0.001$ & $0.3$   & $0.3$   & $0.05  $ & $30$ \\
      \texttt{Poly} & $subset_3$ & $0.001$ & $0.001$ & $0.001$ & $0.4$   & $0.4$   & $0.1$ & $35$ \\
      \texttt{Poly} & $subset_3$ & $0.001$ & $0.001$ & $0.001$ & $0.5$   & $0.5$   & $0.08$    & $10$ \\
       \bottomrule
       \end{tabular}%
     \label{tab:best_1}%
   \end{table}%
  
   \begin{table}[h]
   \centering
   \caption{Best Parameters on CIFAR-100-LT}
     \begin{tabular}{ccccccccc}
     \toprule
     Loss  & subset & \multicolumn{1}{c}{$\eta_1$} & \multicolumn{1}{c}{$\eta_2$} & \multicolumn{1}{c}{$\eta_3$}& \multicolumn{1}{c}{$\alpha$} & \multicolumn{1}{c}{$\beta$} & \multicolumn{1}{c}{$(\gamma-1)^{-1}$} & \multicolumn{1}{c}{$E_k$} \\
     \midrule
    \texttt{Exp}   & $subset_1$ & $0.001$ & $0.001$  & $0.001$ & $0.3$   & $0.3$   & $15$  & $35$\\
    \texttt{Exp}   & $subset_1$ & $0.001$ & $0.001$  & $0.001$ & $0.4$   & $0.4$   & $8$  & $35$ \\
    \texttt{Exp}   & $subset_1$ & $0.001$ & $0.001$  & $0.001$ & $0.5$   & $0.5$   & $8$  & $35$\\
    \texttt{Exp}   & $subset_2$ & $0.001$ & $0.001$  & $0.001$ & $0.3$   & $0.3$   & $15$  & $35$ \\
    \texttt{Exp}   & $subset_2$ & $0.001$ & $0.001$  & $0.001$ & $0.4$   & $0.4$   & $10$  & $35$ \\
    \texttt{Exp}   & $subset_2$ & $0.001$ & $0.001$  & $0.001$ & $0.5$   & $0.5$   & $15$  & $30$ \\
    \texttt{Exp}   & $subset_3$ & $0.001$ & $0.001$  & $0.001$ & $0.3$   & $0.3$   & $15$  & $35$\\
    \texttt{Exp}   & $subset_3$ & $0.001$ & $0.001$  & $0.001$ & $0.4$   & $0.4$   & $15$  & $35$\\
    \texttt{Exp}   & $subset_3$ & $0.001$ & $0.001$  & $0.001$ & $0.5$   & $0.5$   & $10$  & $35$\\
    \texttt{Poly} & $subset_1$ & $0.001$ & $0.001$  & $0.001$ & $0.3$   & $0.3$   & $0.01$ &  $15$ \\
    \texttt{Poly} & $subset_1$ & $0.001$ & $0.001$  & $0.001$ & $0.4$   & $0.4$   & $0.01$  & $3$\\
    \texttt{Poly} & $subset_1$ & $0.001$ & $0.001$  & $0.001$ & $0.5$   & $0.5$   & $0.01$  & $5$ \\
    \texttt{Poly} & $subset_2$ & $0.001$ & $0.001$  & $0.001$ & $0.3$   & $0.3$   & $0.01$ &  $35$ \\
    \texttt{Poly} & $subset_2$ & $0.001$ & $0.001$  & $0.001$ & $0.4$   & $0.4$   & $0.01$  & $20$ \\
    \texttt{Poly} & $subset_2$ & $0.001$ & $0.001$  & $0.001$ & $0.5$   & $0.5$   & $0.03$  & $0$\\
    \texttt{Poly} & $subset_3$ & $0.001$ & $0.001$  & $0.001$ & $0.3$   & $0.3$   & $0.03$ &  $20$\\
    \texttt{Poly} & $subset_3$ & $0.001$ & $0.001$  & $0.001$ & $0.4$   & $0.4$   & $0.1$  & $35$ \\
    \texttt{Poly} & $subset_3$ & $0.001$ & $0.001$  & $0.001$ & $0.5$   & $0.5$   & $0.01$  & $25$ \\
     \bottomrule
     \end{tabular}%
   \label{tab:best_2}%
   \end{table}%
   \begin{table}[h]
     \centering
     \caption{Best Parameters on Tiny-Imagenet-200-LT}
       \begin{tabular}{ccccccccc}
       \toprule
       Loss  & subsets & \multicolumn{1}{c}{$\eta_1$} & \multicolumn{1}{c}{$\eta_2$} & \multicolumn{1}{c}{$\eta_3$} & \multicolumn{1}{c}{$\alpha$} & \multicolumn{1}{c}{$\beta$} & \multicolumn{1}{c}{$(\gamma-1)^{-1}$} & \multicolumn{1}{c}{$E_k$} \\
       \midrule
      \texttt{Exp}  & $subset_1$& $0.001$ & $0.001$ & $0.001$ & $0.3$   & $0.3$   & $30$   & $15$\\
      \texttt{Exp}  & $subset_1$& $0.001$ & $0.001$ & $0.001$ & $0.4$   & $0.4$   & $20 $   & $3 $\\
      \texttt{Exp}  & $subset_1$& $0.001$ & $0.001$ & $0.001$ & $0.5$   & $0.5$   & $20 $   & $3 $\\
      \texttt{Exp}  & $subset_2$& $0.001$ & $0.001$ & $0.001$ & $0.3$   & $0.3$   & $35 $   & $15$ \\
      \texttt{Exp}  & $subset_2$& $0.001$ & $0.001$ & $0.001$ & $0.4$   & $0.4$   & $15 $   & $12$ \\
      \texttt{Exp}  & $subset_2$& $0.001$ & $0.001$ & $0.001$ & $0.5$   & $0.5$   & $20 $   & $12$ \\
      \texttt{Exp}  & $subset_3$& $0.001$ & $0.001$ & $0.001$ & $0.3$   & $0.3$   & $30 $   & $12$ \\
      \texttt{Exp}  & $subset_3$& $0.001$ & $0.001$ & $0.001$ & $0.4$   & $0.4$   & $8 $   & $12$ \\
      \texttt{Exp}  & $subset_3$& $0.001$ & $0.001$ & $0.001$ & $0.5$   & $0.5$   & $8 $   & $15$\\
      \texttt{Poly}  &$ subset_1$ &$ 0.001$ & $0.001$ & $0.001$ & $0.3$   & $0.3$   &$0.03 $    &$0$ \\
      \texttt{Poly}  &$ subset_1$ & $0.001$ & $0.001$ & $0.001$ & $0.4$   & $0.4$   & $0.03 $ &$0$\\
      \texttt{Poly}  &$ subset_1$ & $0.001$ & $0.001$ & $0.001$ & $0.5$   & $0.5$   & $0.01 $ &$0$\\
      \texttt{Poly}  &$ subset_2$ & $0.001$ & $0.001$ & $0.001$ & $0.3$   & $0.3$   & $0.05$  & $0$ \\
      \texttt{Poly}  &$ subset_2$ & $0.001$ & $0.001$ & $0.001$ & $0.4$   & $0.4$   & $0.05$  & $0$ \\
      \texttt{Poly}  &$ subset_2$ & $0.001$ & $0.001$ & $0.001$ & $0.5$   & $0.5$   & $0.05$  & $0$ \\
      \texttt{Poly}  &$ subset_3$ & $0.001$ & $0.001$ & $0.001$ & $0.3$   & $0.3$   & $0.03$ & $8$ \\
      \texttt{Poly}  &$ subset_3$ & $0.001$ & $0.001$ & $0.001$ & $0.4$   & $0.4$   & $0.03$ & $0$\\
      \texttt{Poly}  &$ subset_3$ & $0.001$ & $0.001$ & $0.001$ & $0.5$   & $0.5$   & $0.05 $ & $5 $\\
       \bottomrule
       \end{tabular}%
     \label{tab:addlabel}%
     \end{table}%

 \newpage

 \subsection{Performance Comparison In terms of OPAUC}
   Note that our proposed algorithm can be tailored to deal with OPAUC optimization by setting all $v_+ = 1$. In this sense, we compare our proposed method with the other OPAUC optimization methods to see whether such a degeneralization can work. To this end, we conduct experiments on CIFAR-100-LT dataset and the experimental configurations stay the same as Sec.\ref{exp_detail}. Specifically, the following methods are involved in our comparison a) \texttt{TruncOPAUC}. Note that, aiming at the OPAUC optimization, TruncTPAUC is the same as TruncOPAUC since $\alpha=1.0$. b) \texttt{OPAUC-Poly} and \texttt{OPAUC-Exp}: Perform the OPAUC optimization with setting all $v_+ = 1$. c) \texttt{minmax-OPAUC-Poly} and \texttt{minmax-OPAUC-Exp}: Perform the \texttt{OPAUC-Poly} and \texttt{OPAUC-Exp} optimization with the minimax reformulation in Thm.\ref{thm:reform} by setting all $v_+ = 1$. The results are summarized in Fig \ref{fig:lossop}. The result shows that the OP version of our proposed algorithm still outperforms its truncation-based counterpart in most cases.

 \begin{figure*}[h]  
   \centering
    
     \subfigure[Comparison in terms of the OPAUC metric on subset1, $\alpha=1, \beta=0.3$]{
       \includegraphics[width=0.3\textwidth]{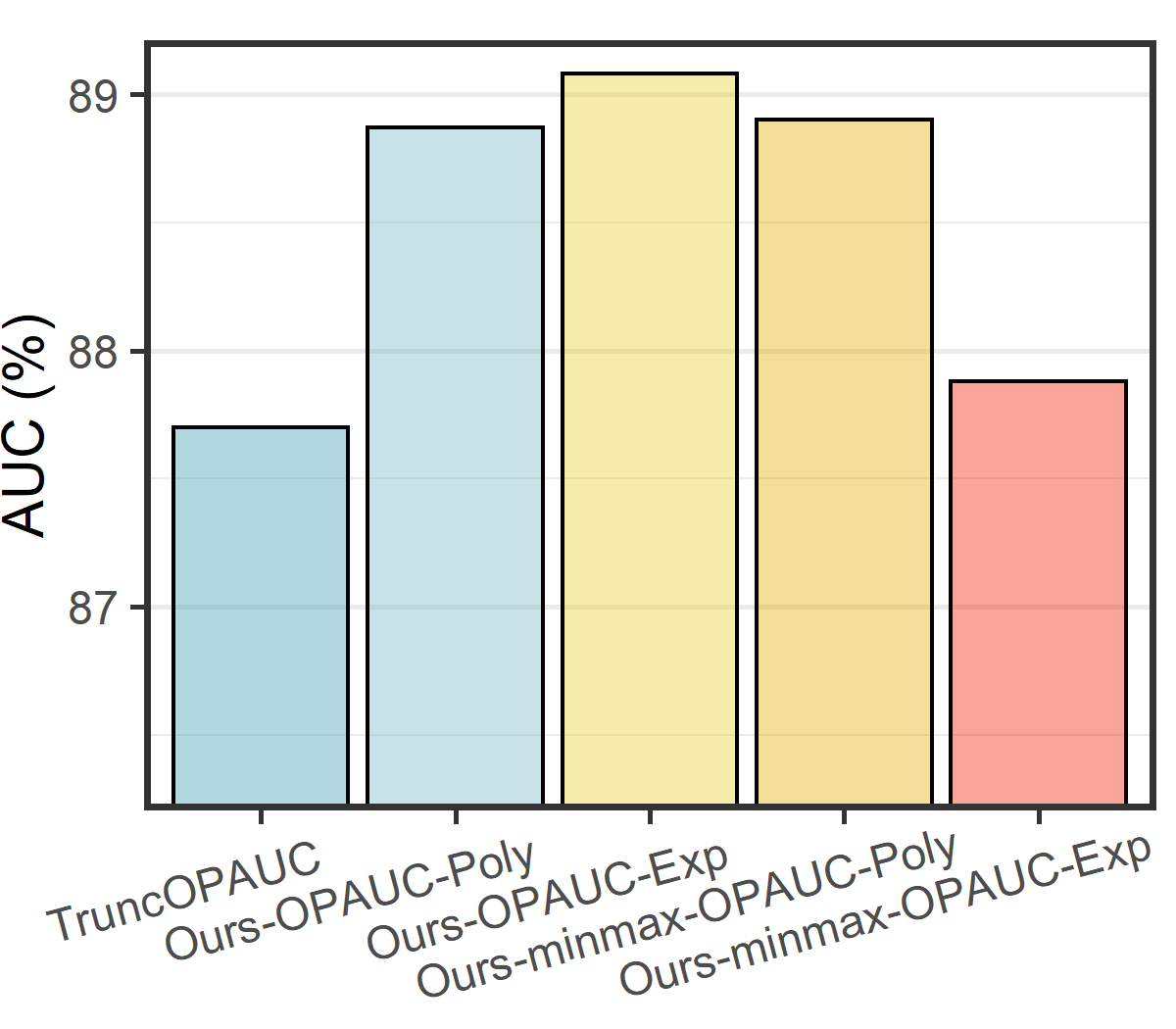} 
     }~~
     \subfigure[Comparison in terms of the OPAUC metric  on subset1, $\alpha=1, \beta=0.4$]{
       \includegraphics[width=0.3\textwidth]{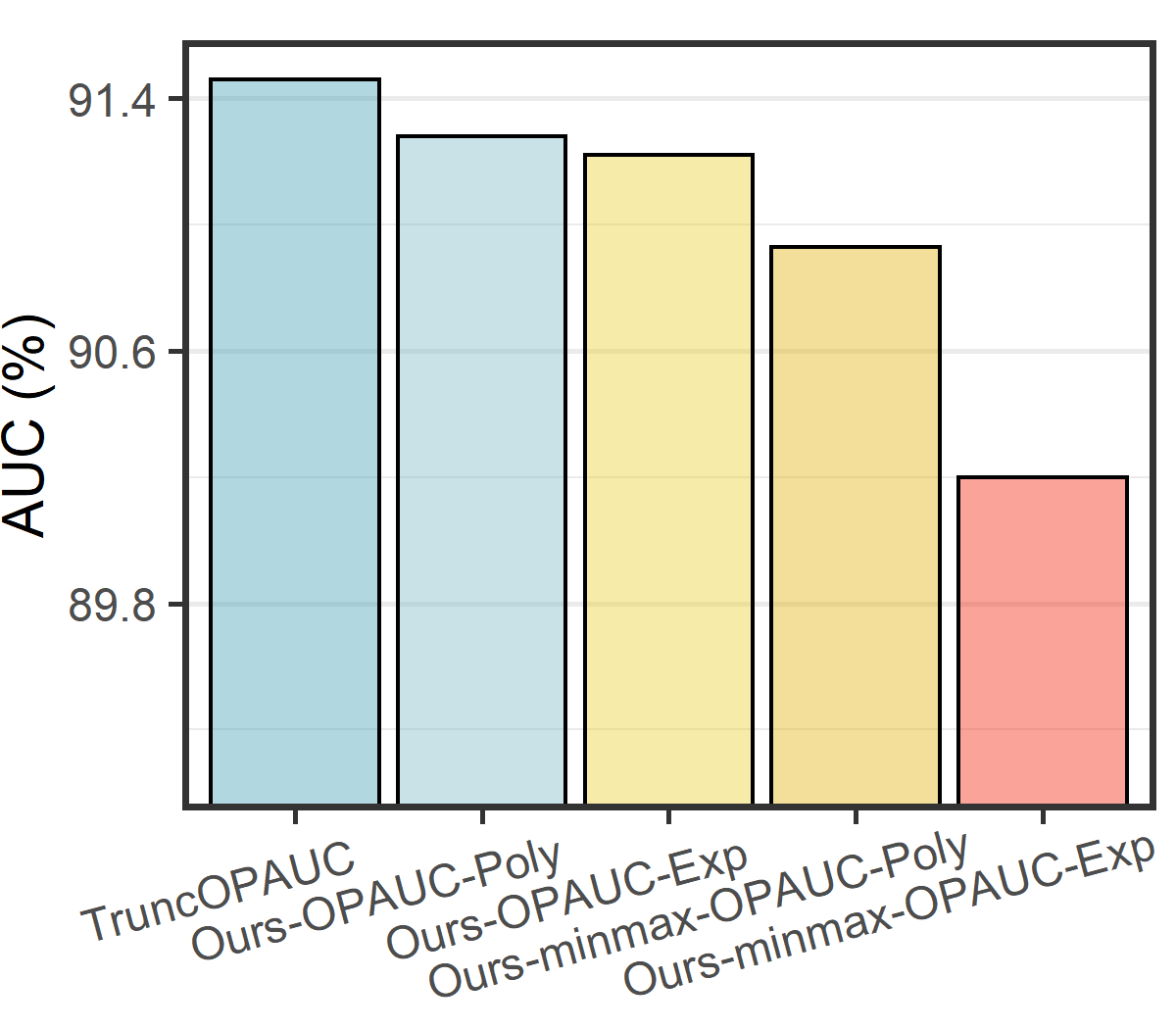} 
     }~~
     \subfigure[Comparison in terms of the OPAUC metric  on subset1, $\alpha=1, \beta=0.5$]{
       \includegraphics[width=0.3\textwidth]{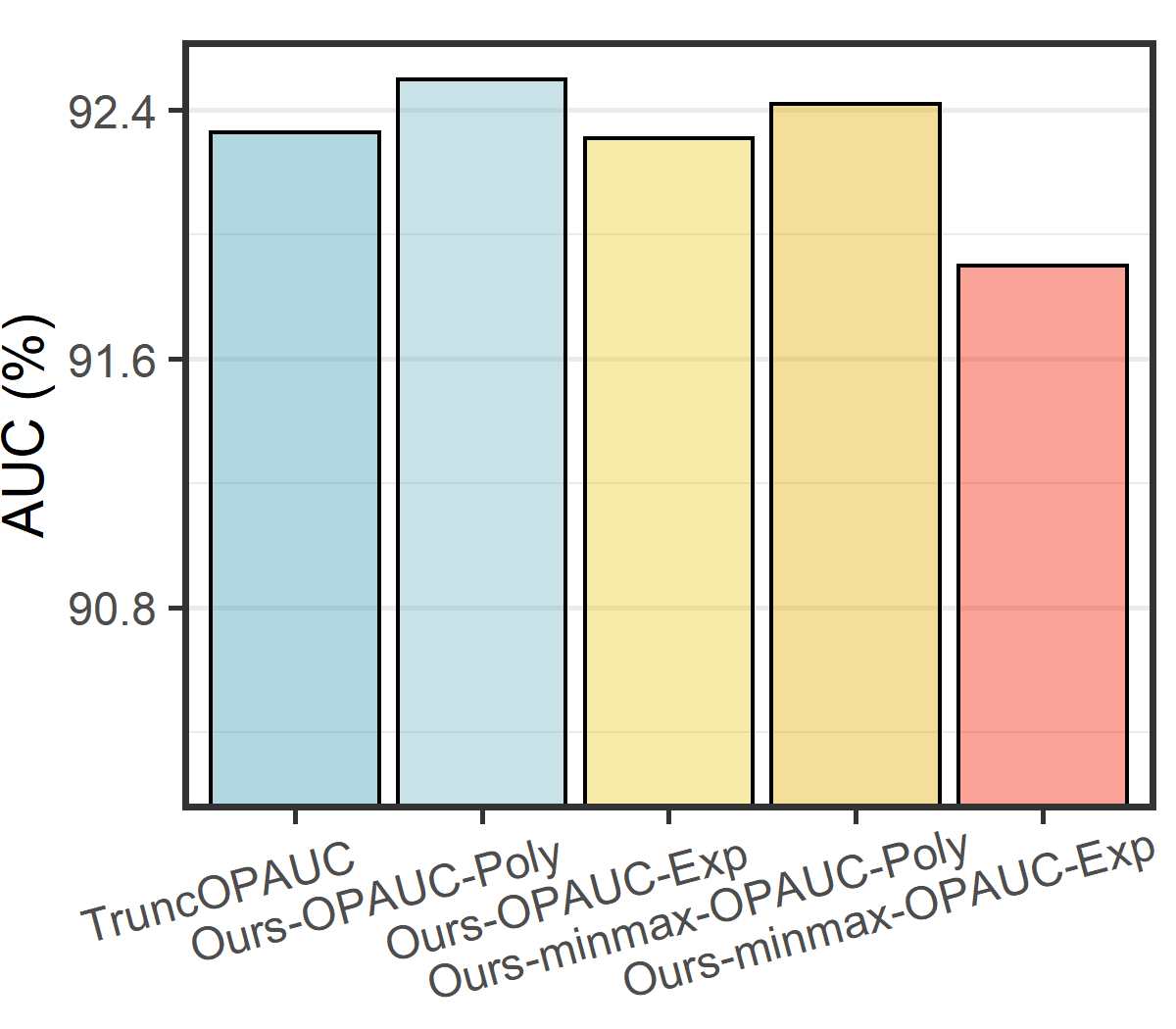} 
     }

     \subfigure[Comparison in terms of the OPAUC metric  on subset2, $\alpha=1, \beta=0.3$]{
       \includegraphics[width=0.3\textwidth]{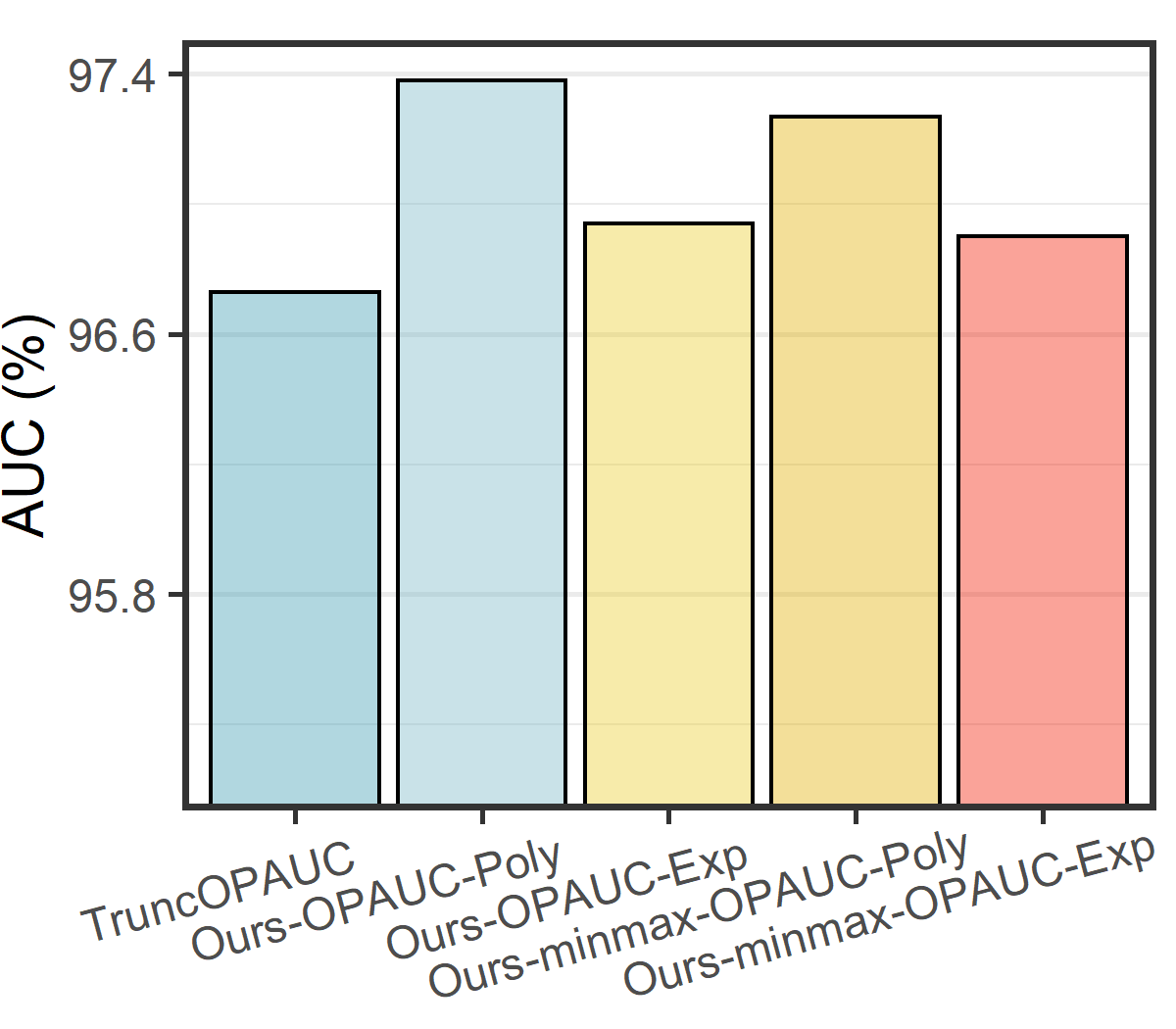} 
     }~~
     \subfigure[Comparison in terms of the OPAUC metric  on subset2, $\alpha=1, \beta=0.4$]{
       \includegraphics[width=0.3\textwidth]{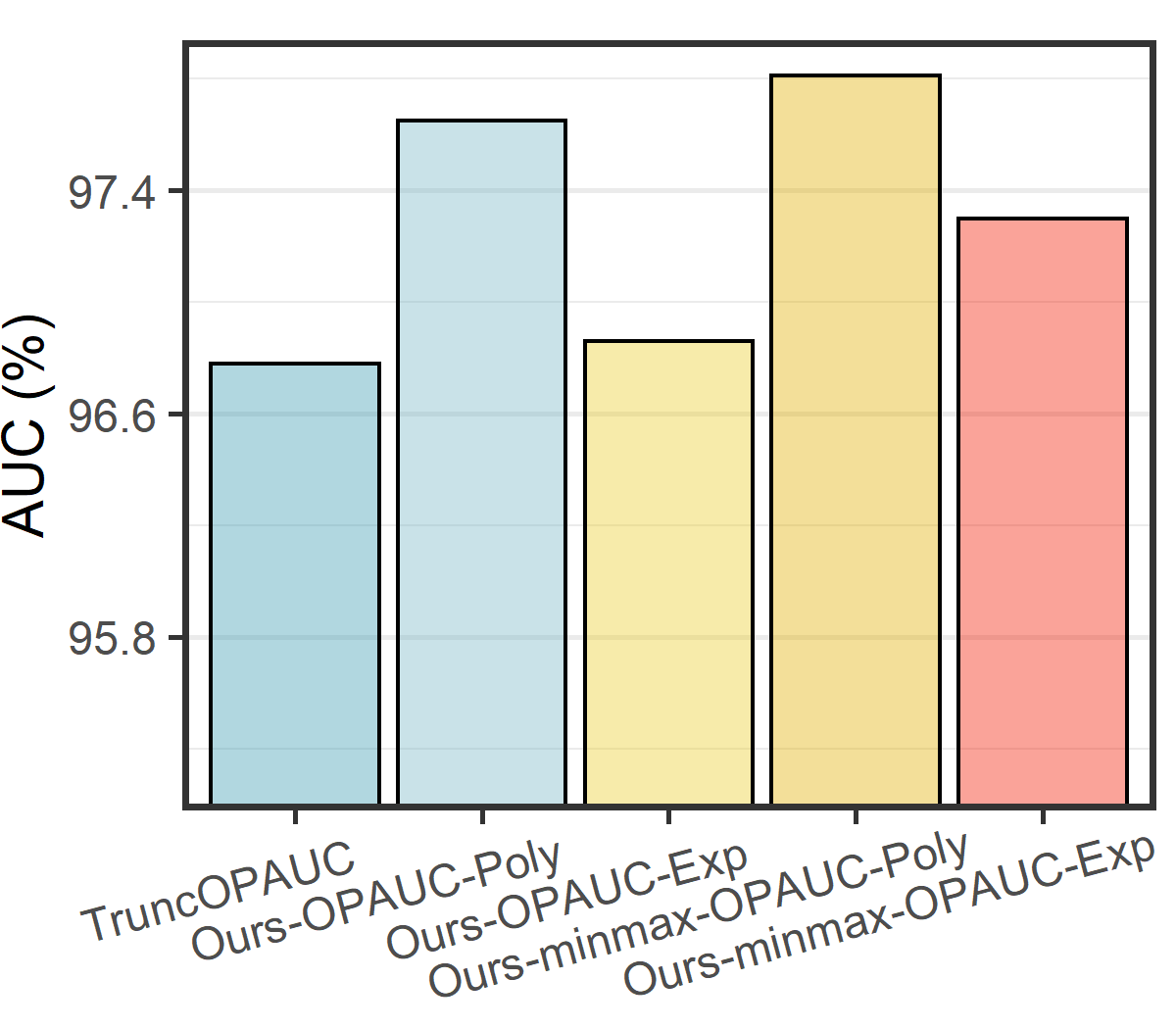} 
     }~~
     \subfigure[Comparison in terms of the OPAUC metric  on subset2, $\alpha=1, \beta=0.5$]{
       \includegraphics[width=0.3\textwidth]{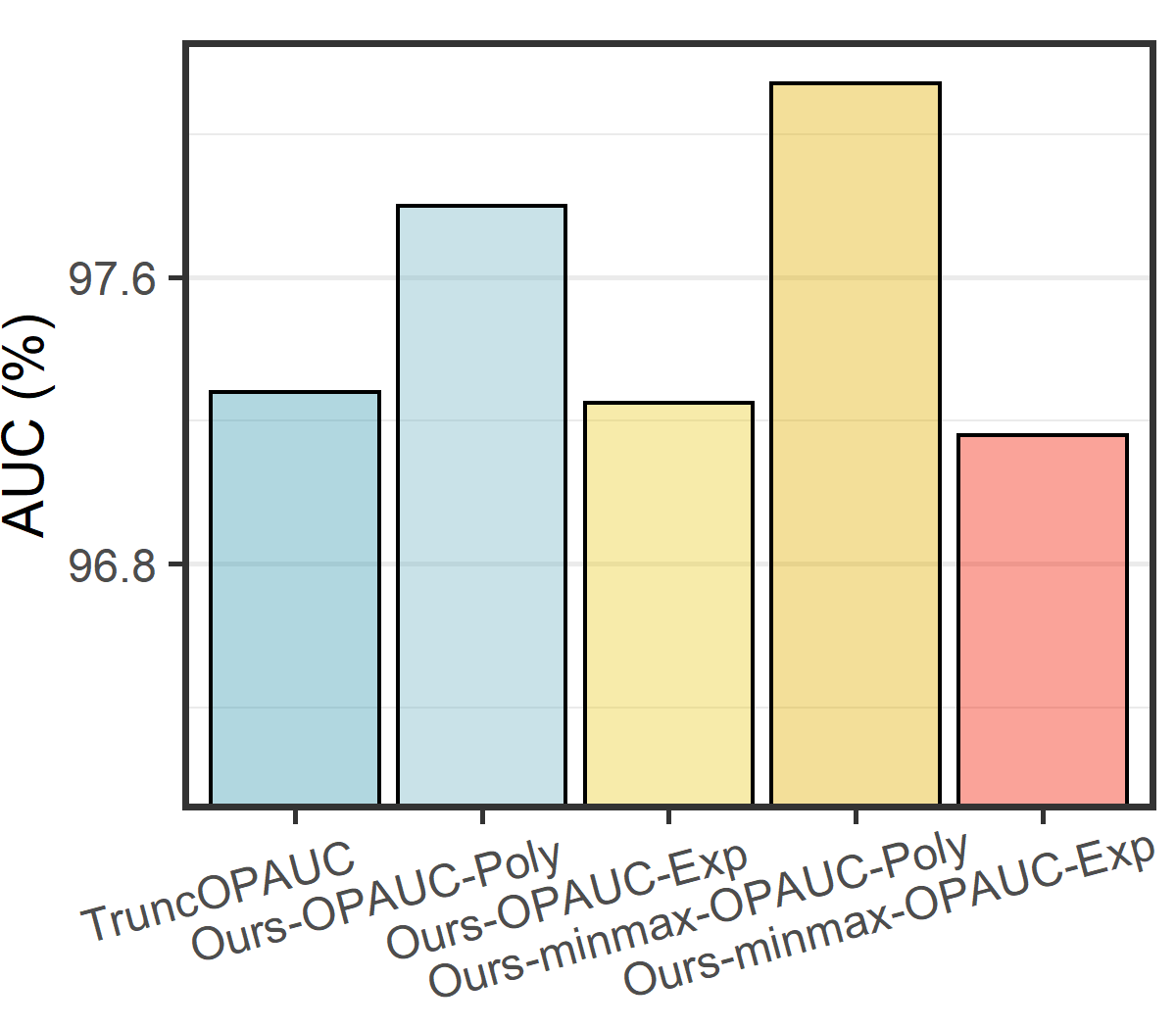} 
     }

     \subfigure[Comparison in terms of the OPAUC metric  on subset3, $\alpha=1, \beta=0.3$]{
       \includegraphics[width=0.3\textwidth]{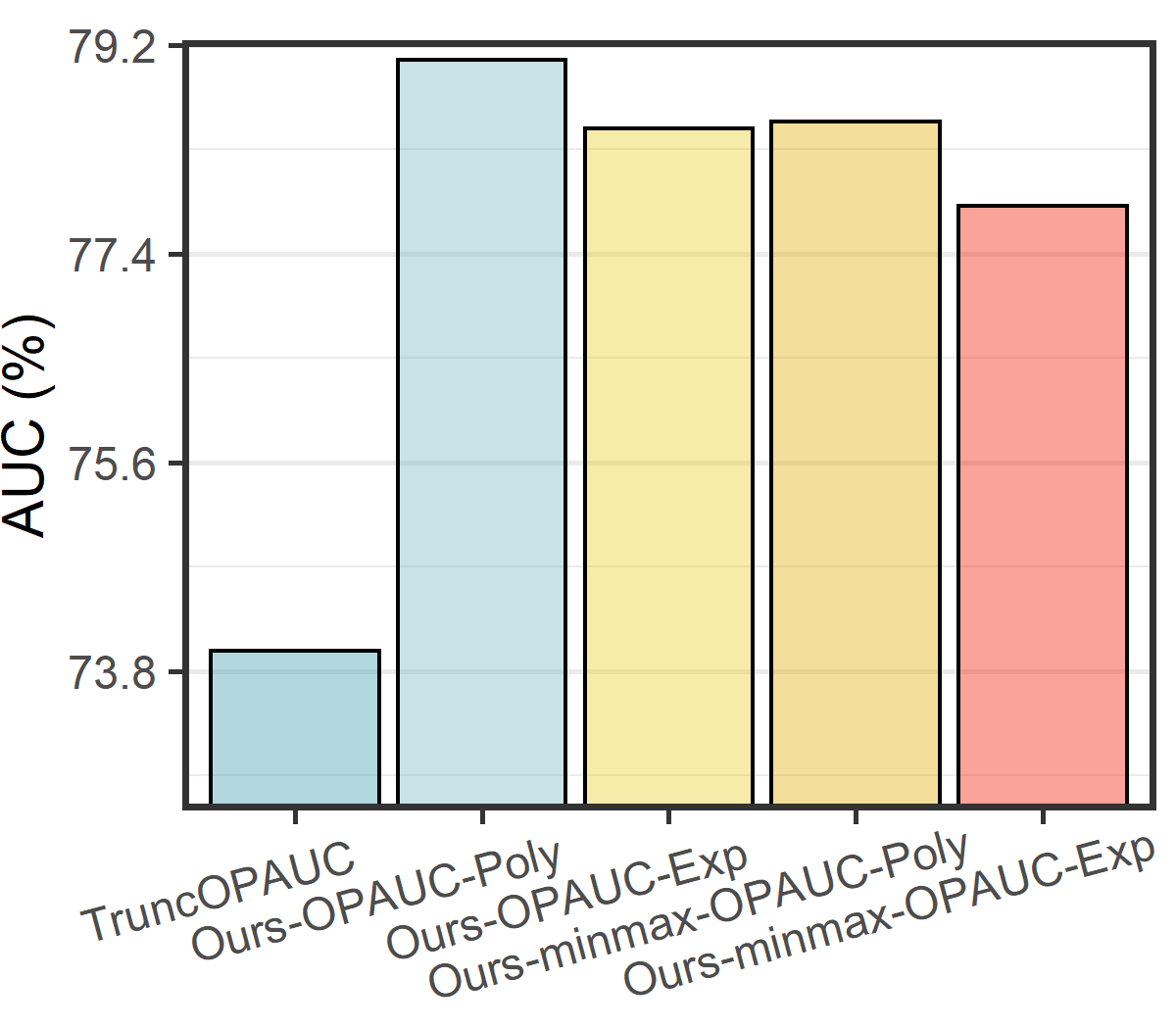} 
     }~~
     \subfigure[Comparison in terms of the OPAUC metric  on subset3, $\alpha=1, \beta=0.4$]{
       \includegraphics[width=0.3\textwidth]{./fig/OPAUC/subset3_0.3.png} 
     }~~
     \subfigure[Comparison in terms of the OPAUC metric  on subset3, $\alpha=1, \beta=0.5$]{
       \includegraphics[width=0.3\textwidth]{./fig/OPAUC/subset3_0.3.png} 
     }
     \caption{\label{fig:lossop} Comparison in terms of the OPAUC metric on CIFAR-100-LT.}
 \end{figure*}

 \begin{figure*}[h]  
   \centering
    
     \subfigure[Comparison for different Loss on subset1, $\alpha=0.3, \beta=0.3$]{
       \includegraphics[width=0.3\textwidth]{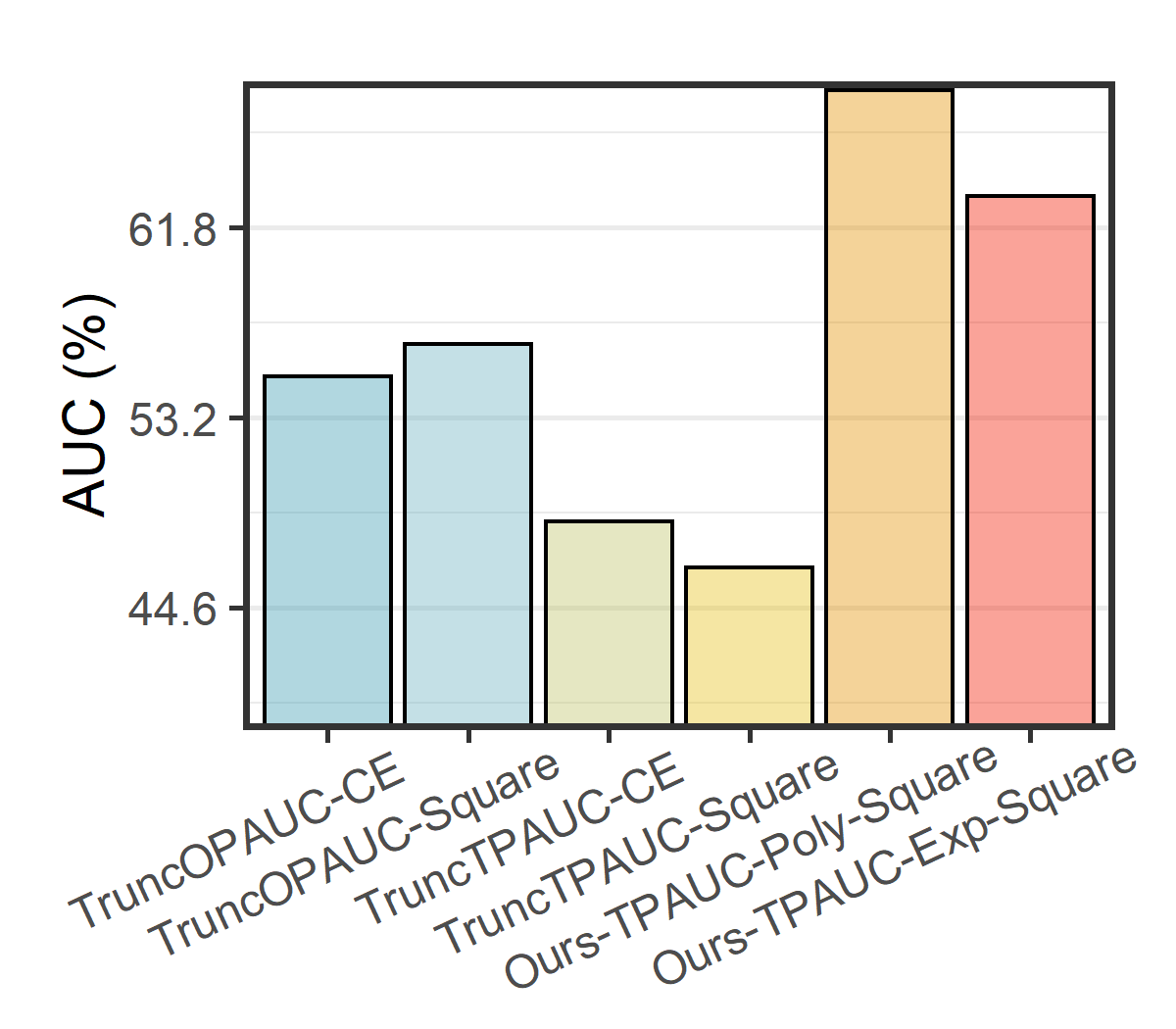} 
     }~~
     \subfigure[Comparison for different Loss  on subset1, $\alpha=0.4, \beta=0.4$]{
       \includegraphics[width=0.3\textwidth]{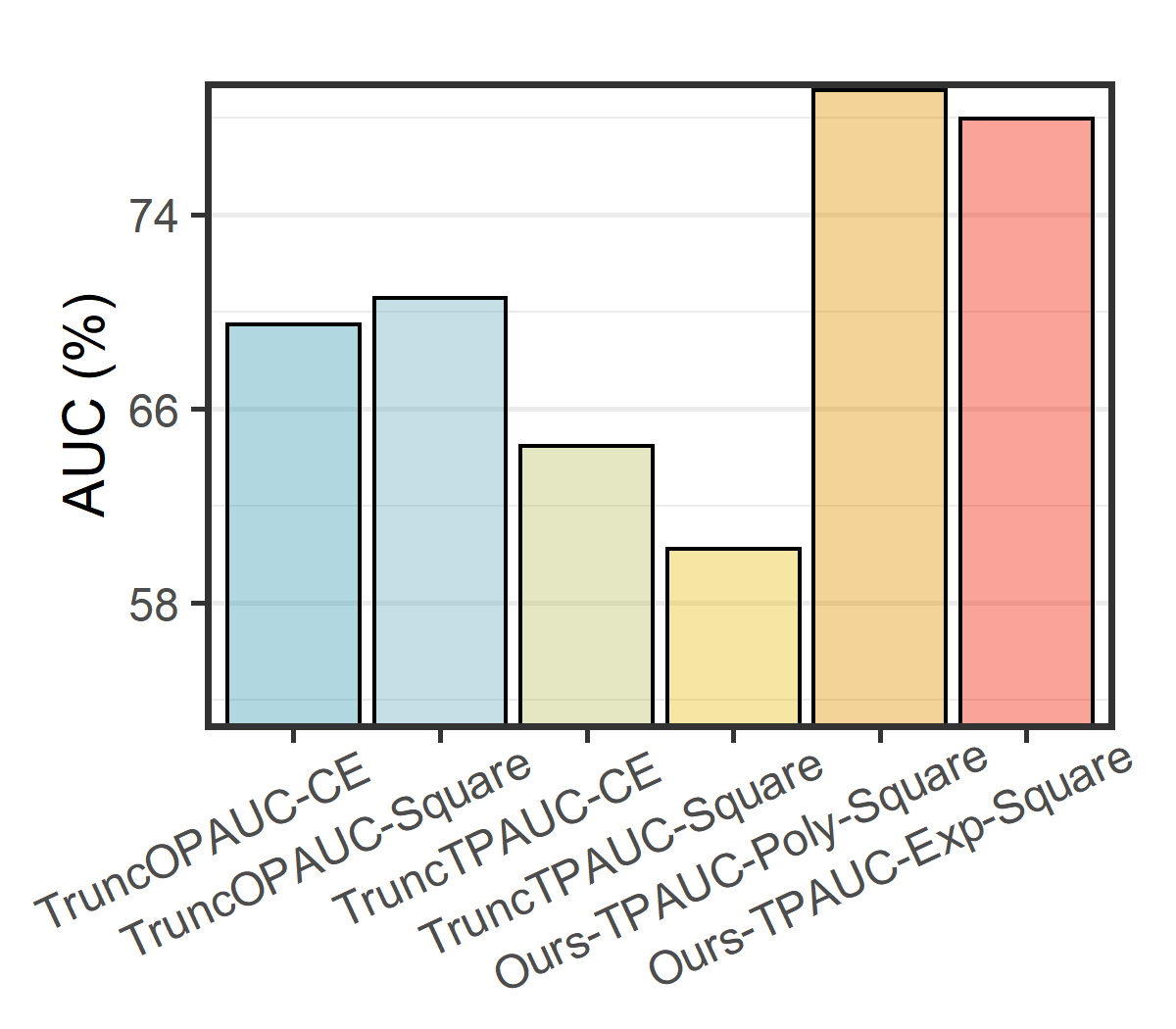} 
     }~~
     \subfigure[Comparison for different Loss  on subset1, $\alpha=0.5, \beta=0.5$]{
       \includegraphics[width=0.3\textwidth]{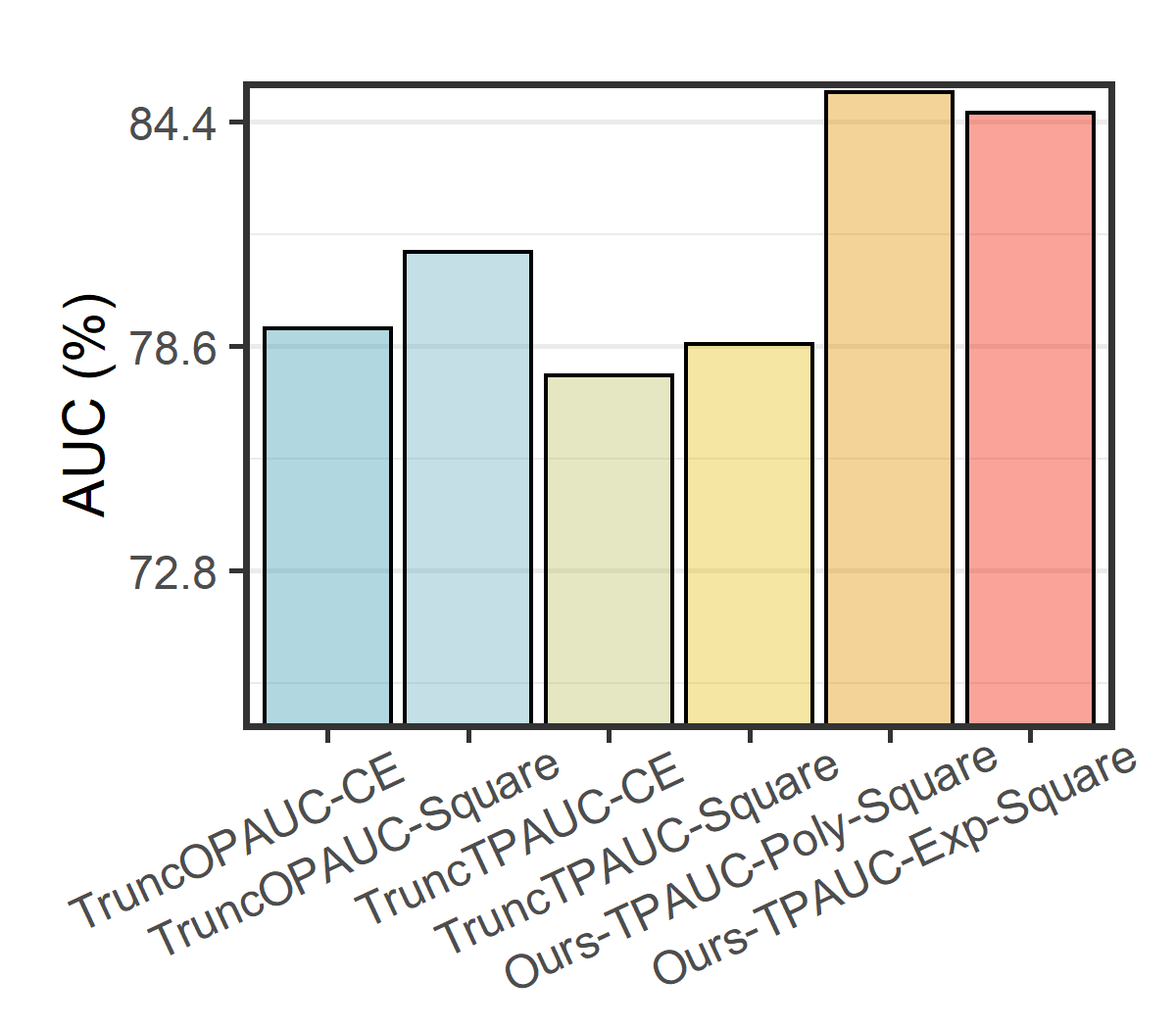} 
     }

     \subfigure[Comparison for different Loss  on subset2, $\alpha=0.3, \beta=0.3$]{
       \includegraphics[width=0.3\textwidth]{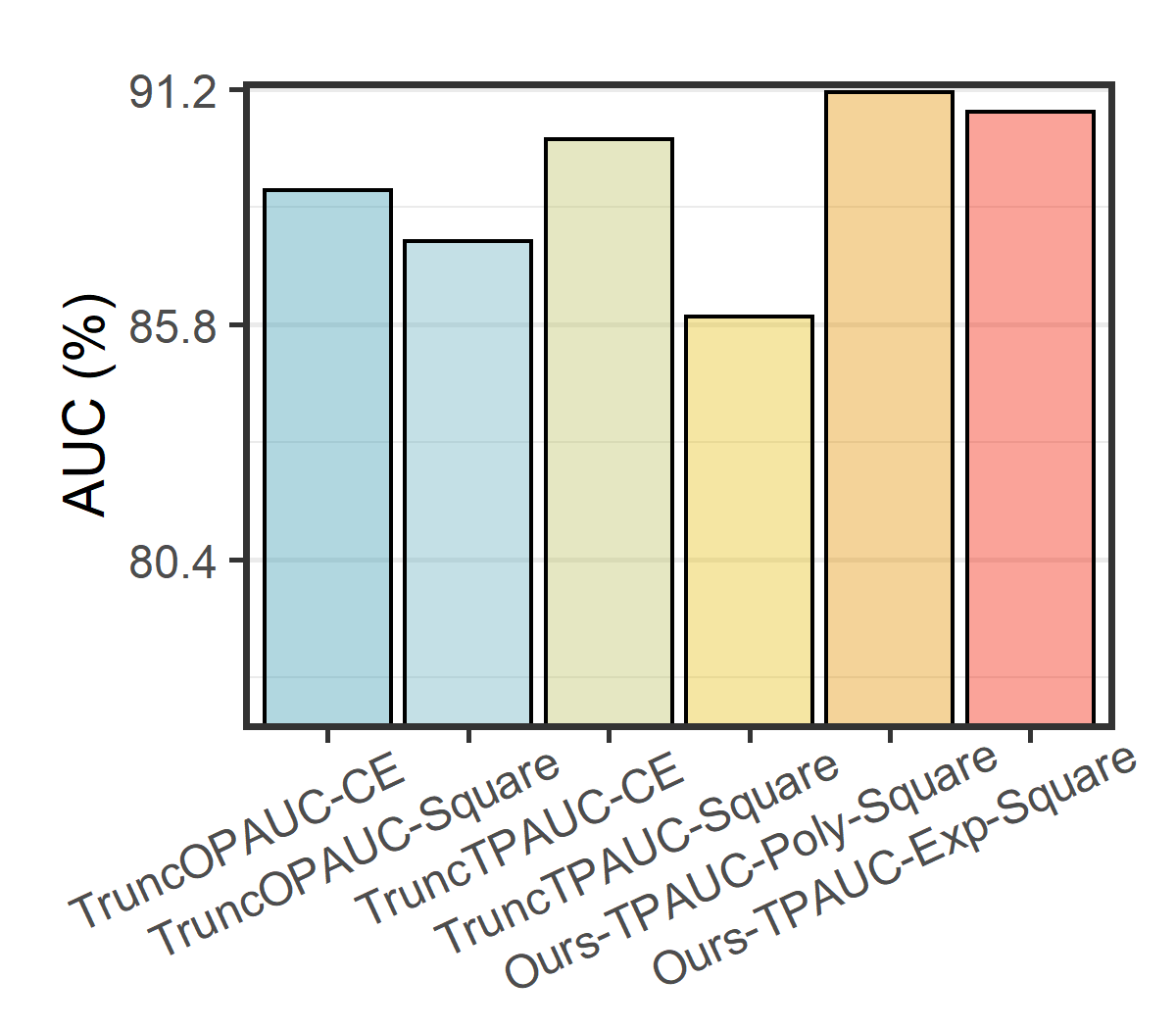} 
     }~~
     \subfigure[Comparison for different Loss  on subset2, $\alpha=0.4, \beta=0.4$]{
       \includegraphics[width=0.3\textwidth]{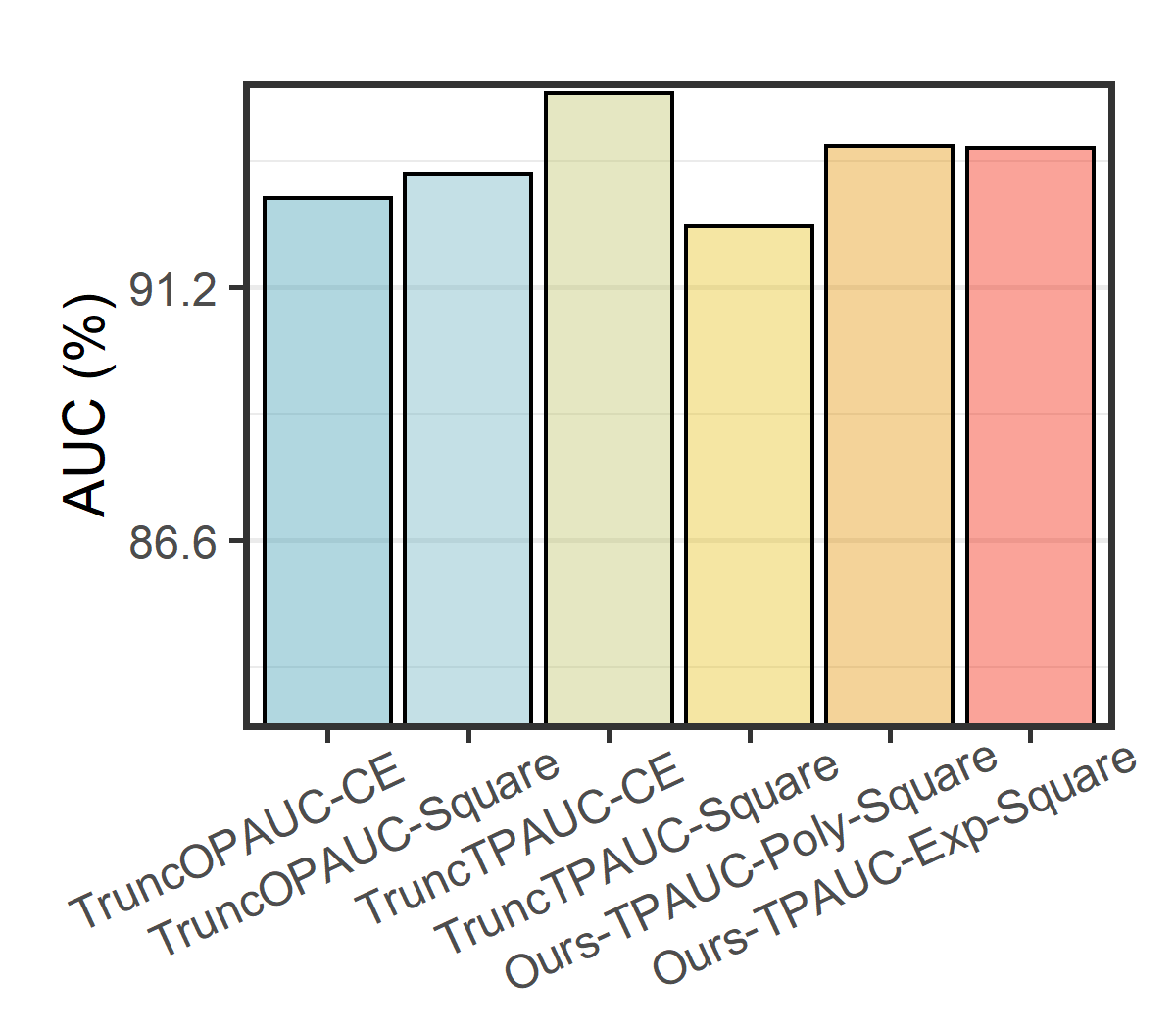} 
     }~~
     \subfigure[Comparison for different Loss  on subset2, $\alpha=0.5, \beta=0.5$]{
       \includegraphics[width=0.3\textwidth]{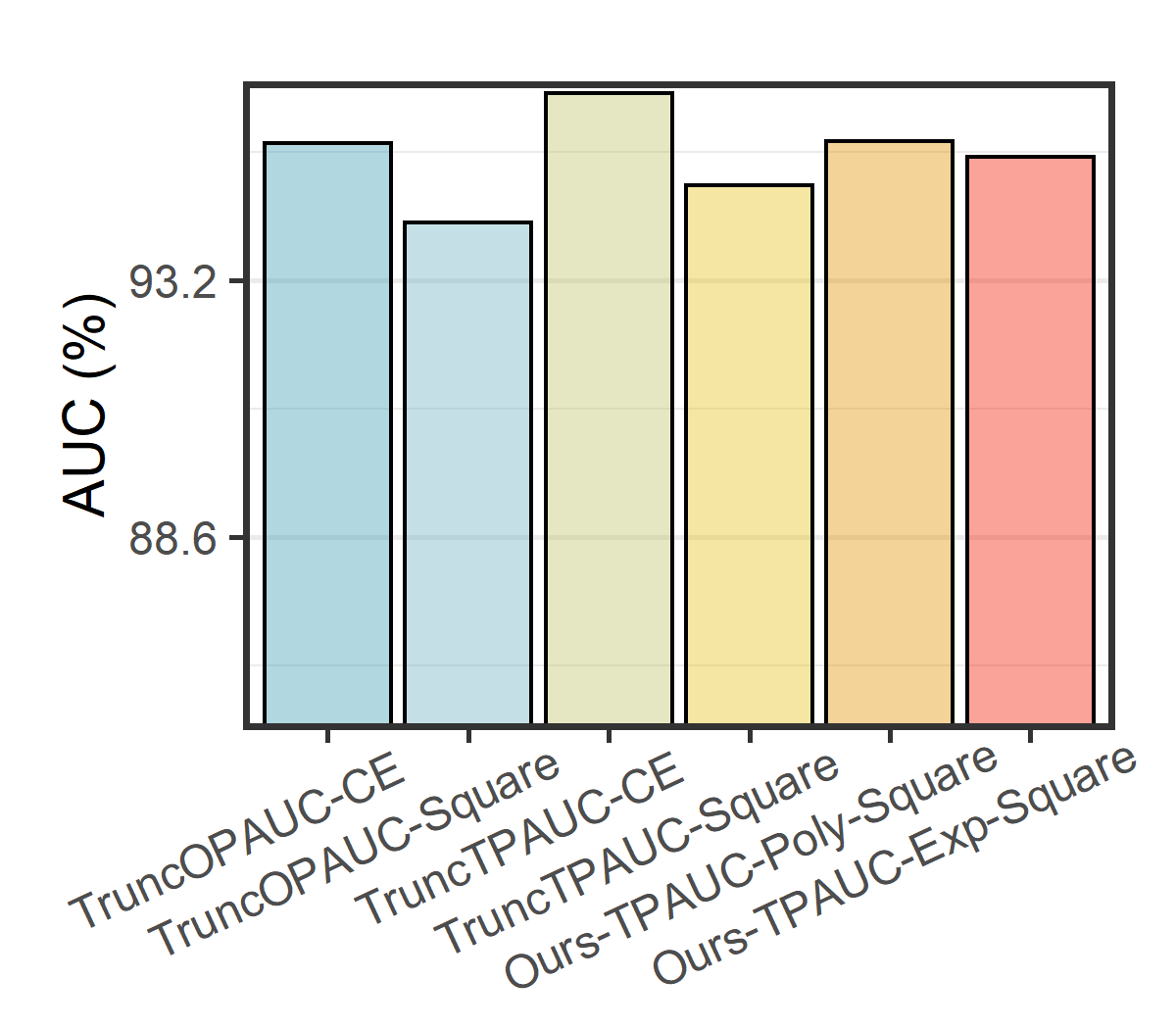} 
     }

     \subfigure[Comparison for different Loss  on subset3, $\alpha=0.3, \beta=0.3$]{
       \includegraphics[width=0.3\textwidth]{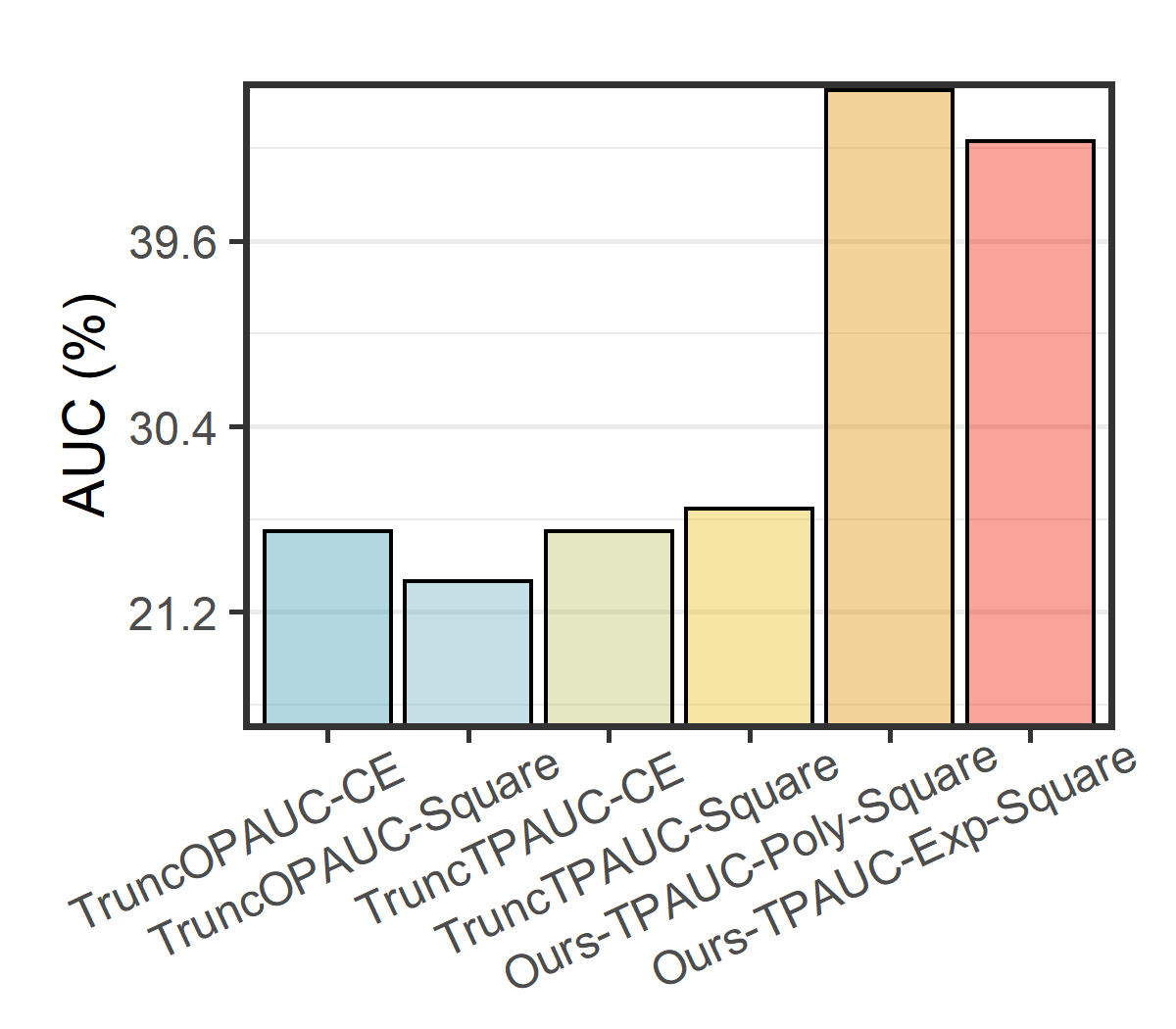} 
     }~~
     \subfigure[Comparison for different Loss  on subset3, $\alpha=0.4, \beta=0.4$]{
       \includegraphics[width=0.3\textwidth]{./fig/CEloss/subset3_0.3.png} 
     }~~
     \subfigure[Comparison for different Loss  on subset3, $\alpha=0.5, \beta=0.5$]{
       \includegraphics[width=0.3\textwidth]{./fig/CEloss/subset3_0.3.png} 
     }
     \caption{\label{fig:lossce} Comparison for CE and Square Loss on CIFAR-100-LT w.r.t. TPAUC.}
 \end{figure*}

 \subsection{Effect of Warm-up Phase}
 In this subsection, we show the sensitivity analysis results w.r.t  the warm-up iterations. First, we conduct the experiments for our method. The results are shown in Fig.\ref{fig:warm1} and Fig.\ref{fig:warm2}. It is easy to see that including a warm-up phase could improve the performance. Second, we show that warm-up is equally important for truncated methods.  Recall that in the main paper, we adopt the square loss as the surrogate loss, which is necessary for an efficient implementation. To be fair, we here implement truncated methods for CE loss, which is a more popular choice for classification problems with two new competitors: \texttt{CE-TruncOPAUC} and \texttt{CE-TruncTPAUC}. They perform Truncation-based OPAUC/TPAUC optimization by replacing $\ell_{sq}$ with a standard CE loss. The results are shown in Tab.\ref{tab:warmup}, where \texttt{w/o warm-up} implies warm-up phase is not employed.   The performances drop dramatically when the warm-up phase is removed. Sometimes for small $\alpha,\beta$ and hard subsets, we even observe a 0 TPAUC on the test set.  A probable reason is that, without warm-up, the models will only look at the hard examples all the time, which will incur severe overfitting, especially when the model used here is a large ResNet. Moreover, the performance drop is a general issue for both CE and square loss. This is also reasonable since the overfitting problem only has a weak connection with the landscape of $\ell$ employed. Finally, we again compare the Trunc-baselines with our proposed method, the results are shown in Fig.\ref{fig:lossce}. It can be seen that our proposed method outperform all the competitors in most cases.

 \begin{table}[htbp]
   \centering
   \caption{Warm-up Effect on CIFAR-100-LT}
   \begin{tabular}{c|c|c|ccc|ccc|ccc}
     \toprule
     \multirow{2}[4]{*}{dataset} & \multirow{2}[4]{*}{methods} & \multirow{2}[4]{*}{surrogate} & \multicolumn{3}{c|}{Subset1} & \multicolumn{3}{c|}{Subset2} & \multicolumn{3}{c}{Subset3} \\
 \cmidrule{4-12}          &       &       & 0.3   & 0.4   & 0.5   & 0.3   & 0.4   & 0.5   & 0.3   & 0.4   & 0.5 \\
     \midrule
     \multirow{8}[8]{*}{CIFAR-100-LT} & \multirow{2}[2]{*}{TruncOPAUC w/o warm-up} & CE    & 0.00  & 69.48  & 76.21  & 81.51  & 90.30  & 92.41  & 0.00  & 0.00  & 27.21  \\
           &       & Square & 13.23  & 77.84  & 86.44  & 83.61  & 89.83  & 92.69  & 0.00  & 4.87  & 38.03  \\
 \cmidrule{2-12}          & \multirow{2}[2]{*}{TruncOPAUC} & CE    & 55.06  & 69.48  & 79.07  & 88.89  & 92.82  & 95.66  & 25.20  & 43.74  & 58.26  \\
           &       & Square & 56.51  & 70.56  & 81.03  & 87.72  & 93.26  & 94.23  & 22.75  & 51.30 & 66.78  \\
 \cmidrule{2-12}          & \multirow{2}[2]{*}{TruncTPAUC w/o warm-up} & CE    & 0.00  & 5.95  & 20.95  & 29.97  & 52.88  & 93.34  & 0.00  & 0.00  & 4.67  \\
           &       & Square & 0.02  & 12.53  & 30.24  & 46.31  & 64.16  & 94.69  & 0.00  & 1.71  & 14.67  \\
 \cmidrule{2-12}          & \multirow{2}[2]{*}{TruncTPAUC} & CE    & 48.49  & 64.46  & 77.84  & 90.05  & 94.73  & 96.55  & 25.20  & 44.19  & 58.04  \\
           &       & Square & 46.42  & 60.23  & 78.66  & 85.99  & 92.31  & 94.90 & 26.34  & 54.69  & 66.77  \\
     \bottomrule
     \end{tabular}%

   \label{tab:warmup}%
 \end{table}%

 \subsection{Additional Effect of $\gamma$} 
 \noindent \textbf{heterogeneous Setting.} In this setting, we will use $\gamma_+$ for positive example weight and $\gamma_-$ for the negative weight with $\gamma_+ \neq \gamma_-$. Note that we fix $E_k = 35$, and the other experimental configurations stay the same as Sec.\ref{exp_detail}. In Fig.\ref{fig:gammapoly2} and Fig.\ref{fig:gammaexp2}, we show the sensitivity in terms of $\gamma$ for \texttt{Poly} and \texttt{Exp} weighting, respectively. Obviously, the performance is sensitive toward the heterogeneous setting. In this sense, adopting heterogeneous weights will generally result in better performance. However, for the sake of fairness, we only report  the overall performance when $\gamma_+ = \gamma_-$.

 \noindent\textbf{Optimal $\gamma$ for different metrics.} In this setting, we will study the relationship between the optimal $\gamma$ and the TPAUC to be optimized. Specifically, $\alpha$ is \textbf{fixed} to $0.2$ for all experiments, and we search the optimal $\gamma$ for $\beta = 0.4, 0.6, 0.8$, respectively. To ensure that the surrogate weighting function is concave, for \texttt{Poly}, the search space  of $\frac{1}{1-\gamma}$ is set to $\{0.0001, 0.0003, 0.0005, \\ 0.001, 0.003, 0.005, 0.01, 0.03, 0.05, 0.1\}$, and the search space for $\gamma$ is set to $\{5, 10, 15, 20, 25, 35, 100, 350, 1000\}$ for \texttt{Exp}. All experiments are conducted on three subsets of CIFAR-100-LT. The optimal $\gamma$ are shown in Fig.\ref{fig:gammabeta}. Note that, for \texttt{Poly} weighting, we will abuse $\gamma = \frac{1}{1-\gamma}$ in rest of our discussion. We have the following observation: For \texttt{Poly}, the optimal $\gamma$ decreases as $\beta$ increases, while for \texttt{Exp}, the optimal $\gamma$ increases as $\beta$ increases. This is consistent with the theoretical property of the weighting function since a smaller $\beta$ requires stronger attention to the bottom positive instances. This means we need the weighting function to provide increasingly sharp weight differences between hard and easy positive examples. To do this, for texttt{Poly}, we need a smaller $\gamma$ for a larger $\beta$; while for \texttt{Exp}, we need a larger $\gamma$ for a larger $\beta$. This is exactly the trend shown in the figure.

 \begin{figure*}[h]  
   \centering
    

     \subfigure[Relationship Between $\gamma$ and $\beta$ for \texttt{Exp} Weighting]{
       \includegraphics[width=0.48\textwidth]{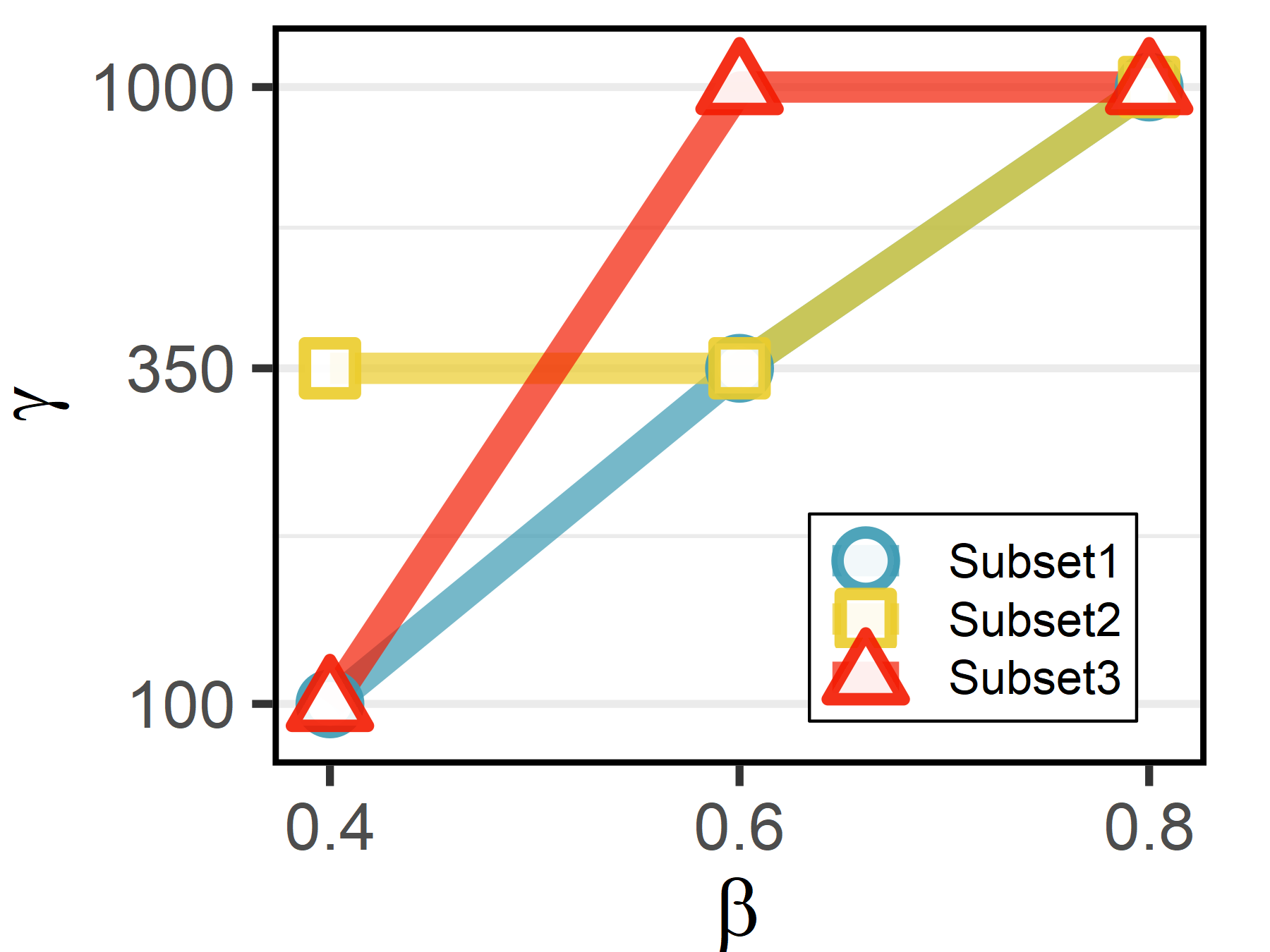} 
     }~~
     \subfigure[Relationship Between $\gamma$ and $\beta$ for \texttt{Poly} Weighting]{
       \includegraphics[width=0.48\textwidth]{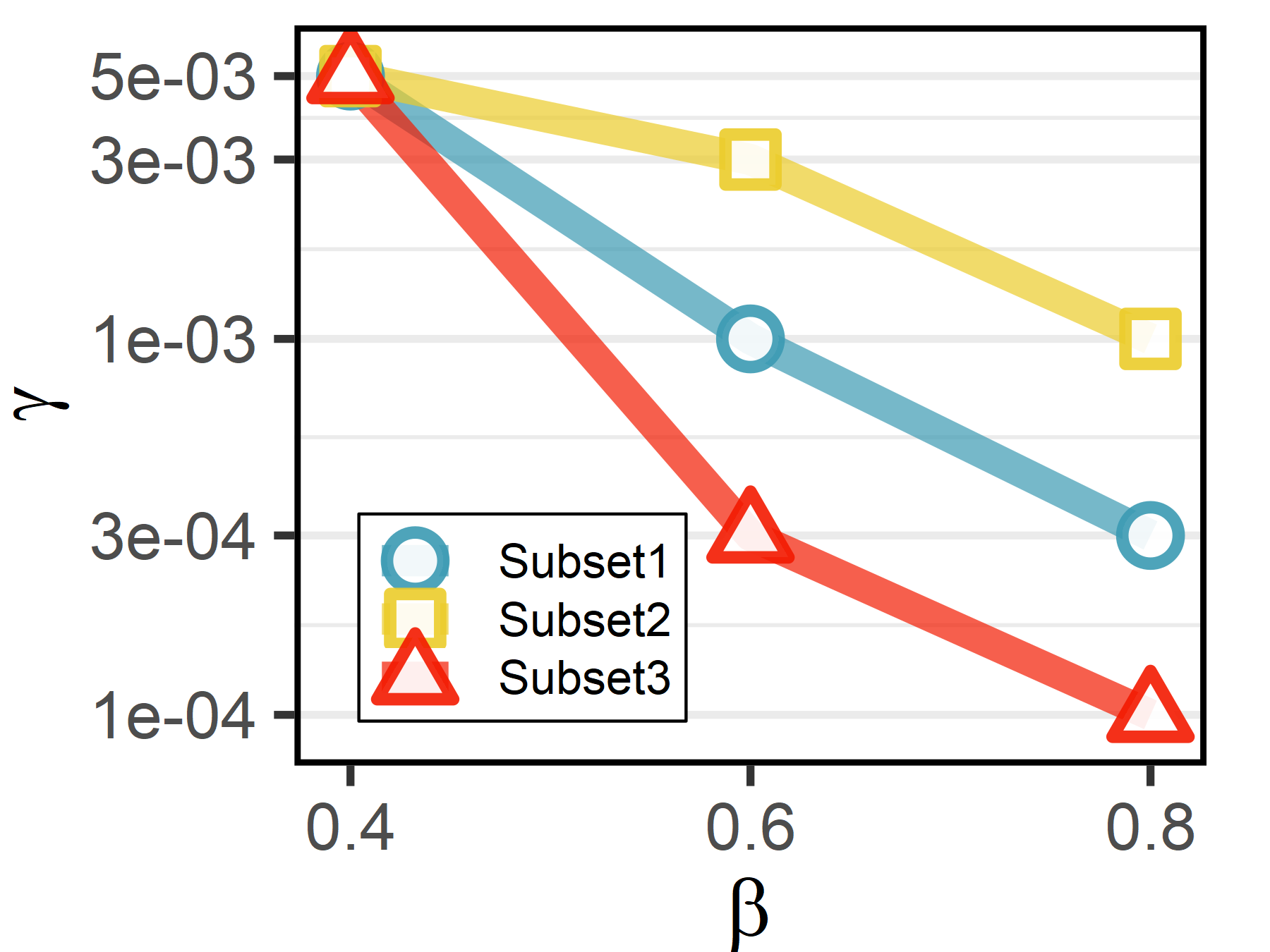} 
     }
     \caption{\label{fig:gammabeta}Optimal $\gamma$ for different $\beta$ with fixed $\alpha=0.2$ on CIFAR-100-LT. Here $(\gamma - 1)^{-1}$ are reported.}
   \end{figure*}


 \begin{figure*}[h]  
   \centering
    
     \subfigure[Effect of $\gamma_+,\gamma_-$ on subset1, $\alpha=0.3, \beta=0.3$]{
       \includegraphics[width=0.3\textwidth]{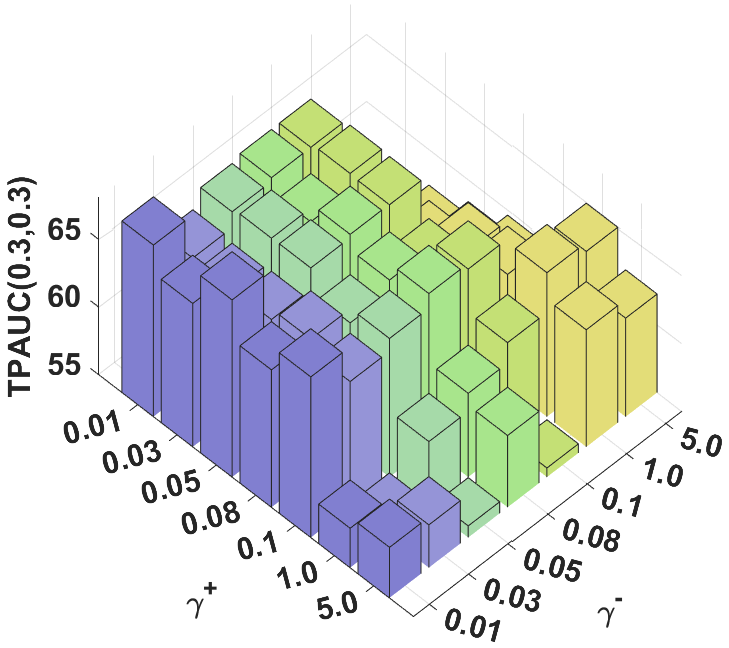} 
     }~~
     \subfigure[Effect of $\gamma_+,\gamma_-$ on subset1, $\alpha=0.4, \beta=0.4$]{
       \includegraphics[width=0.3\textwidth]{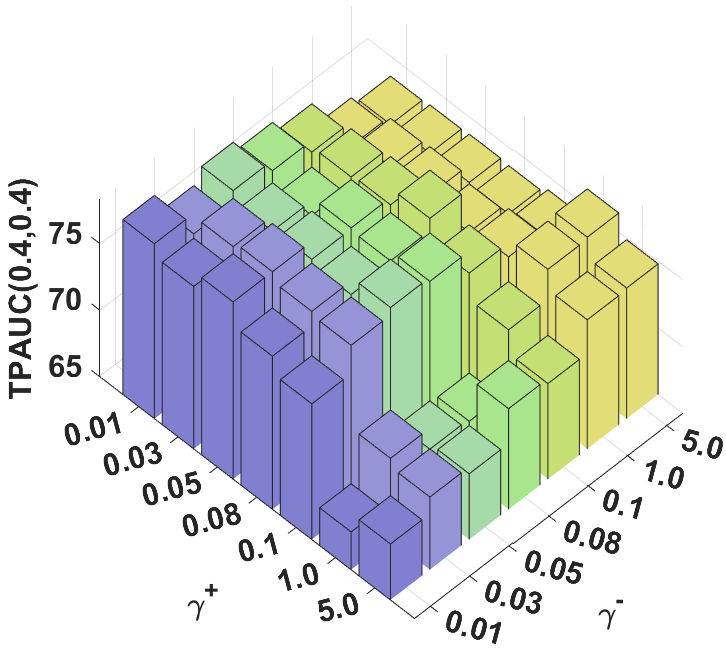} 
     }~~
     \subfigure[Effect of $\gamma_+,\gamma_-$ on subset1, $\alpha=0.5, \beta=0.5$]{
       \includegraphics[width=0.3\textwidth]{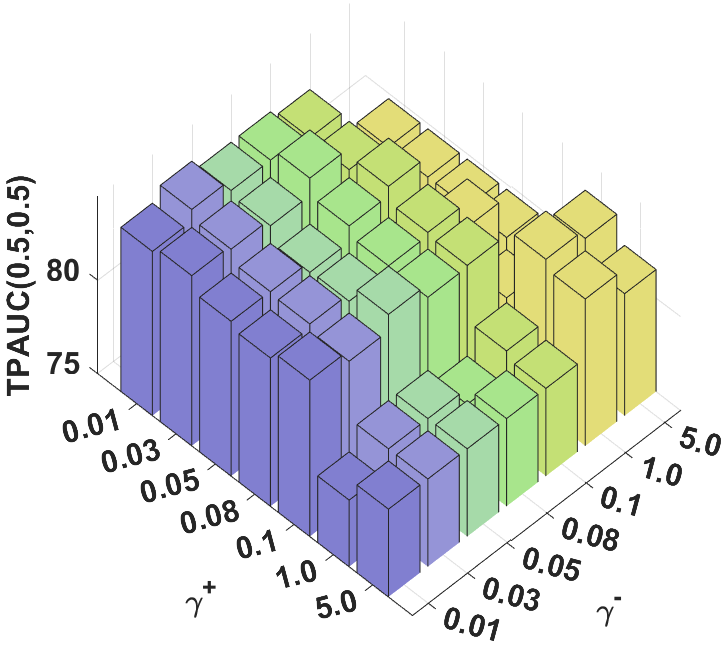} 
     }

     \subfigure[Effect of $\gamma_+,\gamma_-$ on subset2, $\alpha=0.3, \beta=0.3$]{
       \includegraphics[width=0.3\textwidth]{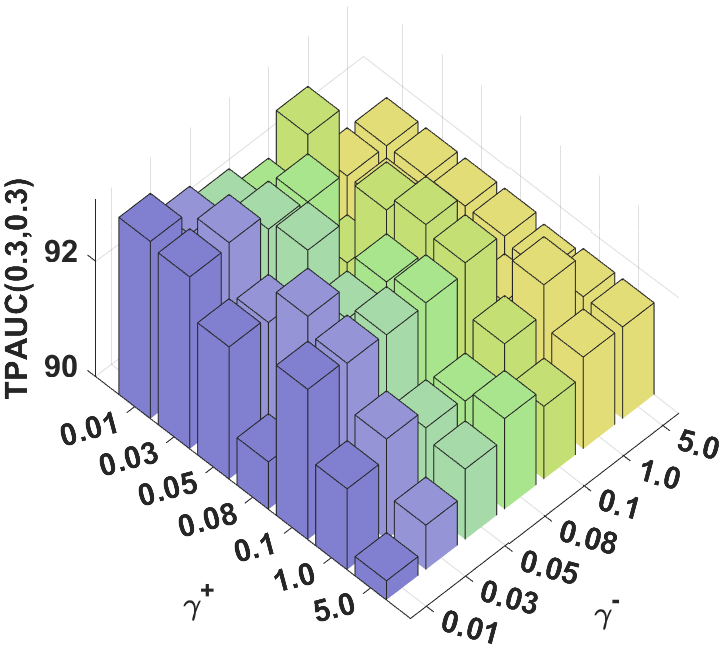} 
     }~~
     \subfigure[Effect of $\gamma_+,\gamma_-$ on subset2, $\alpha=0.4, \beta=0.4$]{
       \includegraphics[width=0.3\textwidth]{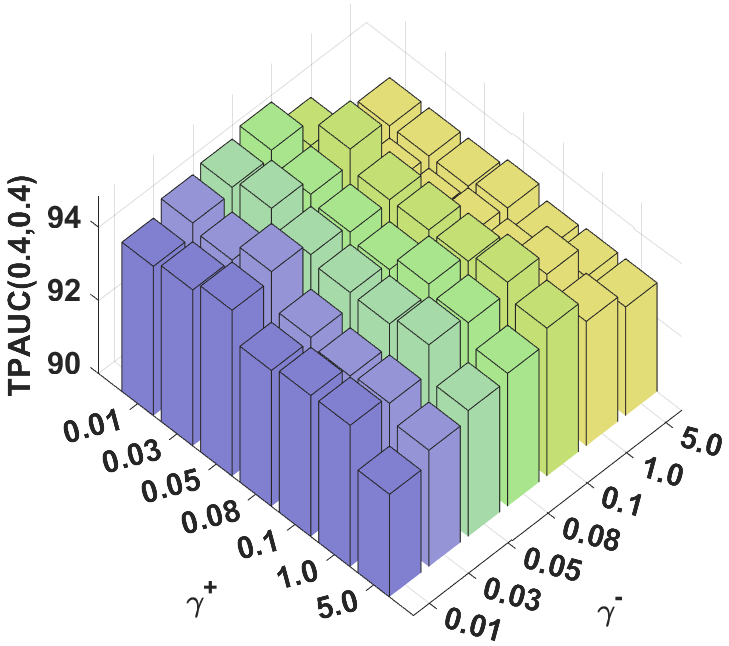} 
     }~~
     \subfigure[Effect of $\gamma_+,\gamma_-$ on subset2, $\alpha=0.5, \beta=0.5$]{
       \includegraphics[width=0.3\textwidth]{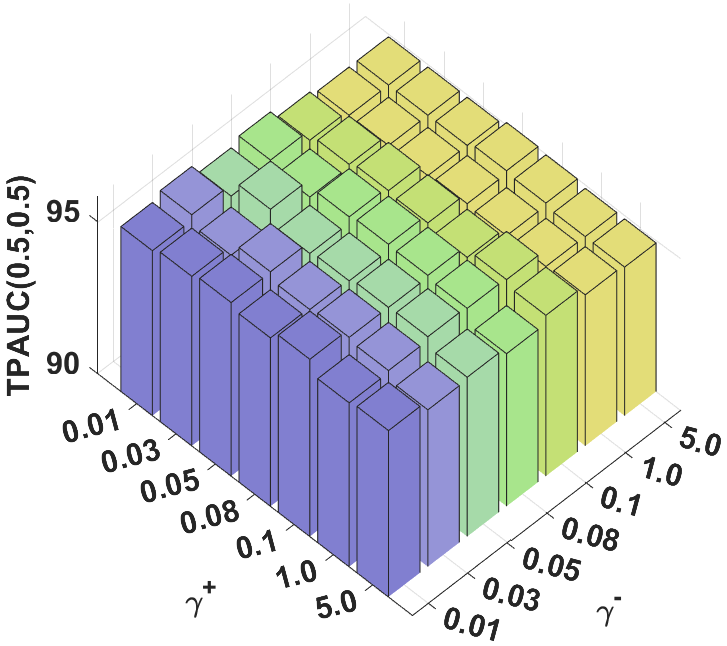} 
     }

     \subfigure[Effect of $\gamma_+,\gamma_-$ on subset3, $\alpha=0.3, \beta=0.3$]{
       \includegraphics[width=0.3\textwidth]{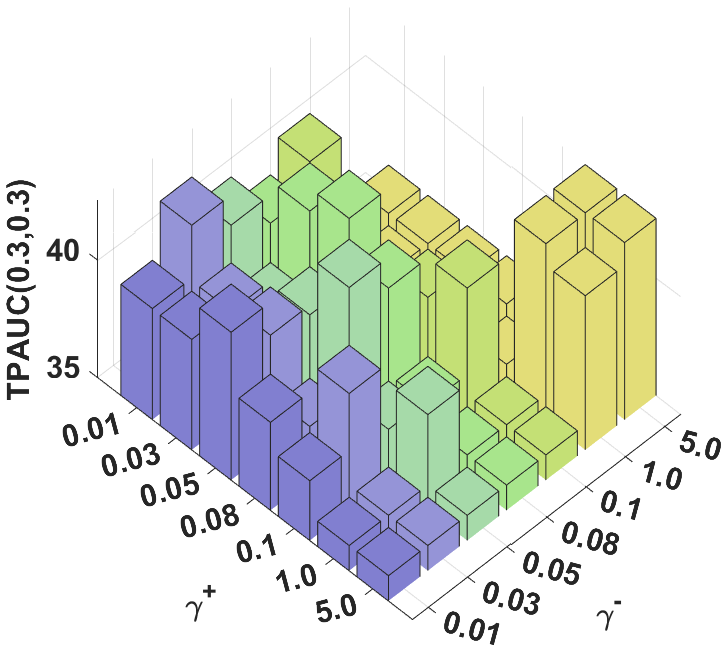} 
     }~~
     \subfigure[Effect of $\gamma_+,\gamma_-$ on subset3, $\alpha=0.4, \beta=0.4$]{
       \includegraphics[width=0.3\textwidth]{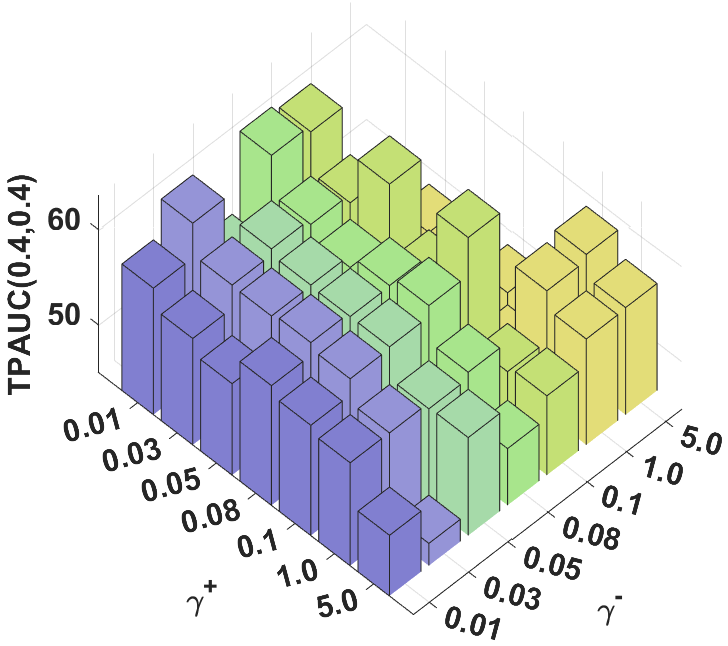} 
     }~~
     \subfigure[Effect of $\gamma_+,\gamma_-$ on subset3, $\alpha=0.5, \beta=0.5$]{
       \includegraphics[width=0.3\textwidth]{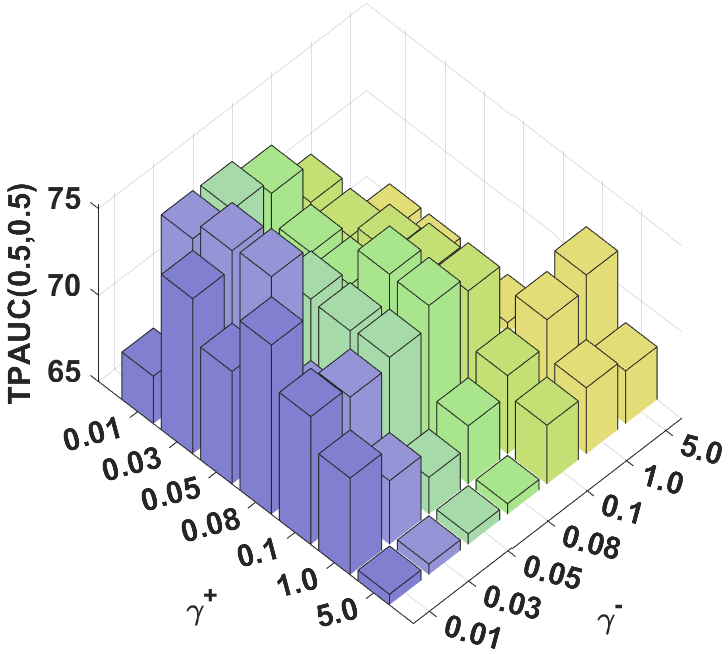} 
     }
     \caption{\label{fig:gammapoly2}Sensitivity analysis of \texttt{Poly} on $\gamma_+,\gamma_{-}$. Experiments are conducted on CIFAR-100-LT.}
 \end{figure*}

 \begin{figure*}[h]  
   \centering
    
     \subfigure[Effect of $\gamma_+,\gamma_-$ on subset1, $\alpha=0.3, \beta=0.3$]{
       \includegraphics[width=0.3\textwidth]{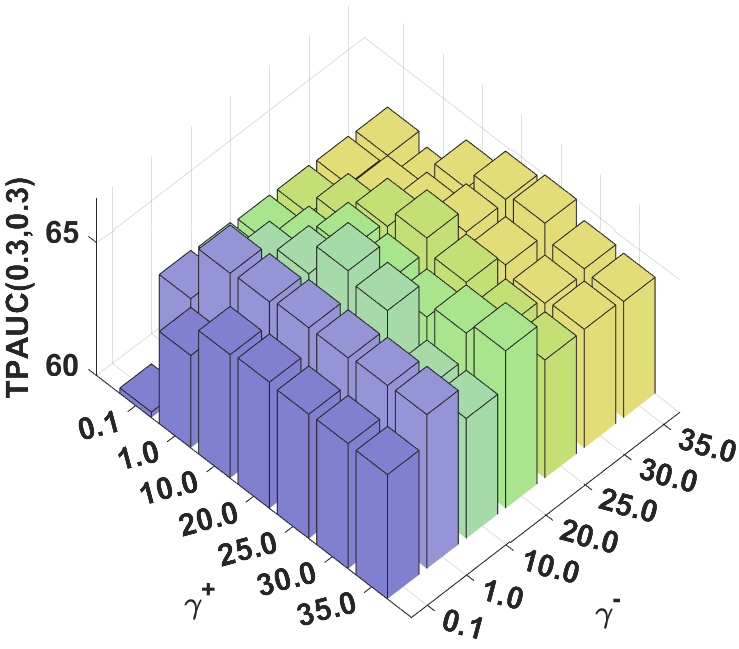} 
     }~~
     \subfigure[Effect of $\gamma_+,\gamma_-$ on subset1, $\alpha=0.4, \beta=0.4$]{
       \includegraphics[width=0.3\textwidth]{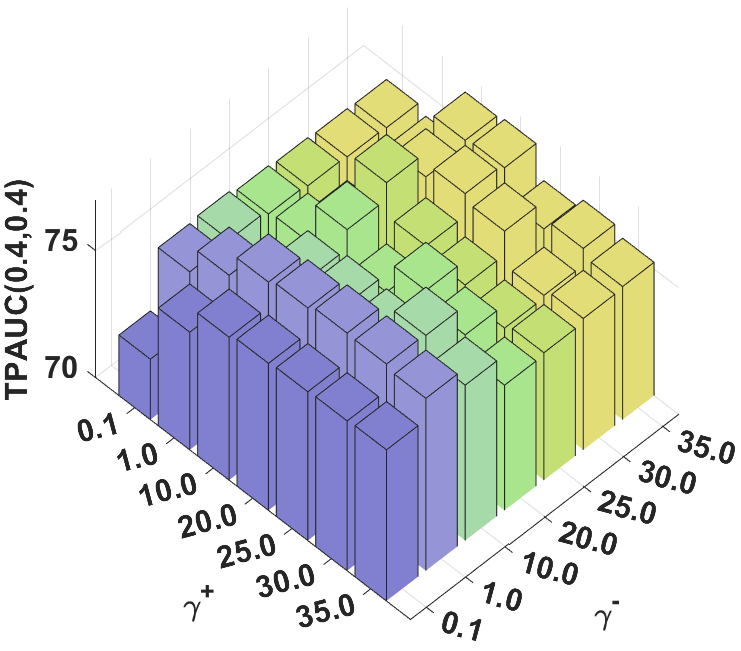} 
     }~~
     \subfigure[Effect of $\gamma_+,\gamma_-$ on subset1, $\alpha=0.5, \beta=0.5$]{
       \includegraphics[width=0.3\textwidth]{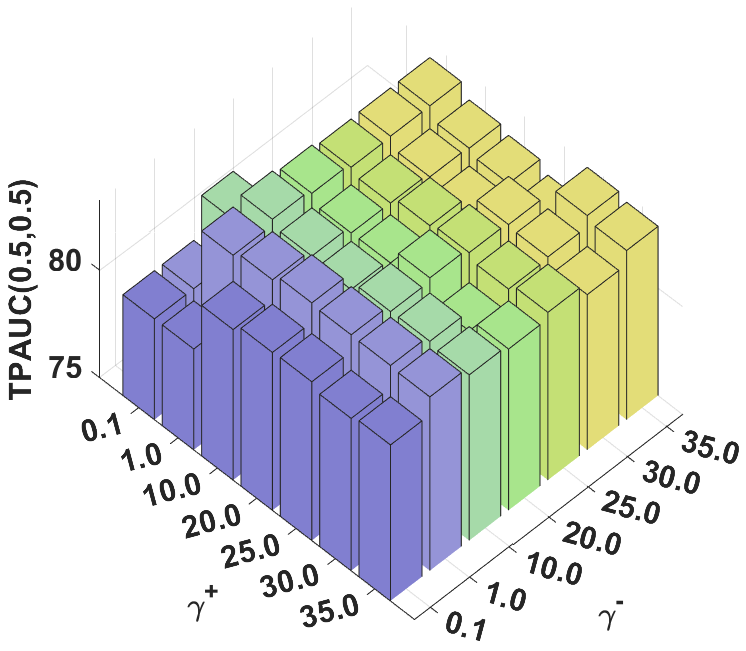} 
     }

     \subfigure[Effect of $\gamma_+,\gamma_-$ on subset2, $\alpha=0.3, \beta=0.3$]{
       \includegraphics[width=0.3\textwidth]{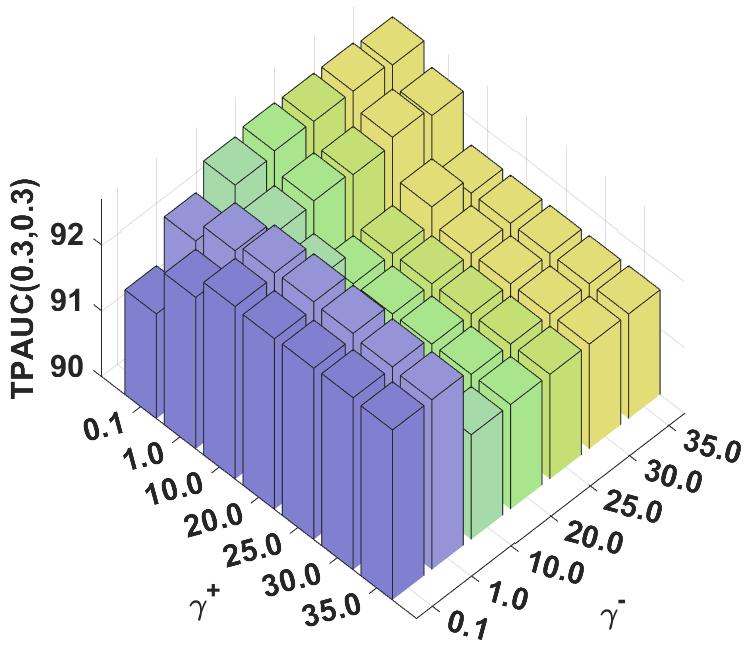} 
     }~~
     \subfigure[Effect of $\gamma_+,\gamma_-$ on subset2, $\alpha=0.4, \beta=0.4$]{
       \includegraphics[width=0.3\textwidth]{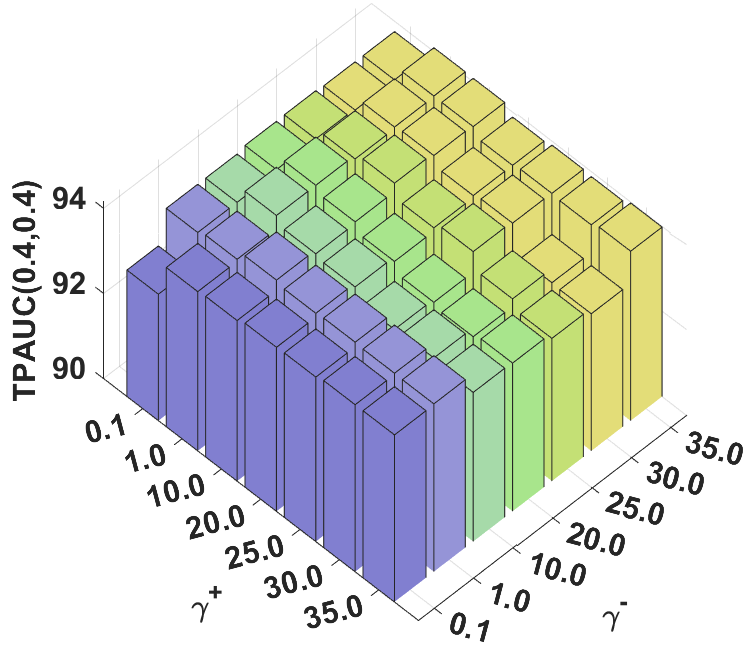} 
     }~~
     \subfigure[Effect of $\gamma_+,\gamma_-$ on subset2, $\alpha=0.5, \beta=0.5$]{
       \includegraphics[width=0.3\textwidth]{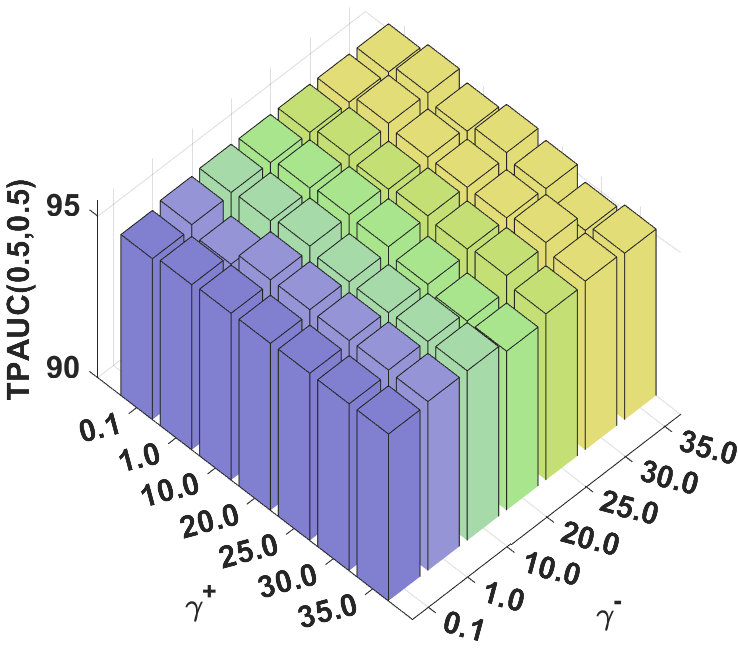} 
     }

     \subfigure[Effect of $\gamma_+,\gamma_-$ on subset3, $\alpha=0.3, \beta=0.3$]{
       \includegraphics[width=0.3\textwidth]{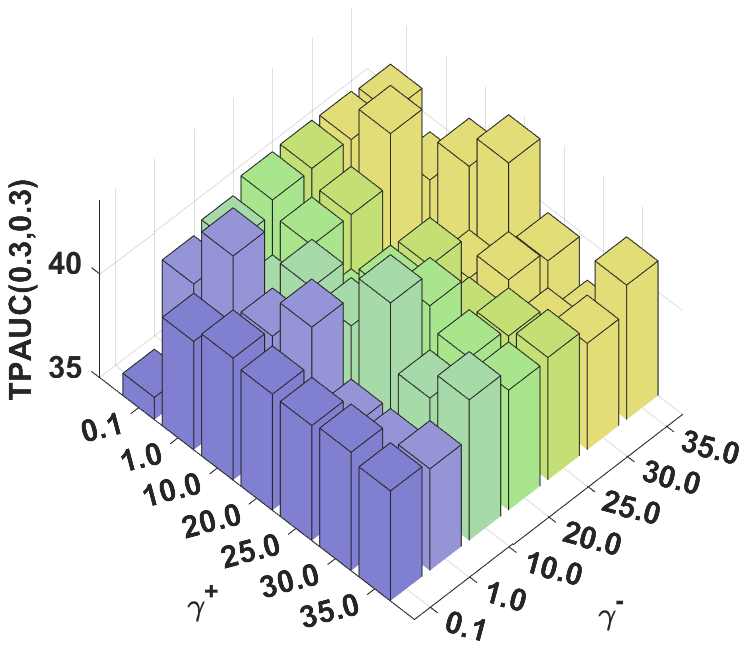} 
     }~~
     \subfigure[Effect of $\gamma_+,\gamma_-$ on subset3, $\alpha=0.4, \beta=0.4$]{
       \includegraphics[width=0.3\textwidth]{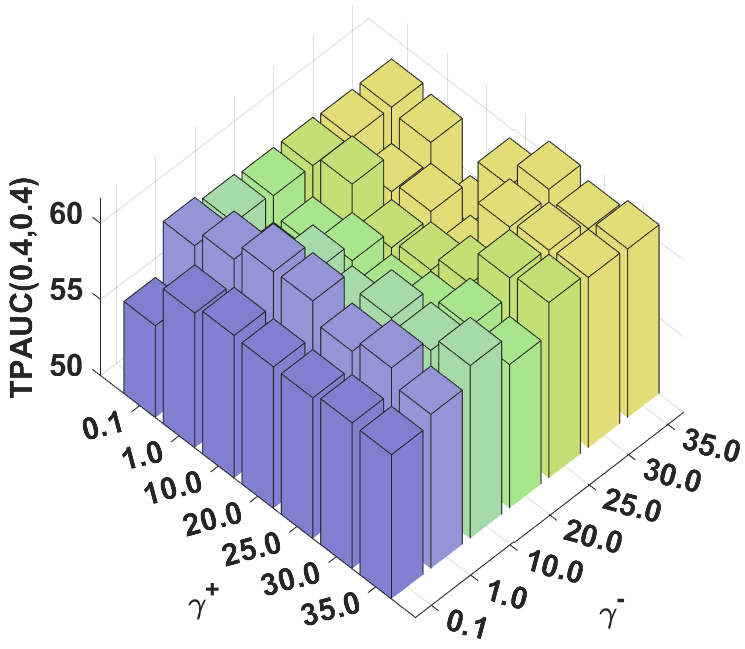} 
     }~~
     \subfigure[Effect of $\gamma_+,\gamma_-$ on subset3, $\alpha=0.5, \beta=0.5$]{
       \includegraphics[width=0.3\textwidth]{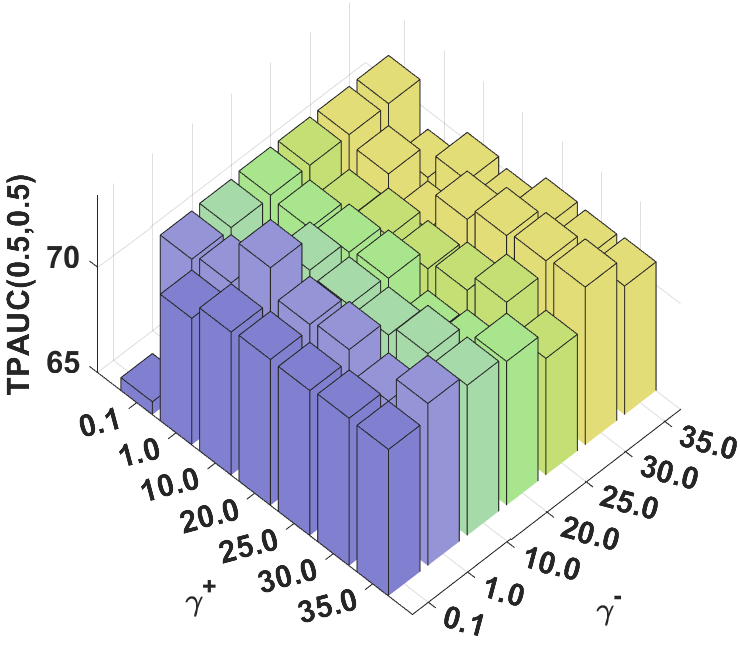} 
     }
     \caption{\label{fig:gammaexp2}Sensitivity analysis of \texttt{Poly} on $\gamma_+,\gamma_-$. Experiments are conducted on CIFAR-100-LT.}

   \end{figure*}
    
 
    

  



  



 \begin{figure*}[th]  
   \centering
  
   \subfigure[Subset1, $\alpha = 0.3, \beta = 0.3$]{
    
     \includegraphics[width=0.31\textwidth]{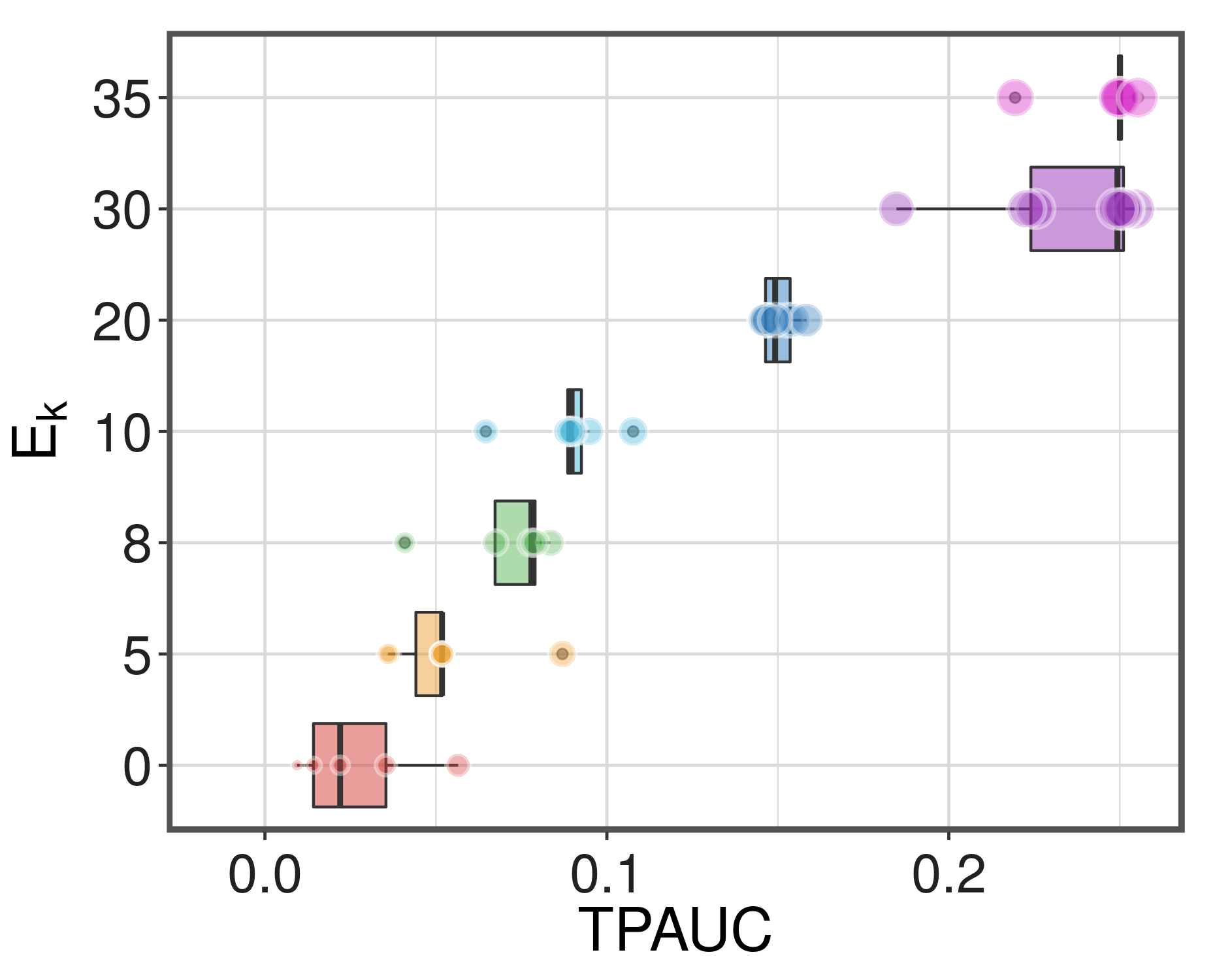} 
   }
   \subfigure[Subset1, $\alpha = 0.4, \beta = 0.4$]{
  
     \includegraphics[width=0.31\textwidth]{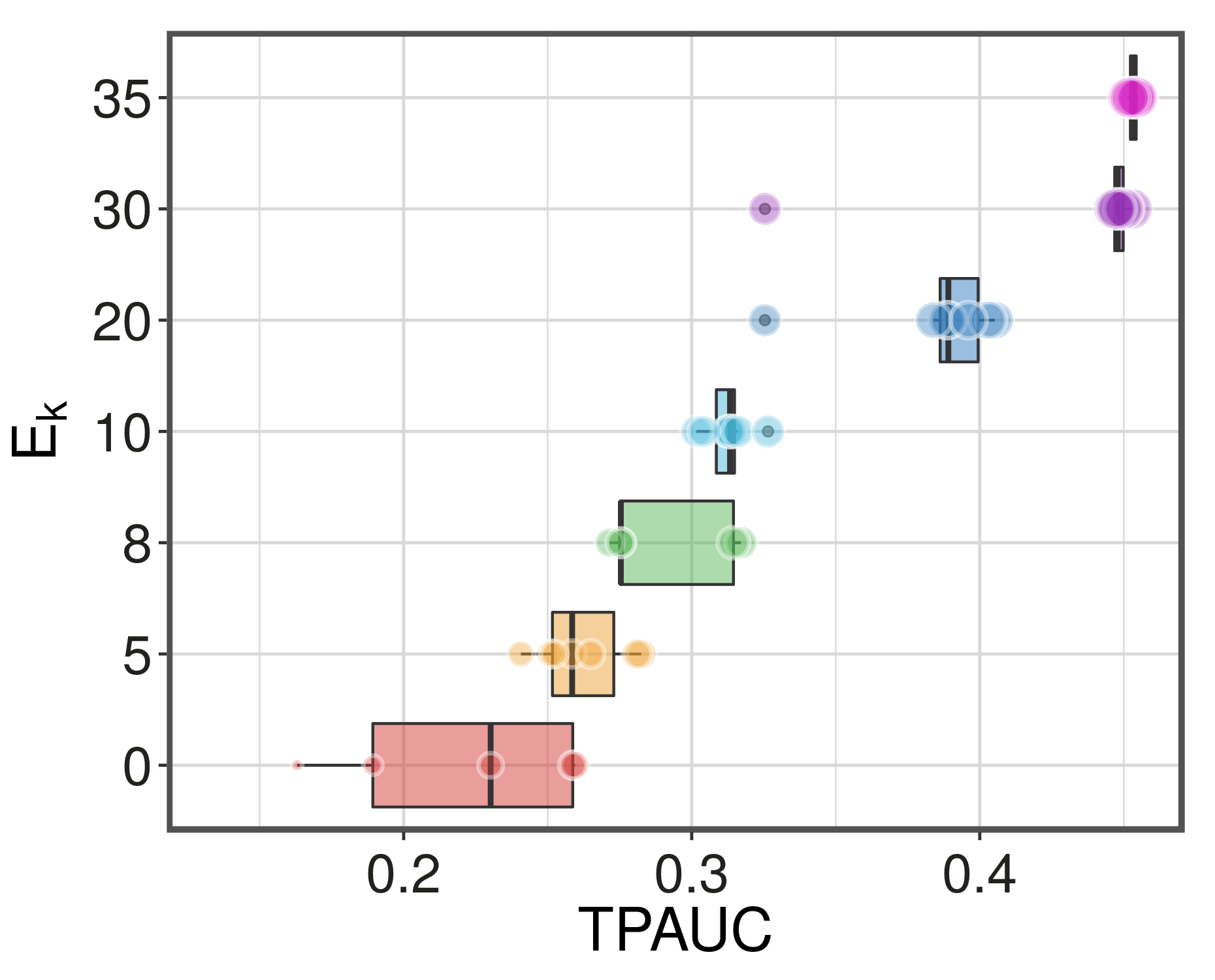} 
   }
   \subfigure[Subset1, $\alpha = 0.5, \beta = 0.5$]{
  
     \includegraphics[width=0.31\textwidth]{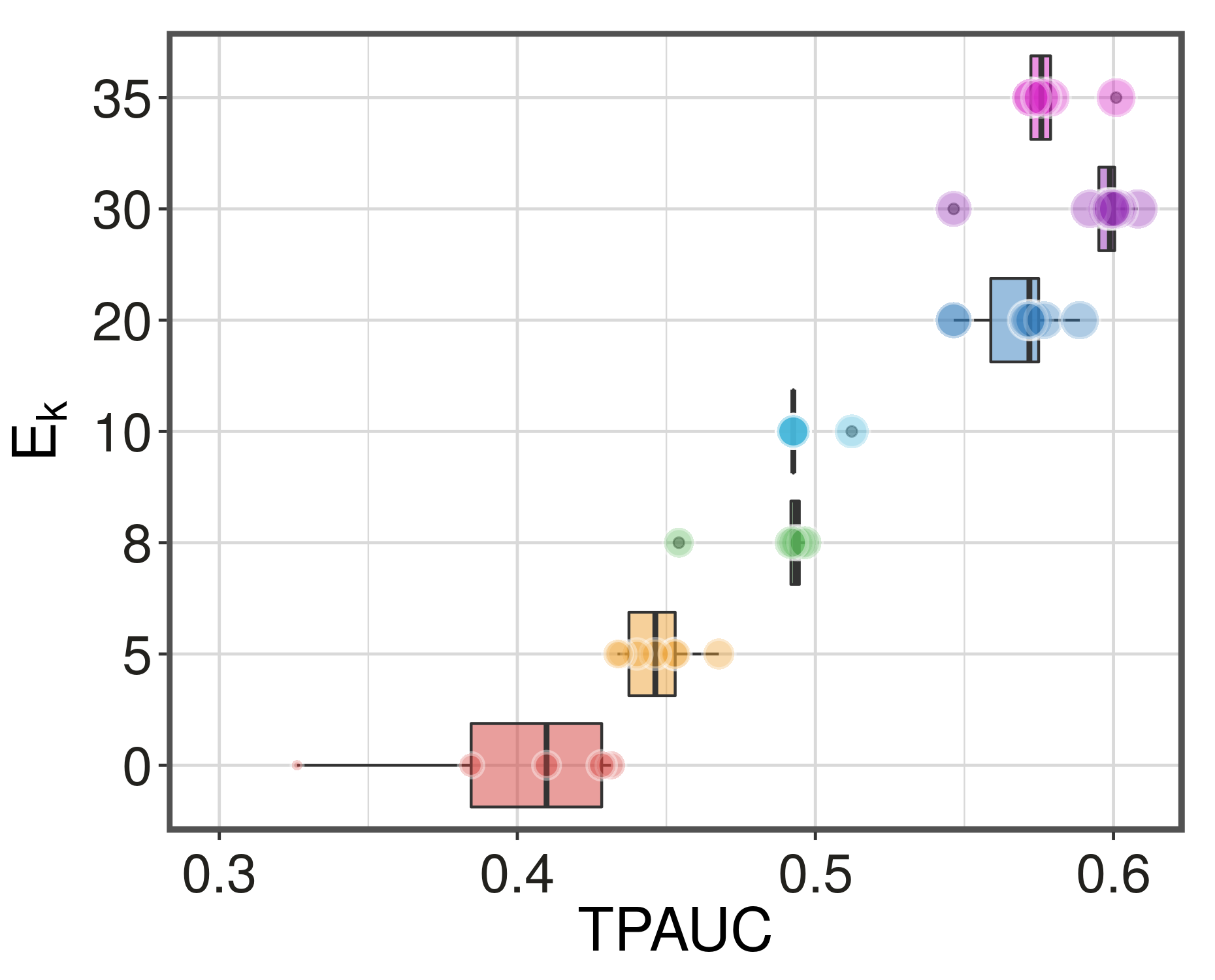} 
   }

     \subfigure[Subset2, $\alpha = 0.3, \beta = 0.3$]{
    
       \includegraphics[width=0.31\textwidth]{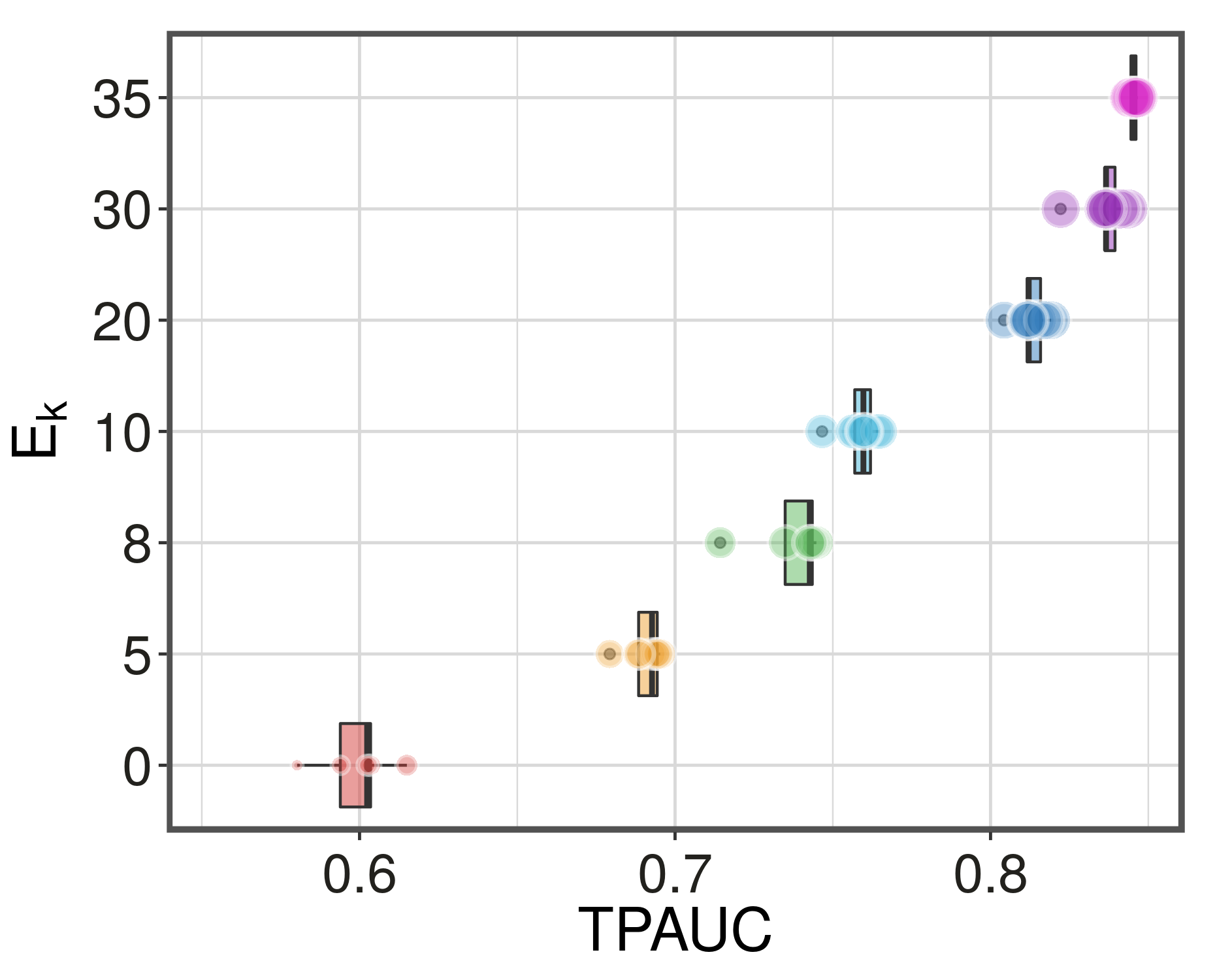} 
     }
     \subfigure[Subset2, $\alpha = 0.4, \beta = 0.4$]{
    
       \includegraphics[width=0.31\textwidth]{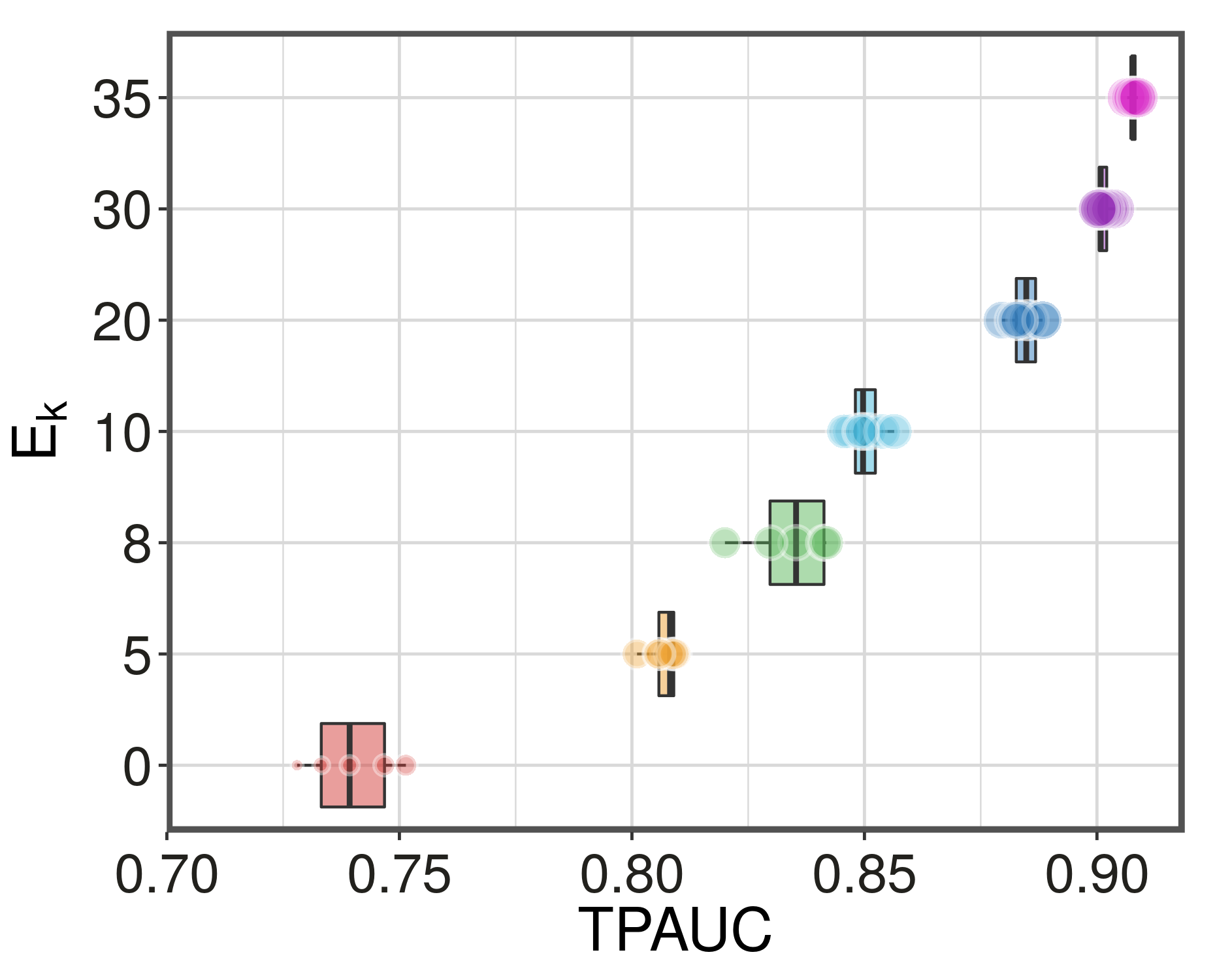} 
     }
     \subfigure[Subset2, $\alpha = 0.5, \beta = 0.5$]{
    
       \includegraphics[width=0.31\textwidth]{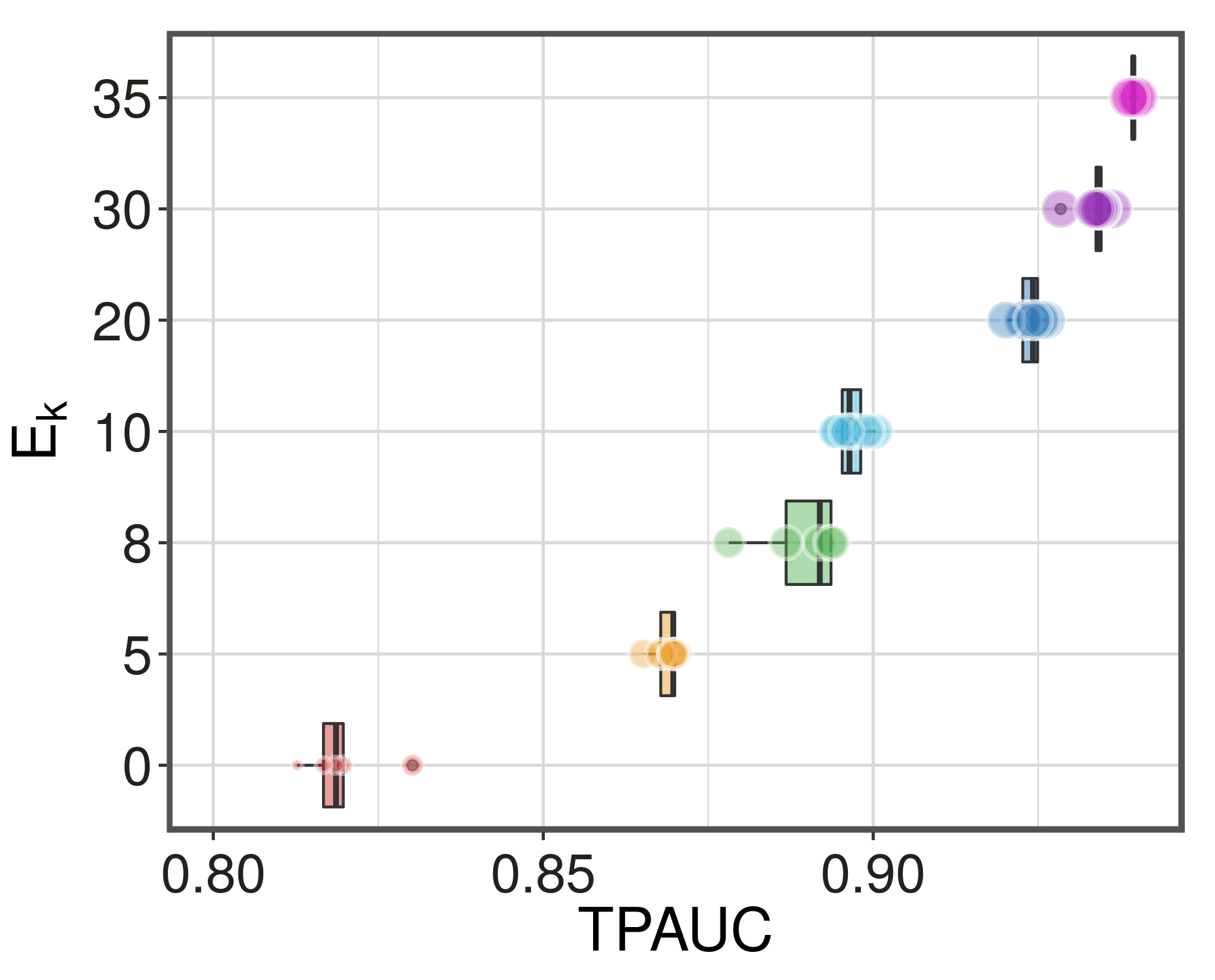} 
     }

     \subfigure[Subset3, $\alpha = 0.3, \beta = 0.3$]{
    
       \includegraphics[width=0.31\textwidth]{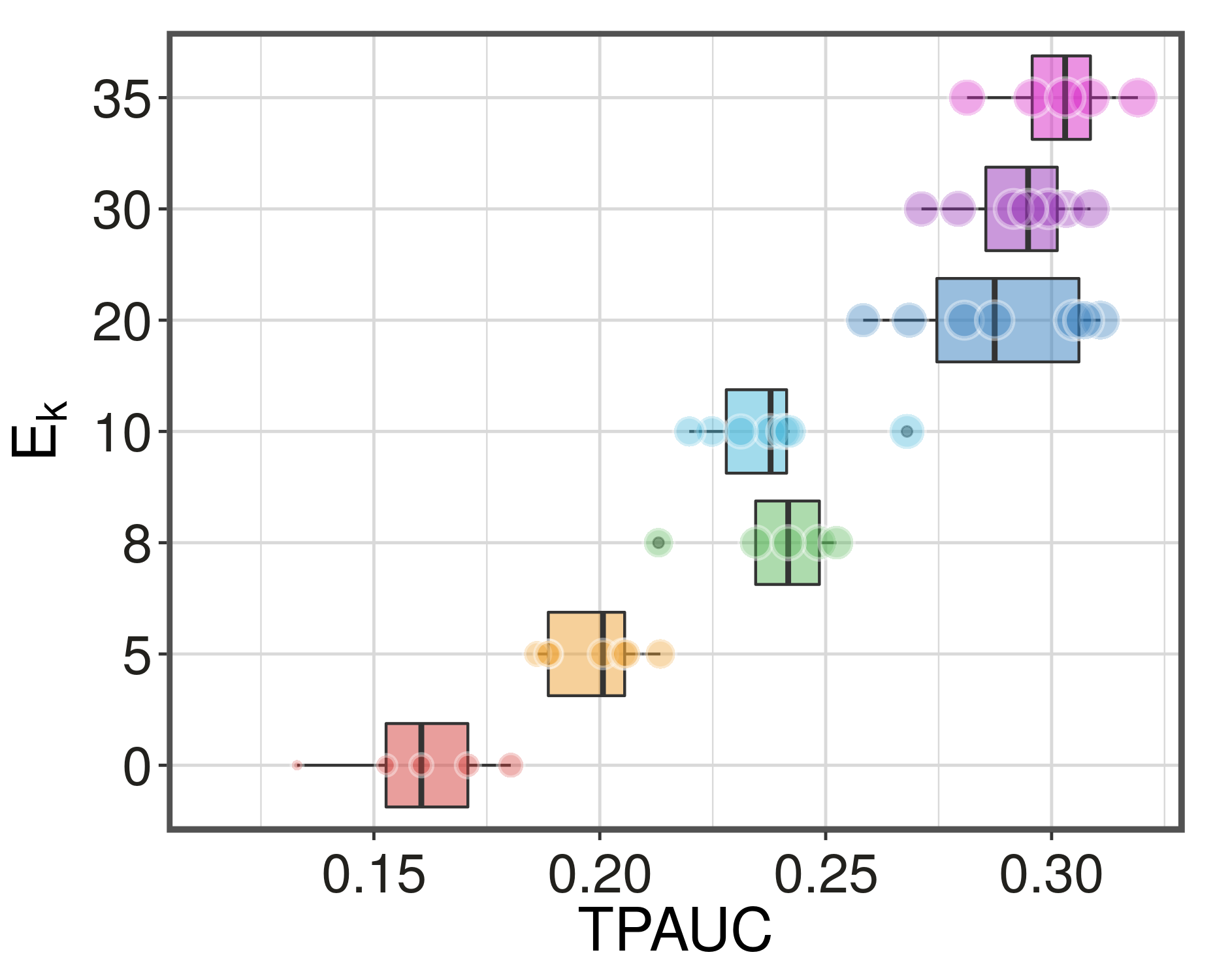} 
     }
     \subfigure[Subset3, $\alpha = 0.4, \beta = 0.4$]{
    
       \includegraphics[width=0.31\textwidth]{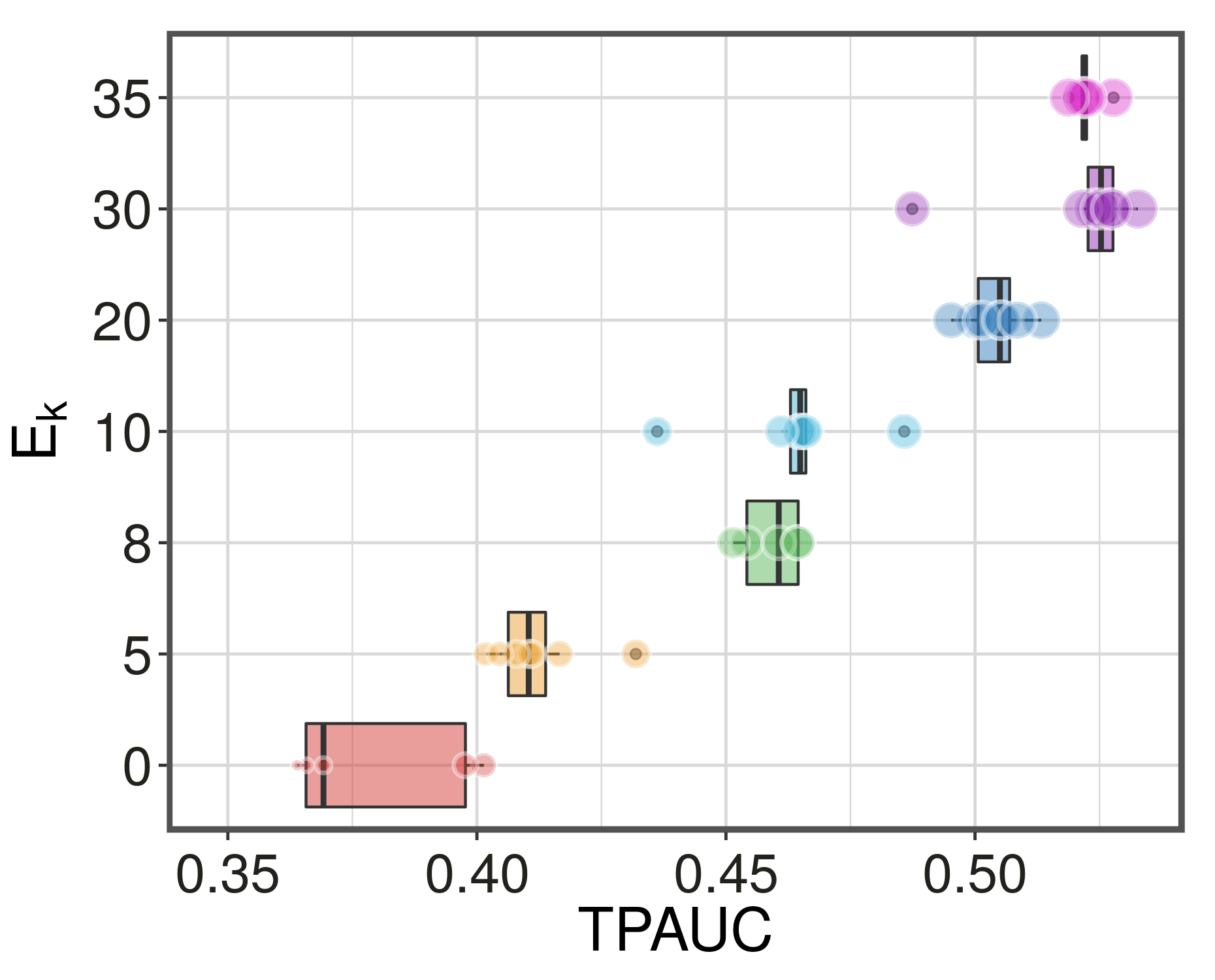} 
     }
     \subfigure[Subset3, $\alpha = 0.5, \beta = 0.5$]{
    
       \includegraphics[width=0.31\textwidth]{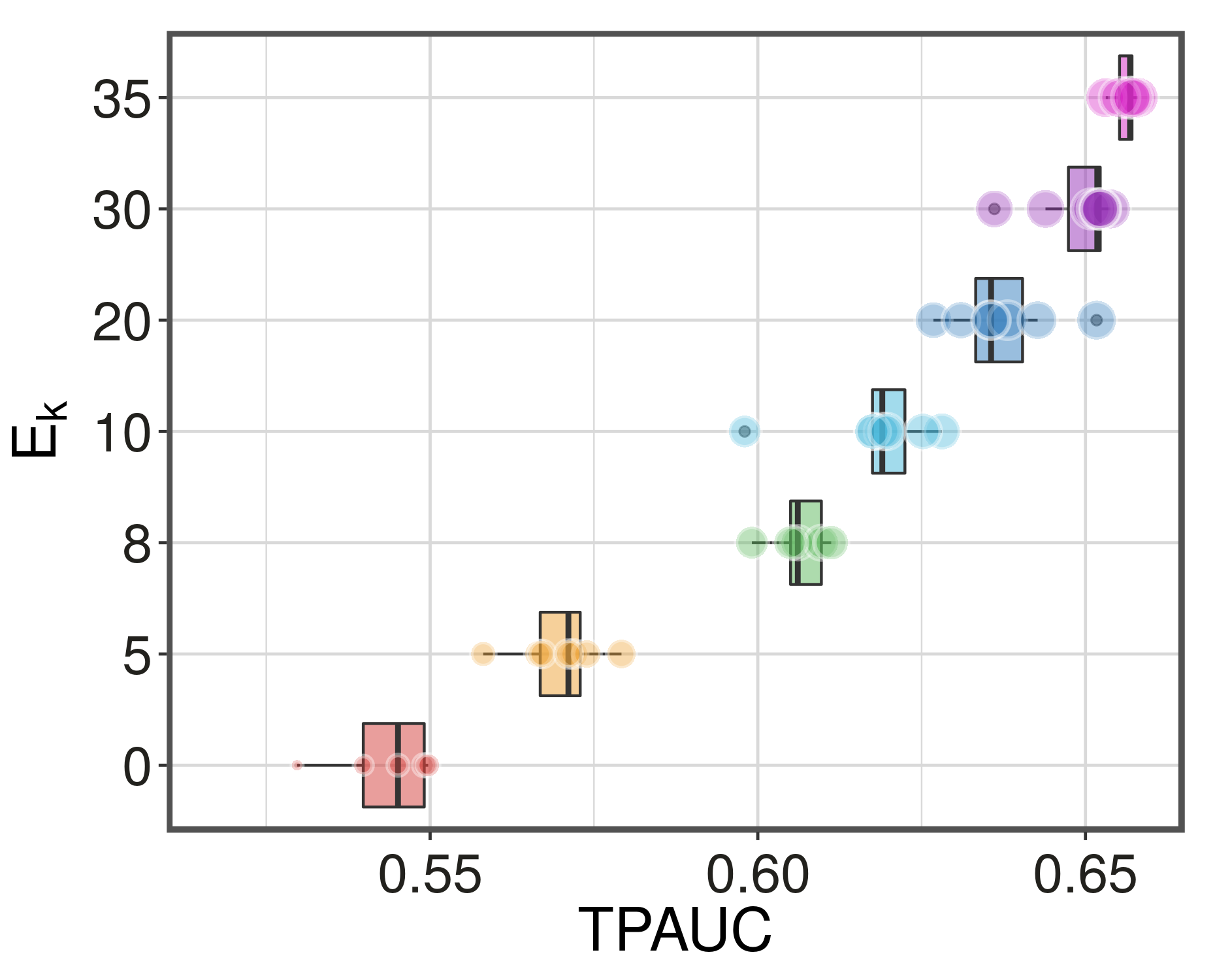} 
     }

     \caption{\label{fig:warm1} Sensitivity analysis on CIFAR-10-LT where TPAUC for \texttt{Exp} with respect to $E_k$. For each Box in the plots,  $E_k$  is fixed as the y-axis value, and the scattered points along the box show the variation of $(\gamma-1)^{-1}$.}
   \end{figure*}

     \begin{figure*}[th]  
       \centering
       \subfigure[Subset1, $\alpha = 0.3, \beta = 0.3$]{
        
         \includegraphics[width=0.31\textwidth]{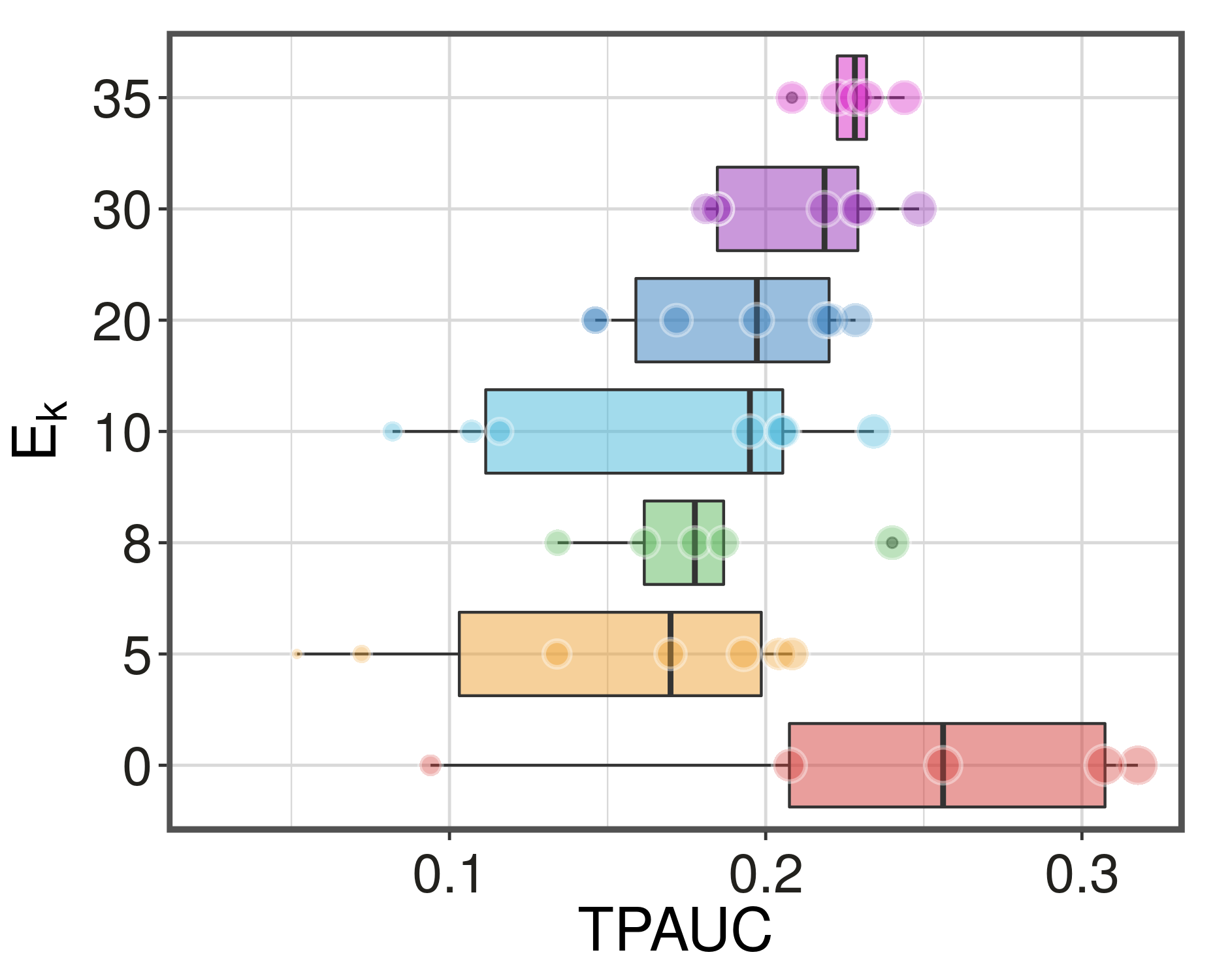} 
       }
       \subfigure[Subset1, $\alpha = 0.4, \beta = 0.4$]{
      
         \includegraphics[width=0.31\textwidth]{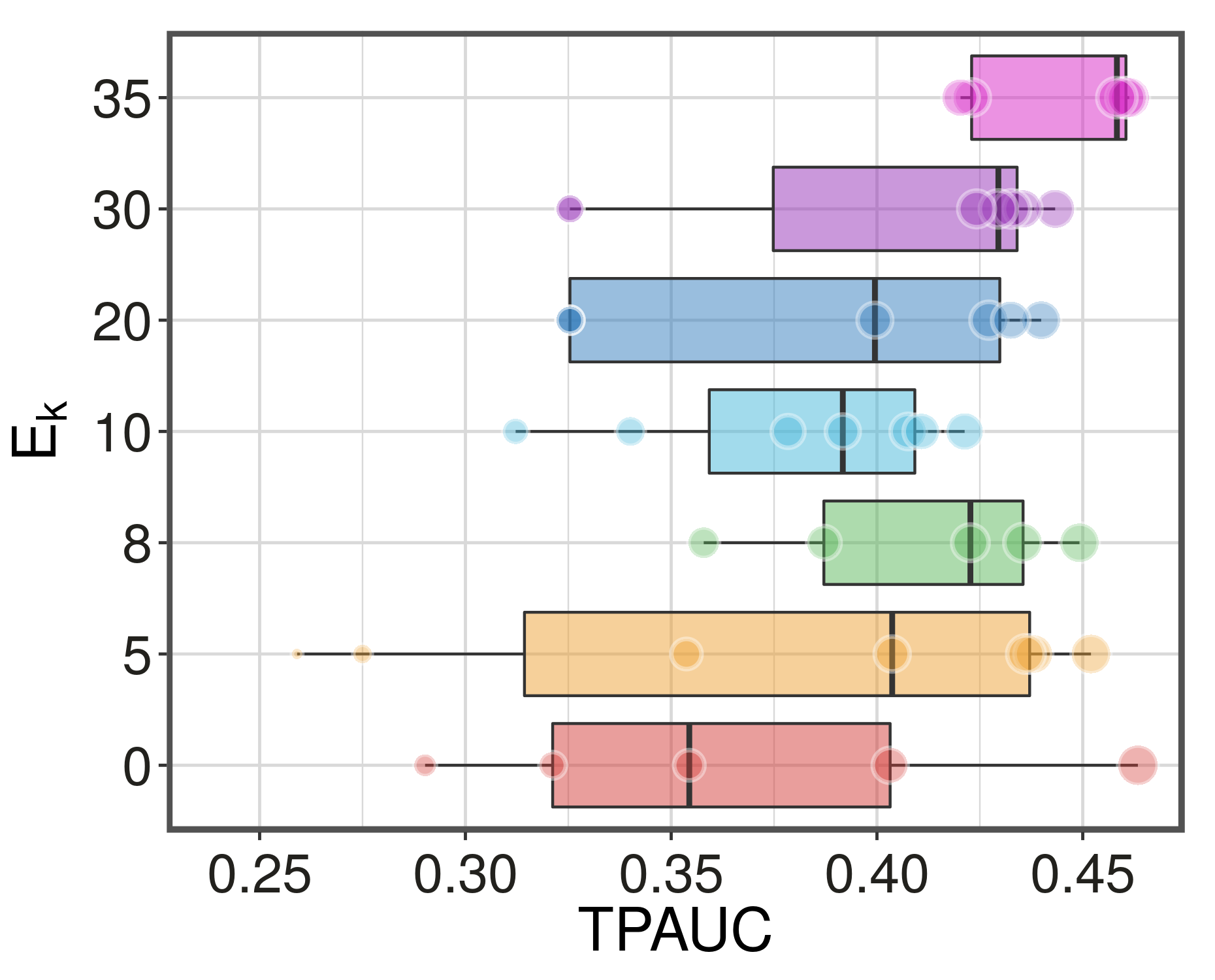} 
       }
       \subfigure[Subset1, $\alpha = 0.5, \beta = 0.5$]{
      
         \includegraphics[width=0.31\textwidth]{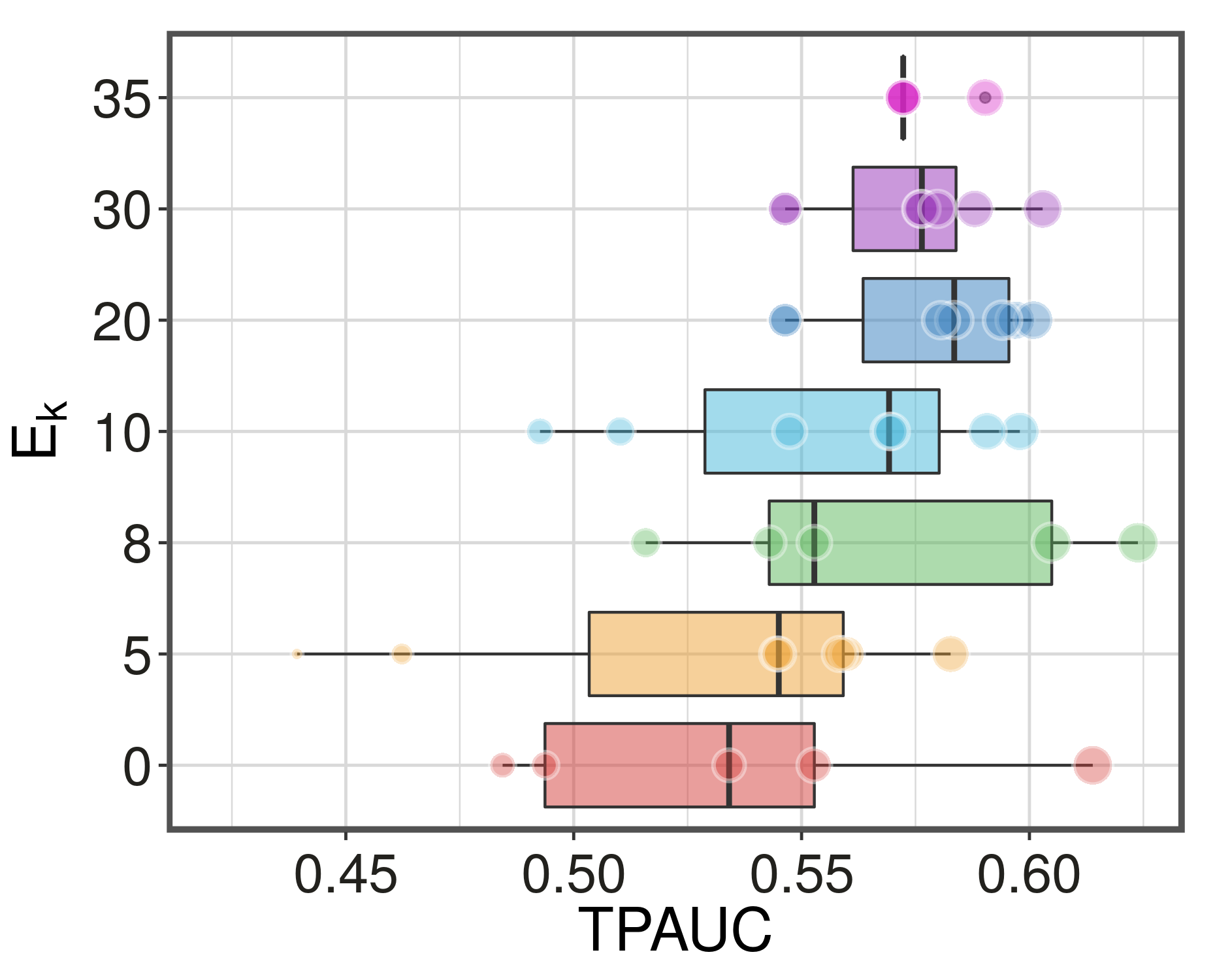} 
       }
  
       \subfigure[Subset2, $\alpha = 0.3, \beta = 0.3$]{
      
         \includegraphics[width=0.31\textwidth]{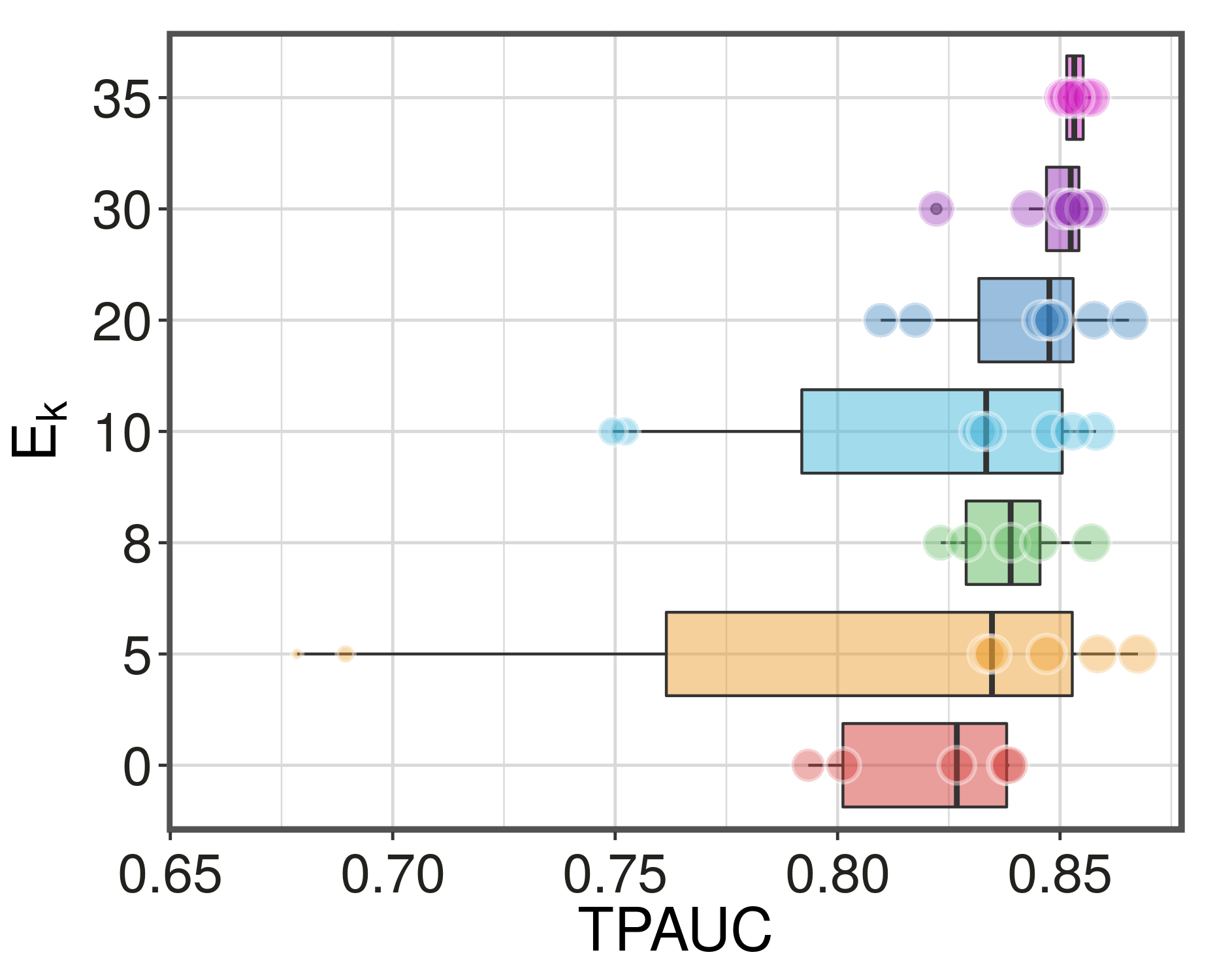} 
       }
       \subfigure[Subset2, $\alpha = 0.4, \beta = 0.4$]{
      
         \includegraphics[width=0.31\textwidth]{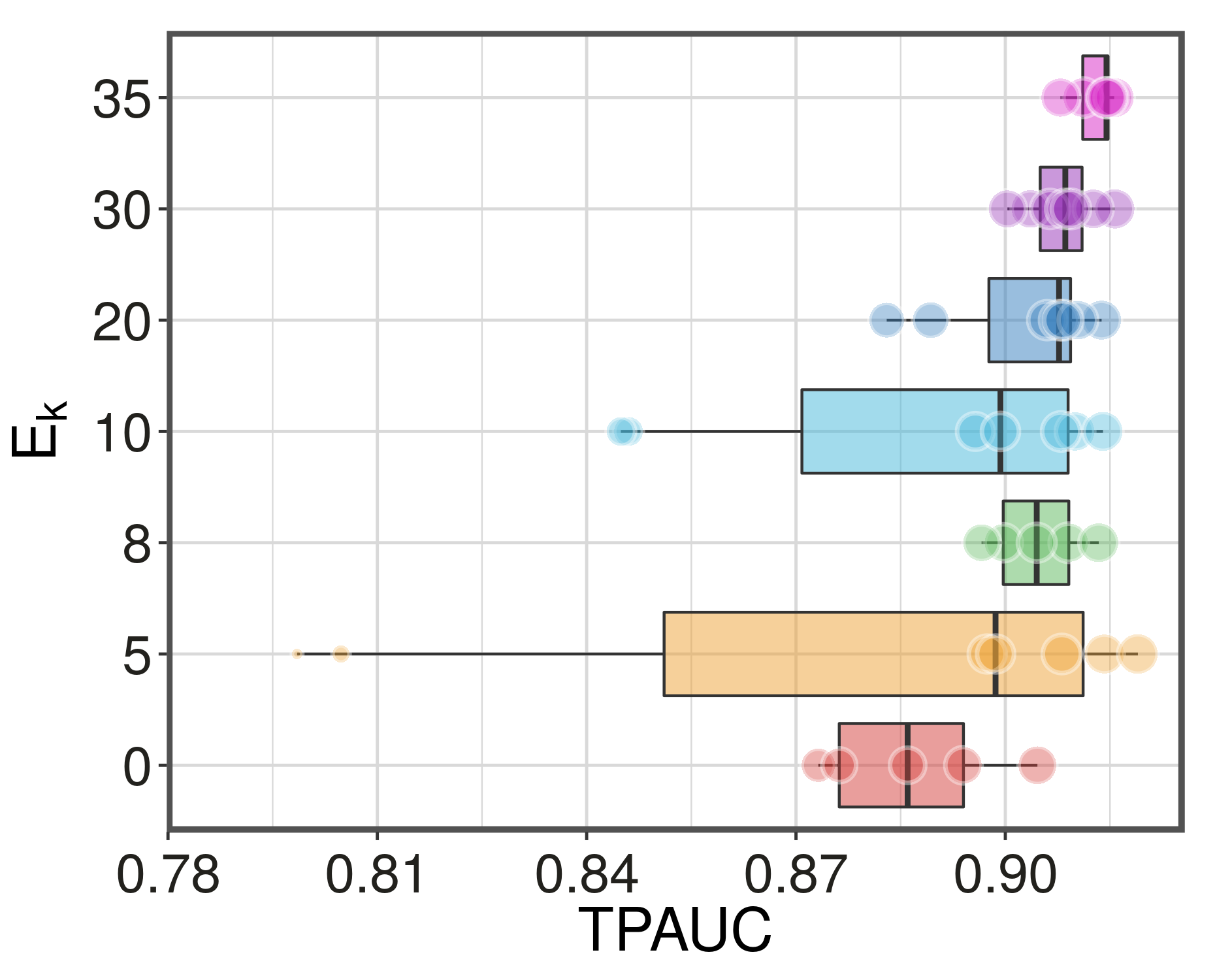} 
       }
       \subfigure[Subset2, $\alpha = 0.5, \beta = 0.5$]{
      
         \includegraphics[width=0.31\textwidth]{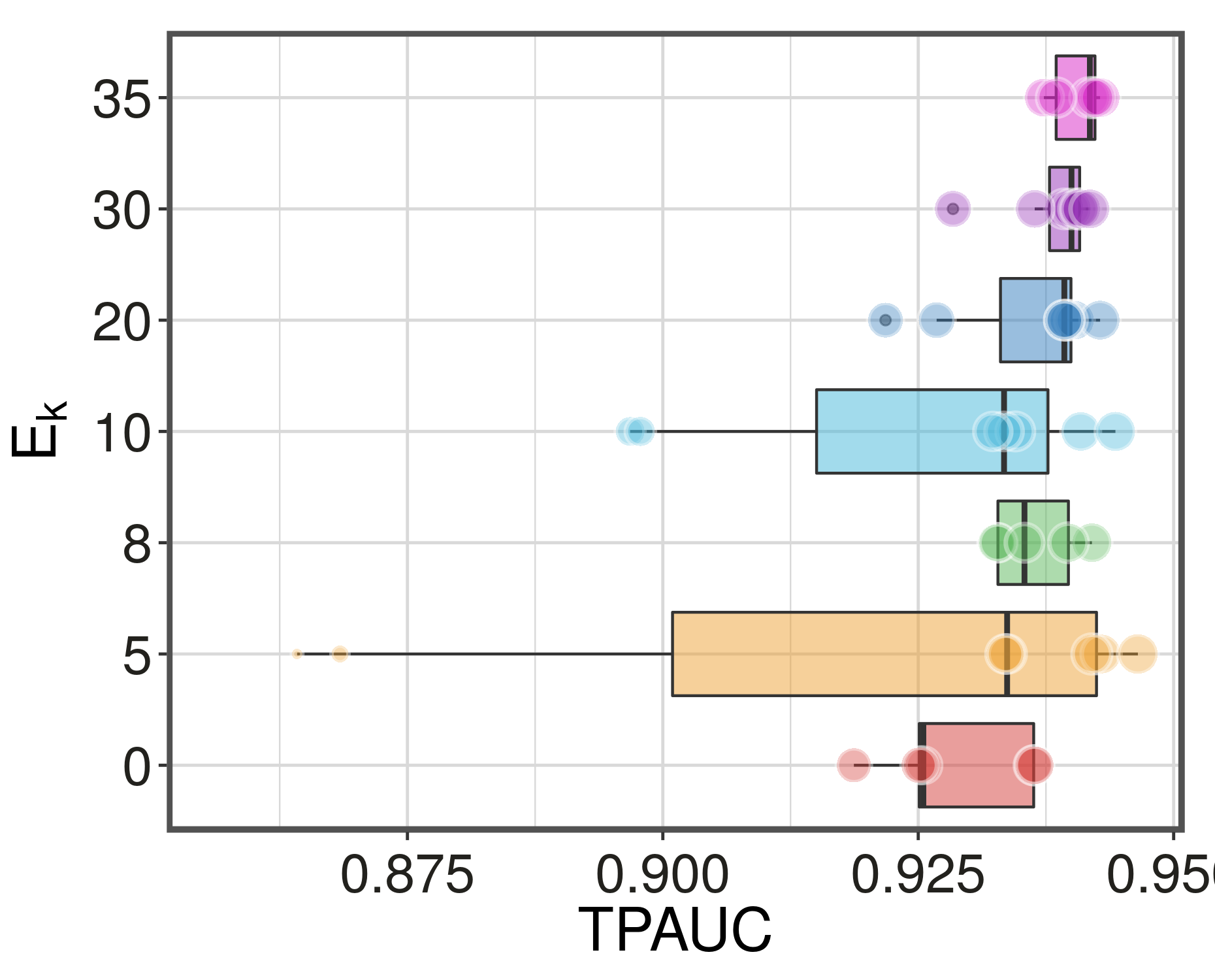} 
       }

         \subfigure[Subset3, $\alpha = 0.3, \beta = 0.3$]{
        
           \includegraphics[width=0.31\textwidth]{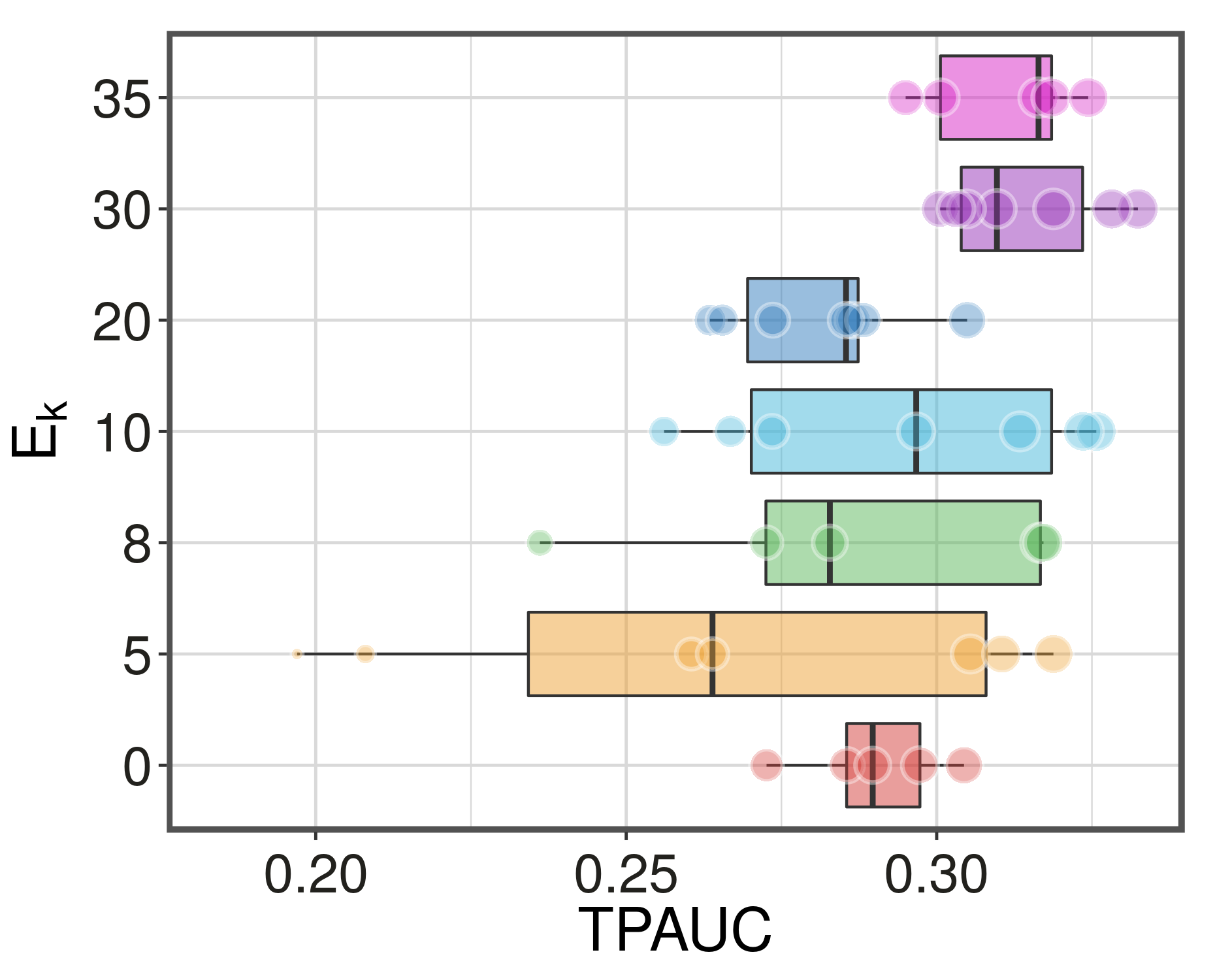} 
         }
         \subfigure[Subset3, $\alpha = 0.4, \beta = 0.4$]{
        
           \includegraphics[width=0.31\textwidth]{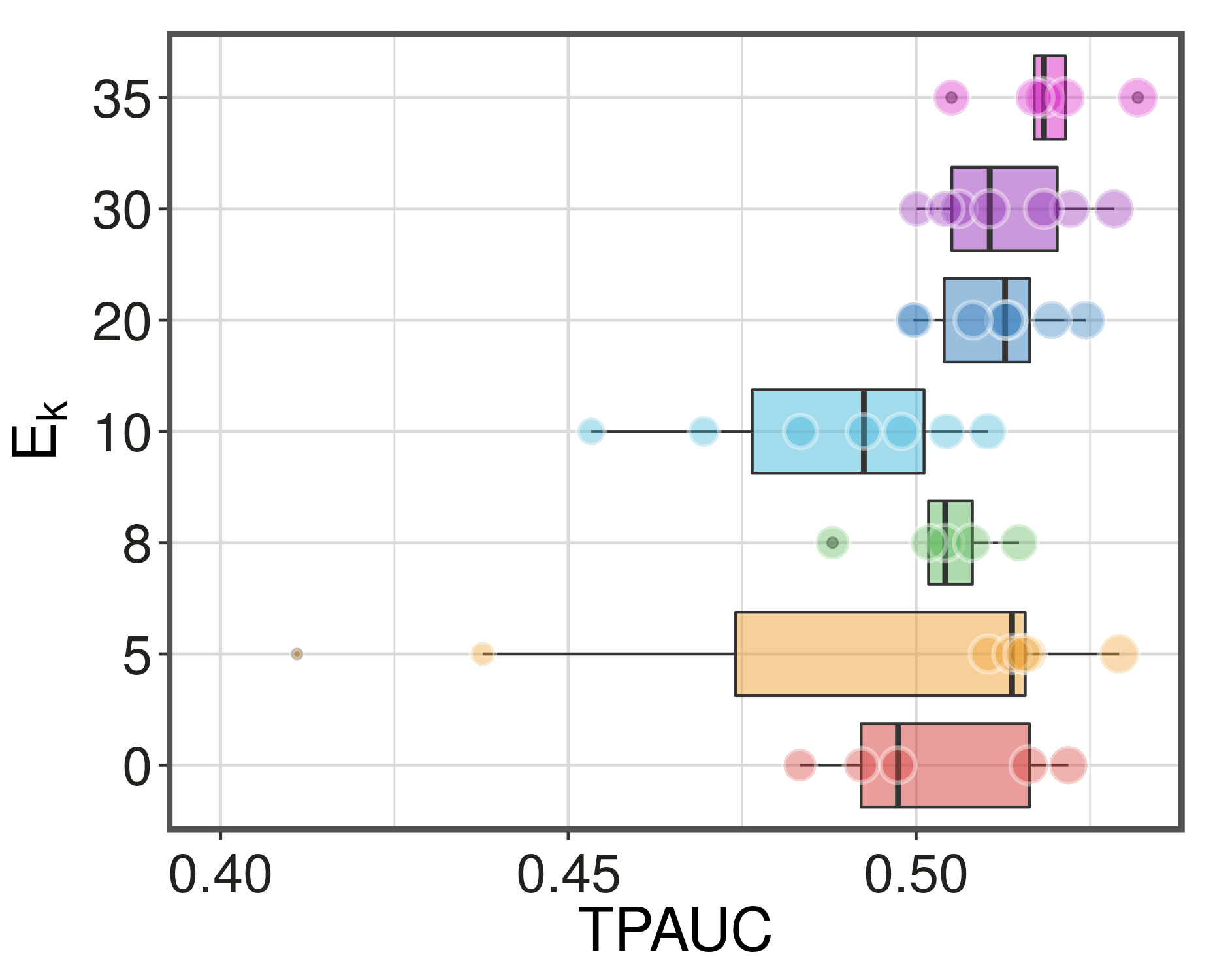} 
         }
         \subfigure[Subset3, $\alpha = 0.5, \beta = 0.5$]{
        
           \includegraphics[width=0.31\textwidth]{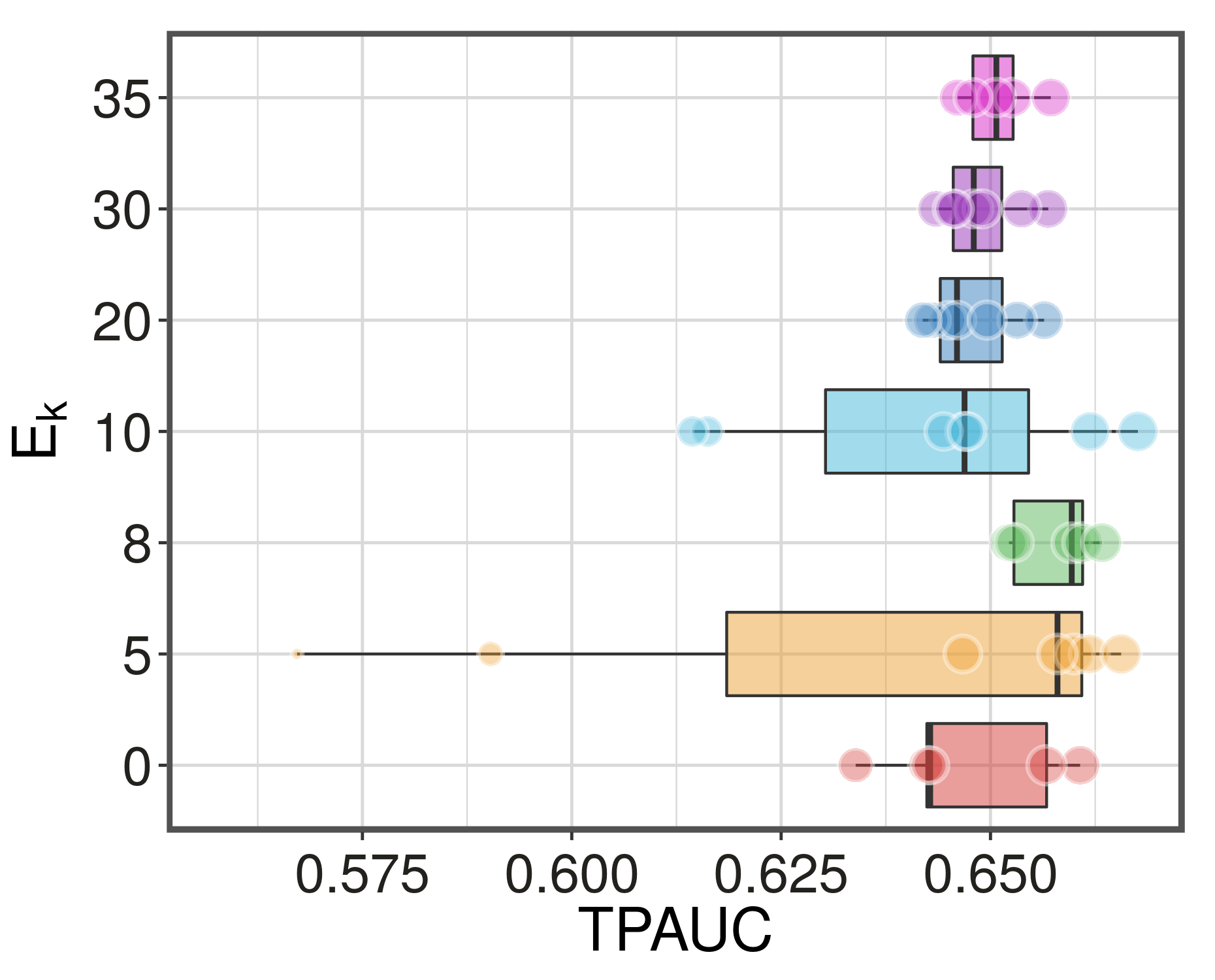} 
         }

           \caption{\label{fig:warm2}Sensitivity analysis on CIFAR-10-LT where TPAUC for \texttt{Poly} with respect to $E_k$. For each Box in the plots,  $E_k$  is fixed as the y-axis value, and the scattered points along the box show the variation of $(\gamma-1)^{-1}$.}
      
       \end{figure*}

\end{document}